%% file: MatrixIRLS_stats_ConvAnalysis.tex
\documentclass[12pt]{article}

\usepackage[T1]{fontenc}
\usepackage[english]{babel}
\usepackage{amsmath,amsfonts,amssymb,amsthm}
\usepackage{mathtools}
\usepackage{graphicx,xcolor}
\usepackage[round,authoryear]{natbib}
\usepackage{bm}
\usepackage{booktabs}
\usepackage{array,tabularx}
\usepackage{url}
\usepackage{nicefrac}
\usepackage[expansion=false]{microtype}
\usepackage{algorithm,algorithmic}
\usepackage{dsfont}
\usepackage{amscd}
\usepackage{multirow}
\usepackage{enumerate}
\usepackage{float}
\usepackage{caption}
\usepackage{subcaption}
\usepackage{enumitem}

\definecolor{citeblue}{RGB}{0,45,154}
\usepackage[colorlinks=true, citecolor=citeblue, linkcolor=citeblue, urlcolor=citeblue]{hyperref}
\usepackage{cleveref}

\usepackage{tikz}
\usepackage{pgfplots}
\pgfplotsset{compat=newest}
\usetikzlibrary{plotmarks}
\usetikzlibrary{arrows.meta}
\usepgfplotslibrary{patchplots}

\newcommand{\BlackBox}{\rule{1.5ex}{1.5ex}}
\renewenvironment{proof}[1][Proof]{\par\noindent{\bf #1\ }}{\hfill\BlackBox\\[2mm]}

\newtheorem{theorem}{Theorem}

\newenvironment{keywords}
{\par\noindent{\footnotesize\textbf{Keywords:}\ }}
{\par}

\input{MatrixIRLS_preamble.tex}

\allowdisplaybreaks[4]

\title{\texorpdfstring{Tight Majorizations and Convergence Rates\\ of Nuclear Norm Minimization IRLS}{Tight Majorizations and Convergence Rates of Nuclear Norm Minimization IRLS}}
\author{%
Christian K\"ummerle\thanks{School of Data, Mathematical, and Statistical Sciences, Department of Computer Science \& Institute of Artificial Intelligence, University of Central Florida, Oviedo, FL 32816, USA (\href{mailto:kuemmerle@ucf.edu}{\texttt{kuemmerle@ucf.edu}}).}
\and Tomas Masak\thanks{Institute for Statistics and Mathematics, Wirtschaftsuniversit\"at Wien, 1020 Vienna, Austria (\href{mailto:tomas.masak@wu.ac.at}{\texttt{tomas.masak@wu.ac.at}}).}
\and Dominik St\"oger\thanks{Department of Mathematics, KU Eichst\"att-Ingolstadt, 85049 Ingolstadt, Germany (\href{mailto:dominik.stoeger@ku.de}{\texttt{dominik.stoeger@ku.de}}).}
}

\begin{document}
\renewcommand*{\thefootnote}{\fnsymbol{footnote}}
\maketitle
\renewcommand*{\thefootnote}{\arabic{footnote}}
\setcounter{footnote}{0}

\input{abstract.tex}
\newpage
\tableofcontents

\input{introduction.tex}
\input{related_work.tex}
\input{algorithm.tex}
\input{main_results.tex}
\input{numerical_experiments.tex}
\input{conclusion.tex}

\input{acknowledgements.tex}

\appendix

\input{contribution_proofs.tex}
\input{majorization_property.tex}
\input{proofs_power_means.tex}
\input{proofs_globalconvergencep1.tex}
\input{proofs_fastlocallinearp1.tex}

\input{literature_proofs.tex}

\bibliography{LinConv_MatrixIRLS}

\end{document}

%% file: MatrixIRLS_preamble.tex
\xdefinecolor{citeblue}{RGB}{0,45,154}
\xdefinecolor{citered}{RGB}{190, 69, 32}
\xdefinecolor{bronze}{RGB}{190,117,45}

\pgfplotsset{plot coordinates/math parser=false}
\newlength\figureheight
\newlength\figurewidth

\DeclareMathOperator{\trace}{tr}

\DeclareMathOperator{\Id}{Id}

\DeclareMathOperator*{\argmin}{arg\,min}

\DeclareMathOperator{\rank}{rank}

\DeclareMathOperator{\N}{\mathbb{N}}
\DeclareMathOperator{\R}{\mathbb{R}}

\DeclareMathOperator{\diag}{diag}

\DeclareRobustCommand{\MatrixIRLSHeading}{%
  \texorpdfstring
    {\textnormal{\texttt{MatrixIRLS}}}
    {MatrixIRLS}%
}

\newcommand{\X}{\mathbf{X}}

\newcommand{\Up}{\mathbf{U}_{\perp}}
\newcommand{\Vp}{\mathbf{V}_{\perp}}

\renewcommand\Re{\operatorname{\mathfrak{Re}}}

\newcommand{\besterrNuc}[2]{\beta_{#2}(#1)_{*}}

\DeclareMathOperator{\dg}{dg}

\newcommand{\onenorm}[1]{\Vert #1 \Vert_1}
\newcommand{\twonorm}[1]{\Vert #1 \Vert_2}

\newcommand\f[1]{\mathbf{#1}} 
\newcommand\hk{^{(k)}}

\newcommand{\xzero}{\f{x}_{\star}}
\newcommand{\xzeromin}{\xzero^{\operatorname{min}}}

\newcommand{\Xzero}{\f{X}_\star}
\newcommand{\innerproduct}[1]{\langle #1 \rangle_{F}}
\newcommand{\NN}{\f{N}}
\newcommand{\XX}{\f{X}}
\newcommand{\Xk}[1]{\XX^{#1}}
\newcommand{\UU}{\f{U}}
\newcommand{\UUstar}{\f{U}_{\star}}
\newcommand{\VV}{\f{V}}
\newcommand{\VVstar}{\f{V}_{\star}}

\newcommand{\SSigma}{\f{\Sigma}}
\newcommand{\vertiii}[1]{{\left\vert\kern-0.25ex\left\vert\kern-0.25ex\left\vert #1 
    \right\vert\kern-0.25ex\right\vert\kern-0.25ex\right\vert}}

\newcommand\Rdd{\R^{d_1 \times d_2}}

\newcommand{\NNk}{\f{N}^{(k)}}
\newcommand{\XXk}{\f{X}^{(k)}}
\newcommand{\XXl}{\f{X}^{(\ell)}}
\newcommand{\XXkplus}{\f{X}^{(k+1)}}

\newcommand{\maxD}{D}
\newcommand{\mind}{d}

\newcommand{\specnorm}[1]{\left\lVert #1 \right\rVert}
\newcommand{\nucnorm}[1]{\left\lVert #1 \right\rVert_{\ast}}
\newcommand{\fronorm}[1]{\left\lVert #1 \right\rVert_{F}}

\newcommand{\overleq}[1]{\overset{#1}{\le}}
\newcommand{\overgeq}[1]{\overset{#1}{\ge}}
\newcommand{\overeq}[1]{\overset{#1}{=}}

\newcommand{\bracing}[2]{\underset{{#1}}{\underbrace{#2}}  }
\newcommand{\epsk}{\varepsilon_k}
\newcommand{\epsl}{\varepsilon_{\ell}}
\newcommand{\W}[2]{W_{#1,#2}}

\newcommand{\DDelta}{\f{\Delta}}

\newcommand{\ZZ}{\f{Z}}

\newcommand{\LL}{\f{L}}
\newcommand{\uu}{\f{u}}
\newcommand{\vv}{\f{v}}
\newcommand{\PUU}[1]{\f{P}_{\UU, #1}}
\newcommand{\PVV}[1]{\f{P}_{\VV, #1}}
\newcommand{\WW}{\f{W}}

\newcommand{\llambdaa}{\boldsymbol{\lambda}}
\newcommand{\PP}{\f{P}}
\newcommand{\AAA}{\f{A}}

\newcommand{\wwdiag}{\f{w}_{\text{diag}}}
\newcommand{\wwoff}{\f{w}_{\text{off}}}
\newcommand{\wwtildediag}{\widetilde{\f{w}}_{\text{diag}}}
\newcommand{\wwtildeoff}{\widetilde{\f{w}}_{\text{off}}}
\newcommand{\aadiag}{\f{a}_{\text{diag}}}
\newcommand{\aaoff}{\f{a}_{\text{off}}}
\newcommand{\zzdiag}{\f{z}_{\text{diag}}}
\newcommand{\zzoff}{\f{z}_{\text{off}}}

\newcommand{\ssigma}{\boldsymbol{\sigma}}
\newcommand{\mmuu}{\boldsymbol{\mu}}

\newtheorem{lemma}{Lemma}
\newtheorem{proposition}{Proposition}
\newtheorem{definition}{Definition}
\newtheorem{corollary}{Corollary}
\newtheorem{remark}{Remark}
\newtheorem{example}{Example}
\crefname{theorem}{theorem}{theorems}
\Crefname{theorem}{Theorem}{Theorems}
\crefname{lemma}{lemma}{lemmas}
\Crefname{lemma}{Lemma}{Lemmas}
\crefname{proposition}{proposition}{propositions}
\Crefname{proposition}{Proposition}{Propositions}
\crefname{definition}{definition}{definitions}
\Crefname{definition}{Definition}{Definitions}
\crefname{corollary}{corollary}{corollaries}
\Crefname{corollary}{Corollary}{Corollaries}
\crefname{remark}{remark}{remarks}
\Crefname{remark}{Remark}{Remarks}
\crefname{example}{example}{examples}
\Crefname{example}{Example}{Examples}
\crefname{appendix}{appendix}{appendices}
\Crefname{appendix}{Appendix}{Appendices}
\makeatletter
\g@addto@macro\appendix{%
  \crefalias{section}{appendix}%
  \crefalias{subsection}{appendix}%
  \crefalias{subsubsection}{appendix}%
}
\makeatother

\crefformat{equation}{\textup{#2(#1)#3}}
\crefrangeformat{equation}{\textup{#3(#1)#4--#5(#2)#6}}
\crefmultiformat{equation}{\textup{#2(#1)#3}}{ and \textup{#2(#1)#3}}
{, \textup{#2(#1)#3}}{, and \textup{#2(#1)#3}}
\crefrangemultiformat{equation}{\textup{#3(#1)#4--#5(#2)#6}}%
{ and \textup{#3(#1)#4--#5(#2)#6}}{, \textup{#3(#1)#4--#5(#2)#6}}{, and \textup{#3(#1)#4--#5(#2)#6}}

\Crefformat{equation}{#2Equation~\textup{(#1)}#3}
\Crefrangeformat{equation}{Equations~\textup{#3(#1)#4--#5(#2)#6}}
\Crefmultiformat{equation}{Equations~\textup{#2(#1)#3}}{ and \textup{#2(#1)#3}}
{, \textup{#2(#1)#3}}{, and \textup{#2(#1)#3}}
\Crefrangemultiformat{equation}{Equations~\textup{#3(#1)#4--#5(#2)#6}}%
{ and \textup{#3(#1)#4--#5(#2)#6}}{, \textup{#3(#1)#4--#5(#2)#6}}{, and \textup{#3(#1)#4--#5(#2)#6}}

%% file: abstract.tex
\begin{abstract}
    Iteratively reweighted least squares (IRLS) methods constitute a natural approach to nuclear norm minimization, but their convergence rates and the role of the weight operator have remained poorly understood.
    This paper establishes sharp convergence rates for IRLS methods for constrained nuclear norm minimization in low-rank recovery.
    A central ingredient is a new majorization analysis for the smoothed nuclear norm: we prove that the harmonic-mean weight operator defines a valid global quadratic majorizer. Furthermore, we show that this weight operator is optimal within the family of power-mean weights, clarifying why it improves over classical one-sided reweighting schemes that use only row- or column-space information. Under a Schatten-1 null space property, we prove global linear convergence of IRLS algorithms using a variety of weight operators, including the harmonic-mean weights. For IRLS with harmonic-mean weights, we prove a dimension-independent, locally linear convergence rate. We provide a counterexample showing that this dimension-independent local rate cannot in general be obtained for IRLS algorithms using one-sided weight operators, which predominate in the literature. Numerical experiments corroborate the theoretical results and illustrate the practical advantage of harmonic-mean reweighting across square, rectangular, and adversarially initialized recovery problems.
\end{abstract}


%% file: introduction.tex
\section{Introduction}\label{sec:introduction}
Optimization methods that treat a matrix variable through its \emph{spectrum} rather than through its entries have moved to the center of large-scale machine learning. The \texttt{Muon} optimizer \citep{JordanJinBoza-Muon2024} replaces the raw gradient (or momentum) matrix $\f{G} = \UU \diag(\ssigma) \VV^\top$ by its \emph{orthogonalized} counterpart $\operatorname{msign}(\f{G}) = \UU_{r} \VV_{r}^\top$, computed in practice by a few Newton--Schulz iterations \citep{Amsel-PolarExpress2025}, and has been adopted for the pre-training of large language models at the trillion-parameter scale \citep{Liu-Moonlight2025,KimiK2-2025,DeepSeekV4-2026}. The fact that makes steps involving $\operatorname{msign}(\f{G})$ principled using spectral geometry is a duality statement: since $\operatorname{msign}(\f{G})$ attains the maximum in $\nucnorm{\f{G}} = \max_{\specnorm{\f{\Delta}} \leq 1} \innerproduct{\f{G}, \f{\Delta}}$, so a \texttt{Muon} step is exactly steepest descent with respect to the spectral norm, with the nuclear norm $\nucnorm{\cdot}$ as the dual norm in which progress is measured \citep{Carlson-AISTATS2015,BernsteinNewhouse-ModularDuality2024,ChenLiLiu-MuonSpectral2025,Pethick-ICML2025}; $\f{G} \mapsto  \operatorname{msign}(\f{G})$ is one of the subgradients of the nuclear norm.

Beyond its role as a dual norm in this context, the nuclear norm has long served as the canonical convex surrogate for matrix rank \citep{Fazel02,rechtfazel_2011}. This second role is central to low-rank recovery: from underdetermined linear measurements $\f{y}=\mathcal{A}(\Xzero)$, where $\mathcal{A}:\Rdd\to\R^m$ and $m\ll d_1d_2$, one seeks to reconstruct a matrix $\Xzero$ of rank $r_\star\ll\min(d_1,d_2)$, or one that is well approximated by such a matrix. Problems for which this type of modeling arises have been prominently studied in signal processing, data science, and machine learning. Instances include phase retrieval problems \citep{Candes13}, the design of recommender systems \citep{CandesRechtMC_09,koren_bell_volinsky,chi2019nonconvex}, blind demixing \citep{ling2017_demixing,jung2018_demixing} and the quantum state tomography \citep{Gross-Quantum2010,tariq2024efficient} problem. For these problems, the benchmark convex estimator is
\begin{equation}\label{eq:nucnorm:min}
    \underset{\XX \in \Rdd}{\min}
    \nucnorm{\XX}
    \qquad\text{subject to }
    \qquad\mathcal{A} (\XX)= \f{y}. 
\end{equation}
Under standard random sensing models, \cref{eq:nucnorm:min} succeeds at near-intrinsic sample sizes. For example, $m\gtrsim r_\star(d_1+d_2)$ Gaussian rank-one measurements suffice, optimally up to constant factors \citep{ROP_CaiZhang2015,NSP_kueng,NSP_Kabanava}, and precise phase transitions have been studied in \citet{romanov2018near}. Thus, the statistical rationale for nuclear norm minimization is mature. However, the computational aspects of this problem remain challenging at scale: \cref{eq:nucnorm:min} admits an exact semidefinite programming formulation, which allows for polynomial time guarantees using generic SDP algorithms \citep{NesterovNemirovskii1994}. The non-smoothness and spectral structure of $\nucnorm{\XX}$ pose challenges for specialized solvers, which may suffer from a slow (e.g., sublinear) convergence rate and the need for repeated full singular value decompositions. 

A structurally suitable algorithmic paradigm for solving \cref{eq:nucnorm:min} is based on a framework going back to the 1930s \citep{Weiszfeld37,Beck-JOTA2015} known as iteratively reweighted least squares (IRLS), which iteratively solves a sequence of weighted least squares problems after implicitly smoothing and majorizing the non-smooth objective by a sequence of quadratic model functions.
IRLS has been widely used for separable non-smooth optimization in computer vision, robust statistics, or compressed sensing, just to name a few \citep{ochs_dosovitskiy_brox_pock,holland_welsch,Daubechies-CPAM2010}.
IRLS is also known as half-quadratic minimization \citep{NikolavaNg05} and is related to the so-called \emph{$\eta$-trick} \citep{bach2019eta} in machine learning.

While IRLS is relatively well understood for separable objectives, both its analysis and its design for spectral objectives such as \cref{eq:nucnorm:min} remain much less settled. IRLS methods for spectral objectives such as \cref{eq:nucnorm:min} use updates of the form
\begin{align} \label{eq:IRLS:matrix:formulation}
    \f{X}^{(k)} &:= 
    \argmin_{ \f{Z} \in \R^{d_1 \times d_2 }}  \ 
    \innerproduct{ \f{Z}, W^{(k)}(\f{Z}) } \; \text{ subject to } \; \mathcal{A}(\f{Z}) = \f{y},
\end{align}
where the positive-definite weight operator $W^{(k)}$ is constructed from spectral information of $\f{X}^{(k-1)}$. The main design choice is precisely how to define this operator.
Following the initial works of \citet{Fornasier11,mohan_fazel}, the majority of papers on IRLS methods for spectral objective minimization \citep{CaiLi2017,Radhakrishnan2025linear,Zhu2025iteratively} use \emph{one-sided} reweighting, where $W^{(k)}$ encodes only column space information of $\Xk{(k-1)}$ and corresponds to a multiplication by a matrix from the left (or the analogue for row space information and multiplication by a matrix from the right), whereas another line of research \citep{Kummerle-JMLR2018,KummerleMayrinkVerdun-ICML2021,Kraemer-2025OneSided} considers different weight operators that use information from both the row and column spaces of $\Xk{(k-1)}$. 
This raises the question of whether different weight operators merely reflect different trade-offs, or whether there is a principled optimal choice for the spectral minimization problem at hand.

The understanding of IRLS methods with iterates \cref{eq:IRLS:matrix:formulation} for nuclear norm minimization is also limited by a lack of available convergence guarantees. The best available results \citep{mohan_fazel,Fornasier11} only imply (global) convergence to a ground truth low-rank matrix, in case where the nuclear norm minimization problem has a unique low-rank solution,
and subsequence convergence to limiting matrices satisfying error bounds in the case of noisy measurements and/or only appropriately low-rank matrices to be recovered \citep{CaiLi2017}. This is somewhat at odds with the mature trajectory analysis of other low-rank matrix recovery methodologies for which rigorous convergence guarantees involving \emph{convergence rates} are available \citep{tu2016low,Chen2018harnessing}.

This paper aims to provide, on the one hand, principled answers to the IRLS design question for \cref{eq:nucnorm:min} and, on the other hand, the first convergence rate analysis of IRLS methods for this problem. These two goals are intertwined. In \Cref{fig:introduction}, we illustrate, for a representative rank-two matrix recovery problem involving rank-one measurements, a typical convergence trajectory of IRLS iterates $\Xk{(k)}$ following the updates of \cref{eq:IRLS:matrix:formulation} simply depending on different weight operator choices. We observe that the objective gap decreases linearly for different choices, but with a notably different linear rate, in which a so-called harmonic-mean weight operator $W^{(k)}$ leads to a notably improved linear rate.

\begin{figure}[h]
     \centering
     \includegraphics[width=0.82\textwidth]{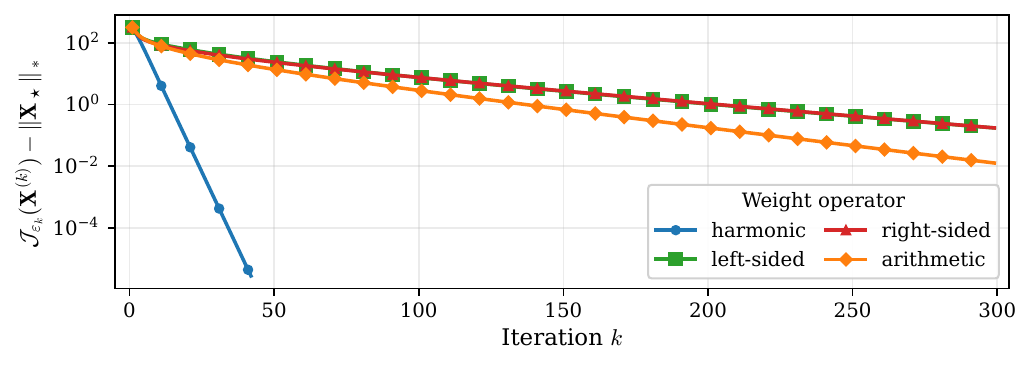}
     \caption{Smoothed nuclear norm gap $\mathcal{J}_{\varepsilon_k}(\XXk)-\nucnorm{\Xzero}$ of IRLS iterates for harmonic, left-sided, right-sided, and arithmetic mean weight operators. 
     	The rank-$2$ ground truth $\Xzero\in\R^{140\times 140}$ and the $m=2520$ rank-one measurements are generated using independent $\mathcal{N}(0,1)$ entries.
     	}
     \label{fig:introduction}
 \end{figure}
In \Cref{fig:introduction}, the objective gap is measured with respect to a family of smoothed nuclear norm objectives $\mathcal{J}_{\varepsilon}: \Rdd \to \R$ given by
 \begin{equation} \label{eq:smoothedell1:objective}
    \mathcal{J}_{\varepsilon}(\XX) := \sum_{i=1}^d j_{\varepsilon}(\sigma_i(\XX)), \quad \text{ with } \quad     j_{\varepsilon}(\sigma) := \begin{cases}
        |\sigma|, & \text{ if } |\sigma| > \varepsilon, \\
        \frac{\sigma^2}{2 \varepsilon}+ \frac{\varepsilon}{2}, & \text{ if } |\sigma| \leq \varepsilon,
    \end{cases}
\end{equation}
which are implicitly minimized by the IRLS updates \cref{eq:IRLS:matrix:formulation}---specifically, by virtue of \cref{eq:IRLS:matrix:formulation} minimizing a quadratic model $Q_{\varepsilon_k}(\,\cdot\mid \Xk{(k)})$ of $\mathcal{J}_{\varepsilon_k}$ about the current iterate, together with a smoothing-parameter update for $\varepsilon$.
An accurate convergence analysis of the resulting IRLS iterates requires that the quadratic model $Q_{\varepsilon_k}(\,\cdot\mid \Xk{(k)})$ \emph{majorizes pointwise} the smoothed objective $\mathcal{J}_{\varepsilon_k}$, which has been only known for models defined by one-sided weight operators \citep{Fornasier11,mohan_fazel}. For the harmonic-mean IRLS method, however, no majorization result has been established, impeding its convergence analysis.

In this paper, we show that the harmonic-mean weight operator not only induces a valid majorizing quadratic model of $\mathcal{J}_{\varepsilon}$, providing the foundation for global convergence guarantees, but also that it is optimal within a family of power-mean weight operators. Building on this, we provide the first convergence analysis for IRLS methods for nuclear norm minimization involving convergence rates, leading also to a sharper local linear rate for the harmonic-mean IRLS method than for classical IRLS methods. We summarize our contributions as follows:

\begin{itemize}[leftmargin=1.0em]
\item
\textbf{The Harmonic-mean Quadratic Model is a Valid Majorizer.}
In \Cref{thm:majorization}, we prove that the quadratic model induced by the harmonic mean weight operator globally majorizes the smoothed nuclear norm $\mathcal{J}_{\varepsilon}$. The proof develops a new spectral comparison argument based on a Sylvester-equation characterization and iterative pinching, thereby overcoming the noncommutativity that obstructs standard separability-based arguments.

\item
\textbf{Harmonic Mean is the Tightest Power Mean.}
Taking $q$-power means of the left- and right-sided weights
yields a Loewner-ordered family of weight operators, and hence an ordered family of quadratic models.
In \Cref{thm:power_means_majorize}, we show that $q\geq-1$ is precisely the
range that guarantees global majorization of
$\mathcal{J}_{\varepsilon}$. Consequently, the harmonic-mean choice
$q=-1$ induces the smallest, and thus tightest, quadratic
majorizer within this family.

\item
\textbf{Global and Local Convergence with Linear Rates.}
\Cref{mainresult:lowrank} establishes global linear convergence of IRLS for every admissible one-sided or power-mean weight operator and any positive-definite initialization of the weights, provided that $\mathcal{A}$
satisfies the null space property (NSP) of order $r$.
\Cref{mainresult:approximatelowrank} extends this linear decay to
approximately low-rank ground truths, up to their best rank-$r$
approximation error. The global contraction factors retain a dependence
on the matrix dimension. By contrast, once the iterates enter a
specified neighborhood of the ground truth,
\Cref{thm:locallinearp1} establishes a local linear rate for
harmonic-mean IRLS whose contraction factor is independent of
$d_1$ and $d_2$. Complementing this positive result,
\Cref{thm:counterexample:leftsided:weight:operator} constructs local
instances satisfying the NSP for which a one-sided IRLS step reduces the
relative error by at most order $r/d$, ruling out an analogous
dimension-independent local rate in general. The separation theoretically justifies the performance gap between the IRLS variants observed in \Cref{fig:introduction}. \Cref{tab:comparison:theory} places these guarantees in the context of the existing IRLS convergence theory, for both sparse vector and low-rank matrix recovery.
\end{itemize}

\begin{table}[!t]
\caption{IRLS convergence results at the convex endpoint $p=1$: the sparse
$\ell_1$ benchmark and the direct nuclear-norm comparison. The final row
records the harmonic-mean specialization of this paper.}
\label{tab:comparison:theory}
\centering
\footnotesize
\setlength{\tabcolsep}{4pt}
\renewcommand{\arraystretch}{1.12}
\begin{tabularx}{\linewidth}{@{}
>{\raggedright\arraybackslash}p{.245\linewidth}
>{\raggedright\arraybackslash}p{.12\linewidth}
>{\raggedright\arraybackslash}p{.25\linewidth}
>{\raggedright\arraybackslash}X@{}}
\toprule
Reference & Weight & Global guarantee & Local / distinguishing result \\
\midrule
\multicolumn{4}{@{}l}{\textcolor{citeblue}{\bfseries
Sparse recovery ($\ell_1$): benchmark}} \\
\addlinespace[1pt]
\citet{Daubechies-CPAM2010}
& entrywise
& conv. under NSP; no rate
& local linear rate; basin $O(\xzeromin)$ \\

\citet{Kummerle-NeurIPS2021}
& entrywise
& linear rate, $1-\frac{C}{\eta_1 d}$
& arbitrary initialization \\

\midrule
\multicolumn{4}{@{}l}{\textcolor{citeblue}{\bfseries
Low-rank recovery (nuclear norm): direct comparison}} \\
\addlinespace[1pt]
\citet{Fornasier11,mohan_fazel}
& one-sided
& conv. under NSP; no rate
& no local rate established \\

\citet{CaiLi2017}
& one-sided
& conv. under RIP; no rate
& approx. rank, noise stability \\

\citet{Kummerle-JMLR2018}
& harm. mean
& n/a
& superlinear of order $2-p$, does not apply to nuclear norm ($p=1$) \\

\addlinespace[2pt]
\textcolor{citeblue}{\textbf{This paper}}
& harm. mean
& \textbf{linear rate, $1-\frac{C}{\eta_1 d}$}
& \textbf{local linear rate $1-c_{\eta_r}$}; basin $O(\sigma_r(\Xzero)/\sqrt{d})$; \textbf{tight majorizer} \\
\bottomrule
\end{tabularx}

\smallskip
\parbox{\linewidth}{\scriptsize
For visual comparison, $d$ denotes the ambient vector dimension in the sparse
block and $d=\min(d_1,d_2)$ in the low-rank block; $\eta_1$ is
the respective order-one NSP constant.}
\end{table}

A preliminary version of the global linear convergence result for nuclear norm IRLS was presented at the 2024 IEEE 13th Sensor Array and Multichannel Signal
Processing Workshop (SAM) and appeared in its proceedings \citep{KuemmerleStoeger-2024SAM}. This four-page workshop paper did not contain proofs nor the majorization, optimality, or fast local-rate results developed here.

\paragraph*{Outline.}
In \Cref{section:algorithm}, we present relevant notions of weight operators and outline the IRLS methodology for nuclear norm minimization, as well as the majorization and optimality results for the harmonic-mean weight operator. \Cref{sec:main:results} contains the convergence results. The numerical experiments of \Cref{sec:simulations} corroborate the predicted advantage of harmonic-mean reweighting across square, rectangular, and adversarially initialized recovery problems (showing in particular that this advantage is not merely an artifact), and investigate the practical trade-offs between different smoothing-parameter schedules. \Cref{sec:conclusion} contains a conclusion and discussion of future research directions. Finally, all proofs of presented main results are contained in \Cref{sec:contribution_proofs}, with some complementary proofs implicitly known in the literature being provided in \Cref{sec:appendix:complementary}.

\paragraph*{Notation.}
Matrices are denoted by bold uppercase letters (such as $\XX \in \R^{d_1 \times d_2}$)
and vectors are denoted by bold lowercase letters ($\f{x} \in \R^d$).
The singular values of a matrix $\XX$ are denoted by $\sigma_1(\XX) \geq \sigma_2(\XX) \geq \ldots \geq 
\sigma_{\min(d_1,d_2)}(\XX) \geq 0$. We use the convention of $\sigma_i = 0$ for $i > \min(d_1, d_2)$.
The Frobenius inner product of two matrices $\XX, \f{Y} \in \R^{d_1 \times d_2}$ is denoted by 
$\innerproduct{\XX, \f{Y}} = \trace(\XX^{\top} \f{Y})$. The spectral norm of a matrix $\XX$ is denoted by $\specnorm{\XX} = \sigma_1(\XX)$ and the nuclear norm of a matrix $\XX$ is denoted by 
$\nucnorm{\XX} = \sum_{i=1}^{\min(d_1,d_2)} \sigma_i(\XX)$. We denote by $\XX \circ \f{Y}$ the Hadamard (entrywise) product of two matrices $\XX$ and $\f{Y}$.
Whenever applicable,
we use the notation of $d= \min(d_1,d_2)$ and $D = \max(d_1,d_2)$, and $[d]:=\{1,2,\ldots,d\}$. Finally,
for $r\in\{0,\ldots,d\}$ and $\XX\in\Rdd$, we denote the error
of the best rank-$r$ approximation of $\XX$ in nuclear norm by 
$
\besterrNuc{\XX}{r}
:=
\min_{\rank(\ZZ)\le r}\nucnorm{\XX-\ZZ}
=
\sum_{i=r+1}^{d}\sigma_i(\XX)$. Equivalently, $\besterrNuc{\XX}{r}$ is the nuclear-norm tail of
the singular-value vector of $\XX$ after its first $r$ entries.

%% file: related_work.tex
\section{Related Work}\label{sec:related_work}
\paragraph{Low-Rank Matrix Recovery.}
Popular methodologies for low-rank matrix recovery impose a rank-\(r\) constraint by construction, using
iterative rank projections \citep{jain2010guaranteed}, factorizations such as
\(\XX=\f{U}\f{V}^{\top}\) \citep{burer2005}, or Riemannian optimization
\citep{tu2016low,Vandereycken13}. See \citet{chi2019nonconvex} for an overview. 
Under a rank-restricted isometry property,
singular value projection converges globally at a linear rate from zero, whereas factorized and Riemannian gradient
methods are commonly shown to converge linearly only after a certified
initialization enters a local basin. 
These methods are generally analyzed together with spectral initialization, 
which is often costly to compute and may be unstable in the presence of noise or outliers.
Global convergence from random initialization has only been proven so far for specific measurement models,
e.g., for rank-one phase retrieval \citep{Chen-MP2019},
and for small random initialization \citep{stoger2021small,soltanolkotabi2025implicit}.
Moreover, standard factorized gradient descent slows as
the condition number of the ground truth increases.
More recent variants can remove condition-number dependence from the local contraction or attain faster local rates, although end-to-end guarantees still depend on the initialization \citep{tong2021_accelerating,CaiWuXia2025,Zilber2022_GNMR,luo2023recursive}. All these formulations constitute nonconvex estimators 
and require the target rank for provable convergence. 

Nuclear norm minimization \cref{eq:nucnorm:min}, in contrast, is the
standard convex estimator and does not fix the rank in its formulation.
Under suitable RIP or null space conditions, it achieves stable recovery
at the degrees-of-freedom sampling order $r(d_1+d_2)$, up to
model-dependent logarithmic factors and independently of $\kappa$
\citep{rechtfazel_2011,Gross2011_recovering,NSP_kueng,NSP_Kabanava},
see also \citet{davenport2016overview,fuchs22_lowrank} for an overview.
Standard first-order schemes such as singular value thresholding may provide 
faster per-iteration complexity than generic semidefinite programming solvers, 
but are typically analyzed with sublinear rates \citep{Cai2010-SingularValueThresholding}. Unlike for nonconvex estimators, algorithms with provable global linear rates under near-optimal sample complexity without $\kappa$-dependence are available for \cref{eq:nucnorm:min} or closely related unconstrained formulations, for example for the restarted accelerated primal-dual method of \citet{Colbrook2022WARPd} and the restarted mirror descent method of \cite{DingWang2026Sharpness}, under suitable sharpness conditions (e.g., Frobenius-robust NSP, cf. \Cref{sec:preliminaries}). The global linear rate we show for IRLS variants in \Cref{mainresult:lowrank} is likewise independent of $\kappa$, but neither requires a specific step size nor restart schedule. The faster local rate of \Cref{thm:locallinearp1} for harmonic mean \texttt{MatrixIRLS} only depends on the NSP constant $\eta_r$ and not directly on the dimension, unlike the certified rates of \citet{Colbrook2022WARPd} and
\citet{DingWang2026Sharpness}.

\paragraph{Low-Rank IRLS.}
The foundational works by \citet{Fornasier11} and \citet{mohan_fazel} proposed the first IRLS algorithms for low-rank matrix recovery, using left-sided reweighting for nuclear norm minimization and right-sided reweighting for Schatten-$p$ quasi-norm minimization with $p \in (0,1]$ (see \Cref{def:weightcore}), respectively. Both works establish global convergence of the respective IRLS algorithm under the NSP assumption (see \Cref{def:NSP:statement}), albeit \emph{without any convergence rate}.
As a tool, they implicitly establish the majorization property \Cref{prop:global:majorization:onesided} for the respective quadratic model functions (cf. \Cref{sec:appendix:concavity:majorization_proof,sec:appendix:variational:majorization_proof} for details).
\citet[Theorem 3.6]{Lai-SIAM-J-NA2013} provide an error bound for the limiting iterate of an IRLS algorithm of an unconstrained variant of \cref{eq:nucnorm:min} under an RIP assumption. On the other hand, the local linear rate analysis that is claimed by \citet[p. 950]{Lai-SIAM-J-NA2013} is not substantiated. \citet{CaiLi2017} established stability guarantees for the limiting iterate of one-sided IRLS for nuclear norm minimization subject to a residual norm inequality constraint. 

\citet{Kummerle-JMLR2018} proposed IRLS algorithm variants for Schatten-$p$ quasi-norm minimization with $p \in (0,1]$ using harmonic-mean weight operators motivated from a perspective to treat column and row space information symmetrically, and showed local superlinear convergence of order $2-p$ in a neighborhood of the ground truth matrix $\Xzero$ under NSP assumptions for these IRLS variants; however, the relevant result \citep[Theorem 11]{Kummerle-JMLR2018} does \emph{not apply} for the nuclear norm case of $p=1$, but is only meaningful in the quasi-norm case where $p \in (0,1)$. \Cref{alg:algo1} studied in this paper is a variant of \citet[Alg.~1]{Kummerle-JMLR2018} specified to $p=1$, but using a different smoothing parameter schedule (see \Cref{sec:smoothing:parameter}). Finally, the proof of the technical result in \citet[Lemma 14]{Kummerle-JMLR2018}, which would imply a majorization property of the harmonic-mean quadratic model, is faulty, as we point out in \Cref{sec:appendix:challenges:harmonic:majorization_proof}. Our proof of \Cref{thm:majorization} overcomes this issue and provides the foundation for a global analysis of \MatrixIRLSHeading{}. A related line of literature \citep{KummerleMayrinkVerdun-ICML2021,GTK24} focuses on the nonconvex log-determinant minimization problem and establishes locally quadratic convergence rates for suitable geometric mean weight operator-based IRLS variants under near-optimal sample complexity assumptions for the respective problems of matrix completion and Euclidean distance geometry. Recently, \citet{Kraemer-2025OneSided} established asymptotic convergence properties of IRLS methods for log-determinant minimization and emphasizes the importance of choosing an appropriate smoothing parameter schedule, particularly at the information-theoretic limit. 

On a related note, the framework of recursive feature machines (RFM), a novel methodology for sample-efficient feature learning \citep{Radhakrishnan2024mechanism}, has been identified to be related to IRLS methods for spectral optimization problems such as \cref{eq:nucnorm:min} if specified to linear models \citep{Radhakrishnan2025linear}. \citet{Radhakrishnan2025linear} provides a derivation of one-sided IRLS methods for log-determinant and nuclear norm minimization and related objectives based on the RFM framework and proposes an SVD-free implementation of one-sided IRLS for log-determinant minimization. However, the proposed SVD-free implementation does not apply to the nuclear norm case, and no rigorous convergence analysis of low-rank IRLS methods has emerged from this framework.

\paragraph{IRLS for Sparse Recovery.}
IRLS algorithms for sparse vector recovery, which precede low-rank IRLS methods, are also relevant for the context of this work \citep{gorodnitsky_rao,rao_kreutz-delgado,wipf_nagarajan}. These algorithms have been well-known as efficient solvers for $\ell_p$-quasi-norm minimization problems for $0<p \leq 1$. \citet{chartrand_yin} and \citet{wipf_nagarajan} proposed and empirically studied 
adaptive smoothing parameter schedules for $\varepsilon$, observing their importance for the success of the methodology. \citet{Daubechies-CPAM2010} analyzes 
an adaptive scheme akin to \cref{eq:IRLS:step_2} for the smoothing parameter $\varepsilon$ of an IRLS algorithm for $\ell_1$-minimization under underdetermined linear measurements and provides a global convergence guarantee for the method under a sparse NSP assumption on the measurement matrix, while also showing locally linear convergence (independent of the ambient dimension), the analogue of \Cref{thm:locallinearp1} in this work, and locally superlinear convergence for IRLS targeted for $\ell_p$-quasi-norm minimization with $0<p<1$. Global linear convergence rates were established in \citet{Kummerle-NeurIPS2021} for $\ell_1$-minimization and in \citet{PKV22,PengKuemmerleVidal-CVPR2023} for related robust estimation problems. 

The fact that vector $\ell_p$-objectives of vectors are separable significantly simplifies the design and analysis of relevant IRLS methods compared to IRLS for spectral optimization, as studied in this paper. 
In particular, sparse IRLS offers fewer degrees of freedom in the design of the weight operator, with no analogue of the off-diagonal entries of the weight operator core matrix (\Cref{def:weightcore}).
We note that the usage of particularly tight quadratic model designs as studied in this work is necessary to achieve convergence rate results for \MatrixIRLSHeading{} that are in line with the ones for sparse recovery (see \Cref{tab:comparison:theory}).

%% file: algorithm.tex
\section{The \MatrixIRLSHeading{} Algorithm and Its Majorization Properties}\label{section:algorithm}

In this section, we first introduce the \texttt{MatrixIRLS} algorithm, the associated weight operators, and the resulting quadratic models. We then show that the harmonic-mean weight operator yields a valid majorizer, before discussing its optimality within the power-mean family.

\subsection{IRLS for Low-Rank Recovery and Basic Properties} \label{sec:algo_IRLS} 
Iteratively reweighted least squares can be interpreted as a \emph{smoothing method} \citep{Chen-MP2012}, where a Huber-type smoothing \citep{Huber-1964} of the nuclear norm is minimized via quadratic majorizing models \citep{Lange-MM2016,SunBabuPalomar-IEEESP2017}.
In particular, IRLS mitigates the non-smoothness of the nuclear norm objective \cref{eq:nucnorm:min}
by working with \emph{$\varepsilon$-smoothed nuclear norms} $\mathcal{J}_{\varepsilon_k}$ of \cref{eq:smoothedell1:objective}.  $\mathcal{J}_{\varepsilon_k}$ is in turn minimized approximately by constructing a quadratic model $Q_{\varepsilon_k}(\,\cdot\mid \Xk{(k)})$ about the current iterate $\XX^k$.

It can be shown that $\mathcal{J}_{\varepsilon}(\cdot)$ is differentiable with a Lipschitz-continuous gradient (see \Cref{sec:appendix:lipschitz:gradients}). But despite the surrogate $\mathcal{J}_{\varepsilon}(\cdot)$ being smooth, minimizing it directly remains challenging because it depends nonlinearly on the singular values of its argument. The IRLS strategy is thus to replace $\mathcal{J}_{\varepsilon}$ locally by quadratic models $Q_{\varepsilon}\left(\cdot \mid \XX \right): \Rdd \to \R$ given suitable reference points $\XX \in \Rdd$ (see \Cref{eq:smoothedell1:IRLSmajorizer} below). If $Q_{\varepsilon}\left(\cdot \mid \XX \right)$ \emph{majorizes} $\mathcal{J}_{\varepsilon}$ globally, IRLS can be interpreted as a Majorization-Minimization (MM) method \citep{Lange-MM2016,SunBabuPalomar-IEEESP2017} intertwined with smoothing \citep{Chen-MP2012}. Since $\mathcal{J}_{\varepsilon}$ is a spectral function, the curvature information entering the quadratic model is naturally expressed in the singular-vector coordinates of the reference matrix $\XX$. Analogously to IRLS algorithms for separable problems \citep{wipf_nagarajan,Daubechies-CPAM2010}, which encode inverse magnitude information into entrywise weights, we incorporate the singular value information into the \emph{weight operator core matrix}, which in turn fixes the curvature of the quadratic model. However, as we will see, the non-diagonal entries of this matrix contain degrees of freedom which can be chosen in different ways, leading to a multitude of possible quadratic models. 

\begin{definition}[Weight Operator Core Matrix]\label{def:weightcore}
	Let $\ssigma \in \R^{\maxD}$ be a nonincreasing vector of singular values
	(padded by zeros for indices larger than $\mind$), and let $\varepsilon>0$ be a smoothing parameter.
	The associated \emph{weight operator core matrix}
	$\f{H}_{\ssigma,\varepsilon} \in \R^{d_1 \times d_2}$ is defined entrywise for $i\in[d_1]$ and $j\in[d_2]$, by one of the following choices:
	\begin{align}
		(\f{H}_{\ssigma,\varepsilon})_{ij}
		&= [\max(\sigma_i,\varepsilon)]^{-1},
		&& \text{left-sided \citep{Fornasier11}},
		\label{eq:leftsided:mean}
		\\
		(\f{H}_{\ssigma,\varepsilon})_{ij}
		&= [\max(\sigma_j,\varepsilon)]^{-1},
		&& \text{right-sided \citep{mohan_fazel}},
		\label{eq:rightsided:mean}
		\\
		(\f{H}_{\ssigma,\varepsilon})_{ij}
		&= 2[\max(\sigma_i,\varepsilon)+\max(\sigma_j,\varepsilon)]^{-1},
		&& \text{harmonic-mean \citep{Kummerle-JMLR2018}}.
		\label{eq:harmonic:mean}
	\end{align}
\end{definition}

The three core matrix types \cref{eq:leftsided:mean,eq:rightsided:mean,eq:harmonic:mean} differ only in the way they compute non-diagonal entries, coinciding on the diagonal. $\f{H}_{\ssigma,\varepsilon}$ of \cref{eq:harmonic:mean} uses the harmonic mean of \cref{eq:leftsided:mean} and \cref{eq:rightsided:mean} of indices corresponding to row and column indices of the matrix, whereas the entries of \cref{eq:leftsided:mean} and \cref{eq:rightsided:mean} depend only on singular values associated to the matrix's row and column indices, respectively. With \Cref{def:weightcore}, we can define the \emph{weight operator} $\W{\XX}{\varepsilon}: \Rdd \to \Rdd$,
associated to a smoothing parameter $\varepsilon > 0$ and a matrix iterate $\XX$, 
specifying a quadratic model $Q_{\varepsilon}(\cdot \mid\XX)$ of $\mathcal{J}_{\varepsilon}(\cdot)$ about $\XX$.

\begin{definition}[{\citeauthor{Kummerle-JMLR2018}, \citeyear{Kummerle-JMLR2018}; \citeauthor{KummerleMayrinkVerdun-ICML2021}, \citeyear{KummerleMayrinkVerdun-ICML2021}}] \label{def:optimalweightoperator}
Let $\varepsilon > 0$ and $\XX \in \Rdd$ be a matrix with the full singular value decomposition 
$\XX= \UU_{\XX} \diag (\ssigma) \VV_{\XX}^\top$, 
where $\UU_{\XX}  \in \R^{d_1 \times d_1}$ and $\VV_{\XX} \in \R^{d_2 \times d_2}$ are orthogonal
and $\diag (\ssigma) \in \Rdd$ is the rectangular diagonal matrix with the extended vector of singular values $\ssigma \in \R^{\maxD}$ of $\XX$ on the diagonal.
Then the  \emph{weight operator}  $\W{\XX}{\varepsilon}: \Rdd \to \Rdd$ 
is defined by
\begin{equation} \label{eq:W:operator:action}
	\W{\XX}{\varepsilon}(\ZZ) = \UU_{\XX} \left[\f{H}_{\ssigma,\varepsilon} \circ (\UU_{\XX}^{\top} \ZZ \VV_{\XX})\right] \VV_{\XX}^{\top}
\end{equation}
for each $\ZZ \in \Rdd$.
Here $\f{H}_{\ssigma,\varepsilon}$ is the 
weight operator core matrix of \Cref{def:weightcore}.
\end{definition}

The weight operator $\W{\XX}{\varepsilon}(\cdot)$ is self-adjoint.
Using \Cref{def:optimalweightoperator}
with any of the core matrices, we define the \emph{quadratic model function of $\mathcal{J}_{\varepsilon}(\cdot)$ about $\XX$} associated to the weight operator $\W{\XX}{\varepsilon}(\cdot)$,
written $Q_{\varepsilon}\left(\cdot \mid \XX \right): \Rdd \to \R$, such that for all $\ZZ \in \Rdd$,
\begin{equation} \label{eq:smoothedell1:IRLSmajorizer}
\begin{split}
Q_{\varepsilon}(\ZZ \mid \XX) 
&:= \mathcal{J}_{\varepsilon}(\XX) 
 + \innerproduct{ \nabla \mathcal{J}_{\varepsilon}(\XX) , \ZZ-\XX }
+ \frac{1}{2} \innerproduct{ \ZZ-\XX, \W{\XX}{\varepsilon} (\ZZ-\XX)}.
\end{split}
\end{equation}

\begin{example}\label{ex:toy-example}
	To make the geometry of the smoothed nuclear norm and its quadratic surrogates (which depend on the weighting choice, be it left-sided, right-sided, or harmonic-mean) explicit, consider a simple $2 \times 2$ example chosen to be nonsymmetric so that the harmonic, left-sided, and right-sided weights are genuinely different.
	Let
	\begin{equation}\label{eq:toy-example}
		X(t)=
		\begin{pmatrix}
			4&3\\
			2&t
		\end{pmatrix},
		\qquad t\in\mathbb R,
	\end{equation}
	corresponding to the situation where only the last entry of a $2\times 2$ matrix is unknown.
	One can immediately see that $t=3/2$ gives the only rank-deficient candidate. It is also straightforward to calculate that
	\[
	\|X(t)\|_*
	=
	\sqrt{29+t^2+8|t-3/2|}
	\]
	and verify that on the $t<3/2$ branch the nuclear norm is decreasing in $t$ while on the $t > 3/2$ branch it is increasing.
	Thus, $t_\star=3/2$ is also the unique nuclear norm minimizer.
	
	\Cref{fig:toy-example} (left) 
	displays the nuclear norm (although it appears piecewise linear in $t$, it is not) 
	with its several smoothed versions. 
	\Cref{fig:toy-example} (right) then shows one step of \texttt{MatrixIRLS} 
	from the initial point $t_0=1/2$ 
	and the smoothing parameter $\varepsilon=0.5$. 
	The three weighting schemes introduced in \Cref{def:weightcore} lead to different surrogates, 
	of which the harmonic-mean surrogate $Q_{0.5}^{\mathrm{harm}}$ is the tightest. 
	In the next step of \texttt{MatrixIRLS}, this surrogate is minimized to obtain $t_1$, $\varepsilon$ is reduced, 
	and a new surrogate is constructed at $t_1$, to be minimized in the next step.
\end{example}

\begin{figure}[t]
	\centering
	\begin{tabular}{@{}cc@{}}
		\includegraphics[width=0.49\textwidth]{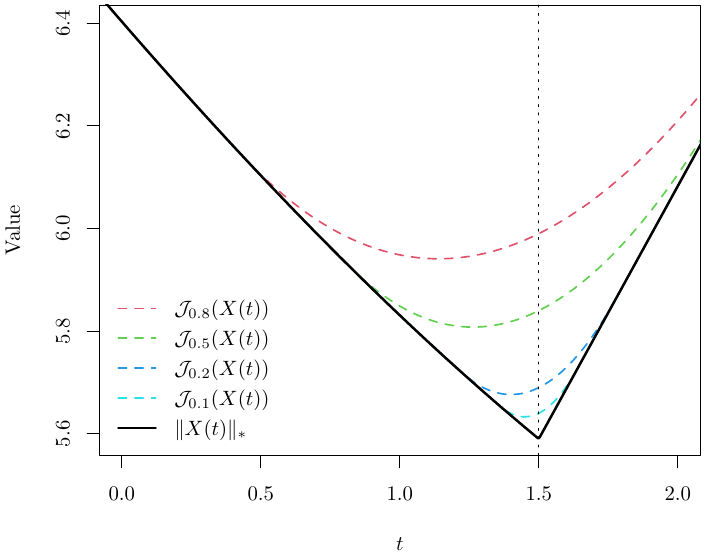} &
		\includegraphics[width=0.49\textwidth]{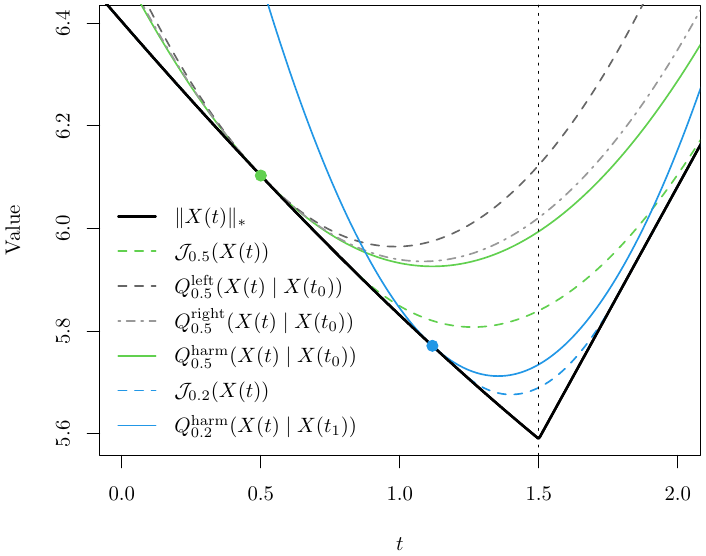}
	\end{tabular}
	\caption{The nuclear norm of \eqref{eq:toy-example} from \Cref{ex:toy-example}. \emph{Left}: The smoothed nuclear norm $\mathcal{J}_\varepsilon$ 
	for $\varepsilon=0.8, 0.5, 0.2$ and $0.1$. 
	\emph{Right}: The left-sided, right-sided, and harmonic-mean surrogates of $\mathcal{J}_{0.5}$ at $t_0=1/2$. The next step of the algorithm with $\varepsilon$ reduced to $0.2$ is also indicated.
	}
	\label{fig:toy-example}
\end{figure}

The following proposition states elementary properties of this expression for weighting schemes introduced above. Exactness of the quadratic model at $\ZZ = \XX$  follows from its definition, and a specific gradient matching leads to a simplification of the quadratic model.
\begin{proposition}[{Gradient Condition and Symmetry of $\boldsymbol{Q}_{\boldsymbol{\varepsilon}}$}] \label{proposition:IRLS:basicproperties}
	\allowbreak Let $\varepsilon >0$,  $\mathcal{J}_{\varepsilon}: \Rdd  \to \R$ be defined as in \cref{eq:smoothedell1:objective} 
  and $Q_{\varepsilon}(\cdot \mid \XX): \Rdd \to \R$ 
  as defined in \cref{eq:smoothedell1:IRLSmajorizer} via weight operator $\W{\XX}{\varepsilon}: \Rdd \to \Rdd$ defined in \Cref{def:optimalweightoperator}. Then the quadratic model and the smoothed surrogate function $\mathcal{J}_{\varepsilon}(\cdot)$ coincide at the reference point, i.e., $Q_{\varepsilon}(\XX\mid\XX) = \mathcal{J}_{\varepsilon}(\XX)$ for $\XX \in \Rdd$. Furthermore, if the weight operator core matrix $\f{H}_{\ssigma, \varepsilon} \in \Rdd$ satisfies $(\f{H}_{\ssigma, \varepsilon})_{ii} = \max(\sigma_i,\varepsilon)^{-1}$ for all $i \in [d]$, then, for any  $\XX \in \Rdd$, the following gradient condition holds: 
		\begin{equation} \label{eq:gradientcondition}
			\W{\XX}{\varepsilon}(\XX) = \nabla\mathcal{J}_{\varepsilon}(\XX) \quad \text{ for each } \XX \in \Rdd,
		\end{equation}
		and the quadratic model satisfies for each $\ZZ,\XX \in \Rdd$ that
		\begin{equation}\label{eq:QZX:equality}
		Q_{\varepsilon}(\ZZ\mid\XX)  
    = 
    \mathcal{J}_{\varepsilon}(\XX) 
    + \frac{1}{2} \innerproduct{ \ZZ, \W{\XX}{\varepsilon}(\ZZ) }
    - \frac{1}{2} \innerproduct{ \XX,\W{\XX}{\varepsilon} (\XX) }.
		\end{equation}
\end{proposition}
The proof of \Cref{proposition:IRLS:basicproperties} is straightforward and provided for completeness in \Cref{sec:appendix:proofbasiclemma}. Inspecting the weight operator core matrix of \Cref{def:weightcore}, we observe that the diagonal condition $(\f{H}_{\ssigma, \varepsilon})_{ii} = \max(\sigma_i,\varepsilon)^{-1}$ for all $i \in [d]$, and thus, the assertion of \Cref{proposition:IRLS:basicproperties}, is satisfied for both one-sided weight operator notions \cref{eq:leftsided:mean} and \cref{eq:rightsided:mean}, and for the harmonic mean weight operator \cref{eq:harmonic:mean}.

Observe that, by \cref{eq:QZX:equality}, minimizing $Q_{\varepsilon}(\ZZ\mid\XX)$ corresponds to minimizing a \emph{reweighted} least squares objective $\innerproduct{ \ZZ, \W{\XX}{\varepsilon}(\ZZ) }$ once the reference point $\XX$ and the smoothing parameter $\varepsilon$ are fixed. This fact gives \emph{iteratively reweighted least squares (IRLS)} methods their name and motivates methods that alternate between updating the reference point $\XX$, the smoothing parameter $\varepsilon$, and the weight operator $\W{\XX}{\varepsilon}(\cdot)$. Instantiating such a method with harmonic-mean weight operators in the quadratic model $Q_{\varepsilon}(\cdot \mid\XX)$ yields the IRLS method of main interest in this paper, which we call \texttt{MatrixIRLS} and state as \Cref{alg:algo1}. When discussing the convergence properties of \texttt{MatrixIRLS} in \Cref{sec:main:results}, we will also refer to variants of \Cref{alg:algo1} which use alternative weight operators instead of harmonic-mean weight operators---in these cases, only the update rule for the governing weight operator $W^{(k)}$ of \cref{eq:IRLS:step_1} changes, whereas all other aspects of the algorithm remain unchanged.

\begin{algorithm}[h]
	\caption{\MatrixIRLSHeading{} for Nuclear Norm Minimization}
	\label{alg:algo1}
	\begin{algorithmic}
		\STATE{\textbf{Input:} Operator $\mathcal{A}: \Rdd \to \R^{m}$,
			data vector $\f{y} \in \mathbb{R}^m$, 
			initial positive definite weight operator $W^{(0)}: \Rdd \to \Rdd$ (default: $W^{(0)} = \Id$),
			rank estimate $\widetilde{r} < \mind$.}
	\STATE{Set $\varepsilon_{-1} = \infty$.}
	\FOR{$k = 0, 1, 2,\ldots$}
	\STATE \textbf{Solve weighted least squares problem:}
	\begin{align} 
		\Xk{(k)} &:= 
		\argmin_{ \ZZ \in \R^{d_1 \times d_2 }}  \ 
		\innerproduct{ \ZZ, W^{(k)}(\ZZ) } \; \text{ subject to } \; \mathcal{A}(\ZZ) = \f{y}. \label{eq:IRLS:step_1} 
	\end{align}
	\vspace*{-2mm}
	\STATE \textbf{Update smoothing parameter} $\varepsilon$:
	\begin{align}
		\varepsilon_{k}&:=  \min\left(\varepsilon_{k-1} ,\,  \besterrNuc{\Xk{(k)}}{\widetilde{r}}/d \right), \label{eq:IRLS:step_2}
	\end{align}
	where $\besterrNuc{\Xk{(k)}}{\widetilde{r}}$
	is the nuclear norm error of the best rank-$\widetilde{r}$ approximation of $\Xk{(k)}$.
	\STATE \textbf{If} $\varepsilon_{k}=0$, i.e., if $\besterrNuc{\Xk{(k)}}{\widetilde{r}}=0$, \textbf{return} $\Xk{(k)}$.
	\STATE \textbf{Update weight operator}: Define $W^{(k+1)} := \W{\Xk{(k)}}{\varepsilon_{k}}$ using weight operator \cref{eq:W:operator:action} from \Cref{def:optimalweightoperator}, where the weight operator core matrix is defined using the \emph{harmonic-mean weights} \cref{eq:harmonic:mean}. 
	\ENDFOR
	\RETURN Sequence $(\Xk{(k)})_{k\geq 0} $.
\end{algorithmic}
\end{algorithm}

We note that it is \emph{not} necessary to compute the entire spectral information of $\Xk{(k)}$ to implement the weight update in \Cref{alg:algo1}. On the contrary, the weight operator update in \Cref{alg:algo1} can be implemented by computing \emph{only the first} $r_{k}$ singular values $\sigma_i(\Xk{(k)})$ and matrices $\f{U}\hk \in \R^{d_1 \times r_{k}}$ and $\f{V}\hk \in \R^{d_2 \times r_{k}}$ with leading $r_{k}$ left and right singular vectors of $\Xk{(k)}$, where $r_{k} := |\{i \in [d]: \sigma_i(\Xk{(k)}) > \varepsilon_{k}\}|$. While not the focus of the present paper, we refer to \citet{KummerleMayrinkVerdun-ICML2021} and \citet{GTK24} for more details on the implementation of such low-rank IRLS algorithms. Furthermore, we note that the smoothing update \cref{eq:IRLS:step_2} is different from the ones proposed in the existing literature \citep{Fornasier11,Kummerle-JMLR2018,KummerleMayrinkVerdun-ICML2021}. The reasons for this choice are detailed in \Cref{sec:smoothing:parameter} and \Cref{sec:smoothing:parameter:experiments}.

The update \cref{eq:IRLS:step_2} returns $\varepsilon_k = 0$ precisely if $\Xk{(k)}$ has rank at most $\widetilde{r}$; the weight operator update is then undefined, which is why \Cref{alg:algo1} returns, and the returned matrix is a feasible matrix of rank at most $\widetilde{r}$, hence equal to a ground truth $\Xzero$ of rank at most $\widetilde{r} = r$ if $\mathcal{A}$ satisfies the NSP of \Cref{def:NSP:statement}: the difference $\f{N} := \Xk{(k)} - \Xzero \in \ker(\mathcal{A})$ has rank at most $2r$, so that $\sum_{i>r} \sigma_i(\f{N}) \leq \sum_{i \leq r} \sigma_i(\f{N}) \leq \eta_r \sum_{i>r} \sigma_i(\f{N})$ by \cref{eq:NSP:definition}, and $\eta_r < 1$ leaves only $\f{N} = \f{0}$. Accordingly, all statements about $\Xk{(k)}$ and $\varepsilon_k$ in the following refer to the iterations carried out by \Cref{alg:algo1}, where a final iteration with $\varepsilon_k = 0$ is covered by the convention $j_{0}(\sigma) := |\sigma|$ in \cref{eq:smoothedell1:objective}, i.e., $\mathcal{J}_{0} = \nucnorm{\cdot}$, which is the pointwise limit of $\mathcal{J}_{\varepsilon}$ as $\varepsilon \to 0^{+}$.

\subsection{Majorization of Quadratic Model Implied by Harmonic-Mean Weights} \label{sec:majorization:properties}

The basic properties of \Cref{proposition:IRLS:basicproperties} are algebraic consequences of the weight
construction. They explain why the quadratic model leads to a weighted least-squares subproblem, but they are insufficient for an analysis of the IRLS algorithm's properties. The missing piece is a statement that connects the weighted least-squares solution $\Xk{(k)}$ of \cref{eq:IRLS:step_1} to progress in the objective value of the $\varepsilon$-smoothed nuclear norm objective $\mathcal{J}_{\varepsilon}(\cdot)$, such as
\begin{equation} \label{eq:minimizer:majorization}
	 Q_{\varepsilon_{k-1}}(\Xk{(k)}\mid\Xk{(k-1)}) \geq \mathcal{J}_{\varepsilon_{k-1}}(\Xk{(k)})
\end{equation}
for a pair of algorithm iterates $\Xk{(k-1)}$ and $\Xk{(k)}$ of \Cref{alg:algo1}. Together with the basic properties from \Cref{proposition:IRLS:basicproperties}, the smoothing parameter update step \cref{eq:IRLS:step_2} and the fact that $\varepsilon \mapsto \mathcal{J}_{\varepsilon}(\XX)$ is monotonically nondecreasing, \cref{eq:minimizer:majorization} implies that the two iterates $\Xk{(k-1)}$ and $\Xk{(k)}$ of \Cref{alg:algo1} satisfy for all $k \geq 1$ that
\begin{equation} \label{eq:J:monotonicity:1}
	\mathcal{J}_{\varepsilon_{k}}(\Xk{(k)}) \leq \mathcal{J}_{\varepsilon_{k-1}}(\Xk{(k)}) \leq Q_{\varepsilon_{k-1}}(\Xk{(k)}\mid\Xk{(k-1)}) \leq Q_{\varepsilon_{k-1}}(\Xk{(k-1)}\mid\Xk{(k-1)}) = \mathcal{J}_{\varepsilon_{k-1}}(\Xk{(k-1)}).
\end{equation}
Therefore, the sequence 
$\left\{ \mathcal{J}_{\varepsilon_k}\left(\Xk{(k)}\right)\right\}_{k}$ 
is nonincreasing. If $\bar{\varepsilon}:=\lim_{k\to\infty}\varepsilon_k>0$, 
standard IRLS arguments, under the remaining assumptions of the respective analyses, then imply that every accumulation point is stationary for the constrained minimization of 
$\mathcal{J}_{\bar\varepsilon}$, cf.~\citet[Theorem 6.11]{Fornasier11} and 
\citet[Theorem 9]{Kummerle-JMLR2018}.
A majorization property such as \cref{eq:minimizer:majorization} is likewise essential for the convergence-rate analysis in \Cref{sec:main:results}.

However, as we discuss in \Cref{sec:appendix:challenges:harmonic:majorization_proof}, existing proofs and analogous strategies fail to establish a corresponding statement for the quadratic model functions defined by harmonic-mean weight operators (and thus, for the setting of \Cref{alg:algo1}), which constitute in some sense a tighter approximation of the smoothed nuclear norm objectives than their one-sided analogues.
As the main result of this paper, we establish such a \emph{global majorization} for $Q_{\varepsilon}(\cdot \mid \XX)$ defined by the harmonic-mean weight operator.

\begin{theorem}[{Global Majorization of Harmonic-Mean Quadratic Model}] \label{thm:majorization}
	Let $\varepsilon > 0$, let $\mathcal{J}_{\varepsilon}: \Rdd  \to \R$ be the $\varepsilon$-smoothed nuclear norm \cref{eq:smoothedell1:objective} and $Q_{\varepsilon}(\cdot \mid \XX): \Rdd \to \R$ be the quadratic model function of \cref{eq:smoothedell1:IRLSmajorizer} defined by the harmonic-mean weight operator \cref{eq:W:operator:action} with core matrix \cref{eq:harmonic:mean}. Then, $Q_{\varepsilon}(\cdot \mid \XX)$ majorizes $\mathcal{J}_{\varepsilon}$ globally, i.e.,
	\begin{equation} \label{eq:majorization:inequality}
	Q_{\varepsilon}(\ZZ \mid\XX) \geq \mathcal{J}_{\varepsilon}(\ZZ)
	\end{equation}
	for each $\ZZ,\XX \in \Rdd$.
\end{theorem}
Our novel proof approach relies on directly establishing a suitable lower bound for the weighted inner product term
$\innerproduct{\ZZ, \W{\XX}{\varepsilon}(\ZZ)}$ using the spectral properties of  $\ZZ$. More precisely, in the harmonic-mean case, we can show that given $\ZZ$, the matrix $\W{\XX}{\varepsilon}(\ZZ)$ solves a Sylvester equation. Using pinching techniques, we can then replace this Sylvester equation iteratively by simpler ones, until we can finally obtain a closed-form expression for the weighted inner product term. We provide a detailed proof of \Cref{thm:majorization} in \Cref{sec:appendix:harmonic:majorization_proof}. We think that it might be of independent interest for other approximation problems involving spectral functions beyond nuclear norm-type problems. A specialization of the novel proof technique to the case of one-sided weights and \Cref{prop:global:majorization:onesided} is discussed in \Cref{sec:appendix:proof_majorization_one_sided_weights}.

\begin{remark}
	While the majorization-minimization (MM) framework \citep{LangeHunterYang-2000,Lange-MM2016} provides an insightful perspective for understanding IRLS, we note that the IRLS smoothing parameter update step of \Cref{alg:algo1} is uncharacteristic of classical MM approaches. Therefore, even with the majorization result at hand, analyzing the convergence properties of the algorithm does not reduce to standard MM arguments, and still poses a notable difficulty, in particular, as it pertains to the convergence rate analysis of \Cref{alg:algo1}.
\end{remark}
\subsection{Choice of the Smoothing Parameter} \label{sec:smoothing:parameter}
It is well-known in the IRLS literature \citep{Daubechies-CPAM2010,AravkinBurkeHe19,PKV22,LermanLiMaunuZhang-2025Global} that the choice of the smoothing parameter $\varepsilon$ is crucial for both the theoretical analysis of an IRLS algorithm and its empirical performance. 
The update rule \cref{eq:IRLS:step_2} proposed in \Cref{alg:algo1} differs from the prevailing update rules in the low-rank IRLS literature, using the scaled $\ell_1$-tail $\besterrNuc{\Xk{(k)}}{\widetilde{r}}/d = \sum_{i=\widetilde{r}+1}^{d} \sigma_i(\Xk{(k)}) / d$ to update the smoothing parameter $\varepsilon$, which requires information about the entire spectrum of $\Xk{(k)}$, whereas other works \citep{Fornasier11,Kummerle-JMLR2018,KummerleMayrinkVerdun-ICML2021,GTK24} all use the $(\widetilde{r}+1)$st singular value $\sigma_{\widetilde{r}+1}(\Xk{(k)})$ in \cref{eq:IRLS:step_2} in lieu of the scaled $\ell_1$-tail:
\begin{equation} \label{eq:IRLS:step_2:singularvalue}
\varepsilon_{k}:=  \min\left(\varepsilon_{k-1} , \sigma_{\widetilde{r}+1}(\Xk{(k)}) \right).
\end{equation}
This has certain computational advantages.

The works of \citet{Kummerle-JMLR2018}, \citet{KummerleMayrinkVerdun-ICML2021}, and \citet{GTK24}
primarily concern nonconvex rank surrogates, such as Schatten-$p$ quasi-norms or the log-determinant, rather than the convex nuclear norm considered here. In the present setting, the update rule \cref{eq:IRLS:step_2:singularvalue} remains computationally attractive, but it is not covered by our global convergence-rate analysis. That analysis exploits the specific coupling between $\varepsilon_k$ and the nuclear-norm tail provided by \cref{eq:IRLS:step_2}. Moreover, the experiments in \Cref{sec:smoothing:parameter:experiments} suggest that the scaled nuclear-norm-tail update recovers the nuclear-norm minimizer over a wider range of sampling factors in the considered experimental regime. These theoretical and empirical considerations motivate
our use of \cref{eq:IRLS:step_2}; they do not imply that this rule
is optimal for every problem instance or performance criterion. A compromise between these two update rules is the Frobenius tail based rule
\begin{equation} \label{eq:IRLS:step_2:frobenius}
 	\varepsilon_{k} := \min\left(\varepsilon_{k-1} , \sqrt{\frac{ \sum_{i=\widetilde{r}+1}^{d} \sigma_i^2(\Xk{(k)}) }{d}}\right).
\end{equation}
However, it is unclear whether either of the rules \cref{eq:IRLS:step_2:singularvalue} and \cref{eq:IRLS:step_2:frobenius} can be used to prove global convergence rates for IRLS as established in \Cref{sec:main:results}.

Alternative smoothing parameter update rules include a linear decrease according to a fixed schedule \citep{mohan_fazel,PengKuemmerleVidal-CVPR2023}, which could also be considered within \Cref{alg:algo1}. While simple in principle, tuning of the per-iteration decrease factor is challenging in practice. Moreover, it is unclear whether convergence guarantees can be obtained with such a rule.

In this paper, we focus on the update rule \cref{eq:IRLS:step_2} for our analyses. We further provide numerical evidence in \Cref{sec:smoothing:parameter:experiments} that suggests that \texttt{MatrixIRLS} with \cref{eq:IRLS:step_2} is able to find the nuclear norm minimizer for a larger range of sampling factors than the other update rules, often within a wall-clock time comparable to IRLS methods using other update rules.

\subsection{Optimality of Harmonic-Mean Weight Operator} \label{sec:majorization:optimality}

In the case of left- and right-sided weight operators \cref{eq:leftsided:mean} and \cref{eq:rightsided:mean}, it can be verified that the weight operator action reduces to left and right multiplication by the matrices $\LL_{\UU}^{-1}$ and $\LL_{\VV}^{-1}$ such that
\begin{equation} \label{eq:onesided:action}
	\W{\XX}{\varepsilon}(\ZZ)
	= \LL_{\UU}^{-1} \ZZ \quad \text{ and } \quad 
	\W{\XX}{\varepsilon}(\ZZ)
	= \ZZ \LL_{\VV}^{-1},
\end{equation}
respectively, where $\LL_{\UU} := \UU_{\XX} \diag ( \lambda_1, \ldots, \lambda_{d_1} ) \UU_{\XX}^\top$ and $\LL_{\VV} := \VV_{\XX} \diag ( \lambda_1, \ldots, \lambda_{d_2} ) \VV_{\XX}^\top$ are square matrices  with $\lambda_i := \max(\sigma_i,\varepsilon)$ for each $i \in [\max(d_1, d_2)]$. In fact, the left-reweighted matrix $\LL_{\UU}^{-1} \ZZ$ was originally used by \citet[equations (2.7) and (2.10)]{Fornasier11} without framing this as an action of a weight operator. The contemporaneous work of \citet{mohan_fazel} used right-reweighted matrices similar to $\ZZ \LL_{\VV}^{-1}$. By transposition of the underlying matrices $\XX$ and $\ZZ$, which preserves the rank, the left-sided and right-sided weight operator notions recover each other, respectively. The harmonic-mean weight operator $\W{\XX}{\varepsilon}$ using \cref{eq:harmonic:mean}, on the other hand, \emph{cannot} be expressed as a simple left or right multiplication by a matrix and was proposed in \citet{Kummerle-JMLR2018} with the rationale of providing a reweighting that appropriately acts both on row and column spaces. Acting on both row and column spaces in this balanced way also has desirable structural consequences, such as preserving symmetry in symmetric problems.

We recall our observation \cref{eq:onesided:action} that the left-sided weight operator using \cref{eq:leftsided:mean} is only informed by the column space of the current iterate $\XX$, and that its action amounts to matrix multiplication of the variable $\ZZ$ from the left. On the other hand, the right-sided weight operator is only informed by the row space of the current iterate $\XX$. This dichotomy between the row and column space information 
has undesirable consequences. For example, when the left-sided weights are applied to the transposed problem as opposed to the original problem \cref{eq:nucnorm:min}---i.e., when $\XX$ is replaced by $\XX^\top$ and the linear operator $\mathcal{A}: \Rdd \longrightarrow \R^m$ is replaced by a compatible one $\mathcal{A}_{\text{tr}}: \R^{d_2 \times d_1} \longrightarrow \R^m$ satisfying $\mathcal{A}(\XX)_{\ell} = \mathcal{A}_{\text{tr}}(\XX^{\top})_{\ell}$, the trajectory of the IRLS iterates $(\Xk{(k)})_{k\geq 0}$ will be different from what it would be if the left-sided weights were applied to the original problem. The same is true for the right-sided weights. 

The harmonic-mean weight operator \cref{eq:harmonic:mean,eq:W:operator:action} avoids this pitfall. However, more generally, an entire family of weight operators that is not subject to this issue can be defined via the notion of \emph{power means}.

\begin{definition}[Power means, \citeauthor{Bullen03}, \citeyear{Bullen03}]\label{def:powermean}
For $-\infty \leq q \leq \infty$, the \emph{$q$-power mean} \\ $\mathcal{M}_{q}(a,b)$ of two numbers $a, b > 0$ is given by
	\begin{equation} \label{eq:def:powermean}
		\mathcal{M}_{q}(a,b) = 
		\begin{cases}
			\min(a,b), & \text{ if } q = -\infty, \\
			\left(\frac{ a^q + b^q }{2}\right)^{\frac{1}{q}}, & \text{ if } q \in (-\infty,0) \cup (0,\infty), \\
			\sqrt{a b}, & \text{ if } q = 0, \\
			\max(a,b), & \text{ if } q = \infty. \\
		\end{cases}
	\end{equation}
\end{definition}

According to the previous definition, the arithmetic mean and the harmonic mean correspond to the $q$-power mean for $q=1$ and $q=-1$, respectively. By extension, power means of the two one-sided weight variants induce an entire family of weight operators. It is well-known \citep[Section III.3, Theorem 1]{Bullen03} that for $a,b\geq 0$ and $-\infty \leq q \leq q' \leq +\infty$, $\mathcal{M}_{q}(a,b) \leq \mathcal{M}_{q'}(a,b)$. This implies a Loewner ordering for the family of power-mean-induced weight operators, which we formalize in \Cref{lem:power_mean_loewner}.

\begin{lemma}[Monotonicity of Power Mean Weight Operators]\label{lem:power_mean_loewner}
	Fix $\varepsilon>0$ and $\XX\in\R^{d_1\times d_2}$ 
	with singular value decomposition $\XX=\UU_{\XX} \,\diag(\ssigma)\,\VV_{\XX}^\top$ as in \Cref{def:optimalweightoperator}. Define $
	\widetilde{\sigma}_i:= \max(\sigma_i,\varepsilon)^{-1}$ for $i \in [d]$ and $\widetilde{\sigma}_i:= \varepsilon^{-1}$ for $i \in [D] \setminus [d]$. For each $q\in[-\infty,\infty]$, 
	let $\f{H}^{(q)}_{\ssigma,\varepsilon}\in\Rdd$ 
	be given entrywise for $i \in [d_1]$ and $j \in [d_2]$ by
	\begin{equation} \label{eq:power:mean:core:matrix}
	\bigl(\f{H}^{(q)}_{\ssigma,\varepsilon}\bigr)_{ij}
	:=
	\mathcal{M}_q(\widetilde{\sigma}_i,\widetilde{\sigma}_j),
	\end{equation}
	where $\mathcal{M}_q(\cdot,\cdot)$ denotes the $q$-power mean from \cref{eq:def:powermean}. 
	Let $\W{\XX}{\varepsilon}^{(q)}$ be the corresponding weight operator defined as in \cref{eq:W:operator:action}, i.e., the operator mapping any $\ZZ\in\Rdd$ to
	\begin{equation} \label{eq:power:mean:weight:operator}
	\W{\XX}{\varepsilon}^{(q)}(\ZZ)
	:=\UU_{\XX}\Bigl(\f{H}^{(q)}_{\ssigma,\varepsilon}\circ(\UU_{\XX}^\top \ZZ\VV_{\XX})\Bigr)\VV_{\XX}^\top.
	\end{equation}
	Then for any $-\infty\le q\le q'\le \infty$, 
	one has the Loewner ordering $
	\W{\XX}{\varepsilon}^{(q)}  \preceq\ \W{\XX}{\varepsilon}^{(q')}
	$ and the associated quadratic models $Q_\varepsilon^{(q)}$ and $Q_\varepsilon^{(q')}$ (cf. \cref{eq:smoothedell1:IRLSmajorizer}) satisfy for all $\ZZ\in\Rdd$ that
	\begin{equation} \label{eq:pm:quad:model:ordering}
	Q_\varepsilon^{(q)}(\ZZ\mid\XX) \leq Q_\varepsilon^{(q')}(\ZZ\mid\XX).
	\end{equation}
\end{lemma}
A proof of \Cref{lem:power_mean_loewner} is given in \Cref{sec:proofs:power_means:optimality}. 

Recalling that the majorization property \cref{eq:majorization:inequality} of the quadratic model $Q_\varepsilon^{(-1)}(\cdot \mid\XX)$ associated to the harmonic-mean weight operator with respect to the smoothed nuclear norm objective  was shown in \Cref{thm:majorization}, we can infer from \Cref{lem:power_mean_loewner} that the majorization property holds for any $Q_\varepsilon^{(q)}(\cdot \mid\XX)$ with $q \geq -1$, which includes geometric and arithmetic mean weight operators. On the other hand, one can ask whether the harmonic mean is at the boundary or whether other $q$-means for $q<-1$, possibly the most extreme min-mean corresponding to $q=-\infty$, also lead to majorization. In \Cref{thm:power_means_majorize}, we establish that within the power-mean weight operator family, the harmonic mean is indeed \emph{optimal}: any power-mean weight operator that is smaller in the Loewner ordering defines a quadratic model that \emph{violates} majorization locally. 

\begin{theorem}[Optimality of Harmonic-Mean Weight Operator] \label{thm:power_means_majorize}
Let $\XX \in \mathbb{R}^{d_1 \times d_2}$
have SVD $\XX=\UU \diag(\boldsymbol{\sigma}) \VV^\top$
such that there are two indices $i,j \in [d]$ with $i \neq j$
whose singular values satisfy $\sigma_i \neq \sigma_j$
and $\sigma_i,\sigma_j > \varepsilon > 0$. Fix $q \in [-\infty , \infty]$ and consider the weight operator $\W{\XX}{\varepsilon}(\cdot):=\W{\XX}{\varepsilon}^{(q)}(\cdot)$ of \cref{eq:power:mean:weight:operator} based on the weight core matrix $\f{H}^{(q)}_{\ssigma,\varepsilon}$ arising as entrywise $q$-power mean of the left- and right-sided core matrices via \cref{eq:power:mean:core:matrix}.
Then the associated quadratic model $Q_\varepsilon^{(q)}(\cdot\mid\XX)$ (cf. \cref{eq:smoothedell1:IRLSmajorizer}) majorizes $\mathcal{J}_\varepsilon(\cdot)$, 
i.e., $Q_{\varepsilon}^{(q)}(\ZZ\mid\XX) \geq \mathcal{J}_{\varepsilon}(\ZZ)$ 
for each $\ZZ \in \Rdd$, if and only if $q \geq -1$.
\end{theorem}
A full proof of \Cref{thm:power_means_majorize} is given in \Cref{sec:proofs:power_means:optimality}. We note that one direction of the necessary and sufficient condition $q \geq -1$ is a direct consequence of \Cref{thm:majorization} and the Loewner ordering of \Cref{lem:power_mean_loewner}. The other direction involves an explicit necessary condition that we obtain from an expression for the Hessian of spectral functions, applied to a perturbation of $\XX$ that is confined to a two-dimensional singular block.

\Cref{thm:power_means_majorize} states that the harmonic mean provides the tightest possible majorization within the class of power mean weight operators. This tightness provides one ingredient towards explaining the superior performance of \texttt{MatrixIRLS} with harmonic-mean weight operators, which we explore in \Cref{sec:simulations}, compared to the one-sided or arithmetic mean IRLS variants---\texttt{MatrixIRLS} optimizes the quadratic model with the pointwise smallest gap to the smoothed surrogate objective among all majorizing power mean variants, which is the mechanism we expect to drive its faster per-iteration progress. A provable consequence of this tightness is an improved, dimension-free local linear convergence rate of \texttt{MatrixIRLS}, which we show with \Cref{thm:locallinearp1} in \Cref{sec:local:linear:convergence}. Numerical experiments of \Cref{sec:linear:convergence:rate:factors,sec:rectangular:lowrank} indicate that the improved local linear rate cannot be observed for IRLS variants using larger power means such as the arithmetic mean variant, nor for the one-sided IRLS variants, for which we provide a counterexample in \Cref{thm:counterexample:leftsided:weight:operator}.
 
\citet{Kummerle-JMLR2018} were the first to propose a harmonic-mean weight operator and observed improved local convergence, but focusing on nonconvex Schatten-$p$ quasi-norm minimization for $0 < p < 1$.
However, we would like to point out that it follows from a generalization of the second-order necessary argument outlined in \Cref{sec:proofs:power_means:necessary}, which is part of the proof of \Cref{thm:power_means_majorize}, that the quadratic model associated to the harmonic-mean weight operator is in fact \emph{not} a valid global majorizer in the nonconvex case of $p < 1$, which lies outside the scope of the present paper.

%% file: main_results.tex
\section{Linear Convergence Rates of IRLS for Nuclear Norm Minimization} \label{sec:main:results}
Building on the majorization results of \Cref{sec:majorization:properties}, we now establish a detailed convergence analysis of \texttt{MatrixIRLS} for nuclear norm minimization (\Cref{alg:algo1}) and distinguish different linear convergence rates that can be shown for IRLS depending on the choice of the weight operator.
\Cref{mainresult:lowrank,mainresult:approximatelowrank} establish the first global linear convergence rates for IRLS methods in this setting (\Cref{sec:global:linear:convergence}). For harmonic-mean \texttt{MatrixIRLS}, \Cref{thm:locallinearp1} further gives a dimension-free local linear rate, which, as we show via \Cref{thm:counterexample:leftsided:weight:operator}, cannot be achieved for IRLS using one-sided weight operators (\Cref{sec:local:linear:convergence}).

\subsection{Preliminaries} \label{sec:preliminaries} 
Our results are based on a well-studied regularity assumption on the measurement operator $\mathcal{A}: \Rdd \to \R^{m}$. In particular, we make the assumption that the measurement operator $\mathcal{A}$ satisfies the NSP \citep{Recht11,Fornasier11,Yi2020}, which is a sufficient (and in a very related form, also necessary) condition for successful recovery of low-rank matrices via nuclear norm minimization.
\begin{definition}[Null Space Property] \label{def:NSP:statement}
	A linear operator $\mathcal{A}: \R^{d_1 \times d_2} \to \R^{m}$ is said to satisfy 
	the \emph{NSP} of order $r \in \N$, $1 \leq r < d$ with constant $0 < \eta_r < 1$ if 
	\begin{equation} \label{eq:NSP:definition}
		\sum_{i=1}^{r} \sigma_i(\f{N}) \leq \eta_r \sum_{i=r+1}^{d} \sigma_i(\f{N})
	\end{equation}
for all  $ \f{N} \in \ker(\mathcal{A})$. Here, $\sigma_i(\f{N})$ denotes the $i$-th largest singular value of the matrix $\f{N}$.
\end{definition}
The NSP holds for several classes of linear measurement operators $\mathcal{A}$.
For example, if the operator $\mathcal{A}$ satisfies the restricted isometry property \citep{rechtfazel_2011} of sufficient order, then the NSP holds \citep[see, e.g.,][Exercise 6.24]{FoucartRauhut13}.
In particular, the NSP of order $r$ for some constant $0 < \eta_r < 1$ holds, e.g., with high probability, if the measurement matrices are random rank-one matrices
with Gaussian factors \citep{ROP_CaiZhang2015} and if the number of measurements satisfies $m = \Omega (r (d_1 +d_2))$, see also \citet{NSP_kueng,NSP_Kabanava}.\footnote{In fact,  \citet{ROP_CaiZhang2015} shows under the said assumption that a variant of \Cref{def:NSP:statement} holds where $\eta_r$ in \cref{eq:NSP:definition} is $1$ and the inequality is strict. A standard modification yields \cref{eq:NSP:definition} for some $0 < \eta_r < 1$.}

Beyond guaranteeing that \(\Xzero\) is the unique nuclear-norm minimizer
compatible with \(\f{y}=\mathcal{A}(\Xzero)\), \Cref{def:NSP:statement}
yields a linear error bound on the feasible set. If \(\Xzero\) has rank \(r\),
then every \(\XX\) with \(\mathcal{A}(\XX)=\mathcal{A}(\Xzero)\) satisfies (see \Cref{lemma:NSPl1min} for a more general version)
\begin{equation} \label{eq:NSP:sharpness}
\,\nucnorm{\XX-\Xzero}
\le
\frac{1+\eta_r}{1-\eta_r} \left(\nucnorm{\XX}-\nucnorm{\Xzero}\right).
\end{equation}
In the language of optimization, this is a \emph{sharpness} condition
for the nuclear norm on the affine constraint
\citep{Burke-SIAM-J-CO1993,Roulet-2020Computational}: the objective gap
grows linearly with the nuclear-norm distance to \(\Xzero\).
Linear error bounds of this type are the standard landscape hypothesis
for linear convergence of first-order methods
\citep{Goffin-1977}.
The same phenomenon originates in sparse recovery, where the analogous
\(\ell_1\)-NSP characterizes exact recovery by basis pursuit
\citep{Cohen-JAMS2009,FoucartRauhut13} and is the assumption behind
the first global linear rates for IRLS for \(\ell_1\)-minimization
\citep{Kummerle-NeurIPS2021}.

For IRLS, however, sharpness alone does not produce a rate: In contrast to known sharpness-based arguments, \Cref{alg:algo1} takes no step size, and \cref{eq:NSP:sharpness} is oblivious to the choice of the weight operator $W^{(k)}$ in \cref{eq:IRLS:matrix:formulation}. Controlling the nuclear norm gap via the majorization property together with appropriate, weight-operator dependent upper bounds on the quadratic forms implied by $W^{(k)}$, which depend on their spectral structure, is key for obtaining linear convergence rates for IRLS algorithms like \Cref{alg:algo1}.

\subsection{Global Linear Convergence of IRLS for Nuclear Norm Minimization} \label{sec:global:linear:convergence}
We show that the iterates $(\Xk{(k)})_{k \geq 1}$ of \Cref{alg:algo1} converge globally to a low-rank ground truth matrix $\Xzero$ with a linear rate whenever the measurement operator $\mathcal{A}$ satisfies the NSP, largely independently of the particular weight operator choice relevant for updating $W^{(k+1)} := \W{\Xk{(k)}}{\varepsilon_{k}}$.
We present in \Cref{sec:mainresult:lowrank} a result for the recovery of exactly low-rank matrices $\Xzero$, before generalizing this in \Cref{sec:mainresult:approximatelowrank} to approximately low-rank ground truths.

\subsubsection{\texorpdfstring
{Rank-$r$ Ground Truth $\Xzero$}
{Rank-r Ground Truth X-star}}  \label{sec:mainresult:lowrank}
Specific to linear operators $\mathcal{A}$ satisfying the NSP of order $r$ of \Cref{def:NSP:statement} with constant $\eta_r > 0$, we define the constants
\begin{equation} \label{eq:A_eta_r:def}
A_{\eta_r} = \frac{\left(\frac{3}{2}+\eta_r\right)\left(1+\eta_r\right)}{1-\eta_r} \quad \text{and} \quad C_{\eta_r} = \frac{\left( \frac{3}{4} - \frac{2\eta_r}{1+\eta_r}  \right)^2 }{3+2\eta_r}.
\end{equation}
Unlike the local result of \Cref{sec:local:linear:convergence}, the global results of this section are not tied to one specific weighting scheme, but hold for an entire family of weight operators, which we fix first for later reference. Define the the weight operator-dependent constant
\begin{equation} \label{eq:c_q:def}
	c_q :=
	\begin{cases}
	1, & \text{ if harmonic, one-sided or power mean weight operator with }q \in [-\infty, 1], \\
	2^{2-1/q}-1, & \text{ if power mean weight operator with }q \in (1, \infty), \\
	3, & \text{ if power mean weight operator with } q = \infty.
	\end{cases}
\end{equation}
\begin{definition}[Admissible Weight Operators] \label{def:admissible:weightoperators}
	We call the weight operators $\W{\Xk{(k)}}{\varepsilon_{k}}$ of \Cref{def:optimalweightoperator} used within \Cref{alg:algo1} \emph{admissible} with constant $c_q$ if, at each iteration $k$, their weight operator core matrix $\f{H}_{\ssigma,\varepsilon}$ of \Cref{def:weightcore} is
		(i) the \emph{harmonic-mean} core matrix \cref{eq:harmonic:mean},
		(ii) a \emph{one-sided} core matrix \cref{eq:leftsided:mean} or \cref{eq:rightsided:mean}, or
		(iii) a \emph{$q$-power mean} core matrix \cref{eq:power:mean:core:matrix} with $q \in [-1,\infty]$,
	where $c_q$ is the associated constant of \cref{eq:c_q:def}.
\end{definition}
Case (i) is the special case $q = -1$ of case (iii), whereas the one-sided operators of case (ii) are not power means. In all cases, $c_q \in [1,3]$, and $c_q = 1$ for the most common choices, cf.\ the discussion after \Cref{mainresult:lowrank}. Since \Cref{alg:algo1} is stated with harmonic-mean weights, a choice other than (i) is to be understood as the corresponding variant of \Cref{alg:algo1} in which only the weight operator update is replaced accordingly, while all other steps remain unchanged.

\begin{theorem}[Global Linear Convergence Rate, Low-Rank Ground Truths]\label{mainresult:lowrank}
Let \\ \allowbreak $\Xzero \in  \mathbb{R}^{d_1 \times d_2}$  be a matrix of rank $r$.
Assume that the measurement operator $\mathcal{A}: \mathbb{R}^{d_1 \times d_2} \longrightarrow \R^m $ satisfies the NSP \cref{eq:NSP:definition} of order $r$ with constant $\eta_r <3/5$. Let $\left(  \XXk\right)_{k\geq 0}$ and $\left(\varepsilon_{k} \right)_{k \geq 0}$ be the \texttt{MatrixIRLS} iterates obtained from \Cref{alg:algo1} with measurements $\f{y} = \mathcal{A}(\Xzero)$, rank estimate $\widetilde{r} = r$ and initializing weight operator $W^{(0)}$, and assume that the weight operators $\W{\Xk{(k)}}{\varepsilon_{k}}$ are admissible in the sense of \Cref{def:admissible:weightoperators}, with associated constant $c_q$ of \cref{eq:c_q:def}.

Then it holds for all iterations $k \in \mathbb{N}_0$ carried out by \Cref{alg:algo1} that
\begin{equation}\label{equ:linearconvergence1}
\mathcal{J}_{\varepsilon_{k}} \left( \XXk \right)  -  \nucnorm{ \Xzero } 
\le 
\left(  1-  \frac{C_{\eta_r}}{c_q\eta_1 d}   \right)^k 
\left(    \mathcal{J}_{\varepsilon_{0}} \left( \XX^{(0)} \right) -  \nucnorm{\Xzero} \right)
\end{equation}
as well as
\begin{equation}\label{ineq:globalconvergenerate1}
\nucnorm{ \XXk - \Xzero} \le 
A_{\eta_r} \left(  1-  \frac{C_{\eta_r}}{c_q\eta_1 d}   \right)^k  \nucnorm{ \XX^{(0)}-\Xzero}
\end{equation}
where $A_{\eta_r}$ and $C_{\eta_r}$ are as in \cref{eq:A_eta_r:def} and $ 0 < \eta_1 \leq \eta_r$ denotes the NSP constant of $\mathcal{A}$ of order~$1$.
\end{theorem}
The proof of \Cref{mainresult:lowrank} can be found in \Cref{sec:quadraticformupperbounds,sec:proofglobalconvergence:lowrank}. Specifically, it is shown that the smoothed objective gap $\mathcal{J}_{\varepsilon_{k}} \left( \XXk \right)  -  \nucnorm{ \Xzero }$ converges $Q$-linearly to zero, and that the iterate norm difference to $\Xzero$ converges $R$-linearly. Note that this theorem implies that it holds that 
$ \nucnorm{\XXk -\Xzero} \le \delta $
after 
$ O_{\eta_r} \left(1+ c_q \eta_1 d  \log_{+} \left(  \frac{\nucnorm{ \Xk{(0)} - \Xzero}}{\delta}  \right)    \right) $
iterations; this corresponds to a number of iterations to reach a fixed accuracy $\delta$ that depends linearly on the dimension $d$, across all considered weight operator choices. A slight dependence on the weight operator choice can be inferred from $c_q$, which satisfies $c_q = 1$ for the most common weight operator choices such as harmonic mean or one-sided whereas $c_q \to 3$ as the $q$-parameter of a power mean weight operator increases with $q \to \infty$. The NSP constant $\eta_r$ (as well as $\eta_1$) can be generally considered as an (undetermined) dimension-free constant, in which case $C_{\eta_r}$ is also dimension-free: If, for example, $\eta_r = 1/10$, the constant $C_{\eta_r}$ satisfies $C_{\eta_r} \approx 0.101$. The constant $A_{\eta_r}$ is furthermore dimension-free and ranges between $1.5 \leq A_{\eta_r} < 8.4$, depending on the value of $\eta_r$.

In \Cref{sec:simulations}, we present numerical experiments that corroborate that linear convergence rates of \Cref{alg:algo1} can be observed in practice, across all considered weight operator variants, and that a linear convergence rate with factor of order $1- \frac{c}{d}$ could indeed describe the correct worst-case behavior of the IRLS algorithm class for nuclear norm minimization; this is done by constructing an adversarial initialization (\Cref{sec:adversarial:initializations}) to define the initial weight operator $W^{(0)}$ of \Cref{alg:algo1}. On the other hand, we see in \Cref{sec:linear:convergence:rate:factors,sec:rectangular:lowrank} that for harmonic-mean weight operators, the generic linear rate of \Cref{alg:algo1} is \emph{dimension-independent} in the sense that it does \emph{not} depend on the dimension $d$. We refer to \Cref{sec:local:linear:convergence} for a local linear convergence result that better captures this generic behavior.

\subsubsection{Approximately Low-Rank Ground Truth} \label{sec:mainresult:approximatelowrank}
We now generalize the result of \Cref{sec:mainresult:lowrank} to the setting in which the ground truth $\Xzero$ is only \emph{approximately} low-rank. Define the $\mathcal{A}$-dependent constants
\begin{equation} \label{eq:B_eta_r:def}
	B_{\eta_r} := \frac{\left(\frac{7}{2}+\eta_r\right)\left(1+\eta_r\right)}{1-\eta_r} \quad \text{ and } \quad \widetilde{C}_{\eta_r} := \frac{ \left(  \frac{1}{2} - \frac{2\eta_r}{1+\eta_r}   \right)^2}{ 3+2\eta_r}
\end{equation}
given its NSP constant $\eta_r$ of order $r$.

\begin{theorem}[Linear Decay to Approximation Floor]
	\label{mainresult:approximatelowrank}
Let \allowbreak $\Xzero \in \Rdd$ be arbitrary. Assume that the measurement operator $\mathcal{A}: \Rdd \longrightarrow \R^m $ satisfies the NSP of order $r$ with constant $\eta_r < 1/3$.
If  $\left(  \XXk\right)_{k\geq 0}$ and $\left(\varepsilon_{k} \right)_{k \geq 0}$ 
are iterates and smoothing parameters of $\texttt{MatrixIRLS}$
with input $\f{y} = \mathcal{A}(\Xzero)$, 
arbitrary initial weight operator $W^{(0)}$, rank estimate $\widetilde{r} = r$ and weight operators $\W{\Xk{(k)}}{\varepsilon_{k}}$ that are admissible in the sense of \Cref{def:admissible:weightoperators}, set
\begin{equation} \label{eq:hatk:def}
\hat{k}:=  \min \left\{ k \in \mathbb{N}_0 :   \besterrNuc{\Xzero}{r}  > \frac{1}{9}\nucnorm{   \XXk  -\Xzero } \right\},
\end{equation}
with the convention $\min\varnothing=\infty$ (which holds whenever $\besterrNuc{\Xzero}{r}=0$).
Then the following two statements hold, in each case for the iterations carried out by \Cref{alg:algo1}.
\begin{enumerate}
\item  For $c_q$ as in \cref{eq:c_q:def} and $A_{\eta_r}$, $B_{\eta_r}$ and $\widetilde{C}_{\eta_r}$ as in \cref{eq:A_eta_r:def} and \cref{eq:B_eta_r:def}, it holds for all $k \in \mathbb{N}_0$ that
\begin{align} \label{equ:linearconvergence2}
\mathcal{J}_{\varepsilon_{k}}( \XXk) -  \nucnorm{\Xzero} 
&\le 
\left(  1-  \frac{\widetilde{C}_{\eta_r}}{c_q \eta_1 d}   \right)^{\min ( k, \hat{k} )} 
\left(     \mathcal{J}_{\varepsilon_{0}}( \XX^{(0)} ) -  \nucnorm{\Xzero} \right), \text{ and} \\
	\label{ineq:approxsparse1}
\nucnorm{\XXk -\Xzero }    
&\le
A_{\eta_r} \left(  1-  \frac{\widetilde{C}_{\eta_r}}{c_q \eta_1  d}   \right)^{\min (k, \hat{k}  )}   
\nucnorm{\XX^{(0)} -\Xzero} + B_{\eta_r}  \besterrNuc{\Xzero}{r},
\end{align}
where $0 < \eta_1 \leq \eta_r$ denotes the order-one NSP constant of $\mathcal{A}$.
\item For all $k \ge \hat{k}$ with $\hat{k}$ as in \cref{eq:hatk:def}, it holds that
\begin{equation}\label{ineq:approxsparse2}
	\nucnorm{ \XXk  -\Xzero}    \le 41 \besterrNuc{\Xzero}{r}.
\end{equation}
\end{enumerate}
\end{theorem}
The proof of \Cref{mainresult:approximatelowrank} is deferred to \Cref{sec:proofglobalconvergence:approximatelowrank}, where we also show that if $\besterrNuc{\Xzero}{r}>0$, then
\[
\hat{k} \le 1 + \frac{c_q \eta_1 d}{\widetilde{C}_{\eta_r}} \log_{+}\!\left(  \frac{A_{\eta_r}}{9-B_{\eta_r}} \cdot \frac{\nucnorm{ \XX^{(0)}  -\Xzero }}{\besterrNuc{\Xzero}{r}} \right),
\]
so inequality \cref{ineq:approxsparse2} holds after at most
$O_{\eta_r}\big( c_q \eta_1 d \log_{+}\left(  \nucnorm{ \XX^{(0)}  -\Xzero } / \besterrNuc{\Xzero}{r} \right) \big) $
iterations, where $\log_{+}(x)=\max(0,\log(x))$. If $\besterrNuc{\Xzero}{r}=0$, then $\hat{k}=\infty$ and \cref{equ:linearconvergence2,ineq:approxsparse1} hold with $\min(k,\hat{k})=k$ for all $k$.

\begin{remark}
Compared to \Cref{mainresult:lowrank}, where we assumed that the $r$-th order NSP constant $\eta_r$ is less than $3/5$, \Cref{mainresult:approximatelowrank} requires the stronger assumption of $0 < \eta_r < 1/3$. As a consequence, the constant $A_{\eta_r}$ has a tighter range of $1.5 \leq A_{\eta_r} < 3.7$ in \Cref{mainresult:approximatelowrank} than in \Cref{mainresult:lowrank}. The constant $B_{\eta_r}$ of \cref{eq:B_eta_r:def} is in the range of $3.5 \leq B_{\eta_r} < 23/3 \approx 7.67$ in \Cref{mainresult:approximatelowrank}. We refer to \cite[Theorem A.1]{Kummerle-NeurIPS2021} for a similar result for IRLS for the $\ell_1$-minimization problem that covers approximately sparse ground truth vectors.
\end{remark}

\subsection{Fast Local Linear Rate of \MatrixIRLSHeading{} with Harmonic-Mean Weights} \label{sec:local:linear:convergence}

The global results of \Cref{sec:global:linear:convergence} are largely indifferent to the underlying weighting scheme: \Cref{mainresult:lowrank,mainresult:approximatelowrank} hold uniformly across the admissible weight operators of \Cref{def:admissible:weightoperators}, whose choice enters these statements only through the constant $c_q \in [1,3]$ of \cref{eq:c_q:def}. In particular, no admissible choice is singled out by the global rates, and none of them escapes the factor $1/d$ in the linear convergence factor.

We now consider the scenario that an iterate of \texttt{MatrixIRLS} has already  entered a specific, local neighborhood of the ground truth $\Xzero$, in which the picture is different, as the choice of the weight operator becomes decisive.
In this case, we show that, specifically for the algorithm variant \Cref{alg:algo1} that uses harmonic-mean weight operators, the iterates converge locally with a fast linear rate, which is independent of the ambient dimension $d$. What drives this improvement is a sharp local upper bound on the weighted quadratic form of the quadratic model function, which is available for harmonic-mean weight operators but \emph{provably fails} for one-sided ones. We make this mechanism precise in \Cref{corollary:quadratictermlocal} and \Cref{thm:counterexample:leftsided:weight:operator} at the end of this section, and delineate in \Cref{rem:dimfree} the weight operators for which a dimension-free local rate remains available.
\begin{theorem}[Dimension-Free Fast Linear Rate of \MatrixIRLSHeading{}]\label{thm:locallinearp1}
Let $\Xzero\in \mathbb{R}^{d_1 \times d_2}$ be a matrix of rank $r$.
Assume that the measurement operator $\mathcal{A}$ satisfies the NSP of \Cref{def:NSP:statement} of order $r$ with constant $\eta_r < 3/5$ and of order $1$ with constant  $\eta_1 \leq \eta_r$. Let $\left(  \XXk\right)_{k\geq 0}$ and $ \left(\varepsilon_{k} \right)_{k \geq 0}$ be the iterates and smoothing parameters of $\texttt{MatrixIRLS}$ with input $\f{y} = \mathcal{A}(\Xzero)$, arbitrary initial weight operator $W^{(0)}$ and $\widetilde{r} = r$. Let $A_{\eta_r}$ be defined as in \cref{eq:A_eta_r:def}. Assume that \texttt{MatrixIRLS}'s weight operators $\W{\Xk{(k)}}{\varepsilon_{k}}$ are defined based on harmonic-mean core matrices \cref{eq:harmonic:mean} and that there is a natural number $\tilde{k}$ such that
\begin{equation}\label{assump:localconvergencerate}
	\nucnorm{\Xk{(\tilde{k})} - \Xzero} 
	\le \frac{\sigma_r (\Xzero )}{9 \max  \left( \eta_1 \sqrt{d}, 2\right) }.
\end{equation}
Then for all iterations $k \ge \tilde{k}$ carried out by \Cref{alg:algo1} it holds that 
\begin{equation}\label{equ:locallinearconvergence1}
      \mathcal{J}_{\varepsilon_{k}} \left( \XXk \right)  -  \nucnorm{ \Xzero } \le \left(  1-  c_{\eta_r}   \right)^{k-\tilde{k}} \left(    \mathcal{J}_{\varepsilon_{\tilde{k}}} \left( \Xk{(\tilde{k})} \right) -  \nucnorm{\Xzero} \right)
\end{equation}
as well as
\begin{equation}\label{ineq:localconvergenerate1}
\nucnorm{ \XXk - \Xzero}\le A_{\eta_r} \left(  1-  c_{\eta_r}   \right)^{k-\tilde{k}}  \nucnorm{ \Xk{(\tilde{k})}-\Xzero}.
\end{equation}
Here, the constant $c_{\eta_r}$ depends only on the NSP constant $\eta_r$ and is given by
\begin{equation} \label{eq:constant:c:eta_r}
c_{\eta_r}
=
\frac{(3-5\eta_r)^{2}(1-\eta_r)^{3}}{8(1+\eta_r)^{2}(3+2\eta_r)\Bigl[8(1-\eta_r)^{3}+(3+\eta_r)^{2}(1+\eta_r)(3+2\eta_r)\Bigr]}.
\end{equation}
\end{theorem}
The proof of \Cref{thm:locallinearp1} can be found in \Cref{sec:prooffastlocallinearp1}.

To parse the definition of $c_{\eta_r}$ in \cref{eq:constant:c:eta_r}, we consider again measurement operators with an order $r$-NSP constant of $\eta_r = 1/10$. In this case, we have $c_{\eta_r} \approx 0.0037$. Notably, this means that the decrease factor $1-c_{\eta_r}$ of \Cref{thm:locallinearp1} \emph{does not depend on $d$}, unlike the decrease factor $1-  \frac{C_{\eta_r}}{\eta_1 d}$ of \Cref{mainresult:lowrank}. This means that for large dimensions of $d \gg 1$, the analysis of \Cref{thm:locallinearp1} leads to a sharper bound than the global rates established above. Remaining in the $\eta_r = 1/10$ example, assuming additionally $\eta_1 = \eta_r$, which is rather pessimistic, \Cref{thm:locallinearp1} leads to a faster rate than \Cref{mainresult:lowrank} if $d > 273$. We note that we did not attempt to optimize the $\eta_r$-dependence of $c_{\eta_r}$ in our proof, which we leave for future work.

By combining \Cref{thm:locallinearp1} with \Cref{mainresult:lowrank},
we obtain that
$ \nucnorm{\XXk -\Xzero} \le \delta  $
after
\begin{equation*}
O_{\eta_r} \left( \eta_1 d \log_{+} \left( \frac{ \max ( \eta_1 \sqrt{d}, 2 ) \nucnorm{ \XX^{(0)} -\Xzero }   }{ \sigma_r (\Xzero) } \right)
+
\log_{+} \left(  \frac{\sigma_r (\Xzero) }{ \max ( \eta_1 \sqrt{d}, 2 ) \delta}  \right)
\right)
\end{equation*}
iterations, where $O_{\eta_r}(\cdot)$ represents the $\eta_r$-dependent $O$-notation and $\log_{+}(x) = \max(0,\log(x))$.

\begin{remark}\label{rem:dimfree}
	A fast, dimension-free linear convergence rate for IRLS can be shown also for $q$-power mean weight operators with $q \in (-1, 0)$ ($q=-1$ corresponds to the harmonic mean); see \Cref{lemma:quadratictermlocal} and \Cref{prop:p1:locallinearrate} in \Cref{sec:prooffastlocallinearp1} for the general argument that also covers this case. For $q \in (-1, 0)$, the $8(1-\eta_r)^3$ in the denominator of $c_{\eta_r}$ in \cref{eq:constant:c:eta_r} becomes \mbox{$(4+ 2^{1-1/q})(1-\eta_r)^3$}, which makes $c_{\eta_r} \to 0$ as $q \to 0^-$ and thus renders the bound \cref{equ:locallinearconvergence1} ineffective. For $q$-power mean weight operators with larger $q \in [0,\infty]$, which includes arithmetic mean weight operators, a statement such as \Cref{thm:locallinearp1} cannot be established anymore using our results, and numerical experiments of \Cref{sec:linear:convergence:rate:factors} suggest that the fast local linear rate cannot be expected in this case.
\end{remark}
We now make the mechanism behind \Cref{thm:locallinearp1} precise. Its proof rests on a more precise estimate for the quadratic term $\innerproduct{ \Xzero-\f{X}, \W{\XX}{\varepsilon} (\Xzero-\XX)}$ of the implied quadratic model function $Q_{\varepsilon}(\cdot \mid \f{X})$ of \cref{eq:smoothedell1:IRLSmajorizer} that becomes available in the case that $\XX$ is close enough to the ground truth $\Xzero$, and which is given by \Cref{lemma:quadratictermlocal} in \Cref{sec:prooffastlocallinearp1}. We state a simplified corollary of it below, which fixes the constant $\eta_r=1/10$ for legibility.

\begin{corollary}[Sharp Local Upper Bound on Weighted Quadratic Form] \label{corollary:quadratictermlocal}
	Assume \\ that the linear measurement operator $\mathcal{A}: \mathbb{R}^{d_1 \times d_2} \rightarrow \mathbb{R}^m$ satisfies the NSP of order $r$ with constant $\eta_r = 1/10 $, that $\Xzero \in \Rdd$ is of rank $r$ and that $ \XX \in \mathbb{R}^{d_1 \times d_2}$  satisfies $\XX - \Xzero \in \ker(\mathcal{A})$. If $\varepsilon =  \frac{\besterrNuc{\XX}{r} }{d} > 0$, $\nucnorm{\XX - \Xzero} 
	\le 
	\frac{ \sigma_r \left( \Xzero \right) }{ \max \left( \eta_1 \sqrt{ d}, 2 \right)  }
$, where $0 < \eta_1 \leq \eta_r$ denotes again the order-one NSP constant of $\mathcal{A}$, and $\W{\XX}{\varepsilon}(\cdot)$ is the harmonic-mean weight operator \cref{eq:W:operator:action}, then
\begin{equation} \label{ineq:sharp:local:bound:corollary}
	\innerproduct{ \Xzero - \XX  , \W{\XX}{\varepsilon} ( \Xzero - \XX ) }
	\le 
	16  \nucnorm{\XX-\Xzero}.
\end{equation}
\end{corollary}
\Cref{corollary:quadratictermlocal} follows immediately from \Cref{lemma:quadratictermlocal} with $q = -1$ and $\vartheta = 1$, using $D_{1/10} = 31/9$.
\paragraph{Impossibility of $d$-Independent Fast Local Rate for One-Sided Weight Operators.}
The dimension-independent upper bound \cref{ineq:sharp:local:bound:corollary} on $\innerproduct{  \Xzero - \XX , \W{\XX}{\varepsilon} (\Xzero - \XX) }$ proportional to $\nucnorm{\XX-\Xzero}$ is \emph{specific to the harmonic-mean weight operator} (and to $q$-power mean weight operators with $q \in [-1, 0)$, cf.\ \Cref{rem:dimfree}): it \emph{cannot hold} for one-sided weight operators (such as those with left-sided or right-sided core matrices \cref{eq:leftsided:mean,eq:rightsided:mean}), even within smaller local neighborhoods of $\Xzero$ than the one defined by \cref{assump:localconvergencerate}. \Cref{thm:counterexample:leftsided:weight:operator} below establishes this through a counterexample for which a $d$-dependent \emph{lower bound} on $\innerproduct{ \Xzero - \XX , \W{\XX}{\varepsilon} ( \Xzero - \XX ) }$ holds if a left-sided weight operator $\W{\XX}{\varepsilon}(\cdot)$ with core matrix \cref{eq:leftsided:mean} is used; the details of this construction are provided in \Cref{sec:counterexample:leftsided:weight:operator}.

\begin{theorem}[No Dimension-Free Fast Local Rate for One-Sided IRLS]
\label{thm:counterexample:leftsided:weight:operator}
For any\\ $r,d \in \mathbb{N}$ with $d \ge 220 r$, there exists a measurement operator $\mathcal{A}: \mathbb{R}^{d \times d} \rightarrow \R^m$ satisfying the NSP of order $r$ with constant $\eta_r = 1/10$ and of order $1$ with $\eta_1 = 1/(11r -1)$, a rank-$r$ matrix $\Xzero \in \mathbb{R}^{d \times d}$ and a matrix $\XX \in \mathbb{R}^{d \times d}$ such that $\XX - \Xzero \in \ker(\mathcal{A})$ and $\nucnorm{\XX - \Xzero} \le \frac{ \sigma_r \left( \Xzero \right) }{ \max \left( \eta_1 \sqrt{ d}, 2 \right)  }$, but also
\begin{equation*}
\innerproduct{ \Xzero - \XX, \W{\XX}{\varepsilon} (\Xzero - \XX ) }
\ge
\frac{d}{220 r}\nucnorm{ \XX - \Xzero},
\end{equation*}
if $\W{\XX}{\varepsilon}(\cdot)$ is the weight operator \cref{eq:W:operator:action} with left-sided core matrix \cref{eq:leftsided:mean} and the smoothing parameter $\varepsilon$ satisfies $\varepsilon = \frac{\besterrNuc{\XX}{r} }{d}$.

Furthermore, in this case, the subsequent IRLS iterate 
\[
\XX^+ = \argmin_{ \f{Z} \in \mathbb{R}^{d \times d} } \innerproduct{ \f{Z}, \W{\XX}{\varepsilon} (\f{Z}) } \text{ subject to } \mathcal{A}(\f{Z}) = \mathcal{A}(\Xzero)
\] 
satisfies
\[
	\nucnorm{\XX^+-\Xzero}
	\geq \left(1-\frac{220 r}{d}\right)\nucnorm{\XX-\Xzero}.
\]
\end{theorem}
An analogous result to \Cref{thm:counterexample:leftsided:weight:operator} can be shown for right-sided weight operators. The different local linear convergence rates of IRLS using harmonic-mean weight operators on the one hand, and of IRLS using one-sided, or large-$q$ power-mean weight operators on the other hand, can also be observed in practice; for generic examples, the dimension-free error decay of \Cref{alg:algo1}'s iterates can often be observed beyond the limited-size local neighborhood defined by \cref{assump:localconvergencerate}. We elaborate on this fact below in \Cref{sec:linear:convergence:rate:factors,sec:rectangular:lowrank}.

%% file: numerical_experiments.tex
\section{Numerical Experiments}\label{sec:simulations}
In this section, we explore the empirical behavior of IRLS for nuclear norm minimization to solve low-rank matrix sensing problems. We focus on the empirical speed of convergence of iterates of \Cref{alg:algo1} given different choices of the weight operator core matrix (see \Cref{def:weightcore}), taking also into account the role of algorithmic initialization and the smoothing parameter update rule \cref{eq:IRLS:step_2}.

\subsection{Setup} \label{sec:setup}
For all experiments, we consider low-rank matrix recovery problems of \emph{matrix sensing} type, where the measurement operator $\mathcal{A}: \Rdd \longrightarrow \R^m$ of \eqref{eq:nucnorm:min} consists of noiseless, \emph{random Gaussian rank-one measurements}. In particular, the $\ell$-th coordinate of $\mathcal{A}(\f{X})$ given the input matrix $\f{X} \in \Rdd$ is
$  \mathcal{A}(\f{X})_\ell = \langle a_\ell b_{\ell}^{\top},\f{X}\rangle $
for each $\ell=1,\ldots, m$, where $a_{\ell},b_{\ell}$ are independent random vectors of length $d_1$ and $d_2$, respectively, with independent, standard normal entries. Such measurements correspond to a simplified, real-valued variant of the measurement setting available in blind deconvolution problems \citep{ahmedBlinddeconv_2014,Li-RapidBlindDeconvolution2019,Ma2020implicit}. We use such a setup as these random rank-one measurements (see \Cref{sec:preliminaries}) are very likely to make $\mathcal{A}$ satisfy an NSP of order $r$ and match the assumptions of the convergence theorems of \Cref{sec:main:results}, as long as $m$ is chosen to be at least proportional to $r (d_1 +d_2)$ with a certain proportionality factor $C > 1$. Another reason is the implied computational cost of rank-one measurements. With $O(m (d_1+d_2))$ entries, they have smaller memory requirements than measurement operators with, e.g., measurement matrices drawn from a dense Gaussian ensemble, and smaller evaluation costs on rank-$r$ matrices with $O(m r (d_1 +d_2))$ instead of $O(m d_1 d_2)$.

\subsection{Linear Convergence Rate Factors for IRLS with Weight Operator Variants} \label{sec:linear:convergence:rate:factors}
In the first experiment, we compare the performance of IRLS for nuclear norm minimization (as defined in \Cref{alg:algo1}), using \emph{different} weight operator variants on typical low-rank matrix recovery problem instances. In the following, we compare IRLS using the updates $W^{(k+1)} := \W{\Xk{(k)}}{\varepsilon_{k}}$, where $\W{\XX}{\varepsilon}$ is as in \Cref{def:optimalweightoperator}, but with different choices for the weight operator core matrix $\f{H}_{\ssigma,\varepsilon}$ in \Cref{def:weightcore}. We consider
\begin{itemize}
    \item a harmonic-mean weight operator core matrix $\f{H}_{\ssigma,\varepsilon}$ as in \cref{eq:harmonic:mean} of \Cref{def:weightcore},
    \item a left-sided weight operator core matrix $\f{H}_{\ssigma,\varepsilon}$ as in \cref{eq:leftsided:mean},
    \item a right-sided weight operator core matrix $\f{H}_{\ssigma,\varepsilon}$ as in \cref{eq:rightsided:mean}, as well as 
    \item an arithmetic mean weight operator core matrix $\f{H}_{\ssigma,\varepsilon} = \f{H}_{\ssigma,\varepsilon}^{(1)}$ as in \cref{eq:power:mean:core:matrix} with $q=1$.
\end{itemize}
For these choices, \Cref{mainresult:lowrank} provides an identical global linear convergence rate under suitable conditions as $c_q$ of \cref{eq:c_q:def} is equal to $1$ in every case. 
For different matrix dimensions $d \in \{30, 60, 90, 110, 140\}$ with $d_1 = d_2= d$, we sample a random low-rank matrix $\Xzero = \f{U}\f{V}^{\top} \in \Rdd$ of rank $r=2$ with random factor matrices $\f{U} \in \R^{d_1 \times r}$ and $\f{V} \in \R^{d_2 \times r}$ with i.i.d.~standard Gaussian entries and run the respective variant of \Cref{alg:algo1} until convergence. The IRLS variants take as input a realization of the rank-one measurement operator $\mathcal{A}$ outlined in \Cref{sec:setup}, that is $\f{y}= \mathcal{A}(\Xzero) \in \R^m$ with  the number of measurements $m = 4.5 r (d_1 + d_2)$. Typical decay curves of the absolute smoothed nuclear norm gap $\mathcal{J}_{\varepsilon_{k}} \left( \XXk \right)  -  \nucnorm{ \Xzero} $ resulting from the respective variant of \Cref{alg:algo1}, are shown in \Cref{fig:ConvergencePlotsWeightOperatorVariants}. In all experiments, we use a \emph{tangent space implementation} for solving the weighted least squares problem \cref{eq:IRLS:step_1} and partial singular value decompositions for the weight operator update of \Cref{alg:algo1}, which generalizes the implementations of \citet{KummerleMayrinkVerdun-ICML2021} and \citet{GTK24}. We refer to these papers and to research code associate to this paper for more details.

\begin{figure}[t]
    \centering
    \includegraphics[width=0.95\textwidth]{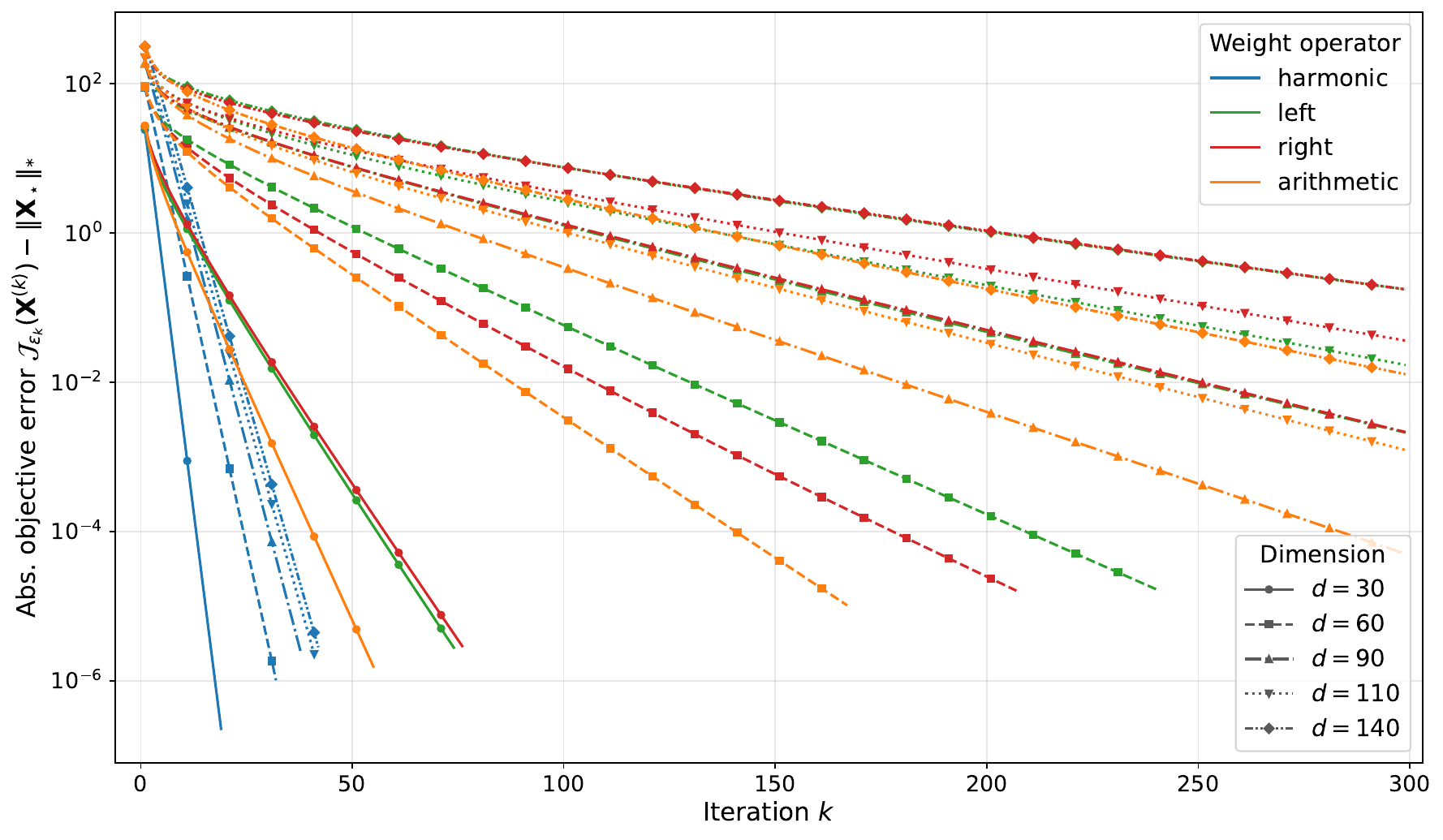}
    \caption{Smoothed nuclear norm gap $\mathcal{J}_{\varepsilon_{k}} \left( \XXk \right) - \nucnorm{\Xzero}$ of IRLS iterates $\XXk$ for different weight operator variants (harmonic, left-sided, right-sided, arithmetic) across matrix dimensions $d \in \{30, 60, 90, 110, 140\}$, rank-$2$ ground truth $\Xzero \in \Rdd$, random Gaussian rank-one measurements with linear scaling.}
    \label{fig:ConvergencePlotsWeightOperatorVariants}
\end{figure}

We observe that while all variants exhibit a linear convergence rate, their multiplicative decrease factors differ substantially, with the harmonic-mean weight operator variant of IRLS (in blue) converging fastest for all considered dimensions, followed by the arithmetic mean variant, while the left-sided and right-sided variants are the slowest. For example, for $d=30$, the harmonic-mean variant of \Cref{alg:algo1} reaches an error threshold of~$10^{-5}$ after 16 iterations, whereas the arithmetic, left-sided, and right-sided variants require 49, 68, and 70 iterations, respectively. For each of the larger dimensions $d \in \{60,90,110,140\}$, this gap widens: the harmonic-mean variant needs at most 40 iterations to reach the threshold, whereas none of the other variants reaches it within 300 iterations, which is the maximal iteration count used in the experiment.

\begin{figure}[t]
    \centering
    \includegraphics[width=\textwidth]{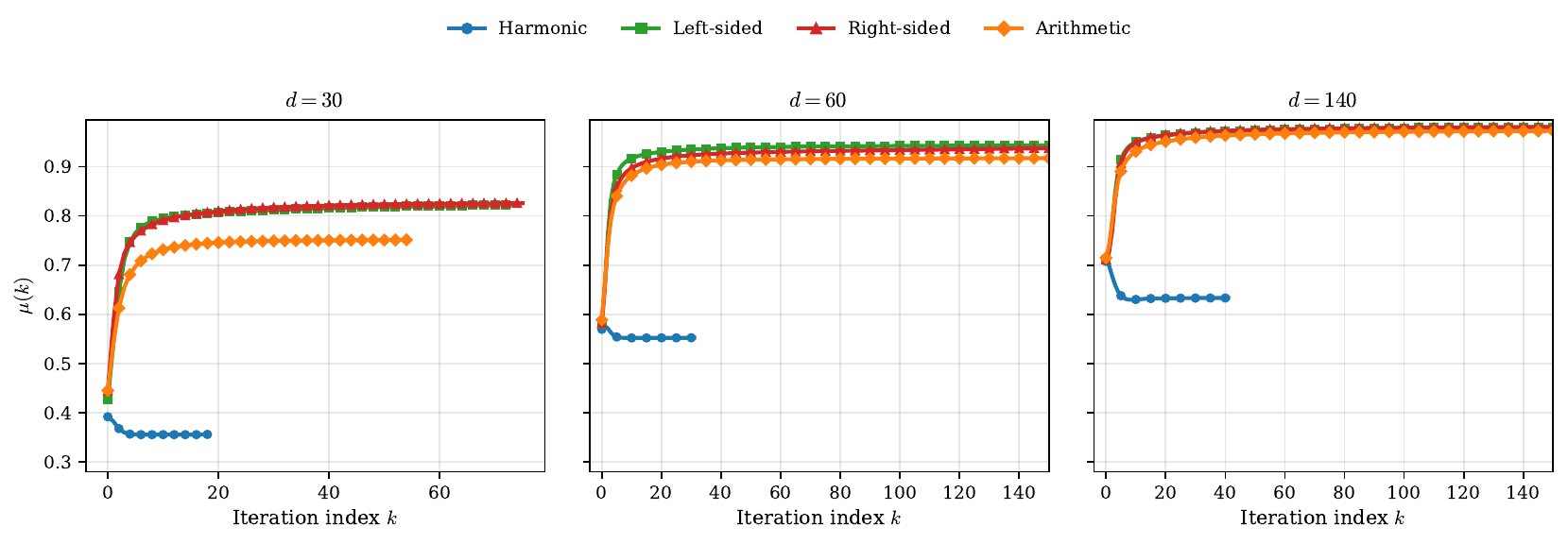}
    \caption{Successive decrease-factor trajectories for the standard initialization $\f{X}^{(0)}$ in the random Gaussian rank-one measurement setup, shown for $d \in \{30, 60, 140\}$. The color pattern and variant order match \Cref{fig:ConvergencePlotsWeightOperatorVariants}: harmonic (blue), left-sided (green), right-sided (red), arithmetic (orange). For $d=60$ and $d=140$, only iteration indices $0$ through $150$ are displayed.}
    \label{fig:DecreaseFactorVariantsStandard}
\end{figure}
To make the contraction explicit, we define for each iteration $k \geq 0$ the pairwise decrease factors
\begin{equation} \label{eq:decrease:factor}
    \mu(k) = \frac{\mathcal{J}_{\varepsilon_{k+1}}(\f{X}^{(k+1)}) - \nucnorm{\Xzero}}{\mathcal{J}_{\varepsilon_{k}}(\f{X}^{(k)}) - \nucnorm{\Xzero}}
\end{equation}
and visualize the corresponding decrease-factor trajectories across iterations $k$ in \Cref{fig:DecreaseFactorVariantsStandard} for the $d \in \{30,60,140\}$ experiments. In this figure, values farther below $1$ indicate a stronger per-iteration decrease of the smoothed nuclear norm gap and thus faster linear convergence. \Cref{fig:DecreaseFactorVariantsStandard} illustrates that at initialization, the decrease factor is similar for all IRLS variants, but deteriorates quickly for the left-sided, right-sided and arithmetic mean variants. On the other hand, for \Cref{alg:algo1} with the harmonic-mean weights, we observe stability or a slight improvement quickly after initialization. This is consistent with the global linear rate analysis of \Cref{mainresult:lowrank}, which characterizes the global behavior of any IRLS variant and which applies already for the first iteration, and the fact that an improved local linear rate can be shown (cf. \Cref{thm:locallinearp1}) for the harmonic-mean IRLS variant once the iterates get closer to the ground truth. The deterioration of the decrease factor at later iterations for the left-sided, right-sided and arithmetic mean variants as $d$ grows suggests that a dimension-independent, fast local convergence rate such as shown in \Cref{thm:locallinearp1} does not hold for these variants. \Cref{fig:ConvergencePlotsWeightOperatorVariants} and \Cref{fig:DecreaseFactorVariantsStandard} suggest that a slight increase in the number of iterations is needed for harmonic-mean IRLS to reach the error threshold of $10^{-5}$. Similarly, the limiting $\mu(k)$ for large $k$ increases with increasing dimension $d$, namely from around $0.36$ for $d=30$, through $0.55$ for $d=60$, to $0.63$ for $d=140$. However, this behavior is entirely compatible with
\Cref{thm:locallinearp1}. As discussed after \Cref{thm:locallinearp1}, its factor $1-c_{\eta_r}$ is an upper bound with a non-optimized constant, which equals $\approx 0.996$ for $\eta_r = 1/10$ and thus lies above all factors observed here. It is moreover dimension-free for a \emph{fixed} NSP constant, whereas our experiments keep the oversampling factor $m / (r(d_1+d_2))$ fixed, so that $\eta_r$ and $\eta_1$ may still vary with $d$.
\subsection{Recovery of Rectangular Low-Rank Matrices} \label{sec:rectangular:lowrank}
In the recent literature on IRLS for low-rank matrix recovery, one-sided variants of IRLS are still often considered \citep{Kraemer-2025OneSided,Radhakrishnan2025linear}, partially motivated by the rectangular structure of many problem instances, where, for example, $d_1$ can be considered as the number of data samples and $d_2$ as the number of features in which case \citet{Radhakrishnan2025linear} recommends right-sided reweighting. This could also be justified specifically for the case of $d_1 \gg d_2$, as in this case the right-sided variant of IRLS requires a matrix decomposition of a smaller matrix,\footnote{\citet{Radhakrishnan2025linear} discusses SVD-free variants of one-sided IRLS under the name \emph{SVD-free lin-RFM} (recursive feature machines). However, this implementation variant is not available for the case of the nuclear norm surrogate (as this corresponds to $\alpha = 1/4$ in the framework of \citet{Radhakrishnan2025linear}, whereas the SVD-free variant is only available for integer multiples of $\alpha = 1/2$).} whereas this perspective would motivate left-sided reweighting for $d_1 \ll d_2$. 

\begin{figure}[t]
    \centering
    \includegraphics[width=\textwidth]{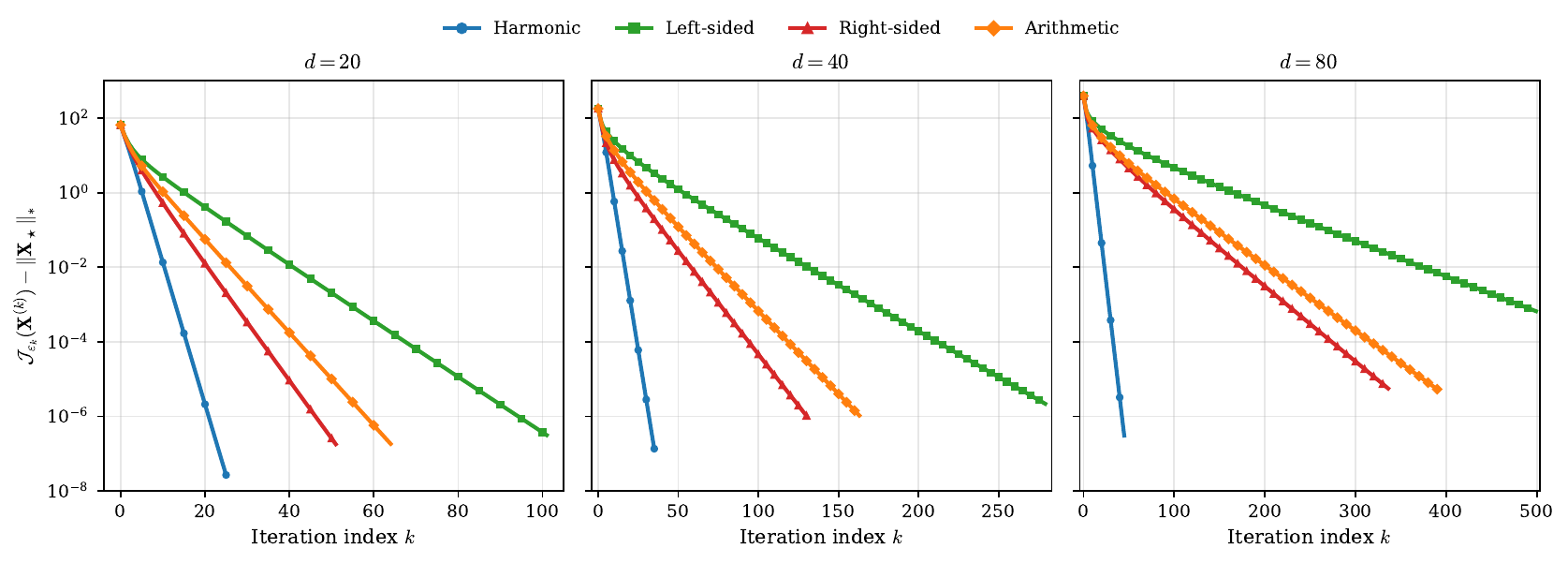}
    \caption{Smoothed nuclear norm gap trajectories $\mathcal{J}_{\varepsilon_{k}} \left( \XXk \right) - \nucnorm{\Xzero}$ for the tall rectangular recovery experiment with standard initialization and dimensions $d \in \{20,40,80\}$. In each panel, a rectangular ground truth $\Xzero \in \R^{4d \times d}$ of rank $r=2$ is recovered from $m = 4.5 r (d_1+d_2)$ random Gaussian rank-one measurements.}
    \label{fig:RectangularTallNucDeltaVariants}
\end{figure}

In any case, the question arises whether one-sided IRLS variants perform sufficiently well for highly rectangular problem instances with $d_1 \gg d_2$, or whether the optimal harmonic-mean variant is still advantageous in this case. To this end, we consider the setup of \Cref{sec:linear:convergence:rate:factors} to recover a rank $r=2$ ground truth $\Xzero \in \Rdd$ from $m = 4.5 r (d_1 + d_2)$ rank-one measurements, but with the tall rectangular dimensions $d_1 = 4 d$ and $d_2 = d$ for $d \in \{20,40,80\}$, so that the aspect ratio $d_1/d_2$ is the only parameter in which the two setups differ.

\begin{figure}[t]
    \centering
    \includegraphics[width=\textwidth]{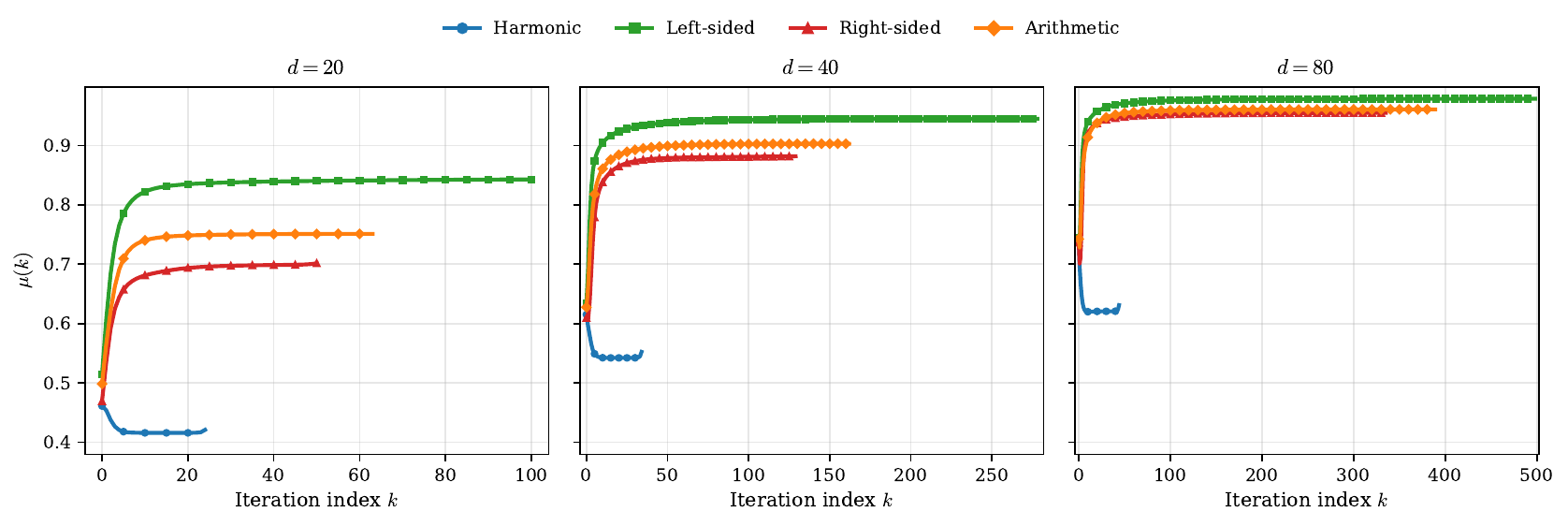}
    \caption{Successive decrease-factor trajectories $\mu(k)$ of \cref{eq:decrease:factor} for the tall rectangular recovery experiment with standard initialization, shown for $d \in \{20,40,80\}$ and rank $r=2$ ground truth $\Xzero \in \R^{4d \times d}$. Each trajectory ends once the respective variant terminates or the maximal iteration count of $500$ is reached.}
    \label{fig:RectangularTallDecreaseFactorVariants}
\end{figure}

We observe in \Cref{fig:RectangularTallNucDeltaVariants} that, for $d = 20$, the right-sided IRLS variant using \cref{eq:rightsided:mean} in \Cref{def:weightcore} clearly converges faster than the left-sided variant using \cref{eq:leftsided:mean}, passing the error threshold of $10^{-5}$ after 40 instead of 81 iterations, whereas the arithmetic mean variant lies in between with 51 iterations and harmonic-mean IRLS is fastest with 19 iterations. This ordering persists for the larger dimensions, but the gaps \emph{widen} substantially as $d$ grows: harmonic-mean IRLS passes the threshold after 28 and 38 iterations for $d=40$ and $d=80$, while the right-sided variant, the best of the alternatives, requires 113 and 324 iterations, and the left-sided variant does not reach the threshold for $d=80$ within the maximal iteration count of 500 used in this experiment.

The decrease factor trajectories $\mu(k)$ of \cref{eq:decrease:factor} depicted in \Cref{fig:RectangularTallDecreaseFactorVariants} visualize this deterioration in a different manner. In each panel, the four variants start from an almost identical $\mu(0)$, in line with the global rate of \Cref{mainresult:lowrank} that applies to all of them alike, but settle at markedly different levels: the limiting decrease factor of harmonic-mean IRLS increases only mildly from about $0.42$ for $d=20$ to about $0.62$ for $d=80$, cf.~\Cref{rem:dimfree}, whereas those of the other three variants reach values between $0.96$ and $0.98$ for $d=80$, which again suggests that a fast local linear rate such as the one of \Cref{thm:locallinearp1} does not hold for these variants.

Overall, this experiment shows that for rectangular low-rank matrix recovery, among the one-sided variants, using nontrivial weights on the \emph{smaller} matrix dimension pays off relatively speaking, but the optimal harmonic-mean variant still is significantly superior, and increasingly so for instances of higher dimensions.

\subsection{Adversarial Initialization} \label{sec:adversarial:initializations}
\Cref{mainresult:lowrank} guarantees a global linear decrease factor of the form
$1-\frac{C_{\eta_r}}{c_q \eta_1 d}$. A natural question is whether the ambient-dimension factor~$d$ in this bound is essentially sharp, or merely an artifact of the proof. While the guarantee applies for every choice of the initial positive definite weight operator
$W^{(0)}$, the default choice $W^{(0)}=\Id$ in \Cref{alg:algo1}
typically yields a first iterate that is already somewhat aligned with the ground truth~$\Xzero$.\footnote{It is easy to see that the solution
$\Xk{(0)}$ of \eqref{eq:IRLS:step_1} is simply
$\Xk{(0)}=\mathcal{A}^{\dagger}(\f{y})$ in this case, where
$\mathcal{A}^{\dagger}$ denotes the pseudo-inverse of the measurement
operator~$\mathcal{A}$.} To probe the worst-case behavior, we therefore track
the initial decrease factor~$\mu(0)$ of \cref{eq:decrease:factor} for an
initial weight operator constructed from an auxiliary reference matrix that is
deliberately chosen to be as poorly aligned with $\Xzero$ as possible.

In particular, to create such an \emph{adversarial initialization}, we first compute the standard nuclear norm minimizer $\f{X}_{\mathrm{nuc}}$ of \eqref{eq:nucnorm:min}. Let $\f{X}_{\mathrm{nuc}}^{(r)} = \f{U}_{\mathrm{nuc}}^{(r)} \diag(\sigma_{\mathrm{nuc}}^{(r)}) \f{V}_{\mathrm{nuc}}^{(r)^{\top}}$ denote its best rank-$r$ approximation, and let
$T = \left\{ \f{X}=  \f{U}_{\mathrm{nuc}}^{(r)} \f{M} + \f{N} \f{V}_{\mathrm{nuc}}^{(r)^{\top}}
, \f{M} \in \R^{r \times d_2}, \f{N} \in \R^{d_1 \times r} \right\} \subset \Rdd$ 
be the tangent space of the fixed-rank manifold at $\f{X}_{\mathrm{nuc}}^{(r)}$ \citep{Vandereycken13}. Writing $P_{T^{\perp}}(\cdot)$ for the orthogonal projection onto the complementary space $T^{\perp}$, we then define
\begin{equation*}
\Xk{(-1)}
\;=\;
\argmin_{\f{X}\in\Rdd}\;
\nucnorm{\f{X}}
\quad\text{subject to}\quad
\mathcal{A}\bigl(P_{T^{\perp}}(\f{X})\bigr)=\f{y}.
\end{equation*}
Relative to the geometry suggested by $\f{X}_{\mathrm{nuc}}$, this forces the measurements to be explained through the $T^{\perp}$ component and thereby yields an auxiliary reference matrix that is poorly aligned with $\Xzero$. Since the two-sided orthogonal projection $P_{T^\perp}(\cdot)$ is nuclear-norm nonexpansive, a minimizer can be chosen in $T^\perp$; in particular, $\mathcal A(\Xk{(-1)})=\f{y}$. We initialize the experimental recurrence by $\varepsilon_{-1}
    :=\besterrNuc{\Xk{(-1)}}{r}/ d$ and $W^{(0)}:=\W{\Xk{(-1)}}{\varepsilon_{-1}}$. The weighted least-squares step \cref{eq:IRLS:step_1} then produces $\Xk{(0)}$.

\begin{figure}[t]
    \centering
    \includegraphics[width=0.92\textwidth]{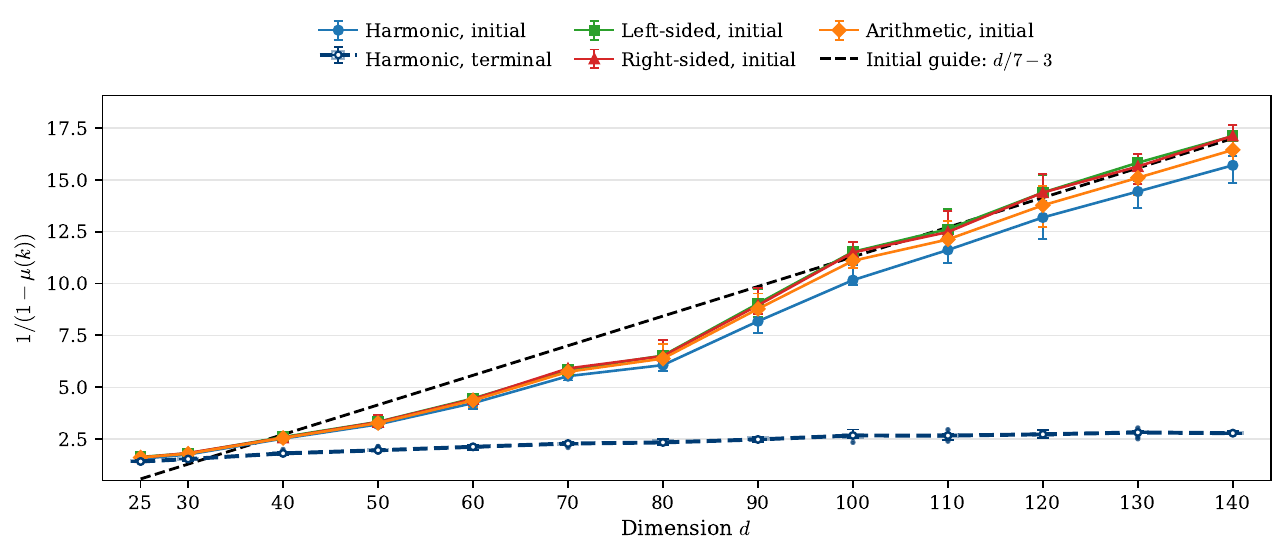}
    \caption{Reciprocal first-algorithmic contraction gaps $1/(1-\mu(0))$ (for initial), measuring the transition from $\Xk{(0)}$ to $\Xk{(1)}$, for the extended recurrence initialized by $(\Xk{(-1)},\varepsilon_{-1})$ and the resulting $W^{(0)}$, across dimensions $d \in \{25,30,40,50,60,70,80,90,100,110,120,130,140\}$ in the square random Gaussian rank-one measurement setup with rank-$2$ ground truth and linear measurement scaling; five terminal iteration average of $1/(1-\mu(k))$ for the harmonic-mean variant. Median values across $10$ seeds with $25\%$ and $75\%$ quantiles.}
\label{fig:ConvergenceRate}
\end{figure}

With this extended adversarial initialization, we revisit the experimental setup of
\Cref{sec:linear:convergence:rate:factors} and recover rank $r=2$ ground truth matrices $\Xzero \in \R^{\mind \times \mind}$ of different sizes $d \in \{25,30,40,50,60,70,80,$ $90,100,110,120,130,140\}$ from $m = 9 r d$ rank-one measurements. As in \Cref{sec:linear:convergence:rate:factors}, we consider different weight operator variants. Revisiting the contraction factor $\mu(0)$ of first non-initialization iteration as defined in \cref{eq:decrease:factor}, we plot the dimension-dependent behavior of $1/(1-\mu(0))$ in \Cref{fig:ConvergenceRate} in box plots across $10$ random seeds.

We observe that across all IRLS variants, $1/(1-\mu(0))$ grows approximately linearly with the dimension $d$, especially over the larger dimensions, with the reference line $d/7-3$ amounting to a suitable fit for the setup considered. We note that the first algorithmic decrease for this adversarial setup is markedly slower than for the standard least-squares initialization illustrated in \Cref{fig:DecreaseFactorVariantsStandard} of \Cref{sec:linear:convergence:rate:factors}, particularly as $d$ grows.

For harmonic mean IRLS specifically, however, that this deterioration is a \emph{transient} effect, as the value of $\mu(k)$ quickly falls to a lower, terminal value. We indicate in \Cref{fig:ConvergenceRate} as \emph{Harmonic, terminal} the distribution of the values of $\frac{1}{5} \sum_{k \text{ among last five iterations}} 1/(1-\mu(k))$ for the harmonic-mean variant. We observe that this value grows only marginally with the dimension $d$, staying well below $3.0$ even in the larger dimension range of $d \in \{100,110,120,130,140\}$. On the other hand, for the other variants, the in-trajectory and terminal values of $1/(1-\mu(k))$ grow linearly with $d$ significantly \emph{faster} than $1/(1-\mu(k))$, similar to \Cref{fig:RectangularTallDecreaseFactorVariants}, with the terminal five-iteration averages growing above $50$ for the one-sided variants and to almost $40$ for the arithmetic mean variant. For harmonic mean IRLS, the observed transition toward the same terminal behavior as for the standard initialization is consistent with the distinction between the global and local regimes in \Cref{sec:main:results}: the local factor $1-c_{\eta_r}$ of \Cref{thm:locallinearp1} depends neither on the dimension $d$ nor on the initialization once a neighborhood as in \cref{assump:localconvergencerate} has been reached.

\subsection{Role of the Smoothing Parameter Update Rules} \label{sec:smoothing:parameter:experiments}
As discussed in \Cref{sec:smoothing:parameter}, the smoothing-parameter update rule is a second important design choice for IRLS. We revisit the experimental setup of \Cref{sec:linear:convergence:rate:factors} and compare the $\ell_1$-tail update \eqref{eq:IRLS:step_2}, the $\ell_2$-tail update \eqref{eq:IRLS:step_2:frobenius}, and the $\ell_\infty$-tail update \eqref{eq:IRLS:step_2:singularvalue}.

In the left panel of \Cref{fig:SmoothingParameterSelected}, we show a typical trajectory for the oversampling factor of $C = 3$. We observe that the $\ell_\infty$-tail rule does not yield convergence.
However, the advantage via faster convergence of the harmonic-mean variant observed in \Cref{sec:linear:convergence:rate:factors,sec:rectangular:lowrank,sec:adversarial:initializations} relative to the one-sided variants persists across all smoothing-parameter update rules.
Increasing the oversampling factor to $C=4$ improves the per-iteration gap decrease for all variants, as illustrated in the right panel of \Cref{fig:SmoothingParameterSelected}, and enables ground truth convergence for the $\ell_\infty$-tail rule-based variants as well.

\begin{figure}[t]
    \centering
   \includegraphics[width=\textwidth]{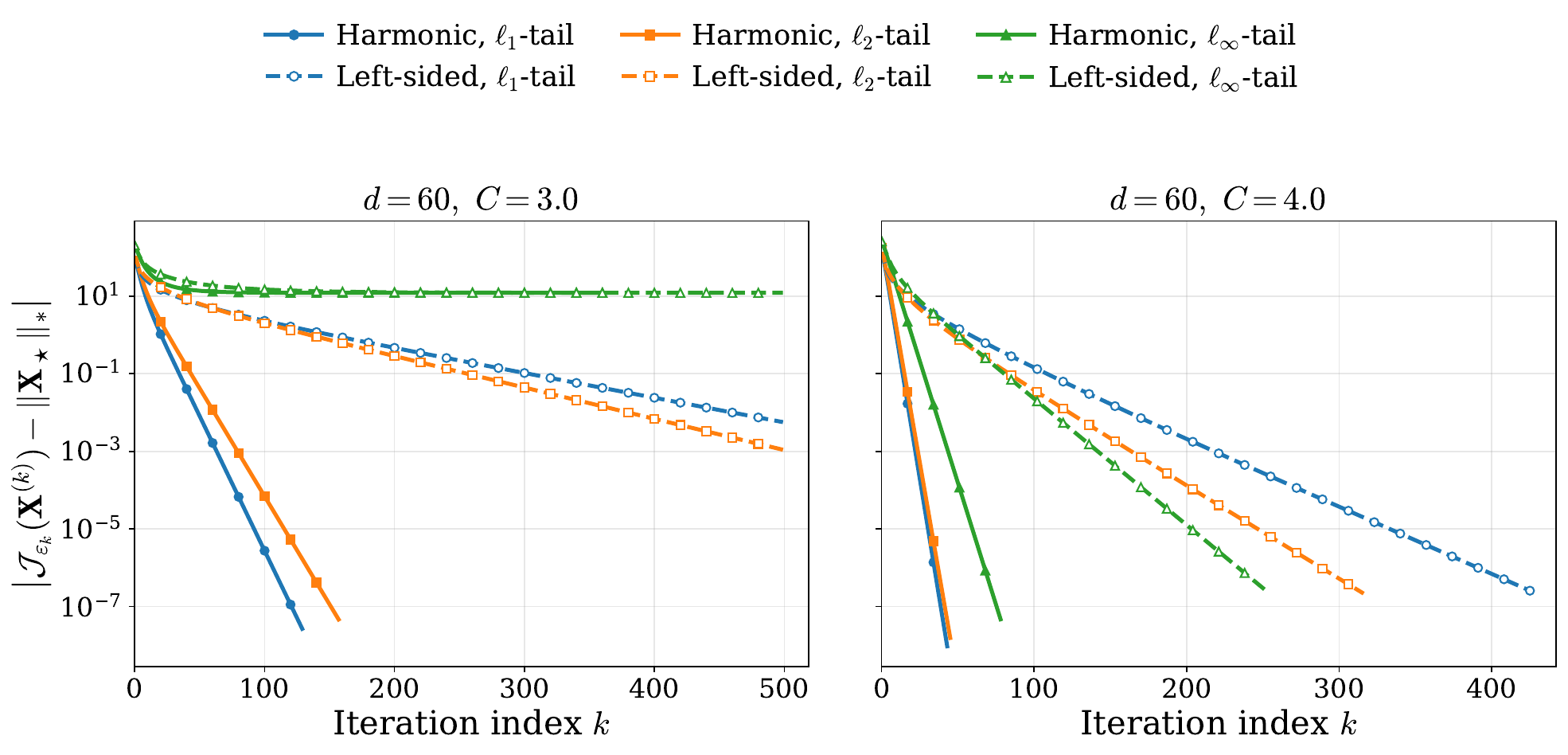}
    \caption{Selected trajectory plots of IRLS variants with different smoothing parameter update rules at dimension $d=60$ with rank-$2$ ground truth and Gaussian rank-one measurements. Left: Objective gap $\left|\mathcal{J}_{\varepsilon_{k}} \left( \XXk \right)- \nucnorm{\Xzero}\right|$ vs. iteration $k$ for $C=3.0$. Right: Analogous experiment with oversampling factor of $C=4.0$.}
    \label{fig:SmoothingParameterSelected}
\end{figure}

A better understanding of the sampling-data dependence of the IRLS variants with different smoothing parameter update rules is provided in \Cref{fig:SmoothingParameterPhaseTransition}. In this experiment, we run $12$ problem instances for each combination of dimension $d \in \{30, 60, 100\}$ and a selection of different oversampling factors  $C \in [2.0, 5.0]$ for the different smoothing parameter update rules. The distribution of the relative Frobenius errors $\|\Xk{(\tilde{k})} - \Xzero\|_F / \|\Xzero\|_F$ is visualized using box plots, where $\tilde{k}$ is the minimum of $500$ and of the first iteration index for which two subsequent algorithm iterates differ by a relative change in Frobenius norm of less than $10^{-10}$. Naturally, the setup is such that the problem becomes easier to solve as $C$ increases. In addition, with the black star, we visualize the smallest value of $C$ for which the nuclear norm minimizer $\f{X}_{\mathrm{nuc}}$ of \cref{eq:nucnorm:min} recovers the ground truth $\Xzero$ (up to a relative error of at most $10^{-4}$) using the splitting conic solver \texttt{SCS} \citep{OCPB-2016} via \texttt{CVXPY} \citep{Diamond-2016cvxpy}.

\begin{figure}[t]
    \centering
   \includegraphics[width=\textwidth]{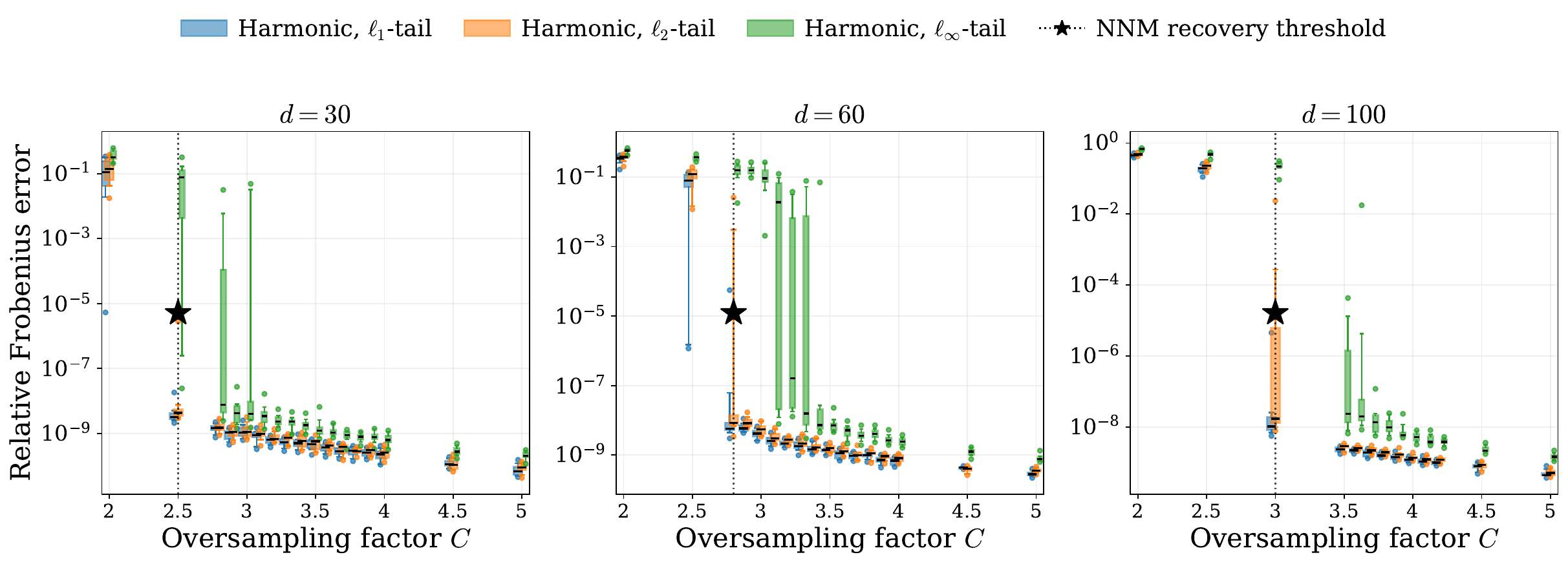}
    \caption{Phase transition experiment for nuclear norm MatrixIRLS with harmonic-mean weight operator, showing distribution of relative Frobenius errors across $12$ problem instances for different oversampling factors $C$ and different smoothing parameter update rules. }
    \label{fig:SmoothingParameterPhaseTransition}
\end{figure}
We observe that the box plots for the $\ell_1$-tail and $\ell_2$-tail variants of \Cref{alg:algo1} track closely the performance of the nuclear norm minimizer $\f{X}_{\mathrm{nuc}}$ in the setting of \Cref{fig:SmoothingParameterPhaseTransition}, recovering $\Xzero$ whenever the nuclear norm minimizer $\f{X}_{\mathrm{nuc}}$ does, with a slight edge of the $\ell_1$-tail variant over the $\ell_2$-tail as can be seen, e.g., for the $d=100$ and $C=3.0$ case. The $\ell_\infty$-tail variant, on the other hand, trails in a low-sampling regime, but transitions also to consistent recovery for around $C = 2.9$ for $d=30$, $C = 3.4$ for $d=60$ and $C = 3.7$ for $d=100$, respectively.

Although the $\ell_\infty$-tail rule is less robust at low sampling levels, it can substantially reduce the cost of an IRLS iteration. For iteration $k$, define the rank envelope and objective gap by
\[
    r_{\mathrm{env}}^{(k)} := \#\{ i : \sigma_i(\Xk{(k)}) > \varepsilon_k \},
    \qquad
    g_k := \left|\mathcal{J}_{\varepsilon_k}(\Xk{(k)})-\|\Xzero\|_*\right|,
\]
respectively. As discussed in \Cref{sec:smoothing:parameter}, both the memory and time requirements of the tangent-space implementation depend on $r_{\mathrm{env}}^{(k)}$. For a fixed iterate, elementary tail-norm inequalities make the candidate $\ell_\infty$-smoothing parameter often larger than its $\ell_1$ and $\ell_2$ counterparts; in the representative runs below, this produces substantially smaller rank envelopes. \Cref{tab:SmoothingParameterDim100Renv} quantifies the resulting tradeoff for $d=100$, rank-$2$ instances with $C\in\{3.0,4.0\}$ and the harmonic-mean weight operator.

For $\tau=10^{-5}$, let $k_\tau := \min\{k:g_k\leq\tau\}$. Because $\Xk{(0)}$ is the least-squares initialization, $k_\tau$ is also the number of completed outer IRLS solves needed to reach the threshold. The IRLS runtimes in \Cref{tab:SmoothingParameterDim100Renv} are the median cumulative solver times over five timing repetitions after one warmup. The \texttt{SCS} rows report the numerical solve time returned by \texttt{CVXPY}'s solver statistics for the same instance of \eqref{eq:nucnorm:min}, using a solver tolerance of $10^{-9}$ and at most $50{,}000$ iterations; CVXPY canonicalization and SCS setup time are excluded. All timings were obtained in CPU double precision on a MacBook Pro with an Apple M4 Pro processor (14 cores, 24\,GB RAM).

\input{results/table2/table2.tex}

The $\ell_\infty$-tail rule maintains $r_{\mathrm{env}}^{(k)}=\widetilde{r}=2$ throughout both representative runs, whereas the $\ell_1$- and $\ell_2$-tail rules use substantially larger envelopes. For $C=4.0$, the lower per-iteration cost of the $\ell_\infty$-tail rule more than offsets its larger iteration count: it reaches the prescribed gap after $125$ iterations in $2.29$ seconds, compared with $43$ iterations and $4.88$ seconds for the $\ell_1$-tail rule and $47$ iterations and $3.53$ seconds for the $\ell_2$-tail rule. The corresponding SCS numerical solve time is $29.32$ seconds; for $C=3.0$, it is $25.37$ seconds. At the lower sampling level $C=3.0$, however, only the $\ell_1$- and $\ell_2$-tail rules reach the IRLS objective-gap threshold; the $\ell_\infty$-tail run terminates at $\tilde{k}=191$ with $g_{\tilde{k}}=5.99\times 10^{1}$. Thus, the computational advantage of a smaller rank envelope must be balanced against robustness at lower sampling levels. In fact, the $\ell_1$-tail rule remains the conservative, theoretically supported default. Comparing the \texttt{MatrixIRLS} runtimes with the SCS solve times of $25.37$ and $29.32$ seconds, respectively, we observe that IRLS can achieve order-of-magnitude speedups over generic SDP solvers, especially in setups farther above the sample complexity phase transition point.

%% file: results/table2/table2.tex
\begin{table}[H]
    \centering
    \caption{Envelope ranks and computational effort for the representative $d=100$, rank-$2$, seed-$42$ smoothing-parameter experiment with the harmonic mean weight operator. For IRLS, runtime is the median cumulative solver time to reach the objective-gap threshold $\tau=10^{-5}$ over 5 timing repetitions after 1 warmup. For \texttt{SCS}, runtime is the solver-reported numerical solve time; CVXPY canonicalization and SCS setup are excluded. A dash denotes a quantity that is not applicable or an objective-gap threshold not reached before termination.}
    \label{tab:SmoothingParameterDim100Renv}
    \begin{tabular}{c l ccc ccc}
        \toprule
        $C$ & Update rule / solver
            & $r_{\mathrm{env}}^{(1)}$
            & $r_{\mathrm{env}}^{(10)}$
            & $r_{\mathrm{env}}^{(\tilde{k})}$
            & $\tilde{k}$
            & $k_\tau$
            & Runtime [s] \\
        \midrule
        $3.0$ & $\ell_1$-tail & $39$ & $21$ & $15$ & $359$ & $269$ & $18.39$ \\
        $3.0$ & $\ell_2$-tail & $29$ & $15$ & $12$ & $726$ & $585$ & $27.30$ \\
        $3.0$ & $\ell_\infty$-tail & $2$ & $2$ & $2$ & $191$ & -- & -- \\
        $3.0$ & \texttt{SCS} & -- & -- & -- & -- & -- & $25.37$ \\
        \midrule
        $4.0$ & $\ell_1$-tail & $38$ & $22$ & $19$ & $57$ & $43$ & $4.88$ \\
        $4.0$ & $\ell_2$-tail & $29$ & $17$ & $14$ & $62$ & $47$ & $3.53$ \\
        $4.0$ & $\ell_\infty$-tail & $2$ & $2$ & $2$ & $153$ & $125$ & $2.29$ \\
        $4.0$ & \texttt{SCS} & -- & -- & -- & -- & -- & $29.32$ \\
        \bottomrule
    \end{tabular}
\end{table}

%% file: conclusion.tex
\section{Conclusion}\label{sec:conclusion}  
In this work, we provide comprehensive answers to fundamental questions regarding both the design and analysis of iteratively reweighted least squares methods for the central spectral optimization problem of nuclear norm minimization, which had remained unanswered since the initial works of \citet{Fornasier11} and \citet{mohan_fazel} despite the recent interest in the methodology in the literature: we frame the algorithm design around the \emph{weight operator} notion and show that a particular choice of weight operator, the harmonic mean weight operator, constitutes an arguably \emph{optimal} choice, which we substantiate with the rigorous majorization and tightness results of \Cref{section:algorithm}.

Building on these results, we provide the first convergence rate analysis of IRLS methods for the problem that is furthermore fine-grained as it, (a) shows that IRLS methods using a variety of weight operators, including traditional ones as well as the harmonic mean one (see \Cref{def:admissible:weightoperators}), exhibit \emph{global linear convergence} from any initialization, and (b) shows that the convergence rate can be improved to a \emph{dimension-free linear rate} once the iterates reach a neighborhood of the low-rank solution, if the correct weight operator notion is used.

The numerical experiments of \Cref{sec:simulations} corroborate this separation: across rectangular and square instances, IRLS with harmonic mean reweighting exhibits a substantially faster empirically observed linear rate than the one-sided or arithmetic mean variants. We further presented adversarial initialization experiments indicating that the dimension dependence appearing in the global linear rate of \Cref{mainresult:lowrank} across all IRLS variants is not merely an artifact of the proof, but also, that it tends to be a transient phenomenon associated with the initial phase of the algorithmic trajectory. Finally, our results highlight the role of the smoothing parameter continuation strategy as a second important algorithmic design choice. The nuclear-norm-tail update used in our theory couples the smoothing parameter directly to the quantity needed in the global convergence argument and, empirically, tracks the nuclear-norm recovery threshold reliably. More computationally economical updates are not presently covered by the convergence theory, so that the generalization of linear rate results to IRLS with $\ell_2$- or $\ell_\infty$-tail updates remains open for future research.

Several other interesting questions remain open: the presented convergence analysis crucially depends on the fact that the linear measurement operator $\mathcal{A}$ satisfies a suitable null space property. However, it is well-known that for some important application scenarios such as low-rank matrix completion, this assumption does not hold. It would be interesting to develop a theory which also covers this scenario. 
Another observation is that although the local contraction factor of \Cref{thm:locallinearp1} is dimension-independent, the certified neighborhood in which the local behavior applies still depends on the dimension. The experiments, however, suggest that the favorable fast linear rate regime for \texttt{MatrixIRLS} can arise substantially earlier than this worst-case basin predicts. Understanding whether the methodology and/or proof technique can be improved to establish a similar result for a dimension-free basin size in a worst-case scenario remains an open problem. 
We also note that our quadratic model majorization and tightness analysis currently only applies to the smoothed nuclear norm objective that is conducive to \cref{eq:nucnorm:min}, but not to the optimization of other spectral objectives such as non-convex Schatten-$p$ quasi-norms with $0<p<1$ or smoothed log-determinants. It remains unknown which weight operator notion defines a tight or the tightest majorizing quadratic model for these spectral objectives, with the geometric mean constituting a popular, but so far theoretically unjustified choice in the latter context \citep{KummerleMayrinkVerdun-ICML2021,Kraemer-2025OneSided}. Finally, it would be interesting to explore how the insights of this work can improve either the design or analysis of related methodologies in machine learning, such as recursive feature machines (RFM) \citep{Radhakrishnan2024mechanism}, iteratively reweighted kernel machines \citep{Zhu2025iteratively}, or \texttt{Muon}-like spectral optimizers in deep learning. The fact that the specialization of RFM to low-rank optimization \citep{Radhakrishnan2025linear} leads to suboptimal IRLS variants indicates that there can be room for improvement via the use of tight majorizing quadratic models.

%% file: acknowledgements.tex
\section*{Acknowledgments}
The authors thank Felix Krahmer for insightful discussions around the paper subject. C.K. and T.M. acknowledge the support of this research by the Mathematisches Forschungsinstitut Oberwolfach through the \emph{Oberwolfach Research Fellows} program. C.K. was supported in part by the grant NSF-2549926.

%% file: contribution_proofs.tex
\section{Proofs of Main Results} \label{sec:contribution_proofs}

In this section, we provide all proofs establishing the main theoretical results of the paper, starting with the global majorization statement \Cref{thm:majorization} of \Cref{sec:majorization:properties} in \Cref{sec:appendix:harmonic:majorization_proof}. We continue with the proof of the optimality result \Cref{thm:power_means_majorize} in \Cref{sec:proofs:power_means:optimality}, before detailing the global linear convergence results \Cref{mainresult:lowrank,mainresult:approximatelowrank} in \Cref{sec:proofglobalconvergence:lowrank} and \Cref{sec:proofglobalconvergence:approximatelowrank}, respectively. Finally, establishing the fast local linear convergence result of \Cref{thm:locallinearp1} is done in \Cref{sec:prooffastlocallinearp1}. The latter section also contains the proof of the impossibility result \Cref{thm:counterexample:leftsided:weight:operator} for dimension-independent fast local rate for one-sided weight operators.

%% file: majorization_property.tex
\subsection{Proof of \Cref{thm:majorization} (Harmonic-Mean Quadratic Model Majorization)} \label{sec:appendix:harmonic:majorization_proof}
The major challenge in the proof of \Cref{thm:majorization} in the case of harmonic-mean weights is to derive a suitable lower bound for the weighted inner product $\innerproduct{ \W{\XX}{\varepsilon} (\ZZ), \ZZ }$.
The argument below isolates this lower bound and then proves the majorization property directly.
Afterwards, in \Cref{sec:appendix:proof_majorization_one_sided_weights}, we explain how the same proof scheme specializes to the simpler one-sided weights, where the required lower bound follows immediately from the Cauchy--Schwarz inequality.

However, in the case of harmonic-mean weights proving a lower bound is significantly more challenging due to the more complicated structure of the weight operator $\W{\XX}{\varepsilon} (\cdot)$. The following lemma provides such a bound.
\begin{lemma}[Lower Bound for Weighted Inner Product of Harmonic Mean] \label{lemma:lower_bound_weighted_inner_product}\mbox{}\\
Let $\varepsilon > 0$
and let $\XX \in \Rdd$ 
with SVD given by $ \XX = \UU_{\XX} \diag (\ssigma) \VV_{\XX}^\top$, where
$\diag (\ssigma) \in \Rdd$ is the rectangular diagonal matrix with the extended vector of nonincreasing singular values $\ssigma \in \R^{\maxD}$ of $\XX$ on the diagonal.
Consider the weight operator
$ W_{\XX, \varepsilon} : \Rdd \to \Rdd $
of \Cref{def:optimalweightoperator} with harmonic-mean core matrix \cref{eq:harmonic:mean}. Define the matrices
$\LL_{\UU} = \UU_{\XX} \diag (\lambda_1, \ldots, \lambda_{d_1}) \UU_{\XX}^\top $
and
$ \LL_{\VV} = \VV_{\XX} \diag (\lambda_1, \ldots, \lambda_{d_2}) \VV_{\XX}^\top $,
where $ \lambda_i := \max \left( \sigma_i, \varepsilon \right)$
for $1 \le i \le \max \left( d_1, d_2 \right)$ with the convention of $\sigma_i = 0$ for $i > d$.

Let now $\ZZ \in \Rdd$
be arbitrary with SVD given by $ \ZZ = \sum_{k=1}^d \sigma_k (\ZZ) \uu_k \vv_k^\top $, where $\sigma_k (\ZZ)$
denotes the $k$-th singular value of $\ZZ$. Then,
it holds that 
\begin{equation*}
    \innerproduct{ \W{\XX}{\varepsilon} (\ZZ), \ZZ }
    \ge
    \sum_{k=1}^d
    \frac{ \sigma_k^2 (\ZZ)}{ 
        \innerproduct{ \LL_{\UU}, \uu_k \uu_k^\top }/2 
        +
        \innerproduct{ \LL_{\VV}, \vv_k \vv_k^\top  }/2
    }.
\end{equation*}
\end{lemma}
As already mentioned,
the proof of Lemma \ref{lemma:lower_bound_weighted_inner_product}
is significantly more involved than the corresponding lower bound for one-sided weights.
We provide the proof of \Cref{lemma:lower_bound_weighted_inner_product} in \Cref{sec:appendix:lb_weighted_ip} below.
The major challenge in the proof of this lemma is to deal with the fact that the inner product $\innerproduct{ \W{\XX}{\varepsilon} (\ZZ), \ZZ }$
no longer has a simple structure as in the case of one-sided weights, where it could be expressed as 
$ \sum_{k=1}^d  \sigma_k^2 ( \ZZ ) \innerproduct{ \LL_{\UU}^{-1}, \uu_k \uu_k^\top } $.

With Lemma \ref{lemma:lower_bound_weighted_inner_product} at hand,
we can now prove \Cref{thm:majorization} in the case of harmonic-mean weight operators.
\begin{proof}[Proof of \Cref{thm:majorization}]
We first compute the terms in the simplified representation
\eqref{eq:QZX:equality} of $Q_\varepsilon(\cdot\mid \XX)$.
Since harmonic-mean weights satisfy
$(\f{H}_{\ssigma,\varepsilon})_{ii}=1/\max(\sigma_i(\XX),\varepsilon)$
on the diagonal, we have
\begin{align*}
    \innerproduct{
        \W{\XX}{\varepsilon}(\XX),\XX
    }
    =&
    \innerproduct{
        \f{H}_{\ssigma,\varepsilon}
        \circ
        \diag(\ssigma),
        \diag(\ssigma)
    }\\
    =&
    \sum_{k \in [d]:\ \sigma_k (\XX) \ge \varepsilon} \sigma_k (\XX)
    +
    \sum_{k \in [d]:\ \sigma_k (\XX) < \varepsilon} 
    \frac{\sigma_k^2 (\XX)}{\varepsilon}.
\end{align*}
Moreover, by definition of $\mathcal{J}_\varepsilon$,
\begin{equation*}
    \mathcal{J}_{\varepsilon} (\XX)
    =
    \sum_{k \in [d]:\ \sigma_k (\XX) \ge \varepsilon} \sigma_k (\XX)
    +
    \sum_{k \in [d]:\ \sigma_k (\XX) < \varepsilon}
    \left(
        \frac{\sigma_k^2 (\XX)}{2 \varepsilon}
        +
        \frac{\varepsilon}{2}
    \right).
\end{equation*}
Inserting these two identities into \eqref{eq:QZX:equality}, we obtain that
\begin{align}
    Q_{\varepsilon}(\ZZ \mid \XX)
    =&
    \frac{1}{2}
    \sum_{k \in [d]:\ \sigma_k (\XX) \ge \varepsilon} \sigma_k (\XX)
    +
    \frac{1}{2}
    \sum_{k \in [d]:\ \sigma_k (\XX) < \varepsilon} \varepsilon
    +
    \frac{1}{2}
    \innerproduct{ \W{\XX}{\varepsilon} (\ZZ), \ZZ }.
    \label{eq:Q_epsilon_intermediate}
\end{align}
Now, note that by definition of $\LL_{\UU}$, 
it holds that
\begin{align*}
    \sum_{k=1}^{d}
    \innerproduct{ \LL_{\UU}, \uu_k \uu_k^\top }
    \overleq{(a)}
    &\underset{\f{z}_1,\ldots, \f{z}_d \text{ orthonormal}}{\max}
    \sum_{k=1}^{d}
    \innerproduct{ \LL_{\UU},  \f{z}_k \f{z}_k^\top }
    \overeq{(b)}
    \sum_{k=1}^{d}
    \lambda_k\\
    \overeq{(c)}
    &\sum_{k \in [d]:\ \sigma_k (\XX) \ge \varepsilon} \sigma_k (\XX)
    +
    \sum_{k \in [d]:\ \sigma_k (\XX) < \varepsilon} \varepsilon,
\end{align*}
where in inequality (a), we have used that
$(\uu_k)_{k=1}^{d}$ is an orthonormal set in $\R^{d_1}$.
In equality (b), we have used the Wielandt minimax principle,
see \citep[Theorem III.3.5]{Bhatia1997MatrixAnalysis},
and that $\lambda_1 \ge \lambda_2 \ge \ldots \ge \lambda_{d_1} \ge 0$
are the eigenvalues of $\LL_{\UU}= \UU_{\XX} \diag ( \lambda_1, \ldots, \lambda_{d_1} ) \UU_{\XX}^\top $.
In equation (c), we used that
$\lambda_k = \max ( \sigma_k (\XX), \varepsilon )$
for $k=1,2,\ldots, d_1$.
Analogously, we obtain that
\begin{equation*}
    \sum_{k=1}^{d}
    \innerproduct{ \LL_{\VV}, \vv_k \vv_k^\top }
    \le
    \sum_{k \in [d]: \ \sigma_k (\XX) \ge \varepsilon} \sigma_k (\XX)
    +
    \sum_{k \in [d]: \ \sigma_k (\XX) < \varepsilon} \varepsilon.
\end{equation*}
Inserting these two inequalities into equation \eqref{eq:Q_epsilon_intermediate} for
$Q_{\varepsilon}(\ZZ \mid \XX)$,
we obtain that
\begin{equation*}
    Q_{\varepsilon} \left( \ZZ  \vert \XX \right)
    \ge  
    \frac{1}{4}
    \sum_{k=1}^{d}
    \innerproduct{ \LL_{\UU}, \uu_k \uu_k^\top }
    +
    \frac{1}{4}
    \sum_{k=1}^{d}
    \innerproduct{ \LL_{\VV}, \vv_k \vv_k^\top }
    +
    \frac{1}{2}
    \innerproduct{ \W{\XX}{\varepsilon} (\ZZ), \ZZ }.
\end{equation*}
By applying Lemma \ref{lemma:lower_bound_weighted_inner_product},
we obtain that
\begin{align}
    Q_{\varepsilon}(\ZZ \mid \XX)
    \ge &
    \frac{1}{2}
    \sum_{k=1}^d 
    \Big[
        \bracing{=:\alpha_k}{
        \frac{\innerproduct{\LL_{\UU}, \uu_k \uu_k^\top } }{2}
        +
        \frac{\innerproduct{\LL_{\VV}, \vv_k \vv_k^\top } }{2}
        }
        +
        \frac{ \sigma^2_k(\ZZ) }{
            \innerproduct{\LL_{\UU}, \uu_k \uu_k^\top }/2
            +
            \innerproduct{\LL_{\VV}, \vv_k \vv_k^\top }/2
        }
    \Big] \nonumber \\
    =& 
    \frac{1}{2}
    \sum_{k=1}^d
    \frac{1}{\alpha_k}
    \left(
        \alpha_k^2
        +
        \sigma^2_k(\ZZ)
    \right).
    \label{eq:Q_epsilon_intermediate4}
\end{align}
Now note that
$\LL_{\UU}$ and $\LL_{\VV}$ are positive semidefinite matrices
with eigenvalues lower bounded by $\varepsilon$.
Thus, we have that $\innerproduct{\LL_{\UU}, \uu_k \uu_k^\top } \ge \varepsilon$
and $\innerproduct{\LL_{\VV}, \vv_k \vv_k^\top } \ge \varepsilon$
for all $k=1,2,\ldots, d$.
In particular, we have that $\alpha_k \ge \varepsilon$ for all $k=1,2,\ldots, d$.
We now estimate the scalar summands in \eqref{eq:Q_epsilon_intermediate4}.
If $\sigma_k(\ZZ) \ge \varepsilon$, then
$\alpha_k^2+\sigma_k^2(\ZZ)\ge 2\alpha_k\sigma_k(\ZZ)$, and hence
\begin{equation*}
    \frac{1}{2\alpha_k}
    \left(
        \alpha_k^2+\sigma_k^2(\ZZ)
    \right)
    \ge
    \sigma_k(\ZZ).
\end{equation*}
If $\sigma_k(\ZZ)<\varepsilon$, then $\alpha_k\ge \varepsilon>\sigma_k(\ZZ)$.
For fixed $\sigma_k(\ZZ)$, the function
$x\mapsto (x^2+\sigma_k^2(\ZZ))/(2x)$ is monotonically increasing for
$x\ge \sigma_k(\ZZ)$.
Therefore,
\begin{equation*}
    \frac{1}{2\alpha_k}
    \left(
        \alpha_k^2+\sigma_k^2(\ZZ)
    \right)
    \ge
    \frac{1}{2\varepsilon}
    \left(
        \varepsilon^2+\sigma_k^2(\ZZ)
    \right)
    =
    \frac{\sigma_k^2(\ZZ)}{2\varepsilon}
    +
    \frac{\varepsilon}{2}.
\end{equation*}
Summing these two cases gives
\begin{equation*}
   \frac{1}{2}
   \sum_{k=1}^{d}
    \frac{1}{\alpha_k}
    \left(
        \alpha_k^2
        +
        \sigma^2_k(\ZZ)
    \right)
    \ge
    \sum_{k \in [d]: \ \sigma_k (\ZZ) \ge \varepsilon}
    \sigma_k (\ZZ)
    +
    \sum_{k \in [d]: \ \sigma_k (\ZZ) < \varepsilon}
    \left(
        \frac{\sigma^2_k (\ZZ)}{2\varepsilon}
        +
        \frac{\varepsilon}{2}
    \right). 
\end{equation*}
Inserting this case distinction 
into Equation \eqref{eq:Q_epsilon_intermediate4},
we obtain that
\begin{equation*}
    Q_{\varepsilon}(\ZZ \mid \XX)
    \ge
    \sum_{k \in [d]: \ \sigma_k (\ZZ) \ge \varepsilon}
    \sigma_k (\ZZ)
    +
    \sum_{k \in [d]: \ \sigma_k (\ZZ) < \varepsilon}
    \left(
        \frac{\sigma^2_k (\ZZ)}{2 \varepsilon}
        +
        \frac{\varepsilon}{2}
    \right)
    =
    \sum_{k=1}^d j_{\varepsilon} (\sigma_k (\ZZ))
    =   
    \mathcal{J}_{\varepsilon} (\ZZ).
\end{equation*}
This completes the proof.
\end{proof}

\subsubsection{Proof of Lemma \ref{lemma:lower_bound_weighted_inner_product} via Iterative Pinching}\label{sec:appendix:lb_weighted_ip}
In this section, we provide a complete, constructive proof of \Cref{lemma:lower_bound_weighted_inner_product}. We first establish a Sylvester equation characterization of the weight operator (\Cref{lemma:sylvester_equation_weight_operator}) and an abstract majorization lemma (\Cref{lemma:majorization_property}), that we subsequently specialize to the case of the relevant Sylvester equation (\Cref{lemma:majorization_property_sylvester}), before introducing the iterative pinching argument via \Cref{lemma:iterative_pinching} and \Cref{lemma:sylvester_solution_explicit}. The proof of \Cref{lemma:lower_bound_weighted_inner_product} is completed at the end of this subsection.

\paragraph{Sylvester Equation and Abstract Majorization Lemma.} \label{sec:appendix:lb_weighted_ip:sylvester}
The first observation we make is that the harmonic-mean weight operator $\W{\XX}{\varepsilon} (\cdot)$
satisfies the Sylvester equation \cref{eq:sylvester_equation_weight_operator} below.
\begin{lemma}[Sylvester Equation Characterization] \label{lemma:sylvester_equation_weight_operator}
    Let $\f{X} \in \Rdd$ and $\varepsilon > 0$.
    Let the SVD of $\XX$ be given by $\XX = \UU_{\XX} \diag(\ssigma) \VV_{\XX}^{\top}$,
    where $\ssigma = (\sigma_1, \ldots, \sigma_d) \in \R^d$ are the singular values of $\f{X}$ ordered in nonincreasing order.
    For $i \in [\max \left( d_1 , d_2 \right)]$,
    set
    $\lambda_i := \max \left( \sigma_i, \varepsilon \right)$ with the convention of $\sigma_i = 0$ for $i > d$.
    Define the vectors
    $\llambdaa_1 := (\lambda_1, \ldots, \lambda_{d_1}) \in \R^{d_1}$
    and
    $\llambdaa_2 := (\lambda_1, \ldots, \lambda_{d_2}) \in \R^{d_2}$.
    Then, 
    for any 
    $\f{Z} \in \Rdd$,
    its image via the harmonic-mean weight operator  $\W{\XX}{\varepsilon} (\ZZ)$ with core matrix \cref{eq:harmonic:mean} is the unique solution of the Sylvester equation
    \begin{equation} \label{eq:sylvester_equation_weight_operator}
        \LL_{\UU} \W{\XX}{\varepsilon} (\ZZ) 
        + \W{\XX}{\varepsilon} (\ZZ) \LL_{\VV} 
        = 2 \ZZ,
    \end{equation}
    where $\LL_{\UU} = \f{U}_{\XX} \diag(\llambdaa_1) \f{U}_{\XX}^{\top}$
    and
    $\LL_{\VV} = \f{V}_{\XX} \diag(\llambdaa_{2}) \f{V}_{\XX}^{\top}$.
\end{lemma}
\begin{remark}
    Note that in the one-sided case of the core matrix \cref{eq:leftsided:mean}, the weight operator $\W{\XX}{\varepsilon} (\cdot)$ satisfies the simpler Sylvester equation
    $ \LL_{\UU} \W{\XX}{\varepsilon} (\ZZ) = \ZZ$.
     It is thus plausible that in the case of harmonic-mean weight operators $\W{\XX}{\varepsilon}(\cdot)$, the relevant Sylvester equation involves both $\LL_{\UU}$ and $\LL_{\VV}$.
\end{remark}
We expect that Lemma \ref{lemma:sylvester_equation_weight_operator} 
is known in the literature \citep[see, e.g.,][]{Bhatia1997MatrixAnalysis}.
However, for the sake of completeness,
we provide a proof here.
\begin{proof}[Proof of \Cref{lemma:sylvester_equation_weight_operator}]
    First, we note that both $\LL_{\UU}$ and $\LL_{\VV}$ are positive definite.
    Then it follows from \citet[Theorem VII.2.3]{Bhatia1997MatrixAnalysis}
    that the Sylvester equation $\LL_{\UU} \WW + \WW \LL_{\VV} = 2 \ZZ$ has a unique solution $\WW  \in \Rdd$.

    In order to show that $\WW = \W{\XX}{\varepsilon} (\ZZ)$
    solves the Sylvester equation,
    we compute that
    \begin{align*}
        \UU_{\XX}^\top \LL_{\UU} \W{\XX}{\varepsilon} (\ZZ)  \VV_{\XX}
        \overeq{(a)}&
       \UU^{\top}_{\XX} \LL_{\UU}
       \left(
        \UU_{\XX}
        \left[
            \f{H}_{\ssigma,\varepsilon}
            \circ
            (\UU_{\XX}^{\top} \ZZ \VV_{\XX})
        \right]
        \right)
        \VV_{\XX}^{\top} \VV_{\XX}\\
        \overeq{(b)}&
         \UU_{\XX}^\top
        \left(
        \UU_{\XX}
        \diag(\llambdaa_1)
        \UU_{\XX}^{\top}
        \right)
        \left(
        \UU_{\XX}
        \left[
            \f{H}_{\ssigma,\varepsilon}
            \circ
            (\UU_{\XX}^{\top} \ZZ \VV_{\XX})
        \right]
        \VV_{\XX}^{\top} \right)
        \VV_{\XX}\\
        =&
        \diag(\llambdaa_1)
        \left[
            \f{H}_{\ssigma,\varepsilon}
            \circ
            (\UU_{\XX}^{\top} \ZZ \VV_{\XX})
        \right].
    \end{align*}
    In step (a), we have used the definition of the weight operator
    $\W{\XX}{\varepsilon} (\ZZ)$,
    see equation \eqref{eq:W:operator:action},
    and in step (b), we have used the definition of $\LL_{\UU}$.
    Note that since $\lambda_i = \max(\sigma_i, \varepsilon)$, the $(i,j)$-th entry of the last expression is 
    then given by
    \begin{align*}
        (\UU_{\XX}^\top \LL_{\UU} \W{\XX}{\varepsilon} (\ZZ)  \VV_{\XX} )_{i,j}
        &=
        \frac{ 2 \max (\sigma_i, \varepsilon)  }
        {\max(\sigma_i, \varepsilon) + \max(\sigma_j, \varepsilon)}
        \cdot
        (\UU_{\XX}^{\top} \ZZ \VV_{\XX})_{i,j}.
    \end{align*}
    Analogously, we compute that
    \begin{align*}
        (\UU_{\XX}^\top \W{\XX}{\varepsilon} (\ZZ) \LL_{\VV}  \VV_{\XX})_{i,j}
        =&
        \frac{ 2\max (\sigma_j, \varepsilon)  }
        {\max(\sigma_i, \varepsilon) + \max(\sigma_j, \varepsilon)}
        \cdot
        (\UU_{\XX}^{\top} \ZZ \VV_{\XX})_{i,j}.
    \end{align*}
    Summing up both terms, we obtain that
    \begin{equation*}
        \UU_{\XX}^\top
        \left(
            \LL_{\UU} \W{\XX}{\varepsilon} (\ZZ)
            +
            \W{\XX}{\varepsilon} (\ZZ) \LL_{\VV}
        \right)
        \VV_{\XX}
        =
        2 \UU_{\XX}^\top \ZZ \VV_{\XX}.
    \end{equation*}
    Multiplying from the left with $\UU_{\XX}$ and from the right with $\VV_{\XX}^{\top}$,
    we obtain equation \eqref{eq:sylvester_equation_weight_operator}.
    This completes the proof.
\end{proof}
We now present a constructive proof of \Cref{lemma:lower_bound_weighted_inner_product}. With the Sylvester characterization of the weight operator of \Cref{lemma:sylvester_equation_weight_operator} at hand, it proceeds by an iterative pinching argument. It provides structural insights into explicit intermediate solutions and the Loewner-monotonicity structure underlying the power-mean family. We outline the overall proof strategy for \Cref{lemma:lower_bound_weighted_inner_product}.
\begin{remark}[Proof strategy]
The Sylvester equation characterization
\begin{equation*}
    \LL_{\UU} \W{\XX}{\varepsilon} (\ZZ) 
    + \W{\XX}{\varepsilon} (\ZZ) \LL_{\VV} 
    = 2 \ZZ
\end{equation*}
of the weight operator
serves as a starting point
for the proof of Lemma \ref{lemma:lower_bound_weighted_inner_product}.
The main challenge is to deal with the fact that
the matrices $\LL_{\UU}, \LL_{\VV}, \ZZ$
do not necessarily commute with each other.
To this end,
we employ a majorization technique.
Namely,
we replace $\LL_{\UU}$ and $\LL_{\VV}$
in the Sylvester equation
by simpler matrices
$$\mathcal{C}_{\UU} (\LL_{\UU})
=  \PP_{\UU} \LL_{\UU} \PP_{\UU}
+ \PP_{\UU,\bot} \LL_{\UU} \PP_{\UU,\bot}$$
and
$$\mathcal{C}_{\VV} (\LL_{\VV})
=  \PP_{\VV} \LL_{\VV} \PP_{\VV}  
+ \PP_{\VV,\bot} \LL_{\VV} \PP_{\VV,\bot},$$
where $\PP_{\UU}$ and $\PP_{\VV}$
are orthogonal projections onto the subspaces
spanned by the left singular and right singular vectors of $\ZZ$,
see below for details.
We will show that the resulting Sylvester equation
\begin{equation*}
    \mathcal{C}_{\UU} (\LL_{\UU}) \widetilde{\WW}
    +
    \widetilde{\WW} \mathcal{C}_{\VV} (\LL_{\VV})
    =
    2 \ZZ
\end{equation*}
admits a solution $\widetilde{\WW} \in \Rdd$
that satisfies
the majorization inequality
\begin{equation*}
    \innerproduct{ \ZZ, \WW }
    \ge 
    \innerproduct{ \ZZ, \widetilde{\WW} },
\end{equation*}
where $\WW = \W{\XX}{\varepsilon} (\ZZ)$.
By repeating this majorization step,
we eventually arrive at a scenario
where both $\LL_{\UU}$ and $\LL_{\VV}$
are replaced by matrices that are diagonal
with respect to the singular vectors of $\ZZ$.
This allows us then to
explicitly compute the solution of the Sylvester equation
and derive the desired lower bound.
\end{remark}

To this end,
we start with the following general majorization lemma.
\begin{lemma}[Abstract Majorization Lemma] \label{lemma:majorization_property}
    Let $\mathcal{B}, \tilde{\mathcal{B}}: \Rdd \to \Rdd$ be linear operators.
    Let $\PP_{\UU} \in \R^{d_1 \times d_1}, \PP_{\VV} \in \R^{d_2 \times d_2}$
    be orthogonal projection matrices.
    Denote by $\PP_{\UU,\bot} := \f{I} - \PP_{\UU}$ 
    and 
    $\PP_{\VV,\bot} := \f{I} - \PP_{\VV}$
    the projections onto the orthogonal complement.
    Define the operation
    $\mathcal{C} (\AAA) := \PP_{\UU} \AAA \PP_{\VV} + \PP_{\UU,\bot} \AAA \PP_{\VV,\bot}$
    for any matrix $\AAA \in \Rdd$.
    Let now $\ZZ \in \Rdd$ be fixed.
    Assume that the following properties hold:
    \begin{enumerate}
        \item 
        Assume that 
        $ \ZZ = \mathcal{C} (\ZZ) $.
        \item $\mathcal{B}$ is self-adjoint, i.e.,
        $\innerproduct{ \mathcal{B}(\WW_1), \WW_2 }
        = \innerproduct{ \WW_1, \mathcal{B}(\WW_2) }$
        for all $\WW_1, \WW_2 \in \Rdd$.
        In addition, $\mathcal{B}$ is positive definite,
        i.e., $\innerproduct{ \mathcal{B}(\WW_1), \WW_1 } > 0$
        for all $\WW_1 \in \Rdd \backslash \{ \f{0} \}$.
        \item $\tilde{\mathcal{B}}$ is self-adjoint and positive definite.
        \item It holds that 
        \begin{equation*}
          \mathcal{C}
          (\tilde{\mathcal{B}} 
          ( 
            \mathcal{C} ( \hat{\WW} ) 
          )
          )
          =
          \mathcal{C} ( \mathcal{B} ( \mathcal{C} ( \hat{\WW} ) ) )
        \end{equation*}
        for any matrix $\hat{\WW} \in \Rdd$. 
        \item It holds that 
        \begin{equation}\label{assumption:offdiagonalzero}
           \PP_{\UU} \tilde{\mathcal{B}} 
           ( \mathcal{C} ( \hat{\WW} )  ) \PP_{\VV,\bot} 
           +
           \PP_{\UU,\bot} \tilde{\mathcal{B}} 
           ( \mathcal{C} ( \hat{\WW} )  ) \PP_{\VV}
           = \f{0}
        \end{equation}
        for any matrix $\hat{\WW} \in \Rdd$.
    \end{enumerate}
    Define $\WW := \mathcal{B}^{-1} (\ZZ)$
    and $\widetilde{\WW} := \tilde{\mathcal{B}}^{-1} (\ZZ)$.
    Then it holds that
    \begin{equation*}
        \innerproduct{ \ZZ, \WW }
        \ge 
        \innerproduct{\ZZ, \widetilde{\WW}}.
    \end{equation*}
\end{lemma}

\begin{proof}
    Let $ \mathcal{T}: (\Rdd, \innerproduct{\cdot, \cdot}) 
    \rightarrow (\mathbb{R}^{d_1 d_2}, \langle \cdot, \cdot \rangle_{\ell^2}) $
    be an isometric vectorization operator
    such that 
    \begin{equation*}
        \mathcal{T} (\AAA) = \begin{pmatrix}
            \aadiag \\
            \aaoff
        \end{pmatrix}
    \end{equation*}
    and 
    such that $\aadiag= \f{0}$ if and only if 
    $\mathcal{C} (\AAA) = \f{0}$
    and
    $\aaoff = \f{0}$ if and only if 
    $ \left( \mathcal{I} - \mathcal{C} \right) (\AAA) = \f{0}$
    for any matrix $\AAA \in \Rdd$.
    (In other words,
    the operator $\mathcal{T}$ splits the vectorization of a matrix $\AAA$
    into the ``diagonal'' part
    $\PP_{\UU} \AAA \PP_{\VV} + \PP_{\UU,\bot} \AAA \PP_{\VV,\bot}$
    and its ``off-diagonal'' part
    $ \PP_{\UU} \AAA \PP_{\VV,\bot} + \PP_{\UU,\bot} \AAA \PP_{\VV} $.)
    Then we can denote the vectorization of the matrices $\ZZ, \WW, \widetilde{\WW} \in \Rdd$ as 
    \begin{equation*}
        \mathcal{T} (\ZZ) = \begin{pmatrix}
            \zzdiag \\
            \zzoff
        \end{pmatrix},
        \quad
        \mathcal{T} (\WW) = \begin{pmatrix}
            \wwdiag \\
            \wwoff
        \end{pmatrix},
        \quad
        \mathcal{T} (\widetilde{\WW}) = \begin{pmatrix}
            \wwtildediag \\
            \wwtildeoff
        \end{pmatrix}.
    \end{equation*}
    Using the operator $\mathcal{T}$, 
    the equation $\ZZ = \mathcal{B} (\WW)$
    can then be rewritten as
    \begin{equation}\label{eq:vectorizedrepresentation}
        \begin{pmatrix}
            \zzdiag \\
            \zzoff
        \end{pmatrix}
        =
        ( \mathcal{T} \circ \mathcal{B} \circ \mathcal{T}^{-1} )
        \begin{pmatrix}
            \wwdiag \\
            \wwoff
        \end{pmatrix}
        =
        \begin{pmatrix}
            \AAA_{\text{1,1}} & \AAA_{\text{1,2}} \\
            \AAA_{\text{1,2}}^\top & \AAA_{\text{2,2}}
        \end{pmatrix}
        \begin{pmatrix}
            \wwdiag \\
            \wwoff
        \end{pmatrix}.
    \end{equation}
    In the last step, we have used that
    the operator $\mathcal{T} \circ \mathcal{B} \circ \mathcal{T}^{-1}$
    can be represented as a symmetric block matrix
    since $\mathcal{B}$ is linear and self-adjoint. 
    In a similar way, the equation
    $\widetilde{\WW} = \tilde{\mathcal{B}}^{-1} (\ZZ)$
    can be rewritten as
    \begin{equation}\label{eq:vectorizedrepresentation:tilde}
        \begin{pmatrix}
            \zzdiag \\
            \zzoff
        \end{pmatrix}
        =
        ( \mathcal{T} \circ \tilde{\mathcal{B}} \circ \mathcal{T}^{-1} )
        \begin{pmatrix}
            \wwtildediag \\
            \wwtildeoff
        \end{pmatrix}
        =
        \begin{pmatrix}
            \tilde{\AAA}_{\text{1,1}} & \tilde{\AAA}_{\text{1,2}} \\
            \tilde{\AAA}_{\text{1,2}}^\top & \tilde{\AAA}_{\text{2,2}}
        \end{pmatrix}
        \begin{pmatrix}
            \wwtildediag \\
            \wwtildeoff
        \end{pmatrix}.
    \end{equation}
    Again,
    in the last equation,
    we have used that
    the operator $\mathcal{T} \circ \tilde{\mathcal{B}} \circ \mathcal{T}^{-1}$
    can be represented as a symmetric block matrix
    since $\tilde{\mathcal{B}}$ is linear and self-adjoint.
    Now we note that since
    $\ZZ = \mathcal{C} (\ZZ)$
    we have that $ \PP_{\UU,\bot} \ZZ \PP_{\VV} = \f{0} $
    and $ \PP_{\UU} \ZZ \PP_{\VV,\bot} = \f{0} $.
    Thus, by the definition of the operator $\mathcal{T}$, 
    it holds that $\zzoff = \f{0}$.

    Now recall the assumption
    that, 
    for any matrix $\hat{\WW} \in \Rdd$,
    we have that 
    $ \mathcal{C} 
    ( \tilde{\mathcal{B}} 
    ( \mathcal{C}
    (\hat{\WW}) ))
    = \mathcal{C} ( \mathcal{B} ( \mathcal{C} (\hat{\WW}) ) ) $.
    Now note that,
    since $\mathcal{T}$ is an isometric isomorphism,
    this implies that 
    $$\mathcal{T} \circ
    ( \mathcal{C} \circ \tilde{\mathcal{B}} \circ \mathcal{C} )
    \circ
    \mathcal{T}^{-1}
    =
    \mathcal{T} \circ
    ( \mathcal{C} \circ \mathcal{B} \circ \mathcal{C} )
    \circ
    \mathcal{T}^{-1}.$$
    We compute that
    \begin{align*}
        \mathcal{T}\circ
        (
        \mathcal{C}
        \circ
        \tilde{\mathcal{B}}
        \circ
        \mathcal{C})
        \circ
        \mathcal{T}^{-1}
    =&
    \left( 
        \mathcal{T} \circ \mathcal{C} \circ \mathcal{T}^{-1}
    \right)
    \circ
    \left(
        \mathcal{T} \circ \tilde{\mathcal{B}} \circ \mathcal{T}^{-1}
    \right)
    \circ
    \left(
        \mathcal{T} \circ \mathcal{C} \circ \mathcal{T}^{-1}
    \right)\\
    =&
    \begin{pmatrix}
        \f{I} & \f{0} \\
        \f{0} & \f{0}
    \end{pmatrix}
    \begin{pmatrix}
        \tilde{\AAA}_{\text{1,1}} & \tilde{\AAA}_{\text{1,2}} \\
        \tilde{\AAA}_{\text{1,2}}^\top & \tilde{\AAA}_{\text{2,2}}
    \end{pmatrix}
    \begin{pmatrix}
        \f{I} & \f{0} \\
        \f{0} & \f{0}
    \end{pmatrix}
    =
    \begin{pmatrix}
        \tilde{\AAA}_{\text{1,1}} & \f{0} \\
        \f{0} & \f{0}
    \end{pmatrix},
    \end{align*}
    where we have identified the operators with their matrix representations.
    In a similar way, we obtain that
    \begin{equation*}
        \mathcal{T}\circ
        (
        \mathcal{C}
        \circ     \mathcal{B}
        \circ
        \mathcal{C})
        \circ
        \mathcal{T}^{-1}
        =
        \begin{pmatrix}
            \AAA_{\text{1,1}} & \f{0} \\
            \f{0} & \f{0}
        \end{pmatrix}.
    \end{equation*}
    Thus, the assumption
    $ \mathcal{C}
    ( \tilde{\mathcal{B}} 
    ( \mathcal{C} ( \hat{\WW} )  ) )
    =
    \mathcal{C} ( \mathcal{B} ( \mathcal{C} ( \hat{\WW} ) ) ) $
    is equivalent to
    $ \tilde{\AAA}_{1,1} = \AAA_{1,1} $
    in the vectorized representation.
    Next, assumption \eqref{assumption:offdiagonalzero}
    is equivalent to $ \tilde{\AAA}_{2,1} = \tilde{\AAA}_{1,2}^\top = \f{0} $
    in the vectorized representation.
    Note that
    $\tilde{\AAA}_{1,2} = \f{0}$
    also implies that
    $\wwtildeoff = \f{0}$
    due to equation \eqref{eq:vectorizedrepresentation:tilde}
    and due to the fact that $\zzoff = \f{0}$, together with the invertibility of $\tilde{\AAA}_{2,2}$ (which is a principal block of the positive definite matrix representing $\tilde{\mathcal{B}}$).

    Inserting these identities into \eqref{eq:vectorizedrepresentation}
    and \eqref{eq:vectorizedrepresentation:tilde},
    we obtain that
    \begin{equation*}
       \begin{pmatrix}
            \zzdiag \\
            \f{0}
        \end{pmatrix}
        =
        \begin{pmatrix}
            \AAA_{\text{1,1}} & \AAA_{\text{1,2}} \\
            \AAA_{\text{1,2}}^\top & \AAA_{\text{2,2}}
        \end{pmatrix}
        \begin{pmatrix}
            \wwdiag \\
            \wwoff
        \end{pmatrix},
        \quad
         \begin{pmatrix}
            \zzdiag \\
            \f{0}
        \end{pmatrix}
        =
        \begin{pmatrix}
            \AAA_{\text{1,1}} & \f{0} \\
            \f{0} & \tilde{\AAA}_{\text{2,2}}
        \end{pmatrix}
        \begin{pmatrix}
            \wwtildediag \\
            \wwtildeoff
        \end{pmatrix}. 
    \end{equation*}
    From the first equation, we obtain that
    \begin{align*}
        \zzdiag = \AAA_{\text{1,1}} \wwdiag + \AAA_{\text{1,2}} \wwoff,\\
        \f{0} = \AAA_{\text{1,2}}^\top \wwdiag + \AAA_{\text{2,2}} \wwoff.
    \end{align*}
    Since $\mathcal{B}$ is positive definite, the block matrix
    in equation \eqref{eq:vectorizedrepresentation} 
    is positive definite as well.
    Then the Schur complement 
    $ \AAA_{1,1} - \AAA_{1,2} \AAA_{2,2}^{-1} \AAA_{1,2}^\top $
    is positive definite as well
    and thus invertible.
    Therefore, we obtain by a direct calculation that
    \begin{equation*}
        \wwdiag 
        =  \left( \AAA_{1,1} - \AAA_{1,2} \AAA_{2,2}^{-1} \AAA_{1,2}^\top \right)^{-1} \zzdiag.
    \end{equation*}
    Moreover, from equation \eqref{eq:vectorizedrepresentation:tilde} we obtain 
    due to $\zzoff = \f{0}$ and $\tilde{\AAA}_{1,2} = \f{0}$
    that
    \begin{align*}
      \wwtildediag =\AAA_{1,1}^{-1} \zzdiag,
    \end{align*}
    since $\AAA_{1,1}$ is invertible as a principal submatrix of a positive definite matrix.
    Now we note that
    \begin{align}
        \innerproduct{ 
            \ZZ, \WW - \widetilde{\WW} 
        }
        \overeq{(a)}&
        \langle{
            \mathcal{T} (\ZZ),
            \mathcal{T} (\WW) - \mathcal{T} (\widetilde{\WW})
        }\rangle_{\ell_2} \nonumber \\
        \overeq{(b)}&
        \langle
            \begin{pmatrix}
                \zzdiag \\
                \f{0}
            \end{pmatrix},
            \begin{pmatrix}
                \wwdiag - \wwtildediag \\
                \wwoff - \wwtildeoff
            \end{pmatrix}
        \rangle_{\ell_2} \nonumber \\
        =&
        \langle
            \zzdiag,
            \wwdiag - \wwtildediag
        \rangle_{\ell_2} \nonumber \\
        \overeq{(c)}&
        \langle
            \zzdiag,
            \left[ \left( \AAA_{1,1} - \AAA_{1,2} \AAA_{2,2}^{-1} \AAA_{1,2}^\top \right)^{-1}
            -
            \AAA_{1,1}^{-1} \right] \zzdiag
        \rangle_{\ell_2}.
        \label{eq:innerproductdifference}
    \end{align}
    Here, in step (a) we have used that $\mathcal{T}$ is an isometric operator,
    and in step (b) we have used the vectorized representations of $\ZZ, \WW$, 
    and $\widetilde{\WW}$.
    Step (c) follows from the formulas for $\wwdiag$ and $\wwtildediag$ derived above.

    Now,
    since $\mathcal{B}$ is positive definite,
    it follows that the block matrix $\AAA_{2,2}$ is positive definite as well.
    Thus, $\AAA_{1,2} \AAA_{2,2}^{-1} \AAA_{1,2}^\top$
    is positive semidefinite.
    It follows that 
    \begin{equation*}
        \AAA_{1,1} - \AAA_{1,2} \AAA_{2,2}^{-1} \AAA_{1,2}^\top
        \preceq
        \AAA_{1,1},
    \end{equation*}
    where $\preceq$ denotes the Loewner partial order on symmetric matrices,
    i.e., $\f{A} \preceq \f{B}$ if and only if $\f{B} - \f{A}$ is positive semidefinite.
    By the operator monotonicity of the inverse function on the cone of positive definite matrices
    we obtain that
    \begin{equation*}
        \left( \AAA_{1,1} - \AAA_{1,2} \AAA_{2,2}^{-1} \AAA_{1,2}^\top \right)^{-1}
        \succeq
        \AAA_{1,1}^{-1}.
    \end{equation*}
    It follows that 
    \begin{equation*}
       \left( \AAA_{1,1} - \AAA_{1,2} \AAA_{2,2}^{-1} \AAA_{1,2}^\top \right)^{-1}
        - \AAA_{1,1}^{-1} \succeq \f{0}, 
    \end{equation*}
    which finally implies that
    \begin{equation*}
       \langle
            \zzdiag,
            \left[ \left( \AAA_{1,1} - \AAA_{1,2} \AAA_{2,2}^{-1} \AAA_{1,2}^\top \right)^{-1}
            -
            \AAA_{1,1}^{-1} \right] \zzdiag
        \rangle_{\ell_2} \ge 0.
    \end{equation*}
    Combined with \eqref{eq:innerproductdifference} this implies the claim.
\end{proof}

As a next step,
we specialize Lemma \ref{lemma:majorization_property}
to the case of Sylvester equations.
\begin{lemma}[Abstract Majorization Lemma for Sylvester Equations] \label{lemma:majorization_property_sylvester} \mbox{}\\
	Let $\ZZ \in \Rdd$ be arbitrary but fixed.
    Let $\LL_{\UU} \in \R^{d_1 \times d_1}$ and $\LL_{\VV} \in \R^{d_2 \times d_2}$ be positive definite matrices.
    Let $\PP_{\UU} \in \R^{d_1 \times d_1}, \PP_{\VV} \in \R^{d_2 \times d_2}$ 
    be orthogonal projection matrices.
    Denote by $\PP_{\UU,\bot} := \f{I} - \PP_{\UU}$ 
    and 
    $\PP_{\VV,\bot} := \f{I} - \PP_{\VV}$
    the projections onto the orthogonal complement.
    Assume that
    $ \PP_{\UU} \ZZ \PVV{\bot}  = \f{0}$
    and $ \PP_{\UU,\bot} \ZZ \PP_{\VV}  = \f{0} $.
    Define the operator
    $\mathcal{C}: \Rdd \to \Rdd$
    by 
    $\mathcal{C} (\AAA) 
    := \PP_{\UU} \AAA \PP_{\VV} + \PP_{\UU,\bot} \AAA \PP_{\VV,\bot}$
    for any matrix $\AAA \in \Rdd$.
    Then the following statements hold true:
    \begin{enumerate}
        \item The equation
        \begin{equation}\label{eq:sylvester:projection}
             \LL_{\UU}  \WW 
            + \WW  \LL_{\VV} 
            =
            2 \ZZ
        \end{equation}
        has a unique solution $\WW \in \Rdd$.
        \item
        Set 
        $\mathcal{C}_{\UU}: \R^{d_1 \times d_1} \to \R^{d_1 \times d_1}$ with $\mathcal{C}_{\UU}(\AAA_1) 
        := \PP_{\UU} \AAA_1 \PP_{\UU} + \PP_{\UU,\bot} \AAA_1 \PP_{\UU,\bot}$
        and $\mathcal{C}_{\VV}: \R^{d_2 \times d_2} \to \R^{d_2 \times d_2}$ with
        $\mathcal{C}_{\VV} (\AAA_2)
        := \PP_{\VV} \AAA_2 \PP_{\VV} + \PP_{\VV,\bot} \AAA_2 \PP_{\VV,\bot}$
        for any matrices $\AAA_1 \in \R^{d_1 \times d_1}$ and $\AAA_2 \in \R^{d_2 \times d_2}$, respectively.
        Then, the equation
        \begin{equation}\label{eq:sylvester:projection:orthogonal}
             \mathcal{C}_{\UU} (\LL_{\UU})  \widetilde{\WW} 
            + \widetilde{\WW}  \mathcal{C}_{\VV} (\LL_{\VV}) 
            =
            2 \ZZ
        \end{equation}
        has a unique solution $ \widetilde{\WW} \in \Rdd$.
        \item It holds that
        \begin{equation*}
            \innerproduct{\ZZ ,\WW - \widetilde{\WW} }
            \ge
            0.
        \end{equation*}
    \end{enumerate}
\end{lemma}

\begin{proof}
    With the same argument as in the proof of Lemma \ref{lemma:sylvester_equation_weight_operator},
    it follows from the theory of Sylvester equations
    that
    Equation \eqref{eq:sylvester:projection} 
    has a unique solution $\WW \in \Rdd$
    since $\LL_{\UU}$ and $\LL_{\VV}$ are positive definite,
    \citep[see, e.g.,][Theorem VII.2.3]{Bhatia1997MatrixAnalysis}.
    This proves the first statement. 

    Now, note that $ \mathcal{C}_{\UU} (\LL_{\UU}) $
    and $ \mathcal{C}_{\VV} (\LL_{\VV}) $ are positive definite as well
    since the pinching operations $\mathcal{C}_{\UU}$ and $\mathcal{C}_{\VV}$ 
    preserve positive definiteness
    as one can easily verify.
    This fact implies that
    equation \eqref{eq:sylvester:projection:orthogonal}
    has a unique solution $\widetilde{\WW} \in \Rdd$ as well.
    This proves the second statement.

    It remains to prove the third statement.
    Define the operator $\mathcal{B}: \Rdd \to \Rdd$
    by
    \begin{equation*}
        \mathcal{B} (\AAA) 
        := \frac{1}{2} \left( \LL_{\UU} \AAA + \AAA \LL_{\VV} \right)
    \end{equation*}
    and define the operator $\tilde{\mathcal{B}}: \Rdd \to \Rdd$
    by
    \begin{equation*}
        \tilde{\mathcal{B}} (\AAA)
        := \frac{1}{2} \left( 
            \mathcal{C}_{\UU} (\LL_{\UU}) \AAA 
            + \AAA \mathcal{C}_{\VV} (\LL_{\VV})
        \right).
    \end{equation*}
    Our goal is to apply the previous Lemma \ref{lemma:majorization_property}.
    For this, we need to verify that the assumptions of Lemma \ref{lemma:majorization_property}
    are satisfied.
    The first assumption is satisfied
    since we have assumed that
    $ \PP_{\UU} \ZZ \PVV{\bot}  = \f{0}$
    and $ \PP_{\UU,\bot} \ZZ \PP_{\VV}  = \f{0} $,
    which is equivalent to $\ZZ = \mathcal{C} (\ZZ)$.
    Next, note that for all matrices $\ZZ_1, \ZZ_2 \in \Rdd$
    it holds that
    \begin{align*}
        \innerproduct{ \mathcal{B} (\ZZ_1), \ZZ_2 }
        =&
        \frac{1}{2} \innerproduct{ \LL_{\UU} \ZZ_1, \ZZ_2 }
        +
        \frac{1}{2} \innerproduct{ \ZZ_1 \LL_{\VV}, \ZZ_2  }\\
        =&
        \frac{1}{2} \innerproduct{ \ZZ_1, \LL_{\UU} \ZZ_2 }
        +
        \frac{1}{2} \innerproduct{ \ZZ_1, \ZZ_2 \LL_{\VV} }\\
        =&
        \innerproduct{ \ZZ_1, \mathcal{B} (\ZZ_2) }.
    \end{align*}
    Moreover, we note that
    \begin{align*}
        \innerproduct{ \mathcal{B} (\ZZ_1), \ZZ_1 }
        =&
        \frac{1}{2} \innerproduct{ \LL_{\UU} \ZZ_1, \ZZ_1 }
        +
        \frac{1}{2} \innerproduct{ \ZZ_1 \LL_{\VV}, \ZZ_1  }\\
        =&
        \frac{1}{2} \innerproduct{ \ZZ_1, \LL_{\UU} \ZZ_1 }
        +
        \frac{1}{2} \innerproduct{ \ZZ_1, \ZZ_1 \LL_{\VV} }\\
        =&
        \frac{1}{2} \trace ( \ZZ_1^{\top} \LL_{\UU} \ZZ_1 )
        +
        \frac{1}{2} \trace (  \ZZ_1 \LL_{\VV} \ZZ_1^{\top})
    \end{align*}
    for any $\ZZ_1 \in \Rdd \backslash \{ \f{0} \}$.
    Since $\LL_{\UU}$ and $\LL_{\VV}$ are positive definite matrices,
    we obtain that
    both $\ZZ_1^{\top} \LL_{\UU} \ZZ_1$
    and
    $\ZZ_1 \LL_{\VV} \ZZ_1^{\top}$
    are positive semidefinite matrices
    which are nonzero.
    Thus,
    both terms in the last equation are positive
    which implies that
    \begin{equation*}
        \innerproduct{ \mathcal{B} (\ZZ_1), \ZZ_1 }
        >
        0
    \end{equation*}
    for all $\ZZ_1 \in \Rdd \backslash \{ \f{0} \}$.
    Thus, we have shown that 
    the operator $\mathcal{B}$ is self-adjoint and positive definite.
    In an analogous way,
    one can show that the operator $\tilde{\mathcal{B}}$ is self-adjoint and positive definite as well
    since we have that $ \mathcal{C}_{\UU} (\LL_{\UU}) $
    and $ \mathcal{C}_{\VV} (\LL_{\VV}) $ are positive definite.
    This shows that the second and third assumptions of Lemma \ref{lemma:majorization_property} 
    are satisfied.

    Now note that
    \begin{equation*}
        \PP_{\UU} \mathcal{C}_{\UU} (\LL_{\UU}) 
        \mathcal{C} (\hat{\WW}) \PP_{\VV}
        \overeq{(a)} 
        \PP_{\UU} \LL_{\UU} \PP_{\UU}  \mathcal{C} (\hat{\WW}) \PP_{\VV}
        \overeq{(b)}
        \PP_{\UU} \LL_{\UU} \mathcal{C} (\hat{\WW}) \PP_{\VV},
    \end{equation*}
    where in step (a) we have used that
    $\PP_{\UU} \mathcal{C}_{\UU} (\LL_{\UU}) = \PP_{\UU} \LL_{\UU} \PP_{\UU}$
    and in step (b) we have used that
    $ \mathcal{C} (\hat{\WW}) \PP_{\VV} = \PP_{\UU} \mathcal{C} (\hat{\WW}) \PP_{\VV} $.
    In a similar way, one can verify the following three identities:
    \begin{align*}
        \PP_{\UU}  
        \mathcal{C} (\hat{\WW}) \mathcal{C}_{\VV} (\LL_{\VV})\PP_{\VV}
        =&
        \PP_{\UU}  \mathcal{C} (\hat{\WW}) \LL_{\VV} \PP_{\VV},\\
        \PP_{\UU,\bot} \mathcal{C}_{\UU} (\LL_{\UU}) \mathcal{C} (\hat{\WW})
        \PP_{\VV,\bot}
        =&
        \PP_{\UU,\bot} \LL_{\UU} \mathcal{C} (\hat{\WW}) \PP_{\VV,\bot},\\
        \PP_{\UU,\bot}  \mathcal{C} (\hat{\WW})\mathcal{C}_{\VV} (\LL_{\VV})
        \PP_{\VV,\bot}
        =&
        \PP_{\UU,\bot}  \mathcal{C} (\hat{\WW}) \LL_{\VV} \PP_{\VV,\bot}.
    \end{align*}
    Using these equations we can show that for all matrices $\hat{\WW} \in \Rdd$, 
    it holds that
    \begin{align*}
        &\mathcal{C} ( \tilde{\mathcal{B}} ( \mathcal{C} (\hat{\WW}) ) )\\
        =
        &\frac{1}{2}
        \mathcal{C}
        \left(
            \mathcal{C}_{\UU} (\LL_{\UU} ) \mathcal{C} (\hat{\WW})
            +
            \mathcal{C} (\hat{\WW}) \mathcal{C}_{\VV} (\LL_{\VV})
        \right)\\
        =&
        \frac{1}{2}
        \PP_{\UU}
        \left(
            \mathcal{C}_{\UU} (\LL_{\UU}) \mathcal{C} (\hat{\WW}) 
            +
            \mathcal{C} (\hat{\WW}) 
            \mathcal{C}_{\VV} (\LL_{\VV})  
        \right) \PP_{\VV}\\
        &+
        \frac{1}{2}
        \PP_{\UU, \bot}
        \left(
            \mathcal{C}_{\UU} (\LL_{\UU}) \mathcal{C} (\hat{\WW}) 
            +
            \mathcal{C} (\hat{\WW}) 
            \mathcal{C}_{\VV} (\LL_{\VV})  
        \right) \PP_{\VV,\bot}\\
        =&
        \frac{1}{2}
        \PP_{\UU}
        \left(
            \LL_{\UU}
            \mathcal{C} (\hat{\WW})
            +
            \mathcal{C} (\hat{\WW})
            \LL_{\VV}
        \right)
        \PP_{\VV}
        +\frac{1}{2}
        \PP_{\UU, \bot}
        \left(
            \LL_{\UU}
            \mathcal{C} (\hat{\WW})
            +
            \mathcal{C} (\hat{\WW})
            \LL_{\VV}   
        \right) \PP_{\VV,\bot}\\
        =&
        \frac{1}{2}
        \mathcal{C} \left(
            \LL_{\UU}
            \mathcal{C} (\hat{\WW})
            +
            \mathcal{C} (\hat{\WW})
            \LL_{\VV}
        \right)
        =
        \mathcal{C} ( \mathcal{B} ( \mathcal{C} (\hat{\WW}) ) ).
    \end{align*}

    This shows that the fourth assumption of Lemma \ref{lemma:majorization_property} is satisfied.

    Finally, we need to verify the fifth assumption of Lemma \ref{lemma:majorization_property}.
    We note that,
    for any matrix $\hat{\WW} \in \Rdd$,
    it holds that
    \begin{align*}
        \PP_{\UU} \tilde{\mathcal{B}} 
        ( \mathcal{C} ( \hat{\WW} )  ) \PP_{\VV,\bot} 
        =
        \frac{1}{2}
        \PP_{\UU} \left(
            \mathcal{C}_{\UU} (\LL_{\UU}) \mathcal{C} (\hat{\WW})
            +
            \mathcal{C} (\hat{\WW}) \mathcal{C}_{\VV} (\LL_{\VV})
        \right) \PP_{\VV,\bot}.
    \end{align*}
    Now note that $\PP_{\UU} \mathcal{C}_{\UU} (\LL_{\UU}) = \PP_{\UU} \LL_{\UU} \PP_{\UU}$
    and $\mathcal{C} ( \hat{\WW}) \PP_{\VV,\bot} = \PP_{\UU,\bot} \hat{\WW} \PP_{\VV,\bot}$.
    Consequently, we obtain that
    $ \PP_{\UU} \mathcal{C}_{\UU} (\LL_{\UU}) \mathcal{C} (\hat{\WW}) \PP_{\VV,\bot} = \f{0} $.
    Similarly, we have
    $ \PP_{\UU} \mathcal{C} (\hat{\WW}) \mathcal{C}_{\VV} (\LL_{\VV}) \PP_{\VV,\bot} = \f{0} $.
    Thus, we obtain that $\PP_{\UU} \tilde{\mathcal{B}} ( \mathcal{C} ( \hat{\WW} )  ) \PP_{\VV,\bot}= \f{0}$.
    In a similar way, one can also show that $\PP_{\UU,\bot} \tilde{\mathcal{B}} 
        ( \mathcal{C} ( \hat{\WW} )  ) \PP_{\VV}
        = \f{0}.$
    It follows that 
    \begin{align*}
        \PP_{\UU} \tilde{\mathcal{B}} 
        ( \mathcal{C} ( \hat{\WW} )  ) \PP_{\VV,\bot} 
        +
        \PP_{\UU,\bot} \tilde{\mathcal{B}} 
        ( \mathcal{C} ( \hat{\WW} )  ) \PP_{\VV}
        = \f{0}.
    \end{align*}
    Thus, the fifth assumption of Lemma \ref{lemma:majorization_property} is satisfied as well
    and we are in a position to apply Lemma \ref{lemma:majorization_property}.

    Since, by definition,
    $\WW$ and $\widetilde{\WW}$
    satisfy
    $\mathcal{B} (\WW) = \ZZ$
    and
    $\tilde{\mathcal{B}} (\widetilde{\WW}) = \ZZ$,
    it
    follows from Lemma \ref{lemma:majorization_property} that
    \begin{align*}
        \innerproduct{\ZZ, \WW -  \widetilde{\WW} }
        =
        \innerproduct{
            \ZZ, 
            \mathcal{B}^{-1} (\ZZ) - \tilde{\mathcal{B}}^{-1}  (\ZZ) 
            }
        \ge 
        0.
    \end{align*}
    This completes the proof.
\end{proof}

\paragraph{Iterative Pinching Argument.} \label{sec:appendix:lb_weighted_ip:iterative}

Let $\LL_{\UU} \in \R^{d_1 \times d_1}, \LL_{\VV} \in \R^{d_2 \times d_2}$ be positive definite matrices.
Let $\uu_1, \ldots, \uu_{d_1} \in \R^{d_1}$ be an orthonormal basis of $\R^{d_1}$
and let $\vv_1, \ldots, \vv_{d_2} \in \R^{d_2}$ be an orthonormal basis of $\R^{d_2}$.
The goal is now to apply the previous lemma iteratively
to obtain a majorization result
for the case that both $\LL_{\UU}$ and $\LL_{\VV}$ are diagonal
in the bases
$ \{ \uu_1, \ldots, \uu_{d_1} \} $
and
$ \{ \vv_1, \ldots, \vv_{d_2} \} $,
respectively.
For this, we define the orthogonal projection matrices
    \begin{align*}
        \PUU{i:j} := \sum_{k=i}^{j} \uu_k \uu_k^{\top}
        \quad
        \text{for }
        1 \le i \le j \leq d_1,\\
        \PVV{i:j} := \sum_{k=i}^{j} \vv_k \vv_k^{\top}
        \quad
        \text{for }
        1 \le i \le j \leq d_2.
    \end{align*}
    Next, we define the operators
    \begin{align*}
        \mathcal{C}_{\UU,i} (\AAA_1) &:= \PUU{1:i} \AAA_1 \PUU{1:i} + \PUU{i+1:d_1} \AAA_1 \PUU{i+1:d_1}
        \quad
        \text{for }
        1 \le i < d_1
    \end{align*}
    and 
    \begin{align*}
        \mathcal{C}_{\VV,i} (\AAA_2)& := \PVV{1:i} \AAA_2 \PVV{1:i} + \PVV{i+1:d_2} \AAA_2 \PVV{i+1:d_2}
        \quad
        \text{for }
        1 \le i < d_2,
    \end{align*}
    for any $\AAA_1 \in \R^{d_1 \times d_1}$ and 
    $\AAA_2 \in \R^{d_2 \times d_2}$, respectively.
    Then, we define recursively
    the matrices
    $ \LL_{\UU,i} := \mathcal{C}_{\UU,i} ( \LL_{\UU, i-1} ) $
    for $i=1,2,\ldots, d_1-1$
    and
    $ \LL_{\VV,i} := \mathcal{C}_{\VV,i} ( \LL_{\VV, i-1} ) $
    for $i=1,2,\ldots, d_2-1$,
    where we set $\LL_{\UU,0} := \LL_{\UU}$
    and $\LL_{\VV,0} := \LL_{\VV}$.

   The following lemma shows that
   after $d_1-1$ and $d_2-1$ iterations,
   the matrices $\LL_{\UU,d_1-1}$ and $\LL_{\VV,d_2-1}$
   are diagonal in the bases
   $ \{ \uu_1, \ldots, \uu_{d_1} \} $
   and  $ \{ \vv_1, \ldots, \vv_{d_2} \} $,
    respectively.
\begin{lemma}\label{lemma:iterative_pinching}
    It holds that
    \begin{equation*}
        \LL_{\UU,d_1-1} = \sum_{k=1}^{d_1} \innerproduct{\LL_{\UU},\uu_k \uu_k^\top} \uu_k \uu_k^\top,
        \quad
        \LL_{\VV,d_2-1} = \sum_{k=1}^{d_2} \innerproduct{\LL_{\VV},\vv_k \vv_k^\top} \vv_k \vv_k^\top.
    \end{equation*}
    In other words,
    the matrices $\LL_{\UU,d_1-1}$ and $\LL_{\VV,d_2-1}$ are diagonal in the bases
    $ \{ \uu_1, \ldots, \uu_{d_1} \} $
    and
    $ \{ \vv_1, \ldots, \vv_{d_2} \} $,
    respectively.
\end{lemma}

\begin{proof}
    We show by induction that for $i=0, 1,2,\ldots, d_1-1$, it holds that
    \begin{equation*}
        \LL_{\UU,i}
        =
        \sum_{k=1}^{i} \innerproduct{\LL_{\UU},\uu_k \uu_k^\top} \uu_k \uu_k^\top
        +
        \PP_{\UU,i+1:d_1} \LL_{\UU} \PP_{\UU,i+1:d_1}.
    \end{equation*}
    The base case $i=0$ follows directly from the definition of $\LL_{\UU,0}$.
    Now assume that the statement holds for some $i \in \{0, 1,2,\ldots, d_1-2\}$.
    We compute that
    \begin{align*}
        \LL_{\UU,i+1}
        =&
        \mathcal{C}_{\UU,i+1} ( \LL_{\UU,i} )\\
        =&
        \PP_{\UU,1:i+1} \LL_{\UU,i} \PP_{\UU,1:i+1}
        +
        \PP_{\UU,i+2:d_1} \LL_{\UU,i} \PP_{\UU,i+2:d_1}\\
        =&
        \PP_{\UU,1:i+1}
        \left(
            \sum_{k=1}^{i} \innerproduct{\LL_{\UU},\uu_k \uu_k^\top} \uu_k \uu_k^\top
            +
            \PP_{\UU,i+1:d_1} \LL_{\UU} \PP_{\UU,i+1:d_1}
        \right)
        \PP_{\UU,1:i+1}\\
        &+
        \PP_{\UU,i+2:d_1}
        \left(
            \sum_{k=1}^{i} \innerproduct{\LL_{\UU},\uu_k \uu_k^\top} \uu_k \uu_k^\top
            +
            \PP_{\UU,i+1:d_1} \LL_{\UU} \PP_{\UU,i+1:d_1}
        \right)
        \PP_{\UU,i+2:d_1}\\
        \overeq{(a)}&
        \PP_{\UU,1:i+1}
        \left(
            \sum_{k=1}^{i} \innerproduct{\LL_{\UU},\uu_k \uu_k^\top} \uu_k \uu_k^\top
        \right)
        \PP_{\UU,1:i+1}\\
        &+
        \PP_{\UU,1:i+1}
        \PP_{\UU,i+1:d_1} \LL_{\UU} \PP_{\UU,i+1:d_1}
        \PP_{\UU,1:i+1}\\
        &+
        \PP_{\UU,i+2:d_1}
        \PP_{\UU,i+1:d_1}
        \LL_{\UU}
        \PP_{\UU,i+1:d_1}
        \PP_{\UU,i+2:d_1}\\
        \overeq{(b)}&
        \sum_{k=1}^{i} \innerproduct{\LL_{\UU},\uu_k \uu_k^\top} \uu_k \uu_k^\top
        +
        \innerproduct{\LL_{\UU},\uu_{i+1} \uu_{i+1}^\top} \uu_{i+1} \uu_{i+1}^\top
        +
        \PP_{\UU,i+2:d_1} \LL_{\UU} \PP_{\UU,i+2:d_1}\\
        =&
        \sum_{k=1}^{i+1} \innerproduct{\LL_{\UU},\uu_k \uu_k^\top} \uu_k \uu_k^\top
        +
        \PP_{\UU,i+2:d_1} \LL_{\UU} \PP_{\UU,i+2:d_1}.
    \end{align*}
    In equation (a), we used that
    $ \PP_{\UU,i+2:d_1} \innerproduct{\LL_{\UU},\uu_k \uu_k^\top} \uu_k \uu_k^\top \PP_{\UU,i+2:d_1}
    = \f{0} $ for all $k=1,\ldots, i$.
    In equation (b), we then used that
    $ \PP_{\UU,1:i+1}\uu_k \uu_k^\top  \PP_{\UU,1:i+1}
    = \uu_k \uu_k^\top$
    for $k=1,2,\ldots, i$,
    that
    $ \PP_{\UU,1:i+1} \PP_{\UU,i+1:d_1} = \uu_{i+1} \uu_{i+1}^\top $,
    and
    that
    $ \PP_{\UU,i+2:d_1} \PP_{\UU,i+1:d_1} = \PP_{\UU,i+2:d_1} $.
    This shows the induction step
    and we have shown that the displayed equation holds for all $i=0,1,\ldots, d_1-1$.

    Thus, we obtain for $i=d_1-1$ that
    \begin{equation*}
        \LL_{\UU,d_1-1}
        =
        \sum_{k=1}^{d_1-1} \innerproduct{\LL_{\UU},\uu_k \uu_k^\top} \uu_k \uu_k^\top
        +
        \PP_{\UU,d_1:d_1} \LL_{\UU} \PP_{\UU,d_1:d_1}
        =
        \sum_{k=1}^{d_1} \innerproduct{\LL_{\UU},\uu_k \uu_k^\top} \uu_k \uu_k^\top.
    \end{equation*}
    This proves the claim for $\LL_{\UU,d_1-1}$.
    The proof for $\LL_{\VV,d_2-1}$ is analogous.
\end{proof}
The next lemma shows that after
arriving at a Sylvester equation
where both $\LL_{\UU}$ and $\LL_{\VV}$ are diagonal
in the bases
$ \{ \uu_1, \ldots, \uu_{d_1} \} $
and
$ \{ \vv_1, \ldots, \vv_{d_2} \} $,
respectively,
we can give an explicit expression for the solution of the Sylvester equation.
\begin{lemma}[Explicit Harmonic-Mean Sylvester Equation Solution] \label{lemma:sylvester_solution_explicit}\mbox{}\\
    Let $\ZZ \in \Rdd$ be an arbitrary matrix
    with SVD given by 
    $ \ZZ = \sum_{i=1}^d \sigma_i (\ZZ) \uu_i \vv_i^{\top} $,
    where $\sigma_1(\ZZ) \ge \sigma_2(\ZZ) \ge \ldots \ge \sigma_d(\ZZ) \ge 0$
    are the singular values of $\ZZ$.
    Here,
    $\{ \uu_1, \ldots, \uu_{d_1} \} \subset \R^{d_1}$
    and
    $\{ \vv_1, \ldots, \vv_{d_2} \} \subset \R^{d_2}$
    are orthonormal bases of 
    $\R^{d_1}$ and $\R^{d_2}$,
    consisting
    of left and right singular vectors of $\ZZ$, respectively.
    Let $ \widetilde{\LL}_{\UU} \in \R^{d_1 \times d_1}, \widetilde{\LL}_{\VV} \in \R^{d_2 \times d_2}$
    be positive definite matrices
    which are diagonal in the bases $\{ \uu_1, \ldots, \uu_{d_1} \}$
    and $\{ \vv_1, \ldots, \vv_{d_2} \}$, respectively.
    Then, the unique solution $\WW \in \Rdd$
    of the Sylvester equation
    \begin{equation*}
        \widetilde{\LL}_{\UU} \WW + \WW \widetilde{\LL}_{\VV} = 2 \ZZ
    \end{equation*}
    is given by
    \begin{equation*}
        \WW
        =
        \sum_{k=1}^d
        \frac{2 \sigma_k (\ZZ)}{
            \innerproduct{ \widetilde{ \LL}_{\UU}, \uu_k \uu_k^\top}
            +
            \innerproduct{\widetilde{\LL}_{\VV}, \vv_k \vv_k^\top}
        }
        \uu_k \vv_k^{\top}.
    \end{equation*}
\end{lemma}

\begin{proof}
Uniqueness follows again from \citet[Theorem VII.2.3]{Bhatia1997MatrixAnalysis}
since $ \widetilde{ \LL}_{\UU}$ and $\widetilde{\LL}_{\VV}$ are positive definite.
To show that the given expression for $\WW$ is indeed the solution of the Sylvester equation,
we compute
\begin{align*}
    & \widetilde{\LL}_{\UU} \WW + \WW \widetilde{\LL}_{\VV}\\
    =&
    \left(
        \sum_{k=1}^{d_1}
        \innerproduct{ \widetilde{ \LL}_{\UU}, \uu_k \uu_k^\top} \uu_k \uu_k^\top
    \right)
    \left(
        \sum_{j=1}^d
        \frac{2 \sigma_j (\ZZ)}{
            \innerproduct{ \widetilde{\LL}_{\UU}, \uu_j \uu_j^\top}
            +
            \innerproduct{ \widetilde{\LL}_{\VV}, \vv_j \vv_j^\top}
        }
        \uu_j \vv_j^{\top}
    \right)\\
    &+
    \left(
        \sum_{j=1}^d
        \frac{2 \sigma_j (\ZZ)}{
            \innerproduct{ \widetilde{\LL}_{\UU}, \uu_j \uu_j^\top}
            +
            \innerproduct{\widetilde{\LL}_{\VV}, \vv_j \vv_j^\top}
        }
        \uu_j \vv_j^{\top}  
     \right)
    \left(
        \sum_{k=1}^{d_2}
        \innerproduct{ \widetilde{\LL}_{\VV}, \vv_k \vv_k^\top} \vv_k \vv_k^\top
    \right)\\
    =&
    \sum_{k=1}^d
    \frac{2 \sigma_k (\ZZ)\innerproduct{\widetilde{\LL}_{\UU}, \uu_k \uu_k^\top}}{
        \innerproduct{\widetilde{\LL}_{\UU}, \uu_k \uu_k^\top}
        +
        \innerproduct{\widetilde{\LL}_{\VV}, \vv_k \vv_k^\top}
    }
    \uu_k \vv_k^{\top}+
    \sum_{k=1}^d
    \frac{2 \sigma_k (\ZZ)\innerproduct{\widetilde{\LL}_{\VV}, \vv_k \vv_k^\top}}{
        \innerproduct{\widetilde{\LL}_{\UU}, \uu_k \uu_k^\top}
        +
        \innerproduct{\widetilde{\LL}_{\VV}, \vv_k \vv_k^\top}
    }
    \uu_k \vv_k^{\top}\\
    =&
    \sum_{k=1}^d
    2 \sigma_k (\ZZ)
    \uu_k \vv_k^{\top}
    =
    2 \ZZ.
\end{align*}
In the first equation we have used that the matrices $\widetilde{\LL}_{\UU}$ and $\widetilde{\LL}_{\VV}$
are diagonal in the bases $\{ \uu_1, \ldots, \uu_{d_1} \}$
and $\{ \vv_1, \ldots, \vv_{d_2} \}$, respectively.
We have shown the desired equation,
which completes the proof.
\end{proof}
Now we have all the ingredients in place to prove Lemma \ref{lemma:lower_bound_weighted_inner_product},
the lower bound on the weighted inner product
$\innerproduct{\ZZ, W_{\XX, \varepsilon} (\ZZ)}$.
\begin{proof}[Proof of Lemma \ref{lemma:lower_bound_weighted_inner_product}]
In the following,
we assume that $\ZZ \in \Rdd$ is arbitrary but fixed.
Then we have seen in Lemma \ref{lemma:sylvester_equation_weight_operator}
that the weight matrix 
$\WW_0 = W_{\XX, \varepsilon} (\ZZ)$
as defined in Lemma \ref{lemma:sylvester_equation_weight_operator}
is the unique solution of the Sylvester equation
\begin{equation*}
    \LL_{\UU} \WW_0 + \WW_0 \LL_{\VV} = 2 \ZZ.
\end{equation*}

Now we denote the SVD of $\ZZ$ by 
$\ZZ = \f{U}_{\ZZ} \diag(\mmuu) \f{V}^{\top}_{\ZZ}$,
where $\mmuu = (\sigma_1(\ZZ), \ldots, \sigma_d(\ZZ)) \in \R^d$ contains the singular values of $\ZZ$.
We denote by $\uu_1, \ldots, \uu_{d_1}$ the columns of $\f{U}_{\ZZ}$
and by $\vv_1, \ldots, \vv_{d_2}$ the columns of $\f{V}_{\ZZ}$.
We define the orthogonal projection matrices by
\begin{align*}
    \PUU{i:j} := \sum_{k=i}^{j} \uu_k \uu_k^{\top}
    \quad
    \text{for }
    1 \le i \le j \le d_1,\\
    \PVV{i:j} := \sum_{k=i}^{j} \vv_k \vv_k^{\top}
    \quad
    \text{for }
    1 \le i \le j \le d_2.
\end{align*}
Now recall that $d=\min  (d_1, d_2)$
and
$ D= \max (d_1, d_2) $.
Next,
we define the operators
\begin{align*}
    \mathcal{C}_{\UU,i} (\AAA_1) 
    &:= 
    \begin{cases}
    \PUU{1:i} \AAA_1 \PUU{1:i} + \PUU{i+1:d_1} \AAA_1 \PUU{i+1:d_1}
    \quad
    &\text{for }
    1 \le i < d_1,\\
    \AAA_1
    \quad
    &\text{for }
    d_1 \le i < D
    \end{cases}\\
    \mathcal{C}_{\VV,i} (\AAA_2)
    &:= 
    \begin{cases}
    \PVV{1:i} \AAA_2 \PVV{1:i} + \PVV{i+1:d_2} \AAA_2 \PVV{i+1:d_2}
    \quad
    &\text{for }
    1 \le i < d_2,\\
    \AAA_2
    \quad
    &\text{for }
    d_2 \le i < D
    \end{cases}
\end{align*}
for matrices $\AAA_1 \in \R^{d_1 \times d_1}$ and $\AAA_2 \in \R^{d_2 \times d_2}$.
Set $\LL_{\UU,0} := \LL_{\UU}$
and $\LL_{\VV,0} := \LL_{\VV}$.
We define recursively 
\begin{align*}
 \LL_{\UU,i} := \mathcal{C}_{\UU,i} ( \LL_{\UU, i-1} )
 \quad
 &\text{ for } i=1,2,\ldots, D-1,\\
 \LL_{\VV,i} := \mathcal{C}_{\VV,i} ( \LL_{\VV, i-1} ) 
\quad
&\text{ for } i=1,2,\ldots, D-1.
\end{align*}
Note that since $\LL_{\UU,0} = \LL_{\UU}$ and $\LL_{\VV,0} = \LL_{\VV}$
are positive definite,
it follows by construction
that $\LL_{\UU,i}$ and $\LL_{\VV,i}$ are positive definite for all $i = 0 ,\ldots, D-1$ as well since the pinching operators
$\mathcal{C}_{\UU,i} (\cdot)$
and
$\mathcal{C}_{\VV,i} (\cdot)$
preserve positive definiteness.

By construction of the projection matrices
$\PUU{1:i}, \PUU{i+1:d_1}, \PVV{1:i}, \PVV{i+1:d_2}$
it holds that
\begin{equation*}
    \PUU{1:i} \ZZ \PVV{i+1:d_2} = \f{0},
    \quad
    \PUU{i+1:d_1} \ZZ \PVV{1:i} = \f{0}
\end{equation*}
for all $i=1,2,\ldots, d-1$.
Thus, we can apply the previous Lemma \ref{lemma:majorization_property_sylvester}
for $1 \le i \leq d-1$.
Now consider the scenario $d \leq i <D$.
If $D=d_1$, we aim to apply Lemma \ref{lemma:majorization_property_sylvester}
with $\PP_{\VV}= \Id$ and $\PP_{\UU}=\PUU{1:i}$.
It follows that $ \PP_{\UU} \ZZ \PP_{\VV,\bot} =\f{0}$ since $ \PP_{\VV,\bot} = \f{0}$.
Moreover, we have $ \PP_{\UU,\bot} \ZZ \PP_{\VV} = \PUU{i+1:D} \ZZ  = \f{0}$,
where we have used $\PP_{\VV} =\Id$ in the first step and the SVD of $\ZZ$ in the second step.
This shows that we can apply \Cref{lemma:majorization_property_sylvester} if $D=d_1$.
If $D=d_2$, we can argue analogously.
Thus, we can also apply \Cref{lemma:majorization_property_sylvester}
in the case $d \le i <D$ both in the case $D=d_1$ and $D=d_2$.

This implies that for $i=1,2,\ldots, D-1$ the Sylvester equation
\begin{equation}\label{eq:sylvester:iteration}
    \LL_{\UU,i} \WW_i + \WW_i \LL_{\VV,i} = 2 \ZZ
\end{equation}
has a unique solution $\WW_i \in \Rdd$
and that
\begin{equation*}
    \innerproduct{ \WW_{i-1} , \ZZ }
    \ge
    \innerproduct{ \WW_i , \ZZ }.
\end{equation*}
It follows that
\begin{equation}\label{eq:inequality_chain}
    \langle \WW_0 , \ZZ \rangle
    \ge
    \langle \WW_1 , \ZZ \rangle
    \ge
    \langle \WW_2 , \ZZ \rangle
    \ge
    \ldots
    \ge
    \langle \WW_{D-1} , \ZZ \rangle.
\end{equation}
Note that by construction we have that
$\LL_{\UU,D-1} = \LL_{\UU,d_1-1} $ 
and $ \LL_{\VV,D-1} = \LL_{\VV,d_2-1}$.
By Lemma \ref{lemma:iterative_pinching},
we have that
\begin{equation*}
    \LL_{\UU,d_1-1} = \sum_{k=1}^{d_1} \innerproduct{ \LL_{\UU}, \uu_k \uu_k^\top } \uu_k \uu_k^{\top},
    \quad
    \LL_{\VV,d_2-1} = \sum_{k=1}^{d_2} \innerproduct{ \LL_{\VV}, \vv_k \vv_k^\top } \vv_k \vv_k^{\top}.
\end{equation*}
Thus,
for $i=D-1$
the solution $\WW_{D-1}$ of the Sylvester equation \eqref{eq:sylvester:iteration}
can be computed
explicitly
and we obtain that 
\begin{equation*}
    \WW_{D-1} 
    = \sum_{k=1}^d
    \frac{2 \sigma_k (\ZZ)}{ 
        \innerproduct{ \LL_{\UU}, \uu_k \uu_k^\top } 
        +
        \innerproduct{ \LL_{\VV}, \vv_k \vv_k^\top }
    }
    \uu_k \vv_k^{\top}.
\end{equation*}
It follows that 
\begin{equation*}
    \innerproduct{ \WW_{D-1} , \ZZ }
    =
    \sum_{k=1}^d
    \frac{2 \sigma_k^2 (\ZZ)}{ 
        \innerproduct{ \LL_{\UU}, \uu_k \uu_k^\top } 
        +
        \innerproduct{ \LL_{\VV}, \vv_k \vv_k^\top }
    }.
\end{equation*}
Combining this equation with 
the inequality chain \eqref{eq:inequality_chain}
and with $\WW_0 = \W{\XX}{\varepsilon} (\ZZ)$,
we have shown that
\begin{equation*}
    \innerproduct{ \W{\XX}{\varepsilon} (\ZZ), \ZZ }
    \ge
    \sum_{k=1}^d
    \frac{2 \sigma_k^2 (\ZZ)}{ 
        \innerproduct{ \LL_{\UU}, \uu_k \uu_k^\top } 
        +
        \innerproduct{ \LL_{\VV}, \vv_k \vv_k^\top }
    }.
\end{equation*}
This completes the proof.
\end{proof}

%% file: proofs_power_means.tex
\subsection{Proof of \Cref{thm:power_means_majorize} (Optimality of Harmonic Mean Quadratic Model)} \label{sec:proofs:power_means:optimality}
We now prove the optimality result \Cref{thm:power_means_majorize} of the quadratic model induced by the harmonic-mean weight operator, which has been presented in \Cref{sec:majorization:optimality}.

\subsubsection{Proof of \Cref{lem:power_mean_loewner} (Monotonicity of Power Mean Weight Operators)}
We begin with proving the Loewner ordering and domination properties of \Cref{lem:power_mean_loewner}, which is a tool to understand optimality of the quadratic models defined by power mean-induced weight operators.

\begin{proof}[Proof of \Cref{lem:power_mean_loewner}]
	We want to show that
	\begin{align*}
	\quad
	\langle \f{Z},W^{(q)}_{\f{X},\varepsilon}(\f{Z})\rangle_F
	\le
	\langle \f{Z},W^{(q')}_{\f{X},\varepsilon}(\f{Z})\rangle_F
	\ \ \text{for all }\f{Z}\in\R^{d_1\times d_2},
	\end{align*}
	where $W^{(q)}_{\f{X},\varepsilon}$ and $W^{(q')}_{\f{X},\varepsilon}$ are the power mean-induced weight operators \cref{eq:power:mean:weight:operator} associated with power means of order $q$ and $q'$ with $q\leq q'$. Fix $\f{Z}\in\R^{d_1\times d_2}$ and set $\widetilde{\f{Z}}:=\UU_{\XX}^\top \f{Z}\VV_{\XX}\in\R^{d_1 \times d_2}$, with $\UU_{\XX} \in \R^{d_1 \times d_1}$ and $\VV_{\XX} \in \R^{d_2 \times d_2}$ being the left and right singular vector matrices of the singular value decomposition of $\XX$. By orthogonality of $\UU_{\XX}$ and $\VV_{\XX}$ and the definition of $W^{(q)}_{\f{X},\varepsilon}$,
	it holds that 
	\begin{align*}
	\langle \f{Z},W^{(q)}_{\f{X},\varepsilon}(\f{Z})\rangle_F
	=&
	\langle 
		\ZZ , 
		\UU_{\XX}\left[ \f{H}^{(q)}_{\ssigma,\varepsilon}
		\circ \bigl(\UU_{\XX}^\top \ZZ \VV_{\XX}  \bigr) \right] \VV_{\XX}^\top
	\rangle_F
	=
	\langle 
		\UU_{\XX}^\top \ZZ \VV_{\XX} , 
		\f{H}^{(q)}_{\ssigma,\varepsilon} \circ \bigl(\UU_{\XX}^\top \ZZ \VV_{\XX}  \bigr) 
			\rangle_F\\
	=&
	\bigl\langle \widetilde{\f{Z}},\ \f{H}^{(q)}_{\ssigma,\varepsilon}\circ \widetilde{\f{Z}}\bigr\rangle_F
	=
	\sum_{i=1}^{d_1} \sum_{j=1}^{d_2} 
	\bigl(\f{H}^{(q)}_{\ssigma,\varepsilon}\bigr)_{ij}\,\widetilde{\f{Z}}_{ij}^{\,2}.
    \end{align*}
	Analogously, we have
	\begin{equation*}
		\innerproduct{
		\f{Z},W^{(q')}_{\f{X},\varepsilon}(\f{Z})}
		=
		\sum_{i=1}^{d_1} \sum_{j=1}^{d_2} 
		\bigl(\f{H}^{(q')}_{\ssigma,\varepsilon}\bigr)_{ij}\,\widetilde{\f{Z}}_{ij}^{\,2}.
	\end{equation*}
	As $q \le q'$, it follows from the monotonicity of power means \citep[Section III.3, Theorem 1]{Bullen03} that for all $i \in [d_1]$ and $j \in [d_2]$,
	$\mathcal{M}_q(\widetilde{\sigma}_i,\widetilde{\sigma}_j)\le \mathcal{M}_{q'}(\widetilde{\sigma}_i,\widetilde{\sigma}_j)$ and, hence, $\bigl(\f{H}^{(q)}_{\ssigma,\varepsilon}\bigr)_{ij}\le \bigl(\f{H}^{(q')}_{\ssigma,\varepsilon}\bigr)_{ij}$. Thus, we obtain the desired Loewner ordering since
	\begin{align*}
	\langle \f{Z},W^{(q)}_{\f{X},\varepsilon}(\f{Z})\rangle_F
	=
	\sum_{i=1}^{d_1} \sum_{j=1}^{d_2}
	\bigl(\f{H}^{(q)}_{\ssigma,\varepsilon}\bigr)_{ij}\,\widetilde{\f{Z}}_{ij}^{\,2}
	\le
	\sum_{i=1}^{d_1} \sum_{j=1}^{d_2}
	\bigl(\f{H}^{(q')}_{\ssigma,\varepsilon}\bigr)_{ij}\,\widetilde{\f{Z}}_{ij}^{\,2}
	=
	\langle \f{Z},W^{(q')}_{\f{X},\varepsilon}(\f{Z})\rangle_F.
	\end{align*}
    Finally, the domination \cref{eq:pm:quad:model:ordering} of the quadratic model $Q_\varepsilon^{(q)}(\cdot \mid\XX)$ by $Q_\varepsilon^{(q')}(\cdot \mid\XX)$ follows from the Loewner ordering and the formula \cref{eq:QZX:equality} as $\langle \XX,W^{(q)}_{\XX,\varepsilon}(\XX)\rangle_F$ and $\langle \XX,W^{(q')}_{\XX,\varepsilon}(\XX)\rangle_F$ both coincide since $\diag(\f{H}^{(q)}_{\ssigma,\varepsilon}) = \diag( \f{H}^{(q')}_{\ssigma,\varepsilon})$.
\end{proof}

\subsubsection{Proofs of \Cref{lemma:2nd_spectral_derivative,lem:necessary} (Second Order Necessary Condition for Majorization)} \label{sec:proofs:power_means:necessary}
Given the established results, the main burden for proving \Cref{thm:power_means_majorize} is to show that majorization is violated when $q < -1$. As we see below, this can be shown via a second order analysis of the smoothed objective $\mathcal{J}_{\varepsilon}$. To this end, we first define  the \emph{symmetrization operator} $S:\R^{d \times d} \to \R^{d \times d}$ and the \emph{antisymmetrization operator} $T:\R^{d \times d} \to \R^{d \times d}$ by
\begin{equation} \label{eq:sym_and_antisym}
	S(\f{Z}) := \frac{1}{2}(\f{Z} + \f{Z}^\top) \qquad \text{and} \qquad
	T(\f{Z}) := \frac{1}{2}(\f{Z} - \f{Z}^\top)
\end{equation}
for any $\f{Z} \in \R^{d \times d}$, which enables us to state an explicit formula for the Hessian of the spectral function $F$ in \Cref{lemma:2nd_spectral_derivative}.

\begin{lemma}[Hessian of Spectral Functions] \label{lemma:2nd_spectral_derivative}
	Let $f:\R_{\geq 0} \to \R$ be a differentiable function 
	with $L$-Lipschitz first derivative $f'$ such that $f'$ is right differentiable at $0$ and $f'(0)=0$.
	Then the spectral function $F: \Rdd \to \R$, $F(\X)= \sum_{i=1}^d f(\sigma_i(\X))$ is differentiable with $L$-Lipschitz gradients $\nabla F(\XX) \in \Rdd$ and furthermore almost everywhere twice differentiable.
    Moreover, 
	$F$ is twice differentiable at $\X \in \Rdd$ 
	if and only if $f$ is twice differentiable at all $\sigma_{1}(\XX),\ldots, \sigma_{\mind}(\XX)$. 
	In that case, if additionally $d_1 \leq d_2$, 
	the Hessian $\nabla^2 F(\X)$ of $F$ at $\X$ is given by
	\begin{equation} \label{eq:spectralfunction:Hessianformula}
		\nabla^2 F(\X)(\f{Z}) = \UU_{\XX} \begin{bmatrix}\f{H}_1 \circ S( \UU_{\XX}^\top \f{Z} \VV_{\XX,1})
			+ 	
			\f{H}_2 \circ T(\UU_{\XX}^\top \f{Z} \VV_{\XX,1})
			& \f{H}_3 \circ (\UU_{\XX}^\top \f{Z} \VV_{\XX,2})\end{bmatrix}
		\VV_{\XX}^\top
	\end{equation}
	for any $\f{Z} \in \Rdd$, 
	where $\X$ has the SVD $\X = \UU_{\XX} \diag(\ssigma) \VV_{\XX}^\top$ 
	with $\UU_{\XX} \in \R^{d_1 \times d_1}$, 
	$\VV_{\XX}= \begin{bmatrix} \VV_{\XX,1} & \VV_{\XX,2} \end{bmatrix} 
	\in \R^{d_2 \times d_2}$, $\VV_{\XX,1} \in \R^{d_2 \times d}$, $\VV_{\XX,2} \in \R^{d_2 \times (d_2-d)}$, 
	$T$ and $S$ are as in \eqref{eq:sym_and_antisym}, and $\f{H}_1 \in \R^{d \times d}$ is such that for $i,j \in [d]$,
	\[
	(\f{H}_1)_{ij} =
	\begin{cases}
		\frac{f'(\sigma_i) - f'(\sigma_j)}{\sigma_i - \sigma_j} & \text{ if } \sigma_i \neq \sigma_j, \\
		f''(\sigma_i) & \text{ if } \sigma_i = \sigma_j,
	\end{cases}
	\]
	the matrix $\f{H}_2 \in \R^{d \times d}$ is such that for $i,j \in [d]$,
	\[
	(\f{H}_2)_{ij} =
	\begin{cases}
		\frac{f'(\sigma_i) + f'(\sigma_j)}{\sigma_i + \sigma_j} & \text{ if } \sigma_i + \sigma_j \neq 0 , \\
		f''(\sigma_i) & \text{ if } \sigma_i = \sigma_j = 0,
	\end{cases}
	\]
	and $\f{H}_3 \in \R^{d_1 \times (d_2-d)}$ is such that for $i \in [d_1], j \in [d_2-d]$,
	\[
	(\f{H}_3)_{ij} =
	\begin{cases}
		f'(\sigma_i)/ \sigma_i & \text{ if } \sigma_i \neq 0 , \\
		f''(\sigma_i) & \text{ if } \sigma_i = 0.	
	\end{cases}
	\]
	\end{lemma}
\begin{proof}[Proof of \Cref{lemma:2nd_spectral_derivative}]
	Since $f$ is differentiable, the spectral function $F$	is differentiable with gradient 
	\[
	\nabla F(\X) = \f{U} \diag\big(\nabla \hat{f}(\sigma(\X))\big) \f{V}^\top 
	= \f{U}\diag\left(\big(f'(\sigma_i)\big)_{i=1}^d\right) \f{V}^\top
	\]
	at any $\XX \in \Rdd$ \citep[Section 7]{LewisSendov}, as $F$ can be written as the composition $\hat{f} \circ \sigma$, where $\hat{f}: \R^{\mind} \to \R$ with $\hat{f}(\sigma) = \sum_{i=1}^{\mind} f(\sigma_i)$. Maps of the form $\XX \mapsto \nabla F(\XX)$ are also called \emph{non-Hermitian Loewner operators} \citep{Loewner34,SunSun08,Ding18} or \emph{generalized matrix functions} \citep{HawkinsBenIsrael73,Noferini17}, and \citet[Theorem 1.1]{andersson2016operator} implies that since $f'$ is $L$-Lipschitz, $\X \mapsto \nabla F(\X)$ is $L$-Lipschitz with respect to the Frobenius norm. Rademacher's theorem then implies that $\nabla F$ is almost everywhere differentiable (with respect to the Lebesgue measure).
	
	Furthermore, it follows from \citet[Theorem 2.2.6]{Yang09} that $f$ is twice differentiable at $\sigma=\sigma(\X)$ if and only if $F$ is twice differentiable at $\X$, and the formula for the Hessian \cref{eq:spectralfunction:Hessianformula} at the points of twice differentiability $\X$ is due to \citet[Theorem 2.2.6]{Yang09} and \citet[Corollary 3.10]{Noferini17}.
\end{proof}
The machinery of \Cref{lemma:2nd_spectral_derivative} can now be used to state a second order condition that is necessary for majorization to hold. The key observation is that such a condition can already be extracted from perturbations of $\XX$ that are confined to a two-dimensional singular block. Along such perturbations, the smoothed nuclear norm objective $\mathcal{J}_{\varepsilon}$ reduces, up to an additive constant, to a spectral function on $\R^{2 \times 2}$, so that \Cref{lemma:2nd_spectral_derivative} only needs to be applied in dimension two and only the two singular values defining the block enter the argument.

\begin{lemma}[Second Order Necessary Condition for Majorization] \label{lem:necessary}
	Let $\varepsilon > 0$ and let $\XX\in\R^{d_1\times d_2}$ have the full singular value decomposition $\XX=\UU_{\XX} \diag(\ssigma) \VV_{\XX}^\top$ of \Cref{def:optimalweightoperator},
	with singular values
	$\sigma_1 \geq \ldots \geq \sigma_d \geq 0$.
	Let $Q_\varepsilon(\cdot\mid \XX)$ be the quadratic model
	\begin{equation*}
	Q_\varepsilon(\XX+\DDelta\mid \XX)
	=
	\mathcal{J}_{\varepsilon}(\XX)
	+\langle \nabla \mathcal{J}_{\varepsilon}(\XX),\DDelta\rangle_F
	+\tfrac12\langle \DDelta, \W{\XX}{\varepsilon} (\DDelta) \rangle_F,
	\end{equation*}
	where $\W{\XX}{\varepsilon}$ is the weight operator \cref{eq:W:operator:action} associated to a
	weight operator core matrix $\f{H}_{\ssigma, \varepsilon} \in \Rdd$ that is symmetric in the sense that
	$(\f{H}_{\ssigma, \varepsilon})_{ij} = (\f{H}_{\ssigma, \varepsilon})_{ji}$ for all $i,j \in [d]$.
	Assume that $Q_\varepsilon(\cdot\mid \XX)$ majorizes $\mathcal{J}_{\varepsilon}$ locally around $\XX$,
	i.e.,
	there exists $\delta>0$ such that
	\begin{equation} \label{eq:local:majorization}
	\mathcal{J}_{\varepsilon}(\XX+\DDelta)\le Q_\varepsilon(\XX+\DDelta\mid \XX)
	\end{equation}
	for all $\DDelta\in\R^{d_1\times d_2}$ with
	$\fronorm{\DDelta} \le \delta$.
	Then for all $i,j \in [d]$
	such that $i \neq j$, $\sigma_i>\varepsilon$ and $\sigma_j>\varepsilon$,
	it holds that
	\begin{equation} \label{eq:necessary:condition}
		(\f{H}_{\ssigma, \varepsilon})_{ij}\ge  \frac{ 2}{\sigma_i+\sigma_j}.
	\end{equation}
	No assumption is made on the remaining singular values $\sigma_k$ with $k \in [d] \setminus \{i,j\}$; in particular, $\mathcal{J}_{\varepsilon}$ is not assumed to be twice differentiable at $\XX$.
\end{lemma}

\begin{proof}[Proof of \Cref{lem:necessary}]
	Let $i,j \in [d]$ satisfy the assumptions of the lemma.
	Interchanging $i$ and $j$ if necessary (which changes neither the assumptions nor the conclusion, as $\f{H}_{\ssigma,\varepsilon}$ is symmetric on $[d] \times [d]$) we may assume that $i < j$, so that $\sigma_i \geq \sigma_j$.

	\emph{Step 1: Reduction to a $2 \times 2$ singular block.}
	Let $\uu_1, \ldots, \uu_{d_1}$ and $\vv_1, \ldots, \vv_{d_2}$ denote the columns of $\UU_{\XX}$ and of $\VV_{\XX}$, respectively, and define the perturbation direction
	\begin{equation} \label{eq:antisym:direction}
		\ZZ := \tfrac{1}{\sqrt{2}} \left( \uu_i \vv_j^\top - \uu_j \vv_i^\top \right) \in \Rdd,
		\qquad \text{which satisfies } \fronorm{\ZZ} = 1 .
	\end{equation}
	Then $\ZZ = \UU_{\XX} \f{E}_{ij} \VV_{\XX}^\top$ with
	$\f{E}_{ij} := \tfrac{1}{\sqrt{2}} ( \f{e}_i \f{e}_j^\top - \f{e}_j \f{e}_i^\top ) \in \Rdd$,
	where $\f{e}_i \in \R^{d_1}$ and $\f{e}_j \in \R^{d_2}$ in the first summand and $\f{e}_j \in \R^{d_1}$ and $\f{e}_i \in \R^{d_2}$ in the second summand, which is well defined since $i,j \in [d] = [\mind]$.
	By orthogonality of $\UU_{\XX}$ and $\VV_{\XX}$, the matrix $\XX + t \ZZ$ has, for each $t \in \R$, the same singular values as $\diag(\ssigma) + t \f{E}_{ij} \in \Rdd$.
	The latter matrix coincides with $\diag(\ssigma)$ outside the rows and columns with indices $i$ and $j$, and is therefore, after applying the permutations that move the indices $i$ and $j$ to the positions $1$ and $2$, block diagonal with blocks
	\begin{equation*}
		\f{B} + t \f{N},
		\quad \text{ where }
		\f{B} := \diag( \sigma_i, \sigma_j ) \in \R^{2 \times 2}
		\text{ and }
		\f{N} := \tfrac{1}{\sqrt{2}} \begin{pmatrix} 0 & 1 \\ -1 & 0 \end{pmatrix} \in \R^{2 \times 2},
	\end{equation*}
	and the rectangular diagonal matrix carrying the remaining singular values $(\sigma_k)_{k \in [d] \setminus \{i,j\}}$.
	As permutation matrices are orthogonal, the singular values of $\XX + t \ZZ$ are, as a multiset, the union of $\{ \sigma_k : k \in [d] \setminus \{i,j\} \}$ and of the two singular values of $\f{B} + t \f{N}$.
	Writing $F_2 : \R^{2 \times 2} \to \R$, $F_2(\f{A}) := j_\varepsilon( \sigma_1 (\f{A}) ) + j_\varepsilon( \sigma_2 (\f{A}) )$, for the spectral function of $j_\varepsilon$ in dimension two, we conclude that
	\begin{equation} \label{eq:block:reduction}
		\mathcal{J}_{\varepsilon}(\XX + t \ZZ) - \mathcal{J}_{\varepsilon}(\XX)
		=
		F_2 ( \f{B} + t \f{N} ) - F_2 ( \f{B} )
		\qquad \text{for all } t \in \R,
	\end{equation}
	since the summands $j_\varepsilon(\sigma_k)$ with $k \in [d] \setminus \{i,j\}$ do not depend on $t$ and cancel.
	In particular, only the two singular values $\sigma_i$ and $\sigma_j$ enter the argument below.

	\emph{Step 2: Application of \Cref{lemma:2nd_spectral_derivative} in dimension two.}
	It is straightforward to check that $j_\varepsilon$ satisfies the assumptions of \Cref{lemma:2nd_spectral_derivative}: it is differentiable with $\tfrac{1}{\varepsilon}$-Lipschitz derivative $j_\varepsilon'$, which is right differentiable at $0$ and satisfies $j_\varepsilon'(0) = 0$.
	We apply \Cref{lemma:2nd_spectral_derivative} with $f = j_\varepsilon$ to the spectral function $F = F_2$ with $d_1 = d_2 = d = 2$ at the point $\f{B}$, whose singular value decomposition is $\f{B} = \f{I}_2 \diag( \sigma_i, \sigma_j ) \f{I}_2^\top$ since $\sigma_i \geq \sigma_j \geq 0$.
	As $\sigma_i > \varepsilon$ and $\sigma_j > \varepsilon$,
	the function $j_{\varepsilon}$ is twice differentiable at $\sigma_i$ and at $\sigma_j$ with
	$j_\varepsilon'(\sigma_i)=j_\varepsilon'(\sigma_j)=1$
	and
	$j_\varepsilon''(\sigma_i)=j_\varepsilon''(\sigma_j)=0$,
	so that \Cref{lemma:2nd_spectral_derivative} guarantees that $F_2$ is twice differentiable at $\f{B}$.
	Since $d_1 = d_2$ in this application, the block $\VV_{\f{B},2}$ is empty and the matrix $\f{H}_3$ does not occur in \cref{eq:spectralfunction:Hessianformula}, while the off-diagonal entries of $\f{H}_1, \f{H}_2 \in \R^{2 \times 2}$ are given by
	\begin{align*}
		\bigl(\f{H}_1\bigr)_{12} = \bigl(\f{H}_1\bigr)_{21}
		&=
		\begin{cases}
			\frac{j_\varepsilon'(\sigma_i) - j_\varepsilon'(\sigma_j)}{\sigma_i - \sigma_j} = 0,
			& \text{ if } \sigma_i \neq \sigma_j, \\
			j_\varepsilon''(\sigma_i) = 0,
			& \text{ if } \sigma_i = \sigma_j,
		\end{cases}\\
		\bigl(\f{H}_2\bigr)_{12} = \bigl(\f{H}_2\bigr)_{21}
		&= \frac{j_\varepsilon'(\sigma_i) + j_\varepsilon'(\sigma_j)}{\sigma_i + \sigma_j} = \frac{2}{\sigma_i + \sigma_j}.
	\end{align*}
	As $\f{N}$ is skew-symmetric, we have $S(\f{N}) = \f{0}$ and $T(\f{N}) = \f{N}$ for the operators \cref{eq:sym_and_antisym}, and therefore the Hessian formula \cref{eq:spectralfunction:Hessianformula} yields that
	$\nabla^2 F_2 (\f{B}) (\f{N}) = \f{H}_1 \circ S(\f{N}) + \f{H}_2 \circ T(\f{N}) = \f{H}_2 \circ \f{N}$
	and, consequently, that
	\begin{equation} \label{eq:hessian:block:value}
		\innerproduct{ \f{N}, \nabla^2 F_2 (\f{B}) (\f{N}) }
		=
		\innerproduct{ \f{N}, \f{H}_2 \circ \f{N} }
		=
		\frac{2}{\sigma_i+\sigma_j} \fronorm{\f{N}}^2
		=
		\frac{2}{\sigma_i+\sigma_j}.
	\end{equation}

	\emph{Step 3: First and second derivative of $\mathcal{G}$ at $0$.}
	Define the one-dimensional function
	\[
	\mathcal{G}(t):=\mathcal{J}_{\varepsilon}(\XX+t\ZZ)-Q_\varepsilon(\XX+t\ZZ\mid \XX)
	\]
	for $t\in\R$.
	By definition of $Q_\varepsilon$ and by the reduction \cref{eq:block:reduction},
	we have
	\begin{align} \label{eq:G:blockform}
	\mathcal{G}(t) &= F_2 ( \f{B} + t \f{N} ) - F_2 ( \f{B} )
	- t \langle \nabla \mathcal{J}_{\varepsilon}(\XX),\ZZ\rangle_F
	- \frac{t^2}{2}\langle \ZZ, \W{\XX}{\varepsilon}(\ZZ)\rangle_F,
	\end{align}
	and, by construction, it holds that $\mathcal{G}(0)=0$.
	By \Cref{lemma:2nd_spectral_derivative}, the spectral function $F_2$ is differentiable with gradient
	$\nabla F_2 (\f{B}) = \f{I}_2 \diag \bigl( j_\varepsilon'(\sigma_i), j_\varepsilon'(\sigma_j) \bigr) \f{I}_2^\top = \f{I}_2$
	\citep[cf.][Section 7]{LewisSendov},
	so that $t \mapsto F_2 ( \f{B} + t \f{N} )$ is differentiable with derivative
	$\innerproduct{ \nabla F_2 ( \f{B} + t \f{N} ), \f{N} }$.
	Since $\f{N}$ has a vanishing diagonal, this derivative equals $\innerproduct{ \f{I}_2, \f{N} } = 0$ at $t = 0$.
	Together with \cref{eq:block:reduction} and the differentiability of $\mathcal{J}_{\varepsilon}$ (see \Cref{sec:appendix:lipschitz:gradients}), this implies that
	\begin{equation*}
		\innerproduct{ \nabla \mathcal{J}_{\varepsilon}(\XX), \ZZ }
		=
		\frac{d}{dt}\Big\vert_{t=0} \mathcal{J}_{\varepsilon}(\XX + t \ZZ)
		= 0,
	\end{equation*}
	and hence, by \cref{eq:G:blockform}, that $\mathcal{G}'(0) = 0$.
	Moreover, twice differentiability of $F_2$ at $\f{B}$ means that $\nabla F_2$ is differentiable at $\f{B}$ with derivative $\nabla^2 F_2 (\f{B})$, so that $t \mapsto \innerproduct{ \nabla F_2 ( \f{B} + t \f{N} ), \f{N} }$ is differentiable at $t = 0$ with derivative $\innerproduct{ \nabla^2 F_2 (\f{B}) (\f{N}), \f{N} }$.
	Therefore, $\mathcal{G}$ is twice differentiable at $t=0$, and \cref{eq:G:blockform} and \cref{eq:hessian:block:value} give
	\begin{equation*}
		\mathcal{G}''(0)
		=
		\innerproduct{ \f{N}, \nabla^2 F_2 (\f{B}) (\f{N}) }
		-
		\innerproduct{ \ZZ, \W{\XX}{\varepsilon}(\ZZ) }
		=
		\frac{2}{\sigma_i+\sigma_j}
		-
		\innerproduct{ \ZZ, \W{\XX}{\varepsilon}(\ZZ) }.
	\end{equation*}

	\emph{Step 4: Conclusion.}
	Since $\fronorm{t \ZZ} = |t|$ by \cref{eq:antisym:direction}, the local majorization assumption \cref{eq:local:majorization} implies that $\mathcal{G}(t) \leq 0 = \mathcal{G}(0)$ for all $|t| \leq \delta$.
	Thus, the function $\mathcal{G}$ has a local maximum at $0$, and therefore $\mathcal{G}''(0)\le 0$.
	By Step 3, this means that
	\begin{equation}\label{eq:necessary_condition_majorization}
		\innerproduct{
			\ZZ, \W{\XX}{\varepsilon}(\ZZ)
		}
		\geq
		\frac{2}{\sigma_i + \sigma_j}.
	\end{equation}
	On the other hand, we have $\UU_{\XX}^\top \ZZ \VV_{\XX} = \f{E}_{ij}$, so that the definition \cref{eq:W:operator:action} of the weight operator and the symmetry of $\f{H}_{\ssigma,\varepsilon}$ yield
	\begin{align*}
		\innerproduct{
			\ZZ, \W{\XX}{\varepsilon}(\ZZ)
		}
		=
		\innerproduct{ \f{E}_{ij}, \f{H}_{\ssigma, \varepsilon} \circ \f{E}_{ij} }
		=
		\frac{(\f{H}_{\ssigma, \varepsilon})_{ij} + (\f{H}_{\ssigma, \varepsilon})_{ji}}{2}
		=
		(\f{H}_{\ssigma, \varepsilon})_{ij}.
	\end{align*}
	Inserting this identity into \cref{eq:necessary_condition_majorization} gives \cref{eq:necessary:condition}.
    This completes the proof.
\end{proof}
\subsubsection{Proof of \Cref{thm:power_means_majorize}}
It now remains to prove \Cref{thm:power_means_majorize}, the statement about optimality of the harmonic-mean weight operator as defined by \cref{eq:harmonic:mean} and \Cref{def:optimalweightoperator}.
\begin{proof}[{Proof of \Cref{thm:power_means_majorize}}]
We note that the statement that $Q_\varepsilon^{(q)}(\cdot\mid\XX)$ majorizes $\mathcal{J}_{\varepsilon}$ for any $q \geq -1$ is a direct consequence of the harmonic mean majorization result \Cref{thm:majorization} and the monotonicity of power mean weight operators (\Cref{lem:power_mean_loewner}). For the reverse direction, we need to show that global majorization of $\mathcal{J}_{\varepsilon}$ is violated when $q < -1$. Let $q<-1$ be arbitrary, and let $i,j \in [d]$ with $i \neq j$ be indices such that $\sigma_i>\varepsilon$, $\sigma_j>\varepsilon$, and $\sigma_i\neq\sigma_j$, which exist by the assumption of \Cref{thm:power_means_majorize}. Since $\sigma_i > \varepsilon$ and $\sigma_j > \varepsilon$, we have $\widetilde{\sigma}_i = \sigma_i^{-1}$ and $\widetilde{\sigma}_j = \sigma_j^{-1}$ in \cref{eq:power:mean:core:matrix}, and $\sigma_i \neq \sigma_j$ implies that $\widetilde{\sigma}_i \neq \widetilde{\sigma}_j$.
As $q < -1$, we obtain by the definition \cref{eq:power:mean:core:matrix} of $\f{H}^{(q)}_{\ssigma,\varepsilon}$ and by the monotonicity of power means, which is strict at distinct arguments \citep[Section III.3, Theorem 1]{Bullen03}, that
\begin{align*}
(\f{H}^{(q)}_{\ssigma,\varepsilon})_{ij}
&= \mathcal{M}_q\!\Big(\sigma_i^{-1},\sigma_j^{-1}\Big)
<
\mathcal{M}_{-1}\!\Big(\sigma_i^{-1},\sigma_j^{-1}\Big)
=
\frac{2}{\sigma_i+\sigma_j}.
\end{align*}
Since power mean core matrices \cref{eq:power:mean:core:matrix} are symmetric, this means that the necessary condition \cref{eq:necessary:condition} of \Cref{lem:necessary} is violated for the index pair $(i,j)$.
By contraposition, \Cref{lem:necessary} therefore implies that $Q_\varepsilon^{(q)}(\cdot\mid \XX)$ does not majorize $\mathcal{J}_{\varepsilon}$ locally around $\XX$, i.e., for every $\delta > 0$, there exists a $\DDelta \in \Rdd$ with $\fronorm{\DDelta} \leq \delta$ such that
\begin{align*}
	\mathcal{J}_{\varepsilon}(\XX+\DDelta) > Q_\varepsilon^{(q)}(\XX+\DDelta\mid \XX).
\end{align*}
In particular,
$Q_{\varepsilon}^{(q)}(\cdot\mid \XX)$ does not majorize $\mathcal{J}_{\varepsilon}$ globally.
This completes the proof for $q<-1$.
\end{proof}

%% file: proofs_globalconvergencep1.tex
\subsection{Proofs of \Cref{mainresult:lowrank,mainresult:approximatelowrank} (Global Linear Convergence)}
\subsubsection{Preliminaries and General Proof Strategy} 
We start by recalling the following lemma 
which states that the NSP induces a reverse triangle inequality.
\begin{lemma}[\citeauthor{Fornasier11}, \citeyear{Fornasier11}, Lemma 6.6]\label{lemma:NSPl1min}
	Assume that the measurement operator $\mathcal{A}: \mathbb{R}^{d_1 \times d_2} \longrightarrow \R^m $ satisfies the NSP of order $r$ of \Cref{def:NSP:statement} for some $\eta_r<1$. Then  for all $\f{Z},  \XX  \in \mathbb{R}^{d_1 \times d_2}$ such that $\mathcal{A} \left( \f{Z} \right) =  \mathcal{A} \left(\XX \right)  $ it holds that
	
	\begin{equation*}
	\nucnorm{\f{Z} - \XX }  \leq \frac{1+\eta_r}{1-\eta_r}\left(    \nucnorm{\XX} -  \nucnorm{\f{Z}}  + 2\besterrNuc{\f{Z}}{r}\right).
	\end{equation*}
	
\end{lemma}

In our proof we will need to relate several times the quantities 
$ \mathcal{J}_{\varepsilon} \left( \XX\right) - \nucnorm{\Xzero }  $, 
$\nucnorm{\XXk - \Xzero} $, and $\besterrNuc{\XXk}{r}$.
This will be achieved via the following inequality.
\begin{lemma}\label{lemma:epscontrol} Let $\Xzero,  \XX \in \mathbb{R}^{d_1 \times d_2}$. Assume that the measurement operator $\mathcal{A}: \mathbb{R}^{d_1 \times d_2} \longrightarrow \R^m $ satisfies the NSP of order $r$ of \Cref{def:NSP:statement} with constant $\eta_r <1$.
	Furthermore, assume $ \mathcal{A} \left(\Xzero\right) = \mathcal{A} \left( \XX \right)$ and $0 \le \varepsilon \le \frac{\besterrNuc{\XX}{r}}{d} $, where the boundary case $\varepsilon = 0$ is understood with the convention $\mathcal{J}_{0} = \nucnorm{\cdot}$.
	Then it holds that
	\begin{equation}\label{ineq:aux1}
	\frac{1-\eta_r}{1+\eta_r}  \nucnorm{\XX -\Xzero } -2  \besterrNuc{\Xzero}{r}  \le     \mathcal{J}_{\varepsilon} \left( \XX\right) - \nucnorm{\Xzero }    \le   \left( \frac{3}{2} + \eta_r \right)  \besterrNuc{\XX}{r}  .
	\end{equation}
\end{lemma}
The proof of \Cref{lemma:epscontrol} is an improved version of the proof of \citet[Lemma B.1]{Kummerle-NeurIPS2021}, which is a corresponding lemma in the $\ell_1$-minimization scenario.
For the sake of completeness, we have included a proof below.
\begin{proof}[Proof of Lemma \ref{lemma:epscontrol}]
	It follows directly from the definition of $ \mathcal{J}_{\varepsilon} \left( \XX \right) $, see \cref{eq:smoothedell1:objective}, that $ \mathcal{J}_{\varepsilon} \left( \XX \right) \ge \nucnorm{\XX}$ for any matrix $\XX$. 
    This yields that
	\begin{align*}
	\mathcal{J}_{\varepsilon} \left(\XX \right) - \nucnorm{\Xzero} &\ge \nucnorm{\XX} - \nucnorm{\Xzero} \\
	&\ge 	\frac{1-\eta_r}{1+\eta_r}  \nucnorm{\XX -\Xzero }  -  2  \besterrNuc{\Xzero}{r},
	\end{align*}
	where in the second inequality we applied Lemma \ref{lemma:NSPl1min}. This proves the first inequality in~\eqref{ineq:aux1}.
	
	It remains to show the second inequality in \cref{ineq:aux1}. Assume now $\varepsilon > 0$.
	For that, we define $I := \{i \in [d]: \sigma_i \left(\XX \right) > \varepsilon \}$.
	It follows that
	\begin{align} 
	\mathcal{J}_{ \varepsilon} \left( \XX \right)  -\nucnorm{\Xzero } &= \sum_{i \in I} \sigma_{i} \left( \XX \right)     + \frac{1}{2} \sum_{i \in I^c} \left( \frac{\sigma_{i} \left( \XX \right) ^2}{\varepsilon}    + \varepsilon  \right)  -\nucnorm{\Xzero} \nonumber\\
	&=
	\nucnorm{\XX}
	+
	\sum_{i \in I^c}
	\left(
	\frac{\left(\varepsilon-\sigma_i(\XX)\right)^2}{2\varepsilon}
	\right)
	-\nucnorm{\Xzero}\nonumber \\
	&\le
	\nucnorm{\XX}
	+
	\frac{|I^c|\varepsilon}{2}
	-
	\nucnorm{\Xzero}\nonumber \\
	&\overleq{(a)}
	\nucnorm{\XX}
	+
	\frac{\besterrNuc{\XX}{r}}{2}
	-
	\nucnorm{\Xzero}.\label{eq:Jupperbd}
	\end{align}
	For inequality $(a)$ we used the assumption $\varepsilon \le \frac{\besterrNuc{\XX}{r}}{d}   $.
	Now note that
	\begin{align}
	\left(    \frac{1-\eta_r}{1+\eta_r}+1\right)\left( \nucnorm{\XX} -\nucnorm{\Xzero} \right)   &\overleq{(a)} \frac{1-\eta_r }{1+\eta_r } \nucnorm{  \XX - \Xzero} - \left(\nucnorm{\Xzero}  - \nucnorm{\XX} \right) \nonumber\\
	&\overleq{(b)}  \left(\nucnorm{\Xzero}  - \nucnorm{\XX}  + 2 \besterrNuc{\XX}{r}  \right) - \left(\nucnorm{\Xzero}  - \nucnorm{\XX} \right) \nonumber\\
	&\leq 2 \besterrNuc{\XX}{r} , \label{eq:caseeps0}
	\end{align}
	where inequality $(a)$ is the reverse triangle inequality and $(b)$ is again due to \Cref{lemma:NSPl1min}.
	By rearranging terms we obtain that 
	\begin{equation*}
	\nucnorm{\XX} -\nucnorm{ \Xzero }  \le \left(1+\eta_r\right) \besterrNuc{\XX}{r}.
	\end{equation*}
	We insert this into \cref{eq:Jupperbd}, which implies that
	\begin{align*}
	\mathcal{J}_{\varepsilon} \left(  \XX \right)  - \nucnorm{ \Xzero }
	& \le  \left(\frac{3}{2}+\eta_r\right)   \besterrNuc{\XX}{r}.
	\end{align*}
	This shows the second inequality in \cref{ineq:aux1}, which finishes the proof, noting that the case $\varepsilon=0$ follows directly from \cref{eq:caseeps0}.
\end{proof}
With \Cref{lemma:epscontrol} in place, we can outline the proof strategy for global linear convergence (\Cref{mainresult:lowrank} and \Cref{mainresult:approximatelowrank}) and local linear convergence with a faster rate (\Cref{thm:locallinearp1}) below.
While our proofs follow the proof strategy presented by \citet{Kummerle-NeurIPS2021}, there are several important differences.
This stems from the fact that for low-rank matrices there is no clear notion of support in contrast to sparse vectors.

Now recall that our goal is to show that $ \mathcal{J}_{\varepsilon_k}(\XXk) - \nucnorm{\Xzero}$ converges linearly to zero.
First, we set $\NNk := \Xzero - \XXk $. Then, we note that for all $t\in \R$ we have that
\begin{align}
\mathcal{J}_{\varepsilon_{k+1}}(\XXkplus)
&\le Q_{\varepsilon_k}(\XXk+t \NNk \mid \XXk) \nonumber \\
&=\mathcal{J}_{\varepsilon_k}(\XXk)
 +  t \, \bracing{=(a)}{ \innerproduct{ \nabla \mathcal{J}_{\varepsilon_k}(\XXk), \NNk  }} 
 + \frac{t^2}{2} 
 \bracing{=(b)}{ \innerproduct{ \NNk, \W{\XXk}{\varepsilon_k} (\NNk) }} 
 \label{proof_strategy}
\end{align}	
in the case that the quadratic model function $Q_{\varepsilon_k}(\cdot \mid \XXk)$ satisfies 
the majorization property of \cref{eq:majorization:inequality}, which we established for weight operators with weight operator core matrix $\f{H}_{\ssigma,\varepsilon_k}$ that corresponds to harmonic \cref{eq:harmonic:mean} in \Cref{thm:majorization}, and with power mean weights \cref{eq:power:mean:core:matrix} with $q \in [-1, \infty]$ in \Cref{thm:power_means_majorize}, and which is well-known in the literature for one-sided weights \cref{eq:leftsided:mean,eq:rightsided:mean}.

Thus, in these cases, in order to get an estimate for the decrease of $\mathcal{J}_{\varepsilon_{k+1}}(\XXkplus) $ made in the $(k+1)$th iteration, we can establish a (negative) upper bound for the (negative) term $(a)$ and a positive upper bound for term $(b)$.
After having established these bounds, we can optimize over $t$ such that the right-hand side in \eqref{proof_strategy} becomes minimal.
As it turns out, we will use the same estimate for term $(a)$ both for the global linear convergence proof and for the local linear convergence proof.
However, for term $(b)$ we will use different estimates, as $\XXk$ being in a neighborhood of $\Xzero$ allows us to derive sharper estimates of (b) than if we have no further information on $\XXk$---at least, if the weight operator $\W{\XXk}{\varepsilon_k}(\cdot)$ is defined with a harmonic-mean weight operator core matrix \cref{eq:harmonic:mean}.

The next lemma below deals with term $(a)$ in \eqref{proof_strategy}.
\begin{lemma}\label{lemma:linearterm}
	Recall that $ \besterrNuc{\XX}{r} = \sum_{i=r+1}^{d} \sigma_i (\XX) $. Assume that
\begin{equation*}
 \mathcal{A} \left(\Xzero \right) = \mathcal{A} \left(\XX \right),
\end{equation*}
that the measurement operator $\mathcal{A}: \mathbb{R}^{d_1 \times d_2} \longrightarrow \R^m $ has the NSP of order $r$ for some $\eta_r<1$ and that $\varepsilon \le \frac{\besterrNuc{\XX}{r}}{d}$.
Then it holds that
\begin{equation*}
\innerproduct{ \nabla \mathcal{J}_{\varepsilon}(\XX) ,\Xzero- \XX}  \le - \left( \frac{1-\eta_r}{1+\eta_r}    - 1/4  \right) \nucnorm{\XX -\Xzero} + \frac{9}{4} \besterrNuc{\Xzero}{r}.
\end{equation*}
\end{lemma}

\begin{proof}
Denote the singular value decomposition of the matrix $\XX$ by $\XX = \UU_{\XX} \diag (\ssigma) \VV_{\XX}^\top$.
We recall from \Cref{proposition:IRLS:basicproperties} that the gradient $\nabla \mathcal{J}_{\varepsilon}(\XX)$ at $\XX$ satisfies 
\begin{align*}
\nabla \mathcal{J}_{\varepsilon}(\XX) 
=
\W{\XX}{\varepsilon} (\XX)
=
\UU_{\XX}
\left[
    \f{H}_{\ssigma,\varepsilon}
	\circ
	\left( \UU_{\XX}^\top \XX \VV_{\XX} \right)
\right]
\VV_{\XX}^\top
= \UU_{\XX} \SSigma_{\varepsilon} \VV_{\XX}^\top,
\end{align*}
where
$\SSigma_{\varepsilon}  
:= \diag\bigg( \Big(\frac{\sigma_i(\XX)}{\max(\sigma_i(\XX),\varepsilon)}\Big)_{i=1}^d 
\bigg)$. 
Thus, we observe that
\begin{align*}
  \innerproduct{\nabla \mathcal{J}_{\varepsilon}(\XX),\Xzero- \XX }  =  \innerproduct{\UU_{\XX} \SSigma_{\varepsilon} \VV_{\XX}^\top  ,\Xzero- \XX }=\innerproduct{\UU_{\XX} \SSigma_{\varepsilon} \VV_{\XX}^\top  ,\Xzero} - \innerproduct{\UU_{\XX} \SSigma_{\varepsilon} \VV_{\XX}^\top  ,\XX }.
\end{align*}
To control the first summand, we apply H\"{o}lder's inequality and obtain that
\begin{align}\label{ineq:intern9}
\innerproduct{\UU_{\XX} \SSigma_{\varepsilon} \VV_{\XX}^\top  ,\Xzero } \le \specnorm{ \SSigma_{\varepsilon}  } \nucnorm{\Xzero}  \le  \nucnorm{\Xzero}.
\end{align}
For the second summand, we note first that
\begin{align*}
\innerproduct{\UU_{\XX} \SSigma_{\varepsilon} \VV_{\XX}^\top  ,\XX} 
&= \innerproduct{\UU_{\XX} \SSigma_{\varepsilon} \VV_{\XX}^\top  ,\UU_{\XX}  \diag (\ssigma) \VV_{\XX}^\top}\\
&= \innerproduct{ \SSigma_{\varepsilon}, \diag (\ssigma) } \\
&= \sum_{i=1}^{d} \frac{\sigma^2_i(\XX)}{\max(\sigma_i(\XX),\varepsilon)},
\end{align*}
using the definition of $\SSigma_\varepsilon$.
With the notation $I:= \left\{ i \in \left[d\right]:    \sigma_i(\XX)  > \varepsilon   \right\} $, we obtain that
\begin{equation}\label{ineq:intern8}
\innerproduct{\UU_{\XX} \SSigma_{\varepsilon} \VV_{\XX}^\top  ,\XX}= \sum_{i \in I} \sigma_i \left( \XX \right) +  \sum_{i \in I^c} \frac{\sigma_i^2 \left( \XX \right) }{\varepsilon}.
\end{equation}
By combining inequalities \eqref{ineq:intern9} and \eqref{ineq:intern8} we obtain that
\begin{align*}
\innerproduct{ \nabla \mathcal{J}_{\varepsilon}(\XX) ,\Xzero - \XX} &\le  \nucnorm{\Xzero} - \sum_{i \in I} \sigma_i \left( \XX \right) -  \sum_{i \in I^c} \frac{\sigma_i \left( \XX \right)^2 }{\varepsilon}\\
&= \nucnorm{\Xzero} - \nucnorm{\XX}  + \sum_{i \in I^c} \sigma_i \left( \XX \right) -  \sum_{i \in I^c} \frac{\sigma_i \left( \XX \right)^2 }{\varepsilon}\\
&= \nucnorm{\Xzero} - \nucnorm{\XX}  + \onenorm{\sigma (\XX)_{I^c}} -  \frac{ \twonorm{\sigma ( \XX )_{I^c}}^2 }{\varepsilon},
\end{align*}
where $\sigma(\XX)_{I^c}$ denotes the vector
which contains the singular values $ ( \sigma_i (\XX) )_{i \in I^c} $
and $\Vert \cdot \Vert_p$ denotes the $\ell_p$-norm of a vector.
Since $\sigma_i(\XX)\le \varepsilon$ for every $i\in I^c$, we obtain that 
\begin{equation*}
\sum_{i \in I^c}
\left(
\sigma_i(\XX)-\frac{\sigma_i(\XX)^2}{\varepsilon}
\right)
\le
\frac{|I^c|\varepsilon}{4},
\end{equation*}
since the scalar function $s \mapsto s-s^2/\varepsilon$ is maximized on $[0,\varepsilon]$ at $s=\varepsilon/2$ with value $\varepsilon/4$.
Hence, we have shown that 
\begin{align*}
\innerproduct{ \nabla \mathcal{J}_{\varepsilon}(\XX) ,\Xzero- \XX} 
 &\le \nucnorm{\Xzero} - \nucnorm{\XX}  + \frac{\varepsilon d}{4}\\
& \overleq{(a)}  \nucnorm{\Xzero} - \nucnorm{\XX}  +  \frac{ \besterrNuc{\XX}{r}}{4}\\
& = \nucnorm{\Xzero} - \nucnorm{\XX}  +  \frac{ \nucnorm{\XX - \XX_r} }{4},
\end{align*}
where $\XX_r$ denotes the best rank-$r$ approximation of $\XX$ and where in inequality $(a)$ we have used the assumption $\varepsilon \le  \frac{\besterrNuc{\XX}{r}}{d}$.
Denoting by $ \f{X}_{\star,r} $ 
the best rank-$r$ approximation of $\Xzero$ and using the Eckart--Young theorem we obtain that
\begin{align*}
\innerproduct{ \nabla \mathcal{J}_{\varepsilon}(\XX) ,\Xzero- \XX} 
& \le  \nucnorm{\Xzero} - \nucnorm{\XX}  + \frac{ \nucnorm{\XX - \f{X}_{\star,r}   }}{4} \\
&\le  \nucnorm{\Xzero} - \nucnorm{\XX}  + \frac{ \nucnorm{\XX -  \Xzero  }}{4}
+\frac{\besterrNuc{\Xzero}{r}}{4} \\
&\overleq{(a)}  - \left( \frac{1-\eta_r}{1+\eta_r}    - 1/4  \right) \nucnorm{\XX -\Xzero} + \frac{9}{4} \besterrNuc{\Xzero}{r},
\end{align*}
where for inequality $(a)$ we applied the reverse triangle inequality, see Lemma \ref{lemma:NSPl1min}.
\end{proof}
\subsubsection{Upper Bounds on Quadratic Forms Implied by Weight Operators} \label{sec:quadraticformupperbounds}

A key ingredient in the global rate proofs are norm bounds of Hadamard products $\f{H}_{\ssigma, \varepsilon} \circ \f{B}$ of the weight operator core matrix $\f{H}_{\ssigma, \varepsilon}$ with arbitrary matrices $\f{B}$, as the weight operator core matrix corresponds to the non-isometric part of the action of the weight operator of \Cref{def:optimalweightoperator}. We state and prove auxiliary results involving such bounds in \Cref{lemma:normbound:hadamard,lemma:hadamardestimate,lemma:posdefcriterion}. The norm bounds are then used in \Cref{lemma:quadraticterm} to upper bound the quadratic terms $\innerproduct{\Xzero - \XX , \W{\XX}{\varepsilon} (\Xzero - \XX) }$, which will be useful to handle the terms $(b)$ in inequality \eqref{proof_strategy}. 

The following lemma provides the key result to be used for power mean core matrices \cref{eq:power:mean:core:matrix}, which include harmonic and arithmetic mean core matrices as special cases.
\begin{lemma} \label{lemma:normbound:hadamard}
	Let $\lambda_1, \ldots, \lambda_d$ be a sequence of positive numbers.
	Set $\llambdaa := (\lambda_1, \ldots, \lambda_d)$ and $\lambda_{\max} := \max_{i \in [d]} \lambda_i$.
	Let $q \in [-\infty, \infty]$ and $\widetilde{\f{H}}_{\llambdaa}^{(q)} \in \R^{d \times d}$ 
	be the matrix defined by
	\begin{equation*}
		(\widetilde{\f{H}}_{\llambdaa}^{(q)})_{ij} 
		:= 
		\mathcal{M}_q (\lambda_i, \lambda_j)
		\quad
		\text{ for }
		i,j \in [d],
	\end{equation*}
	where $\mathcal{M}_q$ is the power mean as in Definition \ref{def:powermean}. Recall $c_q$ from \cref{eq:c_q:def}, i.e., for $q \in [-\infty, \infty]$,
	\begin{equation*}
	c_q =
	\begin{cases}
	1, & \text{ if } q \in [-\infty,1], \\
	2^{2-1/q}-1, & \text{ if } q \in (1,\infty), \\
	3, & \text{ if } q = \infty.
	\end{cases}
	\end{equation*}
	Then for any unitarily invariant norm $\vertiii{\cdot}$ on $\R^{d \times d}$ and any matrix $\f{B} \in \R^{d \times d}$, it holds that
	\begin{equation*}
	\vertiii{ \widetilde{\f{H}}_{\llambdaa}^{(q)} \circ \f{B} }
	\le c_q \lambda_{\max}\, \vertiii{ \f{B} }
	\leq
	\begin{cases}
	\lambda_{\max}\, \vertiii{ \f{B} }, & \text{ if } q \in [-\infty,1], \\
	3 \lambda_{\max}\, \vertiii{ \f{B} }, & \text{ if } q \in (1, \infty].
	\end{cases}
	\end{equation*}
\end{lemma}

An ingredient for showing \Cref{lemma:normbound:hadamard} for a range of values of $q$ is the following lemma, which is a consequence of an inequality by \citet{AndoHornJohnson87}.
\begin{lemma}\label{lemma:hadamardestimate}
Let $\f{A}   \in \R^{d \times d} $ be a positive semidefinite matrix and let $\f{B} \in \R^{d \times d} $ be arbitrary. 
Let $\vertiii{\cdot}$ be any unitarily invariant norm on $\R^{d \times d}$.
Then it holds that
\begin{equation*}
\vertiii{ \f{A}  \circ \f{B} }
\le \underset{i \in \left[d\right]}{\max} \  
\f{A}_{ii}  \  \vertiii{ \f{B} }.
\end{equation*}
\end{lemma}
\begin{proof}[Proof of Lemma \ref{lemma:hadamardestimate}]
From \citet[Equation (3.7.15)]{hadamardproduct_horn} \citep[see also][p.~363, eq.~(35)]{AndoHornJohnson87}, it follows that
\begin{equation*}
\vertiii{ \f{A}  \circ \f{B} } \leq c_1(\f{X}) c_1(\f{Y}) \vertiii{ \f{B} }
\end{equation*}
for square matrices $\f{A},\f{B} \in \R^{d \times d}$, where $\f{X},\f{Y} \in \R^{r \times d}$ are such that $\f{A} = \f{X}^* \f{Y}$ and $c_1(\f{X}), c_1(\f{Y})$ are the maximum Euclidean norms among the columns of $\f{X}$ and $\f{Y}$, respectively.
Since $\f{A}$ is positive semidefinite, 
we can choose $\f{Y} = \f{X} = \SSigma^{1/2} \f{U}^*$ 
where $\f{A} = \f{U} \SSigma \f{U}^*$ is an eigendecomposition of $\f{A}$.
Writing $\SSigma = \diag(\sigma)$, we observe that
\[
c_1(\f{Y}) = c_1(\f{X}) = \max_{i \in [d]} \|\f{X}_{:,i}\|_2 =  \max_{i \in [d]} \sqrt{\sum_{j=1}^{r} \sigma_j \f{U}_{i j}^2} =  \max_{i \in [d]} \sqrt{ (\f{U} \SSigma \f{U}^*)_{ii}} = \max_{i \in [d]} \sqrt{ \f{A}_{ii}}.
\]
This finishes the proof  of Lemma \ref{lemma:hadamardestimate}.
\end{proof}
In the proofs below, $\f{A}$ will be chosen 
as (a square extension of the) weight operator core matrix $\f{H}_{\ssigma, \varepsilon}$ of Definition \ref{def:weightcore}. 

For the case of the harmonic-mean core matrix \cref{eq:harmonic:mean}, which corresponds to a $-1$-power mean $\mathcal{M}_{-1}$ in \Cref{lemma:normbound:hadamard}, \Cref{lemma:hadamardestimate} can be quite directly applied to obtain the desired norm bound of \Cref{lemma:normbound:hadamard} due to the following positive semidefiniteness result.
\begin{lemma}\label{lemma:posdefcriterion}
Let the matrix $\f{A}  \in \R^{d \times d} $ be of the form $\f{A} _{ij} = \frac{1}{\lambda_i + \lambda_j}  $, where $\lambda_i >0$ for all $ i \in \left[d\right] $. Then $\f{A} $ is positive semidefinite.
\end{lemma}
For a proof of \Cref{lemma:posdefcriterion} we refer to \citet[Exercise 1.6.4]{bhatia_book}. 
The key idea is to show that $\f{A}$ is the Gram matrix of suitably chosen vectors in a Hilbert space and, thus, $\f{A}$ is positive semidefinite.

In order to establish \Cref{lemma:normbound:hadamard} across all cases of $q$, a bit more work is needed. We provide the general proof below.

\begin{proof}[Proof of \Cref{lemma:normbound:hadamard}]
 We distinguish several cases.

 \textbf{Case $q=-\infty$:}
 In this case we have
 \begin{equation*}
 (\widetilde{\f{H}}_{\llambdaa}^{(q)})_{ij} =  (\widetilde{\f{H}}_{\llambdaa}^{(-\infty)})_{ij}
 = \min(\lambda_i,\lambda_j)
 = \int_0^{\lambda_{\max}} 1_{t \le \lambda_i} 1_{t \le \lambda_j}\, dt.
 \end{equation*}
 Hence, for any $\f{x} \in \R^d$,
 \begin{align*}
 \f{x}^\top \widetilde{\f{H}}_{\llambdaa}^{(-\infty)} \f{x}
 &= \int_0^{\lambda_{\max}}
 \left(
 \sum_{i=1}^d x_i 1_{t \le \lambda_i}
 \right)^2 dt
 \ge 0.
 \end{align*}
 Thus $\widetilde{\f{H}}_{\llambdaa}^{(-\infty)}$ is positive semidefinite. Since $(\widetilde{\f{H}}_{\llambdaa}^{(-\infty)})_{ii}
 = \min(\lambda_i,\lambda_i)
 = \lambda_i$, \Cref{lemma:hadamardestimate} yields
 \begin{equation*}
 \vertiii{ \widetilde{\f{H}}_{\llambdaa}^{(-\infty)} \circ \f{B} }
 \le \lambda_{\max}\, \vertiii{ \f{B} }.
 \end{equation*}

 \textbf{Case $q \in (-\infty,0)$:}
 We first note that for all $i, j \in [d]$, $\lambda_i, \lambda_j > 0$, and
 \begin{equation*}
 (\widetilde{\f{H}}_{\llambdaa}^{(q)})_{ij}
 = \mathcal{M}_q (\lambda_i, \lambda_j)
 =
 \left( \frac{\lambda_i^q + \lambda_j^q}{2} \right)^{1/q}
 =
 2^{-\frac{1}{q}} \left( \lambda_i^q + \lambda_j^q \right)^{1/q}.
 \end{equation*}
 In order to proceed, we recall that the Gamma function $\Gamma$ is defined by
 $ \Gamma(z) = \int_0^\infty t^{z-1} e^{-t} dt$ for $z \in \mathbb{C}$ with $\Re(z) >0$.
 Now use the substitution $u:=t/x$ for $x>0$.
 Then we have that $dt = x du$.
 Hence, we obtain that
 \begin{equation*}
 \Gamma(z)
 = \int_0^\infty t^{z-1} e^{-t} dt
 =
 \int_0^\infty \left( xu \right)^{z-1} e^{-xu} x du
 =
 x^z \int_0^\infty u^{z-1} e^{-xu} du.
 \end{equation*}
 Then we obtain the well-known integral representation
 \begin{equation*}
 x^{-z} = \frac{1}{\Gamma(z)} \int_0^\infty u^{z-1} e^{-xu} du
 \end{equation*}
 for $x>0$ and $z >0$.
 Using this integral representation, we can write
 \begin{align*}
 \left( \lambda_i^q + \lambda_j^q \right)^{1/q}
 &= \left( \lambda_i^q + \lambda_j^q \right)^{-1/|q|} = \frac{1}{\Gamma(1/|q|)} \int_0^\infty u^{\frac{1}{|q|}-1} e^{-u \left( \lambda_i^q + \lambda_j^q \right)} du.
 \end{align*}
 Now let $\f{x} \in \R^d$ be arbitrary.
 Then we have that
 \begin{align*}
 \f{x}^\top \widetilde{\f{H}}_{\llambdaa}^{(q)} \f{x}
 =&
 \frac{ 2^{ 1/\vert q \vert } }{ \Gamma (1/|q|) }
 \sum_{i,j=1}^d x_i x_j \int_0^\infty u^{\frac{1}{|q|}-1} e^{-u
 \left( \lambda_i^q + \lambda_j^q \right)} du\\
 =&
 \frac{ 2^{ 1/\vert q \vert } }{ \Gamma (1/|q|) }
 \int_0^\infty u^{\frac{1}{|q|}-1}
 \left( \sum_{i=1}^d x_i e^{-u \lambda_i^q} \right)^2
 du
 \ge 0.
 \end{align*}
 It follows that $\widetilde{\f{H}}_{\llambdaa}^{(q)}$ is positive semidefinite.
 Since
 \begin{equation*}
 (\widetilde{\f{H}}_{\llambdaa}^{(q)})_{ii}
 = \mathcal{M}_q (\lambda_i, \lambda_i)
 = \lambda_i,
 \end{equation*}
 Lemma \ref{lemma:hadamardestimate} yields $\vertiii{ \widetilde{\f{H}}_{\llambdaa}^{(q)} \circ \f{B} }
 \le \lambda_{\max}\, \vertiii{ \f{B} }$.

 \textbf{Case $q=0$:}
 In this case
 \begin{equation*}
 (\widetilde{\f{H}}_{\llambdaa}^{(q)})_{ij} = (\widetilde{\f{H}}_{\llambdaa}^{(0)})_{ij}
 = \sqrt{\lambda_i \lambda_j}.
 \end{equation*}
 Hence $\widetilde{\f{H}}_{\llambdaa}^{(0)} = \f{v} \f{v}^\top$ with $\f{v} := (\sqrt{\lambda_1}, \ldots, \sqrt{\lambda_d})^\top$,
 so $\widetilde{\f{H}}_{\llambdaa}^{(0)}$ is positive semidefinite and $(\widetilde{\f{H}}_{\llambdaa}^{(0)})_{ii} = \lambda_i$.
 Therefore Lemma \ref{lemma:hadamardestimate} gives $\vertiii{ \widetilde{\f{H}}_{\llambdaa} \circ \f{B} }
 \le \lambda_{\max}\, \vertiii{ \f{B} }$.

 \textbf{Case $q \in (0,\infty)$:}
 In this case we have that
 \begin{equation*}
 (\widetilde{\f{H}}_{\llambdaa}^{(q)})_{ij}
 = \mathcal{M}_q (\lambda_i, \lambda_j)
 = \left( \frac{\lambda_i^q + \lambda_j^q}{2} \right)^{1/q}.
 \end{equation*}
 Now define the function $f(x,y) = 2^{-1/q} (x^q + y^q)^{1/q}$ for $x,y \ge 0$.
 Then the two-dimensional fundamental theorem of calculus implies that
 \begin{align*}
 f(x,y)&=
 f(x,0) + f(0,y) - f(0,0)
 + \int_0^x \int_0^y \frac{\partial^2 f}{\partial x \partial y} (s,t) dt ds\\
 &=
 \frac{x}{2^{1/q}} + \frac{y}{2^{1/q}}
 + \frac{1-q}{2^{1/q}} \int_0^x \int_0^y s^{q-1} t^{q-1} (s^q + t^q)^{1/q-2} dt ds.
 \end{align*}
 It follows that
 \begin{align*}
 (\widetilde{\f{H}}_{\llambdaa}^{(q)})_{ij}
 &=
 \frac{\lambda_i}{2^{1/q}} + \frac{\lambda_j}{2^{1/q}}
 + \frac{1-q}{2^{1/q}} \int_0^{\lambda_i} \int_0^{\lambda_j}
 s^{q-1} t^{q-1} (s^q + t^q)^{1/q-2} dt ds\\
 &=
 \frac{\lambda_i}{2^{1/q}} + \frac{\lambda_j}{2^{1/q}}
 + \frac{1-q}{2^{1/q}} \int_0^{\infty} \int_0^{\infty}
 s^{q-1} t^{q-1} (s^q + t^q)^{1/q-2} 1_{s \le \lambda_i} 1_{t \le \lambda_j} dt ds.
 \end{align*}
 Denoting by $\f{D} := \diag(\llambdaa)$ and by
 \begin{equation*}
 \f{P}_s := \diag \big(1_{s \le \lambda_i}\big)_{i=1}^d,
 \qquad
 \f{P}_t := \diag \big(1_{t \le \lambda_i}\big)_{i=1}^d,
 \end{equation*}
 we obtain
 \begin{align*}
 \widetilde{\f{H}}_{\llambdaa}^{(q)} \circ \f{B}
 =&
 \frac{1}{2^{1/q}} \f{D} \f{B}
 + \frac{1}{2^{1/q}} \f{B} \f{D}\\
 &+
 \frac{1-q}{2^{1/q}} \int_0^{\infty} \int_0^{\infty}
 s^{q-1} t^{q-1} (s^q + t^q)^{1/q-2}
 \f{P}_s \f{B} \f{P}_t
 dt ds.
 \end{align*}
 Since $\|\f{D}\| = \lambda_{\max}$ and $\|\f{P}_s\|, \|\f{P}_t\| \le 1$,
 the ideal property of unitarily invariant norms yields
 \begin{align*}
 \vertiii{ \widetilde{\f{H}}_{\llambdaa}^{(q)} \circ \f{B} }
 \le &
 \frac{1}{2^{1/q}} \vertiii{ \f{D} \f{B} }
 + \frac{1}{2^{1/q}} \vertiii{ \f{B} \f{D} }\\
 &+
 \frac{|1-q|}{2^{1/q}} \int_0^{\lambda_{\max}} \int_0^{\lambda_{\max}}
 s^{q-1} t^{q-1} (s^q + t^q)^{1/q-2}
 \vertiii{ \f{P}_s \f{B} \f{P}_t }
 dt ds\\
 \le &
 2^{1-\frac{1}{q}} \lambda_{\max}\, \vertiii{ \f{B} }\\
 &+
 \frac{|1-q|}{2^{1/q}} \vertiii{ \f{B} }
 \int_0^{\lambda_{\max}} \int_0^{\lambda_{\max}}
 s^{q-1} t^{q-1} (s^q + t^q)^{1/q-2}
 dt ds.
 \end{align*}
 If $q=1$, the integral term vanishes and therefore
 \begin{equation*}
 \vertiii{ \widetilde{\f{H}}_{\llambdaa}^{(q)} \circ \f{B} }
 \le \lambda_{\max}\, \vertiii{ \f{B} }.
 \end{equation*}
 Assume now that $q \neq 1$.
 Using the substitution $u = s^q$ and $v = t^q$, we obtain
 \begin{align*}
 \int_0^{\lambda_{\max}} \int_0^{\lambda_{\max}}
 s^{q-1} t^{q-1} (s^q + t^q)^{1/q-2}
 dt ds
 &=
 \frac{1}{q^2} \int_0^{\lambda_{\max}^q} \int_0^{\lambda_{\max}^q} (u+v)^{1/q-2} du dv\\
 =
 \frac{1}{q^2(1/q-1)} \int_0^{\lambda_{\max}^q}
 \left[ (u+v)^{1/q-1} \right]_{u=0}^{u=\lambda_{\max}^q}
 dv
 &=
 \frac{1}{q(1-q)} \int_0^{\lambda_{\max}^q}
 \left( (\lambda_{\max}^q + v )^{1/q-1} - v^{1/q-1} \right) dv\\
 =
 \frac{1}{1-q}
 \left[ (\lambda_{\max}^q + v)^{1/q} - v^{1/q} \right]_{v=0}^{v=\lambda_{\max}^q}
 &=
 \frac{\lambda_{\max}}{1-q} (2^{1/q} - 2).
 \end{align*}
 Therefore
 \begin{align*}
 \vertiii{ \widetilde{\f{H}}_{\llambdaa}^{(q)} \circ \f{B} }
 &\le
 \left(
 2^{1-\frac{1}{q}} + \vert 1 - 2^{1-1/q} \vert
 \right)
 \lambda_{\max}\, \vertiii{ \f{B} }.
 \end{align*}
 For $0<q\le 1$, the constant in parentheses equals $1$.
 For $1<q<\infty$, it equals $2^{2-1/q}-1$.

 \textbf{Case $q=\infty$:}
 In this case
 \begin{equation*}
 (\widetilde{\f{H}}_{\llambdaa}^{(q)})_{ij} =  (\widetilde{\f{H}}_{\llambdaa}^{(\infty)})_{ij}
 = \max(\lambda_i,\lambda_j)
 = \lambda_i + \lambda_j - \min(\lambda_i,\lambda_j).
 \end{equation*}
 Let $\f{M} \in \R^{d \times d}$ be defined by $\f{M}_{ij} := \min(\lambda_i,\lambda_j)$ for each $i,j \in [d]$. Then
 \begin{equation*}
 \widetilde{\f{H}}_{\llambdaa}^{(\infty)} \circ \f{B}
 = \diag(\llambdaa)\f{B} + \f{B}\diag(\llambdaa) - \f{M} \circ \f{B}.
 \end{equation*}
 By the ideal property of unitarily invariant norms,
 \begin{equation*}
 \vertiii{ \diag(\llambdaa)\f{B} },
 \ \vertiii{ \f{B}\diag(\llambdaa) }
 \le \lambda_{\max}\, \vertiii{ \f{B} }.
 \end{equation*}
 Moreover, the matrix $\f{M}$ is positive semidefinite by the argument from the case $q=-\infty$,
 and $\f{M}_{ii} = \lambda_i$.
 Hence Lemma \ref{lemma:hadamardestimate} yields
$ \vertiii{ \f{M} \circ \f{B} }
 \le \lambda_{\max}\, \vertiii{ \f{B} }$.
 Therefore
 \begin{equation*}
 \vertiii{ \widetilde{\f{H}}_{\llambdaa}^{(\infty)} \circ \f{B} }
 \le \left(\lambda_{\max}\, \vertiii{ \f{B} } + \lambda_{\max}\, \vertiii{ \f{B} } + \lambda_{\max}\, \vertiii{ \f{B} }\right) = 3 \lambda_{\max}\, \vertiii{ \f{B} }.
 \end{equation*}
\end{proof}

The bounds of \Cref{lemma:normbound:hadamard} can now be used to upper bound the quadratic term $(b)$ in inequality \eqref{proof_strategy} as follows.

\begin{lemma}\label{lemma:quadraticterm}
	Assume that
	\begin{equation*}
		 \mathcal{A} \left(\Xzero \right) = \mathcal{A} \left(\XX \right)
	\end{equation*}
	and that the measurement operator $\mathcal{A}: \mathbb{R}^{d_1 \times d_2} \longrightarrow \R^m $ 
	satisfies the NSP of order $r$ of \Cref{def:NSP:statement} for some constant $0<\eta_r<1$, and the NSP of order $1$ with $0 < \eta_1 < 1$. 
	Let $ \W{\XX}{\varepsilon}(\cdot)$ be defined as in \Cref{def:optimalweightoperator}. Then
	\begin{equation}\label{ineq:claim12}
	\innerproduct{\Xzero - \XX , \W{\XX}{\varepsilon} (\Xzero - \XX) }  \leq  \frac{ \eta_1 \nucnorm{ \Xzero - \XX }^2  }{\varepsilon},
	\end{equation}	
	if the weight operator core matrix $\f{H}_{\ssigma, \varepsilon}$ of $\W{\XX}{\varepsilon}(\cdot)$ corresponds to harmonic mean \cref{eq:harmonic:mean}, one-sided weights \cref{eq:leftsided:mean,eq:rightsided:mean}, or power mean weights \cref{eq:power:mean:core:matrix} with $q \in [-\infty, 1]$, and
	\begin{equation}\label{ineq:quadraticterm:largeqmean}
		\innerproduct{\Xzero - \XX , \W{\XX}{\varepsilon} (\Xzero - \XX) }  \leq \frac{ c_q \eta_1 \nucnorm{ \Xzero - \XX }^2  }{\varepsilon} \leq  \frac{ 3 \eta_1 \nucnorm{ \Xzero - \XX }^2  }{\varepsilon},
		\end{equation}	
	if $\f{H}_{\ssigma, \varepsilon}$ corresponds to power mean weights \cref{eq:power:mean:core:matrix} with $q \in (1,\infty]$, where $c_q$ is the constant \cref{eq:c_q:def}.
\end{lemma}
	
\begin{proof} Let $\NN := \Xzero - \XX $.
We may assume without loss of generality that $ d_1 \le d_2 $.
For a singular value decomposition $\XX = \UU_{\XX} \diag (\ssigma) \VV_{\XX}^\top$ of $\XX$, we set $\f{M} := \UU_{\XX}^{\top} \NN  \VV_{\XX}$.
First, we notice that
\begin{align}
\innerproduct{ \NN , \W{\XX}{\varepsilon} (\NN)  }
&\le \nucnorm{\NN} \specnorm{\W{\XX}{\varepsilon} \left( \NN \right)} \nonumber \\
&= \nucnorm{\NN} \specnorm{ \UU_{\XX} \left[  \f{H}_{\ssigma, \varepsilon} \circ  \f{M} \right]
\VV_{\XX}^\top    } \nonumber \\
&\le \nucnorm{\NN} \specnorm{   \f{H}_{\ssigma, \varepsilon} \circ \f{M}    }.\label{ineq:part2innerproduct}
\end{align}
We now show that $\specnorm{\f{H}_{\ssigma, \varepsilon} \circ \f{M}} \le \frac{c_q \specnorm{\NN}}{\varepsilon}$.

\textbf{One-sided weights:}
For left-sided weights, $(\f{H}_{\ssigma,\varepsilon})_{ij} = \frac{1}{\max(\sigma_i(\XX),\varepsilon)}$ depends only on the row index~$i$.
Hence $\f{H}_{\ssigma,\varepsilon} \circ \f{M} = \diag(h_1,\ldots,h_{d_1}) \f{M}$ where $h_i = \frac{1}{\max(\sigma_i(\XX),\varepsilon)}$. Since $\max_i h_i \leq 1/\varepsilon$, we obtain
\[
\specnorm{\f{H}_{\ssigma,\varepsilon} \circ \f{M}} \le \frac{\specnorm{\f{M}}}{\varepsilon} = \frac{\specnorm{\NN}}{\varepsilon}.
\]
For right-sided weights the argument is analogous with multiplication from the right.

\textbf{Power mean weights (including harmonic mean):}
Define the $d_2 \times d_2$ matrix $\widehat{\f{H}}$ by
\[
\widehat{\f{H}}_{ij}
:= \mathcal{M}_q\!\left(\frac{1}{\max(\hat\sigma_i,\varepsilon)},\, \frac{1}{\max(\hat\sigma_j,\varepsilon)}\right),
\]
where $\hat\sigma_i := \sigma_i(\XX)$ for $i \in [d_1]$ and $\hat\sigma_i := 0$ for $i \in \{d_1+1,\ldots,d_2\}$.
Note that $\max(\hat\sigma_i,\varepsilon) = \varepsilon$ for $i > d_1$, and therefore the first $d_1$ rows of $\widehat{\f{H}}$ coincide with $\f{H}_{\ssigma,\varepsilon}$.
Define $\widehat{\f{M}} \in \R^{d_2 \times d_2}$ by setting $\widehat{\f{M}}_{ij} = \f{M}_{ij}$ for $i \in [d_1]$ and $\widehat{\f{M}}_{ij} = 0$ for $i > d_1$.
Since the rows of $\widehat{\f{H}} \circ \widehat{\f{M}}$ with index $i > d_1$ are all zero, it follows that $
\specnorm{\f{H}_{\ssigma,\varepsilon} \circ \f{M}} = \specnorm{\widehat{\f{H}} \circ \widehat{\f{M}}}$. We can now invoke \Cref{lemma:normbound:hadamard} for $\widehat{\f{H}}$ and $\widehat{\f{M}}$ with the
positive numbers $\lambda_i := \frac{1}{\max(\hat\sigma_i,\varepsilon)}$, $i \in [d_2]$, and for the
spectral norm $\vertiii{\cdot} = \|\cdot\|$, which is unitarily invariant. Since
$\lambda_{\max} = \max_{i \in [d_2]} \lambda_i \le \frac{1}{\varepsilon}$, this yields for every
$q \in [-\infty,\infty]$ that
\[
\specnorm{\widehat{\f{H}} \circ \widehat{\f{M}}}
\le c_q\, \lambda_{\max} \specnorm{\widehat{\f{M}}}
\le \frac{c_q \specnorm{\widehat{\f{M}}}}{\varepsilon}
= \frac{c_q \specnorm{\NN}}{\varepsilon},
\]
where $c_q$ is the constant \cref{eq:c_q:def} and where we used in the last equality that
\[
\specnorm{\widehat{\f{M}}} = \specnorm{\f{M}} = \specnorm{\UU_{\XX}^{\top} \NN \VV_{\XX}} = \specnorm{\NN},
\]
which follows from the unitary invariance of the spectral norm.

In all cases, we have thus shown that $\specnorm{\f{H}_{\ssigma,\varepsilon} \circ \f{M}}
\le \frac{c_q\specnorm{\NN}}{\varepsilon}$, which combined with \eqref{ineq:part2innerproduct} gives
\[
\innerproduct{ \NN , \W{\XX}{\varepsilon} (\NN) } \le \frac{c_q\specnorm{\NN} \nucnorm{\NN}}{\varepsilon}.
\]
This implies that
\[
\innerproduct{ \NN , \W{\XX}{\varepsilon} (\NN) } \leq \frac{c_q\,\eta_1 \nucnorm{\NN}^2}{\varepsilon}.
\]
by the order-one NSP.
Since $c_q = 1$ for $q \in [-\infty,1]$, this proves \cref{ineq:claim12}, including for the harmonic-mean case of $q = -1$, while $1 < c_q \le 3$ for $q \in (1,\infty]$ proves \cref{ineq:quadraticterm:largeqmean}.
\end{proof}

\subsubsection{Proof of \Cref{mainresult:lowrank}} \label{sec:proofglobalconvergence:lowrank}

Having derived estimates for terms $(a)$ and $(b)$ in inequality \eqref{proof_strategy}, see Lemma \ref{lemma:linearterm} and Lemma \ref{lemma:quadraticterm}, the following proposition quantifies the decrease of $ \mathcal{J}_{\varepsilon_{k}} \left( \XXk  \right)$ in each iteration.
\begin{proposition}\label{prop:p1:linearrate}
	Let $\Xzero \in \Rdd$. Assume that the measurement operator $\mathcal{A}: \mathbb{R}^{d_1 \times d_2} \longrightarrow \R^m $ 
	satisfies the NSP \cref{eq:NSP:definition} of order $r$ for some $0 < \eta_r<1$ 
	and that $ \f{y} = \mathcal{A} \left(  \Xzero \right) $.
	Let the IRLS iterates  $\left\{   \XXk \right\}_k$ and $ \left\{ \varepsilon_{k} \right\}_k$ be defined by \cref{eq:IRLS:step_1} and \cref{eq:IRLS:step_2} of \Cref{alg:algo1} with arbitrary positive definite initial weight operator $W^{(0)}$ and rank estimate $\widetilde{r} =r$, and assume that \Cref{alg:algo1} does not return in iteration $k$, i.e., that $\varepsilon_k > 0$, so that the iterate $\XXkplus$ is defined.
	Assume that the weight operator $\W{\XXk}{\varepsilon_k}(\cdot)$ used in each iteration of \Cref{alg:algo1} is defined as in \Cref{def:weightcore} 
	is admissible in the sense of \Cref{def:admissible:weightoperators}, let $c_q$ be the weight operator-dependent constant \cref{eq:c_q:def} of \Cref{lemma:normbound:hadamard}.
	Set
	\begin{equation*}
	\gamma :=
	\begin{cases}
	3/4 \quad & \text{if } \besterrNuc{\Xzero}{r}=0, \\
	1/2 \quad & \text{otherwise}.
	\end{cases}
	\end{equation*}
	If 
	\begin{equation*}
		\besterrNuc{\Xzero}{r}  \le  \frac{1}{9} \nucnorm{\Xzero -\XXk}
	\quad \text{and} \quad
	\gamma > \frac{2\eta_r}{1+\eta_r},
	\end{equation*}
	then it holds that
	\begin{equation*} 
		\mathcal{J}_{\varepsilon_{k+1}}( \XXkplus ) -  \nucnorm{\Xzero} \le \left(  1-  \frac{ C_{\eta_r,\gamma}  }{ c_q \eta_1  d }\right) \left(     \mathcal{J}_{\varepsilon_{k}} \left( \XXk  \right) -  \nucnorm{ \Xzero } \right)
	\end{equation*} 
	where the constant $C_{\eta_r}$ is defined by 
	\begin{equation*}
	C_{\eta_r,\gamma}
	:=
	\frac{\left( \gamma - \frac{2\eta_r}{1+\eta_r}  \right)^2 }{\left( 3 + 2\eta_r \right)}.
	\end{equation*}
\end{proposition}

\begin{proof}
Set $ \NNk := \Xzero - \XXk $. It follows from the majorization of $\mathcal{J}_{\varepsilon_k}(\cdot)$ by the quadratic model function $Q_{\varepsilon_k}(\cdot \mid \XXk)$ (see \Cref{thm:majorization}, \Cref{thm:power_means_majorize}, and \Cref{prop:global:majorization:onesided}) for weight operators in question that for any $t \in \mathbb{R}$, it holds that
\begin{equation} \label{eq:mainidea}
\mathcal{J}_{\varepsilon_{k+1}}(\XXkplus) \leq Q_{\varepsilon_k} (\XXkplus \mid \XXk) \leq Q_{\varepsilon_k}(\XXk+t \NNk \mid \XXk),
\end{equation}
where we used the optimality of $\XXkplus$ in \cref{eq:IRLS:step_1} in the second inequality. Moreover, by the definition of the quadratic objective $Q_{\varepsilon_k}(\cdot \mid \XXk)$, see \cref{eq:smoothedell1:IRLSmajorizer}, it holds that
\begin{equation*}
\begin{split}
&Q_{\varepsilon_k}(\XXk+t \NNk \mid \XXk) - \mathcal{J}_{\varepsilon_k}(\XXk)\\ 
=&  t \, \innerproduct{ \nabla \mathcal{J}_{\varepsilon_k}(\XXk), \NNk  } 
+ \frac{t^2}{2} \innerproduct{ \NNk, \W{\XXk}{\varepsilon_k}(\NNk)}.
\end{split}
\end{equation*}	
Our goal is to minimize the difference 
$Q_{\varepsilon_k}(\XXk+t \NNk \mid \XXk) - \mathcal{J}_{\varepsilon_k}(\XXk)<0$ 
by choosing $t \ge 0$ accordingly. 
By using Lemma \ref{lemma:linearterm} as well as Lemma \ref{lemma:quadraticterm}, it follows that
\begin{align*}
&Q_{\varepsilon_k}(\XXk+t \NNk \mid \XXk) - \mathcal{J}_{\varepsilon_k}(\XXk)\\  
\le &- t \left( \frac{1-\eta_r}{1+\eta_r}    - 1/4  \right) \nucnorm{\NNk} + \frac{9}{4} t \besterrNuc{\Xzero}{r} + t^2 \frac{c_q\eta_1 \nucnorm{ \NNk }^2  }{2\varepsilon_k}\\
=& - t \left(\frac{3}{4}  - \frac{2 \eta_r}{1+\eta_r}  \right) \nucnorm{\NNk} 
+ \frac{9}{4} t  \besterrNuc{\Xzero}{r} + t^2 \frac{c_q \eta_1 \nucnorm{ \NNk }^2  }{2 \varepsilon_k},
\end{align*}
where $c_q$ is the constant \cref{eq:c_q:def}. Consequently, we obtain that
\begin{equation*}
Q_{\varepsilon_k}(\XXk+t \NNk \mid \XXk) - \mathcal{J}_{\varepsilon_k}(\XXk)  \le - t \cdot \bracing{=:b}{ \left( \gamma - \frac{2\eta_r}{1+\eta_r}  \right) \nucnorm{\NNk} } + t^2 \cdot \bracing{=:a}{ \frac{c_q \eta_1 \nucnorm{ \NNk }^2  }{2 \varepsilon_k}},
\end{equation*}
where we have used the assumption $\besterrNuc{\Xzero}{r} \le \frac{1}{9} \nucnorm{\NNk}$. 
The right-hand side is minimized by $t:= \frac{b}{2a}$. 
We obtain that
\begin{align*}
Q_{\varepsilon_k}(\XXk+t \NNk \mid \XXk) - \mathcal{J}_{\varepsilon_k}(\XXk)  
\le  
\frac{-b^2}{4a} =  \frac{- \left( \gamma - \frac{2\eta_r}{1+\eta_r}  \right)^2   \varepsilon_k}{ 2 c_q \eta_1  }.
\end{align*}
Together with inequality \eqref{eq:mainidea} this yields that
\begin{equation*}
\begin{split}
\mathcal{J}_{\varepsilon_{k+1}}( \XXkplus ) - \mathcal{J}_{\varepsilon_{k}}(\XXk)
\leq   \frac{- \left( \gamma - \frac{2\eta_r}{1+\eta_r}  \right)^2   \varepsilon_k}{ 2 c_q \eta_1  }.
\end{split}
\end{equation*}
In particular, we obtain that
\begin{equation}\label{eq:function_J_decreasing}
\mathcal{J}_{\varepsilon_{k+1}}( \XXkplus ) -  \nucnorm{\Xzero} 
\le 
\mathcal{J}_{\varepsilon_{k}}(\XXk) -  \nucnorm{\Xzero} -  \frac{ \left( \gamma - \frac{2\eta_r}{1+\eta_r}  \right)^2  }{ 2 c_q \eta_1  }  \varepsilon_k.
\end{equation}
In order to proceed, we need to bound $\varepsilon_k$ from below.
For that, we note that 
\begin{equation*}
\varepsilon_k = \min\left(\varepsilon_{k-1} ,  \frac{\besterrNuc{\XXk}{r}}{d} \right) = \frac{  \besterrNuc{\XXl}{r}    }{d} 
\end{equation*}
for some $\ell \le k$.
Using \Cref{lemma:epscontrol}, we obtain the inequality chain
\begin{align*}
d \varepsilon_k 
= \besterrNuc{\XXl}{r}  
&\ge   \frac{1}{3/2+\eta_r} \left(\mathcal{J}_{\varepsilon_{\ell}} \left( \XXl     \right) - \nucnorm{\Xzero} \right)\ge   \frac{1}{3/2+\eta_r} \left(\mathcal{J}_{\varepsilon_{k}} \left(\XXk \right) - \nucnorm{\Xzero} \right),
\end{align*}
where in the second inequality we have used that  $\mathcal{J}_{\varepsilon_k} \left( \XXk  \right) \le  \mathcal{J}_{\varepsilon_{\ell}} \left( \XXl  \right) $, which follows from the monotonicity of  $\mathcal{J}_{\varepsilon_k} \left( \XXk  \right) $ in $k$, see \eqref{eq:J:monotonicity:1}.
Plugging this into \cref{eq:function_J_decreasing} leads to 
\begin{align*}
\mathcal{J}_{\varepsilon_{k+1}}(\XXkplus) -  \nucnorm{\Xzero} 
&\le 
\left(  1-  \frac{\left( \gamma - \frac{2\eta_r}{1+\eta_r}  \right)^2 }{2 c_q (3/2+\eta_r) \eta_1  d }\right)\left(     \mathcal{J}_{\varepsilon_{k}}(\XXk ) -  \nucnorm{\Xzero} \right) \\
& = \left(  1-  \frac{\left( \gamma - \frac{2\eta_r}{1+\eta_r}  \right)^2 }{c_q (3+ 2 \eta_r) \eta_1  d }\right)\left(     \mathcal{J}_{\varepsilon_{k}}(\XXk ) -  \nucnorm{\Xzero} \right). 
\end{align*}
The constant $c_q$ of \cref{eq:c_q:def} is equal to $1$ for one-sided, harmonic-mean and power mean weights with $q \in [-1, 1]$, and satisfies $1 < c_q \leq 3$ for power mean weights with $q \in (1, \infty]$, which finishes the proof of \Cref{prop:p1:linearrate}.
\end{proof}

Now, using \Cref{prop:p1:linearrate}, we can prove the main result concerning global linear convergence of IRLS in the case that $\Xzero$ is exactly low-rank, which was stated as \Cref{mainresult:lowrank}.
\begin{proof}[Proof of  \Cref{mainresult:lowrank} ]
Set
\begin{equation*}
C_{\eta_r}
:= C_{\eta_r,\frac{3}{4}} =
\frac{\left( \frac{3}{4} - \frac{2\eta_r}{1+\eta_r}  \right)^2 }{3+2\eta_r}.
\end{equation*}
Since $k$ is an iteration carried out by \Cref{alg:algo1}, we have $\varepsilon_{j} > 0$ in each of the preceding iterations $j = 0,\ldots,k-1$, so that \Cref{prop:p1:linearrate} is applicable in these iterations. Chaining its assertion, inequality \cref{equ:linearconvergence1} follows as $\eta_r < 3/5$ if and only if $3/4 > \frac{2\eta_r}{1+\eta_r}$.
Next, we show inequality \eqref{ineq:globalconvergenerate1}. We note that with $c_q$ as in \cref{eq:c_q:def}, we have that
\begin{align*}
	\frac{1-\eta_r}{1+\eta_r}  \nucnorm{\XXk -\Xzero }  
	&\overleq{(a)} 
	\mathcal{J}_{\varepsilon_k} \left( \XXk\right) - \nucnorm{\Xzero }\\
	&\overleq{(b)}  
	\left(1-   \frac{  C_{\eta_r}  }{c_q \eta_1  d } \right)^k  
	\left(   \mathcal{J}_{\varepsilon_0} \left( \XX^{(0)}\right) - \nucnorm{\Xzero }   \right)\\
 &\overleq{(c)}    \left(\frac{3}{2}+\eta_r\right) \left(1-   \frac{ C_{\eta_r}  }{c_q \eta_1  d }  \right)^k \besterrNuc{\XX^{(0)}}{r} \\
&\overleq{(d)}    \left(\frac{3}{2}+\eta_r\right) \left(1-   \frac{C_{\eta_r}  }{c_q \eta_1  d }  \right)^k \nucnorm{ \XX^{(0)} - \Xzero   },
\end{align*}
where in inequalities $(a)$ and $(c)$ we used Lemma \ref{lemma:epscontrol}. Inequality $(b)$ follows from inequality \eqref{equ:linearconvergence1} and in inequality $(d)$ we used the Eckart--Young theorem. Multiplying both sides by $\frac{1+\eta_r}{1-\eta_r}$ yields \eqref{ineq:globalconvergenerate1}.
\end{proof}
\subsubsection{Proof of \Cref{mainresult:approximatelowrank} } \label{sec:proofglobalconvergence:approximatelowrank}
In order to show \Cref{mainresult:approximatelowrank}, 
the global linear convergence of \Cref{alg:algo1} under the assumption that $\Xzero$ is approximately low-rank, 
we need a slightly more involved argument compared to the proof of \Cref{mainresult:lowrank}. We show the details below.
\begin{proof}[Proof of \Cref{mainresult:approximatelowrank}]
Recall from the statement of  \Cref{mainresult:approximatelowrank} that
\begin{equation*}
\hat{k}=  
\min \left\{ k \in \mathbb{N}_0 :  \besterrNuc{\Xzero}{r} > \frac{1}{9}\nucnorm{ \Xzero -  \XXk } \right\} ,
 \end{equation*}
with the convention $\min\varnothing=\infty$ (so $\hat{k}=\infty$ whenever $\besterrNuc{\Xzero}{r}=0$).
As in the proof of \Cref{mainresult:lowrank}, all indices $k$ considered below are iterations carried out by \Cref{alg:algo1}, so that $\varepsilon_j > 0$ in each preceding iteration $j < k$ and \Cref{prop:p1:linearrate} is applicable in these iterations.

Since $\eta_r < 1/3$, we have $\frac{1}{2} > \frac{2\eta_r}{1+\eta_r}$, so that the constant
\[
\widetilde{C}_{\eta_r} := C_{\eta_r,1/2} =
\frac{ \left(  \frac{1}{2} - \frac{2\eta_r}{1+\eta_r}   \right)^2}{ 3+2\eta_r}
\]
is larger than $0$. Abbreviating $\mu := 1-  \frac{\widetilde{C}_{\eta_r}}{c_q \eta_1 d} \in [0,1)$,
where $c_q$ is the weight operator dependent constant \cref{eq:c_q:def}, we claim that
\begin{equation}\label{ineq:intern4}
\mathcal{J}_{\varepsilon_{k}}( \XXk ) -  \nucnorm{\Xzero}
\le \mu^{\min(k,\hat{k})}
\left(     \mathcal{J}_{\varepsilon_{0}}( \XX^{(0)} ) -  \nucnorm{\Xzero} \right)
\end{equation}
for all $k \in \mathbb{N}_0$.
Indeed, by the minimality in the definition of $\hat{k}$, the condition
$\besterrNuc{\Xzero}{r} \le \frac{1}{9}\nucnorm{ \Xzero -  \XX^{(j)} }$
of \Cref{prop:p1:linearrate} is satisfied for every $j < \hat{k}$, so that \Cref{prop:p1:linearrate},
applied with $\gamma = 1/2$, yields
\begin{equation}\label{ineq:intern4:onestep}
\mathcal{J}_{\varepsilon_{j+1}}( \XX^{(j+1)} ) -  \nucnorm{\Xzero}
\le \mu \left( \mathcal{J}_{\varepsilon_{j}}( \XX^{(j)} ) -  \nucnorm{\Xzero} \right)
\qquad \text{for all } j < \hat{k}.
\end{equation}
(In the case $\besterrNuc{\Xzero}{r} = 0$, we have $\hat{k} = \infty$ and \Cref{prop:p1:linearrate}
applies with $\gamma = 3/4$; since $C_{\eta_r,3/4} \ge \widetilde{C}_{\eta_r}$, inequality
\cref{ineq:intern4:onestep} holds for all $j \in \mathbb{N}_0$ a fortiori.)
For $k \le \hat{k}$, inequality \cref{ineq:intern4} now follows by chaining
\cref{ineq:intern4:onestep} for $j = 0,\ldots,k-1$, whereas for $k > \hat{k}$, it follows from the
case $k = \hat{k}$ together with $\mathcal{J}_{\varepsilon_{k}}( \XXk ) \le
\mathcal{J}_{\varepsilon_{\hat{k}}}( \XX^{(\hat{k})} )$, which is a consequence of the monotonicity
of the sequence $\left\{ \mathcal{J}_{\varepsilon_{k}} \left( \XXk \right) \right\}_k$,
see \cref{eq:J:monotonicity:1}.
Hence, we have shown inequality \cref{equ:linearconvergence2}.

In order to show inequality \eqref{ineq:approxsparse1}, we note first that for all $k \in \mathbb{N}_0$
\begin{align*}
  \frac{1-\eta_r}{1+\eta_r}  
  \nucnorm{\XXk -\Xzero}  -  2 \besterrNuc{\Xzero}{r}
  & \overleq{(a)} \mathcal{J}_{\varepsilon_{k}}( \XXk ) -  \nucnorm{\Xzero}\\
	&\overleq{(b)} \left(  1-  \frac{\widetilde{C}_{\eta_r}}{c_q \eta_1 d}   \right)^{\min(k, \hat{k})} 
\left(  \mathcal{J}_{\varepsilon_{0}}( \XX^{(0)}  ) -  \nucnorm{\Xzero} \right)  \\
& \overleq{(c)}   \left(\frac{3}{2}+\eta_r\right) \left(1-  \frac{\widetilde{C}_{\eta_r}}{c_q \eta_1 d}   \right)^{ \min( k, \hat{k}) }  
\besterrNuc{\XX^{(0)}}{r},
\end{align*}
	where inequalities $(a)$ and $(c)$ follow from Lemma \ref{lemma:epscontrol} and 
	inequality $(b)$ follows from inequality \cref{ineq:intern4}.
By rearranging terms, it follows that
\begin{align}\label{ineq:intern18}
  \nucnorm{\XXk -\Xzero}    
  &\le  A_{\eta_r} \left(  1-  \frac{\widetilde{C}_{\eta_r}}{c_q \eta_1 d}   \right)^{ \min (k, \hat{k})} 
  \besterrNuc{\XX^{(0)}}{r}    + \frac{2(1+\eta_r)}{1-\eta_r} \besterrNuc{\Xzero}{r},
\end{align}
where $A_{\eta_r}$ is as in \eqref{eq:A_eta_r:def}. Denote by $\XX_{\star,r}$ the best rank-$r$ approximation of the matrix $\Xzero$.
Then it follows from the Eckart--Young theorem that
\begin{equation}\label{ineq:intern2}
\besterrNuc{\XX^{(0)}}{r} \le   \nucnorm{\XX^{(0)}  -\XX_{\star,r} }  
 \le  \nucnorm{\XX^{(0)}  -\Xzero } + \besterrNuc{\Xzero}{r},
 \end{equation}
 where in the second inequality we used the triangle inequality.
 Combining inequalities \eqref{ineq:intern18} and \eqref{ineq:intern2} shows that
 \begin{equation*}
  \nucnorm{\XXk -\Xzero}  
  \le  
  A_{\eta_r} \left(  1-  \frac{\widetilde{C}_{\eta_r}}{c_q \eta_1 d}   \right)^{ \min (k, \hat{k})} 
  \nucnorm{\XX^{(0)}  -\Xzero }  
  + B_{\eta_r}  \besterrNuc{\Xzero}{r},
 \end{equation*}
with constant $B_{\eta_r} = \frac{\left(\frac{7}{2}+\eta_r\right)\left(1+\eta_r\right)}{1-\eta_r}$ as in \eqref{eq:B_eta_r:def}, which corresponds to inequality \cref{ineq:approxsparse1}. 
 In order to finish the proof, it remains to show inequality \eqref{ineq:approxsparse2}. For $k \ge \hat{k}$ we can compute that
 \begin{align*}
 \mathcal{J}_{\varepsilon_{k}}( \XXk ) -  \nucnorm{\Xzero} &\overleq{(a)}  \mathcal{J}_{\varepsilon_{\hat{k}}}( \XX^{ (\hat{k} ) } ) -  \nucnorm{\Xzero}  \\
 &\overleq{(b)}  \left( \frac{3}{2}+\eta_r \right) \besterrNuc{\XX^{(\hat{k})}}{r} \\
 &\overleq{(c)} \left( \frac{3}{2}+\eta_r \right) \nucnorm{\XX^{(\hat{k})}  -\Xzero } + \left( \frac{3}{2}+\eta_r \right) \besterrNuc{\Xzero}{r} \\
  &\overleq{(d)} 10 \left( \frac{3}{2}+\eta_r \right) \besterrNuc{\Xzero}{r}.
 \end{align*}
Inequality $(a)$ is due to the monotonicity of the sequence $ \left\{  \mathcal{J}_{\varepsilon_{k}} \left( \XXk \right) \right\}_k $ and inequality $(b)$ follows from Lemma \ref{lemma:epscontrol}.
Furthermore, inequality $(c)$ can be obtained by arguing as in \eqref{ineq:intern2}  and inequality $(d)$ is a direct consequence of the definition of $\hat{k}$, which implies that $\nucnorm{\XX^{(\hat{k})}  -\Xzero } \leq 9 \besterrNuc{\Xzero}{r}$.
Next, Lemma \ref{lemma:epscontrol} 
combined with the above inequality chain
implies that
\begin{align*}
\frac{1-\eta_r}{1+\eta_r}  \nucnorm{\XXk -\Xzero}  -  2 \besterrNuc{\Xzero}{r}
\le 
10\left( \frac{3}{2}+\eta_r \right)
\besterrNuc{\Xzero}{r}.
\end{align*}
By rearranging terms, we obtain
\[
	\nucnorm{\XXk -\Xzero} \leq \frac{1+ \eta_r}{1-\eta_r}\left(10\left( \frac{3}{2}+\eta_r \right) + 2 \right)  \besterrNuc{\Xzero}{r} \leq 41 \besterrNuc{\Xzero}{r},
\] using the assumption that $\eta_r < 1/3$ in the last inequality, which corresponds to inequality \cref{ineq:approxsparse2} for $k \ge \hat{k}$.

It remains to establish the upper bound on $\hat{k}$ stated after the theorem. If $\besterrNuc{\Xzero}{r}=0$, then $\hat{k}=\infty$ and there is nothing to prove, so assume $\besterrNuc{\Xzero}{r}>0$. The bound holds trivially if $\hat{k} = 0$, so assume $\hat{k} \ge 1$. By definition of $\hat{k}$, the strict inequality $\besterrNuc{\Xzero}{r} > \frac{1}{9}\nucnorm{\XXk -\Xzero}$ fails for every $k < \hat{k}$, so in particular
\[
9 \besterrNuc{\Xzero}{r} \le \nucnorm{ \XX^{(\hat{k}-1)} - \Xzero}.
\]
Applying inequality \cref{ineq:approxsparse1} at $k = \hat{k}-1$ (so that $\min(k,\hat{k}) = \hat{k}-1$) and combining the two inequalities yields
\[
(9-B_{\eta_r})\, \besterrNuc{\Xzero}{r}
\le
A_{\eta_r} \left(  1-  \frac{\widetilde{C}_{\eta_r}}{c_q \eta_1  d}   \right)^{\hat{k}-1}
\nucnorm{ \XX^{(0)} -\Xzero}.
\]
The assumption $\eta_r < 1/3$ implies $B_{\eta_r} < 23/3 < 9$, so the left-hand side is positive. Taking logarithms and using $-\log(1-x) \ge x$ for $x \in (0,1)$ gives
\[
\hat{k} \le 1 +
\frac{c_q \eta_1 d}{\widetilde{C}_{\eta_r}}
\log_{+}\left(   \frac{A_{\eta_r}}{9-B_{\eta_r}} \cdot
\frac{\nucnorm{ \XX^{(0)}  -\Xzero }}{\besterrNuc{\Xzero}{r}} \right),
\]
as claimed, where $\log_{+}(x)=\max(0,\log(x))$.
This completes the proof of \Cref{mainresult:approximatelowrank}.
\end{proof}

%% file: proofs_fastlocallinearp1.tex
\subsection{Proofs of \Cref{thm:locallinearp1,thm:counterexample:leftsided:weight:operator} (Local Linear Convergence and Counterexample)} \label{sec:prooffastlocallinearp1}
A limitation of the global linear convergence rate proofs is that they rely on estimates for the quadratic term $ \innerproduct{ \Xzero - \XX , \W{\XX}{\varepsilon} ( \Xzero - \XX) } $ of the quadratic model mismatch such as \Cref{lemma:quadraticterm}, which scale with $1/\varepsilon$. 
In this section, we overcome this limitation in the case of harmonic-mean (and related $q$-power mean) weight operators and obtain a faster linear convergence rate by using a more precise estimate for the quadratic term in the case that $\XX$ is close enough to the ground truth $\Xzero$ in \Cref{sec:quadraticformupperbounds:harmonicmean}, which leads to the proof of \Cref{thm:locallinearp1} in \Cref{sec:fastlocalrate:harmonicmean:proof}. 
On the other hand, we provide in \Cref{sec:counterexample:leftsided:weight:operator} a counterexample exhibiting that this is not possible for one-sided weight operators, such as those with left-sided or right-sided core matrices \cref{eq:leftsided:mean,eq:rightsided:mean}, even within arbitrarily smaller local neighborhoods of $\Xzero$ than the one defined by \cref{assump:localconvergencerate}, which establishes \Cref{thm:counterexample:leftsided:weight:operator}. 
\subsubsection{Tight Upper Bounds on Harmonic-Mean Quadratic Forms} \label{sec:quadraticformupperbounds:harmonicmean}
We proceed with a technical result on a tight upper bound on the quadratic form $\innerproduct{\Xzero-  \XX , \W{\XX}{\varepsilon} ( \Xzero-  \XX ) }$ implied by $q$-power mean weight operators with $q \in [-1,0)$, which applies if the iterate $\XX$ on which the weight operator is based is close enough to the rank-$r$ ground truth $\Xzero$.

\begin{lemma}\label{lemma:quadratictermlocal}
Assume that the linear measurement operator $\mathcal{A}: \mathbb{R}^{d_1 \times d_2} \rightarrow \mathbb{R}^m$ satisfies the NSP of order $r$ with constant $ \eta_r \le  3/5 $, and that for some $\Xzero$ of rank $r$, it holds that
 $ \mathcal{A}(\XX) = \mathcal{A} (\Xzero) $ for some $ \XX \in \mathbb{R}^{d_1 \times d_2}$. Assume that the weight operator $\W{\XX}{\varepsilon}$ is a $q$-power mean weight operator \cref{eq:power:mean:core:matrix} with $q \in [-1,0)$\footnote{This includes the harmonic-mean weight operator as special case with $q = -1$.}, that
\begin{equation}\label{ineq:assump1}
	\nucnorm{\XX - \Xzero} 
	\le 
	\frac{ \sigma_r \left( \Xzero \right) }{ \max \left( \eta_1 \sqrt{ d}, 2 \right)  },  
\end{equation}
where $0 < \eta_1 \leq \eta_r$ denotes again the order-one NSP constant of $\mathcal{A}$, and that
\begin{equation*}
\varepsilon \ge \vartheta \frac{\besterrNuc{\XX}{r} }{d}
\end{equation*}
for some $\vartheta>0$.
Then for the constant $D_{\eta_r} = \frac{3+\eta_r}{1-\eta_r}$, it holds that 
\begin{equation*}
	\innerproduct{ \Xzero-  \XX , \W{\XX}{\varepsilon} (\Xzero-  \XX ) }
	\le 
	\left(2 + 2^{-1/q}+\frac{D_{\eta_r}^2}{\vartheta}\right)  \nucnorm{\XX-\Xzero}.
\end{equation*}
If $\vartheta = \frac{1-\eta_r}{(1+\eta_r)\left(\frac{3}{2}+\eta_r\right)}$, the bound reduces to
\begin{equation} \label{ineq:claim11}
	\innerproduct{ \Xzero-  \XX , \W{\XX}{\varepsilon} ( \Xzero-  \XX ) }
	\le 
	\left(2 + 2^{-1/q}+E_{\eta_r}\right)  \nucnorm{\XX-\Xzero}.
\end{equation}
with $E_{\eta_r} := \frac{(3+\eta_r)^2 (1+\eta_r)\left(\frac{3}{2}+\eta_r\right)}{(1-\eta_r)^3}$.
\end{lemma}
Note that the right-hand side in inequality \eqref{ineq:claim11} is up to constants a factor of $\frac{\eta_1 \nucnorm{\XX-\Xzero}}{\varepsilon}$ smaller than the right-hand side of inequality \eqref{ineq:claim12}, the corresponding inequality in the global linear convergence proof.
In particular, since there is no $\varepsilon$-dependence anymore, we can improve the convergence rate by a factor of $d$ compared to \Cref{mainresult:lowrank}.

For $q=-1$, which corresponds to the harmonic-mean weight operator \cref{eq:W:operator:action}, and an NSP constant of $\eta_r = 1/10$, the right hand side of \cref{ineq:claim11} amounts to $\approx 27.20 \nucnorm{\XX-\Xzero}$, whereas for $\eta_r = 1/2$, it amounts to $298 \nucnorm{\XX-\Xzero}$.

For proving \Cref{lemma:quadratictermlocal}, we use the following elementary inequality about power means (recall \Cref{def:powermean}), as well as an elementary norm bound for matrices in the null space of a measurement operator equipped with the NSP (\Cref{lemma:normbound:nullspace}).  
\begin{proposition} \label{prop:powermeanineq}
For $a,b > 0$ and $q \in [-1,0)$, it holds that
\begin{equation*}
	\min(a,b) \leq \mathcal{M}_q(a,b) \leq 2^{-1/q} \min(a,b).
\end{equation*}
\end{proposition}
\begin{proof}[{Proof of \Cref{prop:powermeanineq}}]
	The first inequality follows from the simple internality property \citep[Section III.1, Theorem 2(a)]{Bullen03}. For the second inequality, we note that for any $q \in [-1,0)$ and $a,b > 0$, it holds that
	\begin{equation*}
		\begin{split}
		\mathcal{M}_q(a,b) = \left( \frac{a^q + b^q}{2} \right)^{1/q} =  \left( \frac{1 + \left(\frac{\max(a,b)}{\min(a,b)}\right)^q}{2} \right)^{1/q}  \min(a,b) &= \left( \frac{2}{1 + \left(\frac{\max(a,b)}{\min(a,b)}\right)^q} \right)^{-1/q}  \min(a,b) \\
		&\le 2^{-1/q} \min(a,b),
		\end{split}
	\end{equation*}
	using that $-1/q > 0$ and that $\frac{\max(a,b)}{\min(a,b)} \ge 1$.
\end{proof}

\begin{lemma} \label{lemma:normbound:nullspace}
Assume that the linear measurement operator $\mathcal{A}: \mathbb{R}^{d_1 \times d_2} \rightarrow \mathbb{R}^m$ satisfies the NSP of order $1$ with constant $0 < \eta_1 < 1$. Then it holds that for any $\f{N} \in \ker(\mathcal{A})$, 
\begin{equation} \label{ineq:normbound:nullspace}
	\specnorm{\f{N}} \leq \frac{\eta_1}{1+\eta_1}\nucnorm{\f{N}}.
\end{equation}
\end{lemma}
\begin{proof}[Proof of \Cref{lemma:normbound:nullspace}]
	From the NSP inequality \cref{eq:NSP:definition} of order $1$, it follows that
	\[
	\specnorm{\f{N}} = \sigma_1(\f{N}) \leq \eta_1 \sum_{i=2}^{d} \sigma_i(\f{N}) = \eta_1 \left(\nucnorm{\f{N}} - \specnorm{\f{N}}\right),
	\]
	which is equivalent to the desired inequality \eqref{ineq:normbound:nullspace}.
\end{proof}

\begin{proof}[Proof of \Cref{lemma:quadratictermlocal}]
We first observe that from Weyl's inequality and assumption \cref{ineq:assump1}, it follows that
\begin{equation*}
	\sigma_r \left( \XX \right) 
	\ge \sigma_r \left( \Xzero  \right) - \specnorm{  \XX - \Xzero}
	\ge \sigma_r \left( \Xzero  \right) - \nucnorm{  \XX - \Xzero} \ge \sigma_r \left( \Xzero  \right)   \left(\frac{\max(\eta_1 \sqrt{d},2)-1}{\max(\eta_1 \sqrt{d},2) }\right).
\end{equation*}
Denote by $\XX = \UU_{\XX} \diag (\ssigma) \VV_{\XX}^\top  $ a full singular value decomposition of $\XX$. From the definition of the weight operator $\W{\XX}{\varepsilon}(\cdot): = \W{\f{X}}{\varepsilon}^{(q)}(\cdot)$ \cref{eq:power:mean:weight:operator} and \cref{eq:power:mean:core:matrix}, it follows that
\begin{align*}
	\innerproduct{ \Xzero - \XX , \W{\XX}{\varepsilon} ( \Xzero - \XX ) }  &= \innerproduct{ \XX - \Xzero , \W{\XX}{\varepsilon} ( \XX - \Xzero ) } \\
	&= \innerproduct{\XX - \Xzero, \UU_{\XX} ( \f{H}_{\ssigma, \varepsilon}^{(q)} \circ ( \UU_{\XX}^\top (\XX - \Xzero)  \VV_{\XX} )) \VV_{\XX}^\top }\\
	&= \innerproduct{\UU_{\XX}^\top ( \XX - \Xzero ) \VV_{\XX} ,  
	\f{H}_{\ssigma, \varepsilon}^{(q)} \circ ( \UU_{\XX}^\top (\XX - \Xzero)  \VV_{\XX} ) }\\
	&= \innerproduct{ (\UU_{\XX}^\top ( \XX - \Xzero ) \VV_{\XX} ) \circ ( \UU_{\XX}^\top ( \XX - \Xzero ) \VV_{\XX} ),  
	\f{H}_{\ssigma, \varepsilon}^{(q)} }.
\end{align*}
Define the set of entries
\begin{equation*}
	S:= 
	\left\{ 
	 (i,j) \in [d_1] \times [d_2]
		:
		1 \le i \le r \text{ or } 1 \le j \le r
	 \right\}
\end{equation*}
and its complement $S^c := [d_1] \times [d_2] \setminus S$.
Denote by $\mathcal{P}_S, \mathcal{P}_{S^c}: \mathbb{R}^{d_1 \times d_2} \rightarrow \mathbb{R}^{d_1 \times d_2}$ the orthogonal projections which sets all entries not belonging to $S$ or $S^c$ to zero, respectively.
It follows that 
\begin{align}
	 &\innerproduct{ (\UU_{\XX}^\top ( \XX - \Xzero ) \VV_{\XX} ) \circ ( \UU_{\XX}^\top ( \XX - \Xzero ) \VV_{\XX} ),  \f{H}_{\ssigma, \varepsilon}^{(q)} }\nonumber \\
	 =&
	 \innerproduct{ \mathcal{P}_S \left(  (\UU_{\XX}^\top ( \XX - \Xzero ) \VV_{\XX} ) \circ ( \UU_{\XX}^\top ( \XX - \Xzero ) \VV_{\XX} )\right),  
	 \f{H}_{\ssigma, \varepsilon}^{(q)} }\nonumber\\
	 &+
	 \innerproduct{ \mathcal{P}_{S^c} \left( (\UU_{\XX}^\top ( \XX - \Xzero ) \VV_{\XX} ) \circ 
	 ( \UU_{\XX}^\top ( \XX - \Xzero ) \VV_{\XX} ) \right) ,  \f{H}_{\ssigma, \varepsilon}^{(q)} }\nonumber\\
	 \overeq{(a)}&
	 \bracing{=:(\mathrm{I})}{\innerproduct{ \mathcal{P}_S   (\UU_{\XX}^\top ( \XX - \Xzero ) \VV_{\XX} )  
	 \circ \mathcal{P}_S ( \UU_{\XX}^\top ( \XX - \Xzero ) \VV_{\XX} ),  
	 \f{H}_{\ssigma, \varepsilon}^{(q)} }}\nonumber\\
	 &+
	 \bracing{=:(\mathrm{II})}{\innerproduct{ \mathcal{P}_{S^c}  (\UU_{\XX}^\top ( \XX - \Xzero ) \VV_{\XX} ) 
	 \circ \mathcal{P}_{S^c} ( \UU_{\XX}^\top ( \XX - \Xzero ) \VV_{\XX} )  ,  \f{H}_{\ssigma,\varepsilon}^{(q)} }}. \label{ineq:intern7}
\end{align}
In equality $(a)$, we used 
that 
$ \mathcal{P}_S (\f{A} \circ \f{B}) = \mathcal{P}_S (\f{A}) \circ \mathcal{P}_S (\f{B}) $ 
for all matrices $\f{A}, \f{B} \in \mathbb{R}^{d_1 \times d_2}$.
We bound the two summands individually.

For bounding the summand $(\mathrm{I})$, we note that \Cref{prop:powermeanineq} implies that for any $(i,j) \in S$, it holds that
\begin{equation} \label{ineq:powermeanineq}
	( \f{H}_{\ssigma,\varepsilon}^{(q)} )_{i,j} \le 2^{-1/q} \min \left( \max(\sigma_i (\XX), \varepsilon)^{-1},  \max(\sigma_j (\XX), \varepsilon)^{-1} \right) \leq 2^{-1/q} \sigma_r (\XX)^{-1}.
\end{equation}
From \cref{ineq:powermeanineq}, it follows that
\begin{align*}
	(\mathrm{I})=&\innerproduct{    \mathcal{P}_S ( \UU_{\XX}^\top ( \XX - \Xzero ) \VV_{\XX} )   \circ \mathcal{P}_S   (\UU_{\XX}^\top ( \XX - \Xzero ) \VV_{\XX} ),  
	 \f{H}_{\ssigma, \varepsilon}^{(q)}     } \nonumber \\
	=& \sum_{(i,j) \in S}   (\f{H}_{\ssigma, \varepsilon}^{(q)})_{i,j}     (\UU_{\XX}^\top ( \XX - \Xzero ) \VV_{\XX} )_{i,j}^2 \nonumber \\
	\overleq{\cref{ineq:powermeanineq}} & \frac{2^{-1/q}}{\sigma_r \left(\XX \right)} \sum_{(i,j) \in S} (\UU_{\XX}^\top ( \XX - \Xzero ) \VV_{\XX} )_{i,j}^2 \nonumber \\
	\leq & \frac{2^{-1/q}}{\sigma_r \left(\XX \right)} \sum_{i=1}^{d_1} \sum_{j=1}^{d_2} (\UU_{\XX}^\top ( \XX - \Xzero ) \VV_{\XX} )_{i,j}^2 \nonumber \\
    = & \frac{2^{-1/q}}{\sigma_r \left(\XX \right)} \innerproduct{   \UU_{\XX}^\top ( \XX - \Xzero ) \VV_{\XX}, 
	 \UU_{\XX}^\top ( \XX - \Xzero ) \VV_{\XX}} \nonumber \\
    \le & \frac{2^{-1/q}}{\sigma_r \left(\XX \right)} \specnorm{  \UU_{\XX}^\top ( \XX - \Xzero ) \VV_{\XX}} 
	\nucnorm{  \UU_{\XX}^\top ( \XX - \Xzero ) \VV_{\XX}} \nonumber \\
	=& \frac{2^{-1/q}}{\sigma_r \left(\XX \right)} \specnorm{ \XX - \Xzero } \nucnorm{ \XX - \Xzero  } \nonumber
\end{align*}
Now, note that as $\mathcal{A}$ satisfies the NSP of order $r$ with $0 < \eta_r \leq 3/5$, it also satisfies the NSP of order $1$ with $0 < \eta_1 \leq 3/5$. Thus, due to Weyl's inequality and \Cref{lemma:normbound:nullspace}, we note that 
\begin{equation}\label{ineq:intern5}
\begin{split}
\frac{\specnorm{ \XX - \Xzero  }}{\sigma_r(\XX)} &\leq \frac{\specnorm{ \XX - \Xzero  }}{\sigma_r(\Xzero) - \specnorm{ \XX - \Xzero  }} \leq \frac{\frac{\eta_1}{(1+ \eta_1)\max(\eta_1 \sqrt{d},2)} \sigma_r(\Xzero)}{\sigma_r(\Xzero) - \frac{\eta_1}{(1+ \eta_1)\max(\eta_1 \sqrt{d},2)} \sigma_r(\Xzero)} \\&= \frac{\eta_1}{(1+\eta_1)\max(\eta_1 \sqrt{d},2) -\eta_1} \leq \frac{1}{\max(\eta_1 \sqrt{d},2)},
\end{split}
\end{equation}
using also the assumption \cref{ineq:assump1} in the second inequality, 
while the last inequality follows straightforwardly from the fact that it is equivalent to
\[
\eta_1 \max(\eta_1\sqrt{d},2)
\le
(1+\eta_1)\max(\eta_1\sqrt{d},2)-\eta_1
\quad\Longleftrightarrow\quad
0\le
\max(\eta_1\sqrt{d},2)-\eta_1.
\]

In order to deal with summand $(\mathrm{II})$ of \cref{ineq:intern7}, we compute that
\begin{align}
	 &\innerproduct{  \mathcal{P}_{S^c} ( \UU_{\XX}^\top ( \XX - \Xzero ) \VV_{\XX} )   
	 \circ \mathcal{P}_{S^c}  (\UU_{\XX}^\top ( \XX - \Xzero ) \VV_{\XX} ),  \f{H}_{\ssigma, \varepsilon}^{(q)} }\nonumber \\
	 =& \innerproduct{ \mathcal{P}_{S^c} ( \UU_{\XX}^\top  \XX   \VV_{\XX} ) 
	 \circ\mathcal{P}_{S^c} ( \UU_{\XX}^\top  \XX   \VV_{\XX} ), \f{H}_{\ssigma, \varepsilon}^{(q)}} 
	 - 2 \innerproduct{ \mathcal{P}_{S^c} ( \UU_{\XX}^\top  \XX   \VV_{\XX} )  
	 \circ \mathcal{P}_{S^c} ( \UU_{\XX}^\top   \Xzero  \VV_{\XX} ), \f{H}_{\ssigma, \varepsilon}^{(q)}}\nonumber\\
	 &+ \innerproduct{ \mathcal{P}_{S^c} ( \UU_{\XX}^\top   \Xzero \VV_{\XX} )  
	 \circ \mathcal{P}_{S^c} ( \UU_{\XX}^\top  \Xzero  \VV_{\XX} ), \f{H}_{\ssigma, \varepsilon}^{(q)} }\nonumber\\
	 =& \innerproduct{- \mathcal{P}_{S^c} ( \UU_{\XX}^\top  \XX   \VV_{\XX} ) 
	 \circ\mathcal{P}_{S^c} ( \UU_{\XX}^\top  \XX   \VV_{\XX} ) -2 \mathcal{P}_{S^c} ( \UU_{\XX}^\top  \XX   \VV_{\XX} ) 
	 \circ   \mathcal{P}_{S^c} ( \UU_{\XX}^\top   (\Xzero - \XX )  \VV_{\XX} ), \f{H}_{\ssigma, \varepsilon}^{(q)} } \nonumber \\
	 &+ \innerproduct{ \mathcal{P}_{S^c} ( \UU_{\XX}^\top   \Xzero \VV_{\XX} ) 
	 \circ \mathcal{P}_{S^c} ( \UU_{\XX}^\top  \Xzero  \VV_{\XX} ), \f{H}_{\ssigma, \varepsilon}^{(q)} }\nonumber\\
	 \le&  
	 -  \innerproduct{2 \mathcal{P}_{S^c} ( \UU_{\XX}^\top  \XX   \VV_{\XX} ) 
	 \circ   \mathcal{P}_{S^c} ( \UU_{\XX}^\top   (\Xzero - \XX )  \VV_{\XX} )+ \mathcal{P}_{S^c} ( \UU_{\XX}^\top   \Xzero \VV_{\XX} )  
	 \circ \mathcal{P}_{S^c} ( \UU_{\XX}^\top  \Xzero  \VV_{\XX} ), \f{H}_{\ssigma, \varepsilon}^{(q)} } \nonumber \\
	 =&
	  2 
	  \bracing{=:(i)}{\innerproduct{  \mathcal{P}_{S^c} ( \diag (\ssigma) ) 
	  \circ   \mathcal{P}_{S^c} ( \UU_{\XX}^\top   (\XX - \Xzero)  \VV_{\XX} ), \f{H}_{\ssigma, \varepsilon}^{(q)} }}\nonumber \\
	 &+ 
	 \bracing{=:(ii)}{\innerproduct{ \mathcal{P}_{S^c} ( \UU_{\XX}^\top   \Xzero \VV_{\XX} )
	 \circ \mathcal{P}_{S^c} ( \UU_{\XX}^\top  \Xzero  \VV_{\XX} ) , \f{H}_{\ssigma, \varepsilon}^{(q)} }}. 
	 \label{intern:locallinear1}
\end{align}
We estimate the two terms individually.
For the first term $(i)$,
we obtain that
\begin{align*}
	&\vert \innerproduct{  \mathcal{P}_{S^c} ( \diag (\ssigma) )
	\circ \mathcal{P}_{S^c}  ( \UU_{\XX}^\top (\XX - \Xzero ) \VV_{\XX} ), \f{H}_{\ssigma, \varepsilon}^{(q)} } \vert \\ 
	\le & \vert \innerproduct{ \f{H}_{\ssigma, \varepsilon}^{(q)} \circ \mathcal{P}_{S^c} ( \diag (\ssigma) ), 
	\mathcal{P}_{S^c}  ( \UU_{\XX}^\top (\XX - \Xzero ) \VV_{\XX} )  } \vert\\ 
	\le & \specnorm{\f{H}_{\ssigma, \varepsilon}^{(q)} \circ \mathcal{P}_{S^c} ( \diag (\ssigma) )} \nucnorm{ \mathcal{P}_{S^c}  
	( \UU_{\XX}^\top (\XX - \Xzero ) \VV_{\XX} )} \\
	\le & \specnorm{\f{H}_{\ssigma, \varepsilon}^{(q)} \circ \mathcal{P}_{S^c} ( \diag (\ssigma) )} \nucnorm{ \XX - \Xzero }. 
\end{align*}
Since $ \f{H}_{\ssigma, \varepsilon}^{(q)} \circ \mathcal{P}_{S^c} ( \diag (\ssigma) )$ 
has only non-zero entries on its diagonal it follows that
\begin{equation*}
	\specnorm{\f{H}_{\ssigma, \varepsilon}^{(q)} \circ \mathcal{P}_{S^c} ( \diag (\ssigma) )}
    = \underset{i \in [d]\setminus [r]}{\max} \left( \f{H}_{\ssigma, \varepsilon}^{(q)} \right)_{ (i,i) }  \sigma_i (\XX) \\
    = \underset{i \in [d]\setminus [r]}{\max} \left( \frac{ \sigma_i (\XX) }{\max  \left(  \sigma_i (\XX), \varepsilon \right) } \right) \\
	\le 1.
\end{equation*}
Thus, we have shown that
\begin{equation}
	\vert(i)\vert
	=
	\vert \innerproduct{  \mathcal{P}_{S^c} ( \diag (\ssigma) ) \circ 
	\mathcal{P}_{S^c}  ( \UU_{\XX}^\top (\XX - \Xzero ) \VV_{\XX} ), \f{H}_{\ssigma, \varepsilon} } \vert 
	\le \nucnorm{ \XX -  \Xzero  }.\label{intern:locallinear2}
\end{equation}
For the second summand $(ii)$ in \eqref{intern:locallinear1}, we note that 
\begin{align}
	 &\innerproduct{  \mathcal{P}_{S^c} ( \UU_{\XX}^\top  \Xzero \VV_{\XX} ) 
	 \circ \mathcal{P}_{S^c}  (\UU_{\XX}^\top \Xzero  \VV_{\XX} ) ,  \f{H}_{\ssigma, \varepsilon}  } \nonumber \\
	 =
	 &\innerproduct{  \mathcal{P}_{S^c} ( \UU_{\XX}^\top  \Xzero \VV_{\XX} )  ,  \f{H}_{\ssigma, \varepsilon} 
	 \circ \mathcal{P}_{S^c}  (\UU_{\XX}^\top \Xzero  \VV_{\XX} ) } \nonumber \\
	 \le 
	 &\nucnorm{ \mathcal{P}_{S^c} ( \UU_{\XX}^\top  \Xzero \VV_{\XX} ) }
	 \specnorm{\f{H}_{\ssigma, \varepsilon} \circ \mathcal{P}_{S^c}  (\UU_{\XX}^\top \Xzero  \VV_{\XX} )} \nonumber \\
	 =
	 &\nucnorm{ \mathcal{P}_{S^c} ( \UU_{\XX}^\top  \Xzero \VV_{\XX} ) }
	 \specnorm{\mathcal{P}_{S^c} (\f{H}_{\ssigma, \varepsilon}  ) \circ \mathcal{P}_{S^c}  (\UU_{\XX}^\top \Xzero  \VV_{\XX} )} \nonumber \\
	 \overleq{(a)} 
	 &\left( \underset{i \in \{r+1,\ldots,d\} }{\max}  \max(\sigma_i(\XX),\varepsilon)^{-1} \right)
	 \nucnorm{ \mathcal{P}_{S^c} ( \UU_{\XX}^\top  \Xzero \VV_{\XX} ) }
	 \specnorm{ \mathcal{P}_{S^c}  (\UU_{\XX}^\top \Xzero  \VV_{\XX} )} \nonumber \\
	 \le 
	 &\frac{1}{\varepsilon}
	 \nucnorm{ \mathcal{P}_{S^c} ( \UU_{\XX}^\top  \Xzero \VV_{\XX} ) }
	 \specnorm{ \mathcal{P}_{S^c}  (\UU_{\XX}^\top \Xzero  \VV_{\XX} )} \nonumber \\
	 \overeq{(b)} &\frac{1}{\varepsilon}
	 \nucnorm{ (\UU_{\XX,r+1:d_1})^\top \UUstar \SSigma_{\star} \VVstar^\top  \VV_{\XX,r+1:d_2} }
	 \specnorm{   (\UU_{\XX,r+1:d_1})^\top \UUstar \SSigma_{\star} \VVstar^\top  \VV_{\XX,r+1:d_2} } \nonumber \\
	 \le &\frac{1}{\varepsilon}
	 \nucnorm{ (\UU_{\XX,r+1:d_1})^\top \UUstar \SSigma_{\star}} 
	 \specnorm{   (\UU_{\XX,r+1:d_1})^\top \UUstar \SSigma_{\star} } 
	 \specnorm{ \VVstar^\top  \VV_{\XX,r+1:d_2} }^2. \label{ineq:intern3}
\end{align}
Inequality $(a)$ follows from \Cref{lemma:normbound:hadamard}, and in equation $(b)$, we used the decomposition
$\UU = [\UU_{\XX,1:r} \ \UU_{\XX,r+1:d_1}]$ and $\VV = [\VV_{\XX,1:r} \ \VV_{\XX,r+1:d_2}]$.
In order to proceed we denote by $\XX_r$ the best rank-$r$ approximation of $\XX$. 
It follows that
\begin{align*}
	 \nucnorm{ (\UU_{\XX,r+1:d_1})^\top \UUstar \SSigma_{\star}} 
	 &=
	 \nucnorm{(\UU_{\XX,r+1:d_1})^\top \Xzero }
	 =
	 \nucnorm{(\UU_{\XX,r+1:d_1})^\top (\Xzero - \XX_r ) }\\
	 &\le 
	 \nucnorm{ \Xzero - \XX_r  }
	 \le
	 \nucnorm{\Xzero - \XX} +
	 \besterrNuc{\XX}{r},
\end{align*}
where we recall that by definition
$ \besterrNuc{\XX}{r} = \nucnorm{\XX - \XX_r} $.
Note that due to the reverse triangle inequality, see Lemma \ref{lemma:NSPl1min}, it follows that
\begin{equation*}
	\nucnorm{\Xzero - \XX}
	\le
	\frac{1+\eta_r}{1-\eta_r} \left( \nucnorm{\Xzero} - \nucnorm{\XX} + 2 \besterrNuc{\XX}{r} \right)
	\le 
	\frac{2(1+\eta_r)}{1-\eta_r} \besterrNuc{\XX}{r},
\end{equation*}
where in the second inequality we have used that $\nucnorm{\Xzero} \le \nucnorm{\XX}$
which follows from the fact that $\Xzero$ is the nuclear norm minimizer. 
Combining the last two inequality chains we obtain that 
\begin{equation*}
	 \nucnorm{ \Up^\top \UUstar \SSigma_{\star}} 
     \le
	D_{\eta_r} \besterrNuc{\XX}{r},
\end{equation*}
using the definition of $\Up = \UU_{\XX,r+1:d_1}$.
This inequality also implies that
\begin{equation*}
	 \specnorm{ \Up^\top \UUstar \SSigma_{\star} }
	 \le \nucnorm{ \Up^\top \UUstar \SSigma_{\star}} 
	 \le D_{\eta_r} \beta_{r} ( \XX )_{\ast}
	 \le D_{\eta_r} \nucnorm{\XX - \Xzero},
\end{equation*}
where in the last inequality we used the Eckart--Young theorem.
Moreover, it follows from Wedin's $\sin\theta$ theorem 
\citep[see, e.g.,][]{Wedin72,SpectralMethods_ChenChiFanMa2021}
that the inner product matrix of $\VVstar$ with $\Vp = \VV_{\XX,r+1:d_2}$ can be bounded in spectral norm such that
\begin{align*}
	\specnorm{ \VVstar^\top  \Vp }
	&\le
	\frac{ \specnorm{ \XX - \Xzero }}{ \sigma_r (\Xzero) - \specnorm{ \XX - \Xzero } }\\
	&\overleq{(a)}
	\frac{2(1+\eta_1)}{2+\eta_1}
	\frac{\specnorm{ \XX - \Xzero } }{ \sigma_r (\Xzero) }\\
	&\overleq{(b)}
	\frac{2\eta_1}{2+\eta_1}
	\frac{\nucnorm{ \XX - \Xzero } }{ \sigma_r (\Xzero) }.
\end{align*}
For inequality $(a)$, set
\[
\alpha
:=
\frac{\specnorm{\XX-\Xzero}}{\sigma_r(\Xzero)}.
\]
Since $\XX-\Xzero \in \ker(\mathcal{A})$, \Cref{lemma:normbound:nullspace} and assumption \eqref{ineq:assump1} imply
\[
\alpha
\le
\frac{\eta_1}{(1+\eta_1)\max(\eta_1\sqrt{d},2)}
\le
\frac{\eta_1}{2(1+\eta_1)}.
\]
Therefore,
\[
\frac{1}{1-\alpha}
\le
\frac{1}{1-\frac{\eta_1}{2(1+\eta_1)}}
=
\frac{2(1+\eta_1)}{2+\eta_1},
\]
and consequently
\[
\frac{ \specnorm{ \XX - \Xzero }}{\sigma_r(\Xzero)-\specnorm{ \XX - \Xzero }}
=
\frac{\alpha}{1-\alpha}
\le
\frac{2(1+\eta_1)}{2+\eta_1}
\frac{\specnorm{ \XX - \Xzero }}{\sigma_r(\Xzero)}.
\]
This proves $(a)$. Inequality $(b)$ follows again from \Cref{lemma:normbound:nullspace}.
Inserting the last three inequalities into \eqref{ineq:intern3} we obtain that
\begin{align*}
	\innerproduct{  \mathcal{P}_{S^c} ( \UU_{\XX}^\top  \Xzero \VV_{\XX} ) 
	\circ \mathcal{P}_{S^c}  (\UU_{\XX}^\top \Xzero  \VV_{\XX} ) ,  \f{H}_{\ssigma,\varepsilon}  }
	&\le 
    \frac{4D_{\eta_r}^2\eta_1^2 \besterrNuc{\XX}{r}^2\nucnorm{ \XX - \Xzero }^2 }{(2+\eta_1)^2\varepsilon \cdot  \sigma_r (\Xzero)^2}\\
	&\overleq{(a)}
    \frac{4D_{\eta_r}^2 d \eta_1^2 \nucnorm{ \XX - \Xzero }^3 }{ (2+\eta_1)^2\vartheta \sigma_r (\Xzero)^2},
\end{align*}
where we used that by assumption $\varepsilon \ge \vartheta \besterrNuc{\XX}{r}/d$ and the Eckart--Young theorem in the last inequality.
Using assumption \eqref{ineq:assump1} we have
\begin{equation*}
\frac{d\eta_1^2\nucnorm{\XX-\Xzero}^2}{\sigma_r(\Xzero)^2} \le 1,
\end{equation*}
because either $\eta_1\sqrt d \ge 2$, in which case $\nucnorm{\XX - \Xzero} \le \sigma_r(\Xzero)/(\eta_1\sqrt d)$, or $\eta_1\sqrt d < 2$, in which case $d\eta_1^2<4$ and $\nucnorm{\XX-\Xzero}\le \sigma_r(\Xzero)/2$.
Hence, we obtain that
\begin{equation}
	(ii)
	=
	\innerproduct{  \mathcal{P}_{S^c} ( \UU_{\XX}^\top  \Xzero \VV_{\XX} ) 
	\circ \mathcal{P}_{S^c}  (\UU_{\XX}^\top \Xzero  \VV_{\XX} ) ,  \f{H}_{\ssigma, \varepsilon}  }
	\le
    \frac{4D_{\eta_r}^2}{(2+\eta_1)^2\vartheta} \nucnorm{\XX - \Xzero }.
	\label{intern:locallinear3}
\end{equation}
Inserting inequalities \eqref{intern:locallinear2} and \eqref{intern:locallinear3} into inequality \eqref{intern:locallinear1}, we obtain for summand $(\mathrm{II})$ that
\begin{equation}\label{ineq:intern6}
\innerproduct{  \mathcal{P}_{S^c} ( \UU_{\XX}^\top ( \XX - \Xzero ) \VV_{\XX} )   
\circ \mathcal{P}_{S^c}  (\UU_{\XX}^\top ( \XX - \Xzero ) \VV_{\XX} ),  \f{H}_{\ssigma, \varepsilon} }
\le  \left(2+\frac{D_{\eta_r}^2}{\vartheta}\right) \nucnorm{\XX - \Xzero}.
\end{equation}
Inserting our estimates for summand $(\mathrm{I})$ and summand $(\mathrm{II})$, see inequalities \eqref{ineq:intern5} and \eqref{ineq:intern6}, into \eqref{ineq:intern7}, we obtain that
\begin{equation*}
	\innerproduct{ \XX - \Xzero ,\W{\XX}{\varepsilon}( \XX - \Xzero ) }
	\le
	\left(2 + 2^{-1/q}+\frac{D_{\eta_r}^2}{\vartheta}\right) \nucnorm{ \XX - \Xzero }.
\end{equation*}
This completes the proof of the lemma.
\end{proof}
Recall that by definition of $\varepsilon_k$ we have that $\varepsilon_k \le \frac{\besterrNuc{\XXk}{r}}{d}$.
The following technical lemma gives an explicit lower bound for $\varepsilon_k$.
In particular, \Cref{lemma:epslowerbound} verifies the condition on $\varepsilon_k$ in \Cref{lemma:quadratictermlocal} with an explicit $\eta_r$-dependent constant as used in the bound \cref{ineq:claim11}.
\begin{lemma}\label{lemma:epslowerbound}
    Assume that the linear measurement operator $\mathcal{A}: \mathbb{R}^{d_1 \times d_2} \rightarrow \mathbb{R}^m$ satisfies the NSP of order~$r$ with constant $ \eta_r < 1 $.
	Moreover, assume that $\Xzero$ has rank~$r$.
    Let $\left(  \XXk\right)_{k\geq 0}$ and $ \left(\varepsilon_{k} \right)_{k \geq 0}$ be the iterates and smoothing parameters of $\texttt{MatrixIRLS}$ with input $\f{y} = \mathcal{A}(\Xzero)$, arbitrary initial weight operator $\f{W}^{(0)}$, $\widetilde{r} = r$, based on $q$-power mean weight operators with $q \in [-1,\infty]$ or on one-sided weight operators $\W{\XXk}{\varepsilon_k}(\cdot)$.
	Then, for all natural numbers $k \ge 1$, it holds that 
	\begin{equation*}
		\epsk \ge  \vartheta \frac{\besterrNuc{\XXk}{r}}{d},
	\end{equation*}
	where
	\begin{equation*}
	\vartheta
	:=
	\frac{1-\eta_r}{\left(1+\eta_r\right)\left(\frac{3}{2}+\eta_r\right)}.
	\end{equation*}
\end{lemma}

\begin{proof}
Recall that by definition $ \epsk= \min \left( \varepsilon_{k-1} ,\frac{\besterrNuc{\XXk}{r}}{d} \right)$.
Choose $\ell \le k$ such that $ \epsk = \epsl = \frac{\besterrNuc{\Xk{(\ell)}}{r}}{d} $.
Then we obtain that 
\begin{align*}
	\epsk 
	&= \frac{\besterrNuc{\Xk{(\ell)}}{r}}{d}\\
	&\overgeq{(a)}  \frac{ \mathcal{J}_{\epsl} (\XX^{(\ell)})  - \nucnorm{\Xzero} }{\left(\frac{3}{2}+\eta_r\right)d}\\
	&\overgeq{(b)} \frac{ \mathcal{J}_{\epsk} (\XX^{(k)})  - \nucnorm{\Xzero} }{\left(\frac{3}{2}+\eta_r\right)d}\\
	&\overgeq{(c)} \frac{1-\eta_r}{\left(1+\eta_r\right)\left(\frac{3}{2}+\eta_r\right)d} \nucnorm{\XXk -\Xzero}\\
	&\overgeq{(d)} \frac{1-\eta_r}{\left(1+\eta_r\right)\left(\frac{3}{2}+\eta_r\right)d} \besterrNuc{\XXk}{r}.
\end{align*}
Inequality $(a)$ follows from Lemma \ref{lemma:epscontrol} and inequality $(b)$ follows from the monotonicity of the sequence $ \left\{  \mathcal{J}_{\varepsilon_k} \left( \XX^{(k)} \right) \right\}_{k \ge 1} $.
For inequality $(c)$ we again used Lemma \ref{lemma:epscontrol}.
Inequality $(d)$ follows from the Eckart--Young theorem since $\Xzero$ has rank $r$.
\end{proof}

\subsubsection{{Proof of \Cref{thm:locallinearp1} (Dimension-Free Fast Linear Rate of \MatrixIRLSHeading{})}} \label{sec:fastlocalrate:harmonicmean:proof}
Now we have all ingredients in place to prove the key result \Cref{prop:p1:locallinearrate}, which shows that $ \mathcal{J}_{\varepsilon_k} \left( \XXk \right) $ decreases linearly at a dimension-free rate when the iterate $\XXk$ is close enough to the true solution $\Xzero$, before providing the complete proof of \Cref{thm:locallinearp1}. As discussed in \Cref{sec:local:linear:convergence}, we can prove suitable statements for harmonic-mean weight operators and for $q$-power mean weight operators with $q \in [-1,0)$.

\begin{proposition}[One-Step Local Contraction]\label{prop:p1:locallinearrate}
	Let $\Xzero \in \mathbb{R}^{d_1 \times d_2}$ have rank $r$.
	Assume that the measurement operator 
	$\mathcal{A}: \mathbb{R}^{d_1 \times d_2} \longrightarrow \R^m $ 
	satisfies the NSP of order $r$ with constant $0<\eta_r<3/5$
	and that 
	$ \f{y} = \mathcal{A} \left(  \Xzero \right) $.
	Let the IRLS iterates  
	$\left\{   \XXk \right\}_k$ and $ \left\{ \varepsilon_{k} \right\}_k$ 
	be defined by \cref{eq:IRLS:step_1} and \cref{eq:IRLS:step_2} 
	with rank estimate $\widetilde r=r$ and with fixed $q$-power mean weight operators for some $q\in[-1,0)$, and assume that \Cref{alg:algo1} does not return in iteration $k$, i.e., that $\varepsilon_k > 0$, so that the iterate $\XXkplus$ is defined.
	Assume that
	\begin{equation}\label{intern:localconvergence1}
	    \nucnorm{\XXk - \Xzero}  
		\le  
		\frac{  \sigma_r ( \Xzero ) }{\max \left( \eta_1 \sqrt{d}, 2 \right)}. 
    \end{equation}
	where $0 < \eta_1 \leq \eta_r$ denotes again the order-one NSP constant of $\mathcal{A}$. Set
	\begin{equation*}
	K_{\eta_r,q}
	:=
	2+2^{-1/q}+
	\frac{(3+\eta_r)^2(1+\eta_r)\left(\frac{3}{2}+\eta_r\right)}{(1-\eta_r)^3}
	\end{equation*}
	and
	\begin{equation*}
	\widehat c_{\eta_r,q}
	:=
	\frac{(3-5\eta_r)^{2}}{16(1+\eta_r)^{2}(3+2\eta_r)K_{\eta_r,q}}.
	\end{equation*}
	Then it holds that
	\begin{equation}
		\mathcal{J}_{\varepsilon_{k+1}}( \XXkplus ) -  \nucnorm{\Xzero} 
	\le 
    \left(
	1-  \widehat c_{\eta_r,q} \right)
	\left(     \mathcal{J}_{\varepsilon_{k}} \left( \XXk  \right) -  \nucnorm{ \Xzero } \right).
	\end{equation}
	In particular, for the harmonic-mean weight operator, i.e., $q=-1$, this implies the same estimate with
	\begin{equation*}
	c_{\eta_r}
	=
	\frac{(3-5\eta_r)^{2}(1-\eta_r)^{3}}{8(1+\eta_r)^{2}(3+2\eta_r)\Bigl[8(1-\eta_r)^{3}+(3+\eta_r)^{2}(1+\eta_r)(3+2\eta_r)\Bigr]}
	\end{equation*}
	in place of $\widehat c_{\eta_r,-1}$.
\end{proposition}

\begin{proof}
The proof of this proposition is structurally similar to the proof of \Cref{prop:p1:linearrate}.
Again, we define $ \NNk := \Xzero - \XXk $.
By the monotonicity of $\varepsilon \mapsto \mathcal{J}_{\varepsilon}(\XXkplus)$, the majorization property for $q$-power mean weights with $q\ge -1$, and the optimality of $\XXkplus$ in \cref{eq:IRLS:step_1}, for any $t \in \mathbb{R}$ we have that
\begin{equation}\label{ineq:intern51}
\mathcal{J}_{\varepsilon_{k+1}}(\XXkplus) \leq Q_{\varepsilon_k} (\XXkplus \mid \XXk) \leq Q_{\varepsilon_k}(\XXk+t \NNk \mid \XXk).
\end{equation}
Moreover, by the definition of the quadratic objective $Q_{\varepsilon_k}(\cdot \mid \XXk)$, see \cref{eq:smoothedell1:IRLSmajorizer}, it holds that
\begin{equation*}
\begin{split}
&Q_{\varepsilon_k}(\XXk+t \NNk,\XXk) - \mathcal{J}_{\varepsilon_k}(\XXk)\\ 
=&  t \, \innerproduct{ \nabla \mathcal{J}_{\varepsilon_k}(\XXk), \NNk  } 
\end{split}
\end{equation*}	
As in the proof of \Cref{prop:p1:linearrate}, our goal is to minimize the difference 
\begin{equation*}
Q_{\varepsilon_k}(\XXk+t \NNk \mid \XXk) - \mathcal{J}_{\varepsilon_k}(\XXk)
\end{equation*}
by choosing $t$ accordingly.
Note that the assumptions of Lemma \ref{lemma:quadratictermlocal} are fulfilled due to assumption \eqref{intern:localconvergence1} and Lemma \ref{lemma:epslowerbound}, with the choice
\[
\vartheta_{\eta_r}:=
\frac{1-\eta_r}{\left(1+\eta_r\right)\left(\frac{3}{2}+\eta_r\right)}.
\]
Set
\begin{equation*}
\delta_{\eta_r}
:=
\frac{1-\eta_r}{1+\eta_r} - \frac{1}{4}.
\end{equation*}
Thus, we apply Lemma \ref{lemma:linearterm} and Lemma \ref{lemma:quadratictermlocal} and obtain that
\begin{align*}
Q_{\varepsilon_k}(\XXk+t \NNk \mid \XXk) - \mathcal{J}_{\varepsilon_k}(\XXk)  
&\le - t \bracing{=:b}{\delta_{\eta_r}\nucnorm{\NNk}} + \bracing{=:a}{\frac{K_{\eta_r,q}}{2} \nucnorm{ \NNk }} \cdot t^2.
\end{align*}
The right-hand side is minimized by setting $t:= \frac{b}{2a}$. This yields that
\begin{align*}
Q_{\varepsilon_k}(\XXk+t \NNk \mid \XXk) - \mathcal{J}_{\varepsilon_k}(\XXk)  
\le 
\frac{-b^2}{4a}
=
\frac{- \delta_{\eta_r}^2 \nucnorm{\NNk} }{2K_{\eta_r,q} }.
\end{align*}
Combining this estimate with inequality \eqref{ineq:intern51} it follows that 
\begin{equation*}
\mathcal{J}_{\varepsilon_{k+1}}(\XXkplus)    
- \nucnorm{\Xzero}
\le 
\mathcal{J}_{\varepsilon_k}(\XXk)
-\nucnorm{\Xzero}
-
\frac{ \delta_{\eta_r}^2 \nucnorm{\Xzero - \XXk} }{2K_{\eta_r,q}}.
\end{equation*}
Now note that Lemma \ref{lemma:epscontrol} and the fact that $\Xzero$ has rank $r$ imply
\begin{equation*}
	\mathcal{J}_{\varepsilon_k} (\XXk) - \nucnorm{\Xzero}
	\le
	\left(\frac{3}{2}+\eta_r\right)\besterrNuc{\XXk}{r}
	\le
	\left(\frac{3}{2}+\eta_r\right)\nucnorm{\XXk-\Xzero}.
\end{equation*}
Equivalently,
\[
\nucnorm{\Xzero-\XXk}
\ge
\frac{\mathcal{J}_{\varepsilon_k} (\XXk) - \nucnorm{\Xzero}}{\frac{3}{2}+\eta_r}.
\]
Inserting this into the above inequality, we obtain that
\begin{equation*}
\mathcal{J}_{\varepsilon_{k+1}}(\XXkplus)    
- \nucnorm{\Xzero}
\le 
\left(
	1- \frac{ \delta_{\eta_r}^2 }{2K_{\eta_r,q}\left(\frac{3}{2}+\eta_r\right) }
\right)
\left(
\mathcal{J}_{\varepsilon_k}(\XXk)-\nucnorm{\Xzero}
\right).	
\end{equation*}
It remains to verify that the bracketed contraction factor equals $1-\widehat c_{\eta_r,q}$. Combining the two fractions inside $\delta_{\eta_r}$,
\[
\delta_{\eta_r}
=
\frac{1-\eta_r}{1+\eta_r}-\frac{1}{4}
=
\frac{4(1-\eta_r)-(1+\eta_r)}{4(1+\eta_r)}
=
\frac{3-5\eta_r}{4(1+\eta_r)},
\]
so that $\delta_{\eta_r}^{2}=\frac{(3-5\eta_r)^{2}}{16(1+\eta_r)^{2}}$. Writing $\frac{3}{2}+\eta_r=\frac{3+2\eta_r}{2}$ cancels the factor $2$ in front of $K_{\eta_r,q}$, and we arrive at
\[
\frac{\delta_{\eta_r}^{2}}{2K_{\eta_r,q}\left(\frac{3}{2}+\eta_r\right)}
=
\frac{(3-5\eta_r)^{2}}{16(1+\eta_r)^{2}(3+2\eta_r)\,K_{\eta_r,q}}
=
\widehat c_{\eta_r,q},
\]
which establishes the asserted contraction.

For the harmonic-mean weight operator $q=-1$, we have $2^{-1/q}=2$ and hence
\[
K_{\eta_r,-1}
=
4+\frac{(3+\eta_r)^{2}(1+\eta_r)\left(\frac{3}{2}+\eta_r\right)}{(1-\eta_r)^{3}}
=
\frac{8(1-\eta_r)^{3}+(3+\eta_r)^{2}(1+\eta_r)(3+2\eta_r)}{2(1-\eta_r)^{3}},
\]
where the second equality follows by putting both summands over the common denominator $(1-\eta_r)^{3}$ and using $\frac{3}{2}+\eta_r=\frac{3+2\eta_r}{2}$. Inserting this into $\widehat c_{\eta_r,-1}$ yields
\[
\widehat c_{\eta_r,-1}
=
\frac{(3-5\eta_r)^{2}\,(1-\eta_r)^{3}}{8\,(1+\eta_r)^{2}\,(3+2\eta_r)\bigl[8(1-\eta_r)^{3}+(3+\eta_r)^{2}(1+\eta_r)(3+2\eta_r)\bigr]}
=
c_{\eta_r},
\]
which is the displayed harmonic-mean form. This completes the proof.
\end{proof}

With \Cref{prop:p1:locallinearrate} in place, we can prove the main result regarding local linear convergence,  \Cref{thm:locallinearp1}, with a dimension-free rate, for \texttt{MatrixIRLS} using harmonic-mean weight operators. 
\begin{proof}[Proof of \Cref{thm:locallinearp1}]
We will prove this theorem by induction.
Recall that
\begin{equation*}
A_{\eta_r}
:=
\frac{\left(\frac{3}{2}+\eta_r\right)\left(1+\eta_r\right)}{1-\eta_r}.
\end{equation*}
Since $\eta_r<3/5$, we have
\begin{equation}\label{ineq:Aetar:less:nine}
	A_{\eta_r}
	<
	\frac{\left(\frac{3}{2}+\frac{3}{5}\right)\left(1+\frac{3}{5}\right)}{1-\frac{3}{5}}
	=
	\frac{42}{5}
	<9.
\end{equation}
Moreover, the displayed formula for $c_{\eta_r}$ gives $0<c_{\eta_r}<1$.
We immediately observe that in the base case $k= \tilde{k}$
both inequalities \eqref{equ:locallinearconvergence1} and \eqref{ineq:localconvergenerate1} hold.

Now assume that the statement holds for some $k \ge \tilde{k}$ for which \Cref{alg:algo1} carries out iteration $k+1$, i.e., for which $\varepsilon_k > 0$.
To show the induction step we first note that due to assumption \eqref{assump:localconvergencerate}, inequality \eqref{ineq:localconvergenerate1}, \eqref{ineq:Aetar:less:nine}, and $(1-c_{\eta_r})^{k-\tilde k}\le 1$, we have that 
\begin{equation*}
	\nucnorm{\XXk - \Xzero}
	\le
	A_{\eta_r}\left(1-c_{\eta_r}\right)^{k-\tilde k}
	\frac{\sigma_r(\Xzero)}{9\max  \big( \eta_1 \sqrt{d}, 2\big)}
	\le
	\frac{\sigma_r (\Xzero ) }{\max  \big( \eta_1 \sqrt{d}, 2\big) }.
\end{equation*}
Thus, we can apply \Cref{prop:p1:locallinearrate} with $q=-1$ and obtain that
\begin{equation} \label{ineq:intern59}
	\begin{split}
		\mathcal{J}_{\varepsilon_{k+1}}( \XXkplus ) -  \nucnorm{\Xzero} 
	&\le 
    \left(
	1-c_{\eta_r} \right)
	\left(     \mathcal{J}_{\varepsilon_{k}} \left( \XXk  \right) -  \nucnorm{ \Xzero } \right)\\
	&\le 
    \left(
	1-c_{\eta_r} 
	\right)^{k+1- \tilde{k}}
	\left(     \mathcal{J}_{\varepsilon_{\tilde{k}}} \left( \XX^{(\tilde{k})}  \right) -  \nucnorm{ \Xzero } \right)
	\end{split}
\end{equation}
In the second inequality we used the induction hypothesis.
This proves inequality \eqref{equ:locallinearconvergence1} for $k+1$.
It remains to prove inequality \eqref{ineq:localconvergenerate1}. 
For that, we note that 
\begin{align*}
	\nucnorm{\XX^{(k+1)} -\Xzero}
	&\overleq{(a)}
	\frac{1+\eta_r}{1-\eta_r} \left( \mathcal{J}_{\varepsilon_{k+1}} (\XXkplus) - \nucnorm{\Xzero} \right)\\
	&\overleq{(b)}
	\frac{1+\eta_r}{1-\eta_r} \left( 1- c_{\eta_r} \right)^{k+1-\tilde{k}} \left( \mathcal{J}_{\varepsilon_{\tilde{k}}} (\XX^{(\tilde{k})}) - \nucnorm{\Xzero} \right)\\
	&\overleq{(c)}
	A_{\eta_r} \left( 1- c_{\eta_r} \right)^{k+1-\tilde{k}} \besterrNuc{\XX^{(\tilde{k})}}{r} \\
	&\overleq{(d)}
	A_{\eta_r} \left( 1- c_{\eta_r} \right)^{k+1-\tilde{k}} \nucnorm{ \XX^{(\tilde{k})} -\Xzero}.
\end{align*}
For inequality $(a)$ and inequality $(c)$, we used \Cref{lemma:epscontrol}, whereas inequality $(b)$ is due to \eqref{ineq:intern59}.
Finally, inequality $(d)$ is a consequence of the Eckart--Young theorem and the fact that the matrix $\Xzero$ has rank $r$.
This proves inequality \eqref{ineq:localconvergenerate1} for $k+1$.
Thus, the proof is complete.
\end{proof}

\subsubsection{Proof of \Cref{thm:counterexample:leftsided:weight:operator} (Counterexample for One-Sided IRLS)}
\label{sec:counterexample:leftsided:weight:operator}
In this section, we provide an explicit counterexample substantiating \Cref{thm:counterexample:leftsided:weight:operator}, which shows that it is in general not possible to obtain a dimension-independent upper bound on $\innerproduct{ \Xzero - \XX , \W{\XX}{\varepsilon} (\Xzero - \XX  ) }$ if a one-sided weight operator such as \cref{eq:leftsided:mean,eq:rightsided:mean} is used. Leveraging a lower bound on $\innerproduct{ \XX - \Xzero , \W{\XX}{\varepsilon} ( \XX - \Xzero ) }$, we show that in this example, the nuclear norm error $\nucnorm{ \XX^+ - \Xzero}$ of the next IRLS iterate $\XX^+$ cannot decrease by more than a factor of $(1- c r /d)$ for some constant $c>0$ compared to the nuclear norm error before the IRLS step.

We provide the argument for left-sided weight operators and note that it can be easily adapted to right-sided weight operators by transposition of the underlying matrices.

\begin{proof}[Proof of \Cref{thm:counterexample:leftsided:weight:operator}]
	We give an explicit construction. 
	Let $d \ge 220 r$ and set $h:=10r$.
	Let $\left\{\f{e}_{1},\ldots,\f{e}_{d}\right\}$ and
	$\left\{\f{f}_{1},\ldots,\f{f}_{d}\right\}$ be orthonormal bases of
	$\R^d$.
	Fix $\alpha>0$. We choose an angle $\theta>0$ satisfying
	\begin{equation}\label{eq:counterexample:theta:choice}
	0<\theta\le
	\min\left(
	\frac{\pi}{2},
	\frac{1}{11r\max\left(\frac{\sqrt d}{11r-1},2\right)}
	\right).
	\end{equation}
	For $i=1,\ldots,r$, define
	\[
	\f{u}_{i}:=\cos(\theta)\f{e}_{i}+\sin(\theta)\f{e}_{r+i},
	\qquad
	\f{p}_{i}:=-\sin(\theta)\f{e}_{i}+\cos(\theta)\f{e}_{r+i}.
	\]
	Set $a:=2\alpha\sin(\theta/2)$ and define
	\[
	\Xzero:=\alpha\sum_{i=1}^{r}\f{e}_{i}\f{f}_{i}^{\top},
	\qquad
	\XX:=
	\alpha\sum_{i=1}^{r}\f{u}_{i}\f{f}_{i}^{\top}
	+
	a\sum_{j=1}^{h}\f{e}_{2r+j}\f{f}_{r+j}^{\top}.
	\]
	Since $d\ge 220r$, all indices used in this definition are at most $d$.
	Moreover, \eqref{eq:counterexample:theta:choice} implies
	$a\le \alpha\theta\le \alpha/(22r)<\alpha$, so the nonzero singular values of
	$\XX$ are $\alpha$ repeated $r$ times and $a$ repeated $h=10r$ times.  Thus
	\begin{equation}\label{eq:counterexample:betaeps}
	\besterrNuc{\XX}{r}=10ra,
	\qquad
	\varepsilon=\frac{\besterrNuc{\XX}{r}}{d}=\frac{10ra}{d}.
	\end{equation}
	
	Let $\DDelta:=\XX-\Xzero$.  Then
	\[
	\DDelta=
	\alpha\sum_{i=1}^{r}\left(\f{u}_{i}-\f{e}_{i}\right)\f{f}_{i}^{\top}
	+
	a\sum_{j=1}^{h}\f{e}_{2r+j}\f{f}_{r+j}^{\top}.
	\]
	The rank-one terms in this decomposition are mutually orthogonal on both the
	left and the right, and
	\[
	\left\|\alpha\left(\f{u}_{i}-\f{e}_{i}\right)\right\|_{2}
	=2\alpha\sin(\theta/2)=a.
	\]
	Consequently, the singular values of $\DDelta$ are $a$ repeated $r+h=11r$
	times, and hence
	\begin{equation}\label{eq:counterexample:singularvalues:Delta}
	\sum_{i=1}^{r}\sigma_i(\DDelta)=ra,
	\qquad
	\sum_{i=r+1}^{d}\sigma_i(\DDelta)=10ra,
	\qquad
	\nucnorm{\DDelta}=11ra.
	\end{equation}
	
	We now define the measurement operator. Let
	$\f{G}_{1},\ldots,\f{G}_{d^{2}-1}$ be an orthonormal basis of the
	Frobenius-orthogonal complement of $\DDelta$ in $\R^{d\times d}$, and set
	\[
	\mathcal{A}(\f{Z})
	:=
	\left(\innerproduct{\f{Z},\f{G}_{\ell}}\right)_{\ell=1}^{d^2-1}.
	\]
	Then $\ker(\mathcal{A})=\operatorname{span}\{\DDelta\}$ and, in particular,
	$\XX-\Xzero=\DDelta\in\ker(\mathcal{A})$.  For every nonzero
	$\NN\in\ker(\mathcal{A})$, we have $\NN=t\DDelta$ for some $t\ne0$.
	Thus \eqref{eq:counterexample:singularvalues:Delta} gives
	\[
	\sum_{i=1}^{r}\sigma_i(\NN)
	=
	\frac{1}{10}\sum_{i=r+1}^{d}\sigma_i(\NN),
	\]
	so $\mathcal{A}$ satisfies the NSP of order $r$ with constant $\eta_r=1/10$.
	The same singular value calculation shows that $\mathcal{A}$ satisfies the NSP
	of order $1$ with constant
	\[
	\eta_1=\frac{1}{11r-1},
	\]
	because $\sigma_1(\DDelta)=a$ and
	$\sum_{i=2}^{d}\sigma_i(\DDelta)=(11r-1)a$.
	
	It remains to check the local neighborhood condition.  Since
	$\sigma_r(\Xzero)=\alpha$, \eqref{eq:counterexample:theta:choice} implies
	\[
	\nucnorm{\XX-\Xzero}
	=11ra
	\le 11r\alpha\theta
	\le
	\frac{\alpha}{\max\left(\frac{\sqrt d}{11r-1},2\right)}
	=
	\frac{\sigma_r(\Xzero)}
	{\max\left(\eta_1\sqrt d,2\right)}.
	\]
	
	\medskip
	\noindent\textbf{Lower bound on the left-sided quadratic form.}
	Let $\W{\XX}{\varepsilon}(\cdot)$ be the weight operator with the left-sided
	core matrix \eqref{eq:leftsided:mean}.  The vectors $\f{p}_1,\ldots,\f{p}_r$
	are orthogonal to the column space of $\XX$; hence they can be chosen as part of
	the zero left-singular-vector block in a full SVD of $\XX$.  On these rows the
	left-sided weight is equal to $1/\varepsilon$.  Furthermore,
	\[
	\f{p}_{i}^{\top}\DDelta\f{f}_{i}
	=
	\alpha\f{p}_{i}^{\top}\left(\f{u}_{i}-\f{e}_{i}\right)
	=
	\alpha\sin\theta
	=
	a\cos(\theta/2).
	\]
	Writing $\XX=\UU_{\XX}\diag(\ssigma)\VV_{\XX}^{\top}$ and
	$\f{M}:=\UU_{\XX}^{\top}\DDelta\VV_{\XX}$, the left-sided definition gives
	\[
	\innerproduct{\DDelta,\W{\XX}{\varepsilon}(\DDelta)}
	=
	\sum_{k,\ell}
	\frac{\left|M_{k\ell}\right|^{2}}{\max(\sigma_k(\XX),\varepsilon)}.
	\]
	All summands are nonnegative.  Keeping only the entries corresponding to the
	rows $\f{p}_i$ and columns $\f{f}_i$ yields
	\[
	\innerproduct{\DDelta,\W{\XX}{\varepsilon}(\DDelta)}
	\ge
	\frac{1}{\varepsilon}\sum_{i=1}^{r}
	\left|\f{p}_{i}^{\top}\DDelta\f{f}_{i}\right|^{2}
	=
	\frac{ra^{2}\cos^{2}(\theta/2)}{\varepsilon}.
	\]
	Using \eqref{eq:counterexample:betaeps} and
	\eqref{eq:counterexample:singularvalues:Delta}, we obtain
	\[
	\innerproduct{\DDelta,\W{\XX}{\varepsilon}(\DDelta)}
	\ge
	\frac{d\,a\,\cos^{2}(\theta/2)}{10}
	=
	\frac{d}{110r}\cos^{2}(\theta/2)\nucnorm{\DDelta}.
	\]
	Since $\theta\le\pi/2$, $\cos^{2}(\theta/2)\ge1/2$, and therefore
	\[
	\innerproduct{\XX-\Xzero,\W{\XX}{\varepsilon}(\XX-\Xzero)}
	\ge
	\frac{d}{220r}\nucnorm{\XX-\Xzero}.
	\]
	
	\medskip
	\noindent\textbf{One-step lower bound.}
	The feasible set of the weighted least-squares problem is the affine line
	\[
	\{\f{Z}:\mathcal{A}(\f{Z})=\mathcal{A}(\Xzero)\}
	=
	\{\Xzero+\lambda\DDelta:\lambda\in\R\}.
	\]
	Set
	\[
	A_{\DDelta}:=\innerproduct{\DDelta,\W{\XX}{\varepsilon}(\DDelta)},
	\qquad
	b_{\DDelta}:=\innerproduct{\W{\XX}{\varepsilon}(\XX),\DDelta}.
	\]
	The minimizer on the affine line is
	\[
	\XX^+=\Xzero+\lambda_+\DDelta,
	\qquad
	\lambda_+=1-\frac{b_{\DDelta}}{A_{\DDelta}}.
	\]
	Since all nonzero singular values of $\XX$ are larger than $\varepsilon$,
	\[
	\W{\XX}{\varepsilon}(\XX)
	=
	\sum_{i=1}^{r}\f{u}_{i}\f{f}_{i}^{\top}
	+
	\sum_{j=1}^{h}\f{e}_{2r+j}\f{f}_{r+j}^{\top}.
	\]
	Consequently,
	\[
	b_{\DDelta}
	=
	r\alpha(1-\cos\theta)+ha
	=
	ra\sin(\theta/2)+10ra
	\le 11ra.
	\]
	Combining this with
	$A_{\DDelta}\ge d\,a\cos^{2}(\theta/2)/10$ gives
	\[
	1-\lambda_+
	=
	\frac{b_{\DDelta}}{A_{\DDelta}}
	\le
	\frac{110r}{d\cos^{2}(\theta/2)}
	\le
	\frac{220r}{d}.
	\]
	Since $d\ge 220r$, this shows $\lambda_+\in[0,1]$.  Hence
	\[
	\nucnorm{\XX^+-\Xzero}
	=
	\lambda_+\nucnorm{\DDelta}
	\ge
	\left(1-\frac{220r}{d}\right)\nucnorm{\XX-\Xzero},
	\]
	which proves the stated lower bound for the subsequent IRLS iterate.
	\end{proof}

%% file: literature_proofs.tex
\section{Complementary Proofs} \label{sec:appendix:complementary}
In this section, we provide proofs for statements that are known in the literature (explicitly or implicitly) and point out an incorrect proof in the literature. In particular, we recall a proof for the properties of \Cref{proposition:IRLS:basicproperties} of the IRLS quadratic model in \Cref{sec:appendix:proofbasiclemma}, establish Lipschitz continuity of $\nabla \mathcal{J}_{\varepsilon}(\cdot)$ in \Cref{sec:appendix:lipschitz:gradients}, and provide existing proofs for majorization of one-sided IRLS quadratics in \Cref{sec:appendix:prior:art:majorization:proofs}. In \Cref{sec:appendix:challenges:harmonic:majorization_proof}, we discuss to what extent the majorization result has remained elusive for previous works, in particular, by pointing out a key inaccuracy in the corresponding arguments of \citet{Kummerle-JMLR2018}.
\subsection{Proof of \Cref{proposition:IRLS:basicproperties} (Properties of the IRLS Quadratic Model)} \label{sec:appendix:proofbasiclemma}
As mentioned in \Cref{sec:algo_IRLS}, the proof of  \Cref{proposition:IRLS:basicproperties} is rather straightforward, and applies for the quadratic model functions $Q_{\varepsilon}(\cdot \mid \XX)$ corresponding to any of the considered weight operator notions considered in this paper, including one-sided weight operators, the harmonic-mean weight operator, and all $q$-mean weight operators. Two ingredients are used: (i) the diagonal elements of the core matrix $\f{H}_{\ssigma, \varepsilon} \in \Rdd$ satisfy $(\f{H}_{\ssigma, \varepsilon})_{ii} = \max(\sigma_i,\varepsilon)^{-1}$ for all $i \in [d]$, which yields the gradient condition and can be easily verified for all considered weight operator notions; and (ii) the Hadamard form \cref{eq:W:operator:action} of \Cref{def:optimalweightoperator}, which implies that $\W{\XX}{\varepsilon}(\cdot)$ is self-adjoint and thus yields the simplified expression \cref{eq:QZX:equality}.

\begin{proof}[Proof of \Cref{proposition:IRLS:basicproperties}]
Let us start by computing the gradient $\nabla \mathcal{J}_{\varepsilon}(\f{X})$ 
of the smoothed nuclear norm surrogate $\mathcal{J}_{\varepsilon}(\f{X})$
defined in \cref{eq:smoothedell1:objective}. 
Note that $\mathcal{J}_{\varepsilon}(\f{X})$ is a convex spectral function of $\f{X}$, 
and thus by \cite[Proposition 6.2]{LewisSendov}, 
the gradient itself is given by the spectral formula
    \[
    \nabla \mathcal{J}_{\varepsilon}(\f{X}) 
    = \UU_{\XX} \diag\big( j'_\varepsilon(\sigma_i (\XX) ) \big)_{i=1}^{d} \VV_{\XX}^\top,
    \]
    where $\f{X} = \UU_{\XX} \diag(\ssigma) \VV_{\XX}^\top$ is the singular value decomposition of $\f{X}$. We note that
    \begin{equation*}
    j'_\varepsilon(\sigma) =
    \begin{cases}
    	1, & \sigma > \varepsilon,\\[4pt]
    	\sigma/\varepsilon, & \sigma \leq \varepsilon.
    \end{cases} = \frac{\sigma}{\max(\sigma, \varepsilon)},
    \end{equation*}
    It follows that 
    \begin{equation} \label{eq:smoothedsurrogate:gradientformula}
    \nabla \mathcal{J}_{\varepsilon}(\f{X}) 
    = \UU_{\XX} \dg\left( \frac{\sigma_i (\XX) }{\max(\sigma_i (\XX),\varepsilon)} \right)_{i=1}^{d} \VV_{\XX}^\top.
    \end{equation}  
	Note that, 
    regardless whether harmonic-mean weights or one-sided weights are used,
    the diagonal elements of $\f{H}_{\ssigma, \varepsilon} \in \Rdd$ 
    are given by $(\f{H}_{\ssigma, \varepsilon})_{ii} = \max(\sigma_i,\varepsilon)^{-1}$.
	Then the claim follows by simply inserting the definition of the weight operator 
    $\W{\XX}{\varepsilon}(\cdot)$, see \cref{eq:W:operator:action}.
    Indeed, we have that
	\begin{align*}
	\W{\XX}{\varepsilon}(\f{X}) 
    &= \UU_{\XX} \left[\f{H}_{\ssigma, \varepsilon} \circ (\UU_{\XX}^{\top} \f{X} \VV_{\XX})\right] \VV_{\XX}^{\top}= \UU_{\XX}\left[\f{H}_{\ssigma, \varepsilon} \circ \dg(\ssigma) \right] \VV_{\XX}^{\top}\\
    &= \UU_{\XX} \dg\left( \frac{\sigma_i (\XX) }{\max(\sigma_i (\XX),\varepsilon)} \right)_{i=1}^{d} \VV_{\XX}^{\top}=  \nabla \mathcal{J}_{\varepsilon}(\f{X}).
	\end{align*}
	This proves the gradient condition \eqref{eq:gradientcondition}.
	
	We now derive the specific form of $Q_{\varepsilon} (\cdot \mid \f{X})$ in equation \eqref{eq:QZX:equality}.
	First, note that \cref{eq:W:operator:action} implies that $\W{\XX}{\varepsilon}(\cdot)$ is linear and self-adjoint with respect to the Frobenius inner product.
	Indeed, for all $\f{A},\f{B} \in \Rdd$, orthogonality of $\UU_{\XX}$ and $\VV_{\XX}$ gives
	\begin{align*}
	    \innerproduct{\W{\XX}{\varepsilon}(\f{A}),\f{B}}
	    &=
	    \innerproduct{
	        \f{H}_{\ssigma,\varepsilon}\circ(\UU_{\XX}^{\top}\f{A}\VV_{\XX}),
	        \UU_{\XX}^{\top}\f{B}\VV_{\XX}
	    }\\
	    &=
	    \innerproduct{
	        \UU_{\XX}^{\top}\f{A}\VV_{\XX},
	        \f{H}_{\ssigma,\varepsilon}\circ(\UU_{\XX}^{\top}\f{B}\VV_{\XX})
	    }
	    =
	    \innerproduct{\f{A},\W{\XX}{\varepsilon}(\f{B})},
	\end{align*}
	where the middle identity holds for every real core matrix $\f{H}_{\ssigma,\varepsilon}$ (symmetry of $\f{H}_{\ssigma,\varepsilon}$ is not required), since Hadamard multiplication is entrywise and thus self-adjoint for the Frobenius product.
	Using the definition \cref{eq:smoothedell1:IRLSmajorizer}, the gradient condition \eqref{eq:gradientcondition}, and this self-adjointness, we obtain
	\begin{align*}
	Q_{\varepsilon}(\ZZ\mid\XX)
	&=
	\mathcal{J}_{\varepsilon}(\XX)
	+
	\innerproduct{\W{\XX}{\varepsilon}(\XX),\ZZ-\XX}
	+
	\frac{1}{2}
	\innerproduct{\ZZ-\XX,\W{\XX}{\varepsilon}(\ZZ-\XX)}\\
	&=
	\mathcal{J}_{\varepsilon}(\XX)
	+
	\innerproduct{\W{\XX}{\varepsilon}(\XX),\ZZ}
	-
	\innerproduct{\W{\XX}{\varepsilon}(\XX),\XX}\\
	&\quad
	+
	\frac{1}{2}
	\innerproduct{\ZZ,\W{\XX}{\varepsilon}(\ZZ)}
	-
	\innerproduct{\ZZ,\W{\XX}{\varepsilon}(\XX)}
	+
	\frac{1}{2}
	\innerproduct{\XX,\W{\XX}{\varepsilon}(\XX)}\\
	&=
	\mathcal{J}_{\varepsilon}(\XX)
	+
	\frac{1}{2}
	\innerproduct{\ZZ,\W{\XX}{\varepsilon}(\ZZ)}
	-
	\frac{1}{2}
	\innerproduct{\XX,\W{\XX}{\varepsilon}(\XX)}.
	\end{align*}
	This proves \eqref{eq:QZX:equality}. Finally, the statement $Q_{\varepsilon}(\XX\mid\XX) = \mathcal{J}_{\varepsilon}(\XX)$ follows immediately as 
	$\frac{1}{2}
	\innerproduct{\XX,\W{\XX}{\varepsilon}(\XX)}
	-
	\frac{1}{2}
	\innerproduct{\XX,\W{\XX}{\varepsilon}(\XX)} = 0$.
\end{proof}

\subsection{Lipschitz Gradients of Smoothed Nuclear Norm} \label{sec:appendix:lipschitz:gradients}
In \Cref{sec:algo_IRLS}, it was claimed that the $\varepsilon$-smoothed nuclear norm $\mathcal{J}_{\varepsilon}(\cdot)$ is differentiable with a Lipschitz-continuous gradient. We now provide a proof of this claim.

In particular, we observe that the gradient $\X \mapsto \nabla \mathcal{J}_{\varepsilon}(\X)$ of \cref{eq:smoothedsurrogate:gradientformula} is a non-Hermitian Loewner operator \citep{Loewner34,SunSun08,Ding18}. The framework of \citep{andersson2016operator} provides an exact quantification of the Lipschitz properties of such operators if the function that is applied to each singular value is the same, which is the case here with $g_\varepsilon: \sigma \mapsto \frac{\sigma}{\max(\sigma,\varepsilon)}$. $g_\varepsilon$ is Lipschitz continuous with Lipschitz constant $L = \frac{1}{\varepsilon}$, so by \cite[Theorem 1.1]{andersson2016operator}, the gradient $\X \mapsto \nabla \mathcal{J}_{\varepsilon}(\X)$ is Lipschitz continuous with Lipschitz constant $L = \frac{1}{\varepsilon}$ with respect to the Frobenius norm, which proves the claim.

\subsection{Prior Art of Majorization Proofs for Low-Rank IRLS Algorithms} \label{sec:appendix:prior:art:majorization:proofs}

For quadratic model functions defined by one-sided weight operators \cref{eq:leftsided:mean,eq:rightsided:mean} and associated IRLS methods \citep{Fornasier11,mohan_fazel}, \cref{eq:minimizer:majorization} follows from the \emph{global majorization property} as established in
previous works \citep[Section 2.3.2]{Fornasier11,christian_thesis}, which we restate below for completeness.
\begin{proposition}[{Global Majorization of One-Sided Quadratic Models}] \label{prop:global:majorization:onesided}
	Let $\varepsilon > 0$, let $\mathcal{J}_{\varepsilon}: \Rdd  \to \R$ be the $\varepsilon$-smoothed nuclear norm \cref{eq:smoothedell1:objective} and $Q_{\varepsilon}(\cdot \mid \XX): \Rdd \to \R$ be the quadratic model function of \cref{eq:smoothedell1:IRLSmajorizer} defined by the left-sided or right-sided weight operator \cref{eq:W:operator:action} using core matrix \cref{eq:leftsided:mean} or \cref{eq:rightsided:mean}, respectively. Then, $Q_{\varepsilon}(\cdot \mid \XX)$ majorizes $\mathcal{J}_{\varepsilon}$ globally:
	\begin{equation} \label{eq:majorization_property:proofsection}
		Q_{\varepsilon}(\ZZ \mid\XX) \geq \mathcal{J}_{\varepsilon}(\ZZ) \qquad \text{for each } \ZZ,\XX \in \Rdd.
	\end{equation}
\end{proposition} 
It is easy to see that \Cref{prop:global:majorization:onesided} implies \cref{eq:minimizer:majorization} by choosing $\ZZ = \Xk{(k)}$ and $\XX = \Xk{(k-1)}$. For one-sided weight operators as used in the algorithms of \citet{Fornasier11} and \citet{mohan_fazel}, \Cref{prop:global:majorization:onesided} can be shown, for example, using variational arguments via the definition of a suitable auxiliary functional, which is in line with the literature on IRLS methods for separable objectives such as $\ell_1$-type norms \citep{GemanReynolds92,CharbonnierBFAB97,Daubechies-CPAM2010,ochs_dosovitskiy_brox_pock}.

In the remainder of this section, we provide two different proofs for \cref{eq:majorization_property:proofsection} in the case of one-sided weights in \Cref{sec:appendix:concavity:majorization_proof} and \Cref{sec:appendix:variational:majorization_proof}, respectively, and a third proof in \Cref{sec:appendix:proof_majorization_one_sided_weights} via a specialization of the proof strategy of \Cref{sec:appendix:harmonic:majorization_proof}. Finally, we discuss why these proof strategies do not extend to the case of harmonic-mean weights in \Cref{sec:appendix:challenges:harmonic:majorization_proof}.

\subsubsection{Proof of \Cref{prop:global:majorization:onesided} Using Concavity Arguments} \label{sec:appendix:concavity:majorization_proof}
Apart from low-rank matrix optimization problems, the vast majority of the theory of iteratively reweighted least squares algorithms is specialized to coordinatewise separable objectives such as $\ell_1$-norms or other functions that are sums of coordinatewise vector functions. For such cases, a variety of works have established frameworks for how to derive suitable quadratic auxiliary objectives $Q_{\varepsilon}(\cdot \mid \f{X})$ that provably majorize a given (smoothed) surrogate function. We refer to \citet{GemanReynolds92} and \citet{CharbonnierBFAB97} for classical works in the image processing literature (in which this methodology is commonly referred to as \emph{half-quadratic minimization}) and the work of \citet{ochs_dosovitskiy_brox_pock} for an exposition tailored to computer vision.

It turns out that the proof strategies tailored to separable objectives can be adapted 
to show the majorization property \cref{eq:majorization_property:proofsection} 
for one-sided weight operators with core matrices \cref{eq:leftsided:mean} or \cref{eq:rightsided:mean} relatively straightforwardly. 
The key idea is to use the concavity of the term $\mathcal{J}_{\varepsilon}(\ZZ)$ in a matrix variable after a suitable change of variables, 
and has been detailed by \citet[Section 2.3.2]{christian_thesis} for rank surrogate optimization. We present the proof below.

\begin{proof}[First Proof of \cref{eq:majorization_property:proofsection} for One-Sided Weights]
We show the statement without loss of generality for the case of left-sided weight operators. We recall that the smoothed nuclear norm surrogate \cref{eq:smoothedell1:objective} is given by $\mathcal{J}_{\varepsilon}(\ZZ) := \sum_{i=1}^d j_{\varepsilon}(\sigma_i(\ZZ))$ with real-valued penalization functions $j_{\varepsilon}: \R \to \R$ such that
\begin{equation*}
    j_{\varepsilon}(\sigma) := \begin{cases}
         |\sigma|, & \text{ if } |\sigma| > \varepsilon, \\
         \frac{\sigma^2}{2 \varepsilon}+ \frac{\varepsilon}{2}, & \text{ if } |\sigma| \leq \varepsilon.
     \end{cases}
\end{equation*} 
Accordingly, we define the function $g_\varepsilon: \R_{\geq 0} \to \R$ by
\begin{align*}
g_\varepsilon(\lambda) 
:=
j_{\varepsilon}(\sqrt{\lambda})
= 
\begin{cases}
	\sqrt{\lambda}, & \text{ if } \lambda > \varepsilon^2, \\
	\frac{\lambda}{2 \varepsilon}+ \frac{\varepsilon}{2}, & \text{ if } 0 \leq \lambda \leq \varepsilon^2.
\end{cases}
\end{align*}
For a PSD matrix $\f{M} \in \R^{d_1 \times d_1}$, 
let $G_\varepsilon(\f{M}) := \mathrm{tr}[ g_\varepsilon(\f{M})]$,
where $g_\varepsilon(\f{M})$ is defined via the standard functional calculus for symmetric matrices,
i.e.,
if $\f{M} = \UU \diag(\lambda_1, \ldots, \lambda_{d_1}) \UU^\top$ is the eigendecomposition of $\f{M}$,
then $g_\varepsilon(\f{M}) = \UU \diag(g_\varepsilon(\lambda_1), \ldots, g_\varepsilon(\lambda_{d_1})) \UU^\top$.
Next, we notice that for any $\f{Z} \in \mathbb{R}^{d_1 \times d_2}$ we have
\begin{equation} \label{eq:J:and:G:connection}
\mathcal{J}_{\varepsilon}(\ZZ) 
=
\sum_{i=1}^d j_{\varepsilon}(\sigma_i (\ZZ))
=
\sum_{i=1}^d g_\varepsilon(\sigma_i^2 (\ZZ))
=
G_\varepsilon(\ZZ\ZZ^\top) - \max(0, d_1-d_2) \frac{\varepsilon}{2}.
\end{equation}
Since $g_\varepsilon'(\lambda) = \frac{1}{2 \max(\sqrt{\lambda},\varepsilon)}$ is decreasing on the interval $(0, \infty)$, 
the function $g_\varepsilon$ is concave on this interval, and so is the induced spectral trace function $G_\varepsilon$ \citep[see, e.g.,][Theorem 2.10]{carlen2010trace}.    
We also have that 
$\nabla G_\varepsilon(\XX\XX^\top) = \UU_{\XX} \diag(g_\varepsilon'(\sigma^2_i (\XX) ))_{i=1}^{d_1} \UU_{\XX}^\top$
by \citet[Proposition 6.2]{LewisSendov}, 
where we recall that 
\begin{align*}
\XX\XX^\top = \UU_{\XX} \diag (\sigma_1^2 (\XX), \ldots, \sigma_{d_1}^2 (\XX)) \UU_{\XX}^\top
\end{align*}
is an eigendecomposition of $\XX\XX^\top$ with square matrix $\UU_{\XX} \in \R^{d_1 \times d_1}$ and the convention that $\sigma_i(\XX) = 0$ for $i > d = \min(d_1, d_2)$.
Thus,
we obtain that
\begin{align*}
\nabla G_\varepsilon(\XX\XX^\top) 
&=
\frac{1}{2} \UU_{\XX} \diag\left( \frac{1}{ \max(\sigma_i (\XX),\varepsilon)} \right)_{i=1}^{d_1} \UU_{\XX}^\top\\
&=
\frac{1}{2} 
\left(
\UU_{\XX} \diag\left( \max(\sigma_i (\XX),\varepsilon) \right)_{i=1}^{d_1} \UU_{\XX}^\top
\right)^{-1}
= \frac{1}{2} \LL_\UU^{-1},
\end{align*}
where $\LL_{\UU} = \UU_{\XX} \diag ( \max(\sigma_1(\XX),\varepsilon), \ldots, \max(\sigma_{d_1}(\XX),\varepsilon) ) \UU_{\XX}^\top \in \R^{d_1 \times d_1}$, cf. also \Cref{sec:algo_IRLS}. Due to the concavity of the entrywise functions $g_\varepsilon$ and \citep[Proposition 6.1]{LewisSendov}, we know that $G_\varepsilon: \R^{d_1 \times d_1} \to \R$ is also a concave function. Consequently, we can upper bound $G_\varepsilon(\ZZ\ZZ^\top)$ by its linearization in $\XX\XX^\top$ such that
\begin{align*}
     \mathcal{J}_{\varepsilon}(\ZZ)
&= G_\varepsilon(\ZZ\ZZ^\top) - \max(0, d_1-d_2) \frac{\varepsilon}{2} \\
&\leq G_\varepsilon(\XX\XX^\top) 
+ \langle \nabla G_\varepsilon(\XX\XX^\top), \ZZ\ZZ^\top - \XX\XX^\top\rangle - \max(0, d_1-d_2) \frac{\varepsilon}{2} \\
&=
G_\varepsilon(\XX\XX^\top)
+
\frac{1}{2} \innerproduct{ \LL_{\UU}^{-1} \ZZ , \ZZ }
-
\frac{1}{2} \innerproduct{ \LL_{\UU}^{-1} \XX , \XX } - \max(0, d_1-d_2) \frac{\varepsilon}{2} \\
&\overeq{(a)}
\mathcal{J}_{\varepsilon}(\XX)
+
\frac{1}{2} \innerproduct{ W_{\XX,\varepsilon} (\ZZ) , \ZZ }
-
\frac{1}{2} \innerproduct{ W_{\XX,\varepsilon} (\XX) , \XX }\\
&\overeq{(b)}  Q_\varepsilon(\ZZ \mid \XX),
\end{align*}
where we used \cref{eq:J:and:G:connection} for $\ZZ = \XX$ in equality (a)
as well as that $\WW_{\XX,\varepsilon} (\ZZ) = \LL_{\UU}^{-1} \ZZ$
for all $\ZZ \in \Rdd$.
In equality (b), we used the consequence \cref{eq:QZX:equality} of the gradient condition \cref{eq:gradientcondition}. This concludes the proof.
\end{proof}

Extending this proof strategy to weight operators $\W{\XX}{\varepsilon} (\cdot)$ with harmonic-mean core matrices \cref{eq:harmonic:mean} has remained elusive to the authors; a simple change of variables does not suffice as the weight operator cannot be represented as simply a left- (or right-)matrix multiplication in this case.

\subsubsection{Proof of \Cref{prop:global:majorization:onesided} Using Variational Calculus} \label{sec:appendix:variational:majorization_proof}
The one-sided majorization property can also be derived from the auxiliary variational formulation
underlying the IRLS-M algorithm of \citet{Fornasier11}.
We spell out the argument for the left-sided weights; the right-sided case follows by applying the
same argument to the transposed matrices.

\begin{proof}[Second Proof of \cref{eq:majorization_property:proofsection} for One-Sided Weights]
Let $\XX \in \Rdd$ be fixed and assume that the left-sided weight core matrix
\cref{eq:leftsided:mean} is used.
For a symmetric positive definite matrix $\f{M}\in\R^{d_1\times d_1}$, define
\begin{equation} \label{eq:variational:left:auxiliary:functional}
    \Phi_{\varepsilon}^{L}(\ZZ,\f{M})
    :=
    \frac{1}{2}\innerproduct{\f{M}\ZZ,\ZZ}
    +
    \frac{1}{2}\trace(\f{M}^{-1}),
\end{equation}
on the restricted domain for the second matrix variable $\f{M}$ such that $0 \prec \f{M} \preceq \varepsilon^{-1}\Id_{d_1}$. This is the left-sided auxiliary functional used by \citet[Section 5]{Fornasier11}. If $\ZZ = \UU_{\ZZ}\dg(\ssigma(\ZZ))\VV_{\ZZ}^{\top}$ is a full singular value decomposition and
we use the convention $\sigma_i(\ZZ)=0$ for $i>d=\min(d_1,d_2)$, then minimizing the preceding
functional over $\f{M}$ gives
\begin{equation} \label{eq:variational:left:minimizer}
    \f{M}_{\ZZ,\varepsilon}^{L}
    =
    \UU_{\ZZ}
    \diag\left(
        \frac{1}{\max(\sigma_i(\ZZ),\varepsilon)}
    \right)_{i=1}^{d_1}
    \UU_{\ZZ}^{\top}.
\end{equation}
To see this, rotate the auxiliary variable into the left singular-vector basis and write
$\f{N}:=\UU_{\ZZ}^{\top}\f{M}\UU_{\ZZ}$.
Then $0\prec \f{N}\preceq \varepsilon^{-1}\Id_{d_1}$ and, with
$a_i:=\sigma_i^2(\ZZ)$ for $i=1,\ldots,d_1$,
\begin{equation*}
    \Phi_{\varepsilon}^{L}(\ZZ,\f{M})
    =
    \frac{1}{2}\sum_{i=1}^{d_1} a_i \f{N}_{ii}
    +
    \frac{1}{2}\trace(\f{N}^{-1}).
\end{equation*}
For every positive definite $\f{N}$ and every coordinate vector $\f{e}_i$, the Cauchy--Schwarz
inequality gives
\begin{equation*}
    1
    =
    \left(
    (\f{N}^{-1/2}\f{e}_i)^\top(\f{N}^{1/2}\f{e}_i)
    \right)^2
    \leq
    (\f{N}^{-1})_{ii}\f{N}_{ii},
\end{equation*}
and hence $(\f{N}^{-1})_{ii}\geq 1/\f{N}_{ii}$.
Since $0<\f{N}_{ii}\leq \varepsilon^{-1}$, it follows that
\begin{align*}
    \Phi_{\varepsilon}^{L}(\ZZ,\f{M})
    &\geq
    \frac{1}{2}
    \sum_{i=1}^{d_1}
    \left(
        a_i \f{N}_{ii}
        +
        \frac{1}{\f{N}_{ii}}
    \right)\\
    &\geq
    \frac{1}{2}
    \sum_{i=1}^{d_1}
    \min_{0<m\leq \varepsilon^{-1}}
    \left(
        a_i m
        +
        \frac{1}{m}
    \right).
\end{align*}
The scalar minimizer is $m_i^{\star}=1/\max(\sigma_i(\ZZ),\varepsilon)$: if
$\sigma_i(\ZZ)\geq\varepsilon$, this is the unconstrained critical point
$m=1/\sigma_i(\ZZ)$, whereas if $\sigma_i(\ZZ)<\varepsilon$ the minimum over
$(0,\varepsilon^{-1}]$ is attained at the boundary $m=\varepsilon^{-1}$.
Choosing $\f{N}=\diag(m_1^{\star},\ldots,m_{d_1}^{\star})$ makes the preceding lower bounds
equalities. Transforming back from $\f{N}$ to $\f{M}$ proves
\cref{eq:variational:left:minimizer}.
At this minimizer, the $i$-th scalar contribution is
$j_{\varepsilon}(\sigma_i(\ZZ))$ for $i\leq d$ and $\varepsilon/2$ for $i>d$.
Consequently,
\begin{equation} \label{eq:variational:left:envelope}
    \min_{0 \prec \f{M} \preceq \varepsilon^{-1}\Id_{d_1}}
    \Phi_{\varepsilon}^{L}(\ZZ,\f{M})
    =
    \mathcal{J}_{\varepsilon}(\ZZ)
    +
    \frac{d_1-d}{2}\varepsilon .
\end{equation}
The additional constant appears only when $d_1>d_2$, because the remaining
$d_1-d$ zero eigenvalue directions of $\ZZ\ZZ^\top$ each contribute $\varepsilon/2$.

We now freeze the auxiliary variable at the minimizer associated with the base point $\XX$.
Writing $\LL_{\UU}
    :=
    \UU_{\XX}
    \diag\left(\max(\sigma_i(\XX),\varepsilon)\right)_{i=1}^{d_1}
    \UU_{\XX}^{\top}$, the minimizer \eqref{eq:variational:left:minimizer} for $\ZZ=\XX$ is
$\f{M}_{\XX,\varepsilon}^{L}=\LL_{\UU}^{-1}$. It turns out that this coincides exactly (see also \Cref{sec:algo_IRLS}) with the action of the left-sided weight operator $\W{\XX}{\varepsilon}(\cdot)$ \cref{eq:leftsided:mean,eq:W:operator:action} such that $
    \W{\XX}{\varepsilon}(\ZZ)=\LL_{\UU}^{-1}\ZZ
    =
    \f{M}_{\XX,\varepsilon}^{L}\ZZ$.
Using \eqref{eq:variational:left:envelope} first at $\ZZ$ and then at $\XX$, we obtain
\begin{align*}
    \mathcal{J}_{\varepsilon}(\ZZ)
    +
    \frac{d_1-d}{2}\varepsilon
    &\leq
    \Phi_{\varepsilon}^{L}(\ZZ,\f{M}_{\XX,\varepsilon}^{L})\\
    &=
    \frac{1}{2}\innerproduct{\W{\XX}{\varepsilon}(\ZZ),\ZZ}
    +
    \frac{1}{2}\trace(\LL_{\UU})\\
    &=
    \frac{1}{2}\innerproduct{\W{\XX}{\varepsilon}(\ZZ),\ZZ}
    +
    \mathcal{J}_{\varepsilon}(\XX)
    +
    \frac{d_1-d}{2}\varepsilon
    -
    \frac{1}{2}\innerproduct{\W{\XX}{\varepsilon}(\XX),\XX}.
\end{align*}
After cancelling the dimension-dependent constant, \cref{eq:QZX:equality} yields
\begin{align*}
    \mathcal{J}_{\varepsilon}(\ZZ)
    &\leq
    \mathcal{J}_{\varepsilon}(\XX)
    +
    \frac{1}{2}\innerproduct{\W{\XX}{\varepsilon}(\ZZ),\ZZ}
    -
    \frac{1}{2}\innerproduct{\W{\XX}{\varepsilon}(\XX),\XX} =
    Q_{\varepsilon}(\ZZ\mid\XX),
\end{align*}
which proves \cref{eq:majorization_property:proofsection} for left-sided weights.
The right-sided case follows analogously by replacing $\ZZ\ZZ^\top$ by $\ZZ^\top\ZZ$ and
$\LL_{\UU}^{-1}\ZZ$ by $\ZZ\LL_{\VV}^{-1}$.
\end{proof}

\subsubsection{Proof of \Cref{prop:global:majorization:onesided} Based on \Cref{sec:appendix:harmonic:majorization_proof}}\label{sec:appendix:proof_majorization_one_sided_weights}

For completeness, we provide an application of the proof strategy of \Cref{sec:appendix:harmonic:majorization_proof} to the quadratic model majorization result of \Cref{prop:global:majorization:onesided} for one-sided weight operators, providing an alternative for the existing proofs presented in the preceding sections. Compared to the harmonic-mean arguments of \Cref{sec:appendix:lb_weighted_ip}, the argument is shortened significantly.

\begin{proof}[Third Proof of \cref{eq:majorization_property:proofsection} for One-Sided Weights]
We show the argument for left-sided weights; the proof for right-sided weights follows by transposition.
The computation of $\innerproduct{\W{\XX}{\varepsilon}(\XX),\XX}$ and the resulting representation of $Q_\varepsilon(\ZZ\mid \XX)$ are identical to the beginning of the proof of \Cref{thm:majorization} in \Cref{sec:appendix:harmonic:majorization_proof}, because all admissible weight cores have diagonal entries
$(\f{H}_{\ssigma,\varepsilon})_{ii}=1/\max(\sigma_i(\XX),\varepsilon)$.
Thus, the only ingredient specific to the one-sided case is the lower bound for
$\innerproduct{\W{\XX}{\varepsilon}(\ZZ),\ZZ}$. We recall from \Cref{sec:algo_IRLS} that the action of the left-sided weight operator $\W{\XX}{\varepsilon}(\cdot)$ is given by $$\W{\XX}{\varepsilon} (\ZZ) =
\LL_{\UU}^{-1} \ZZ $$
with 
\begin{equation*}
\LL_{\UU}
=
\UU_{\XX} \diag (\lambda_1, \ldots, \lambda_{d_1}) \UU_{\XX}^\top,
\end{equation*}
where $\UU_{\XX}$ is the matrix of left singular vectors of $\XX$ and $\lambda_i = \max(\sigma_i (\XX), \varepsilon)$ for $1 \le i \le d_1$.
Let $\ZZ=\sum_{k=1}^d \sigma_k(\ZZ)\uu_k\vv_k^\top$ be an SVD of $\ZZ$.
Then
\begin{align*}
\innerproduct{ \W{\XX}{\varepsilon} (\ZZ), \ZZ }
    =
    \sum_{k=1}^d \sigma_k^2 (\ZZ) \innerproduct{ \LL_{\UU}^{-1}, \uu_k \uu_k^\top }
    \ge
    \sum_{k=1}^d
    \frac{\sigma_k^2 (\ZZ)}{ \innerproduct{\LL_{\UU}, \uu_k \uu_k^\top } }
\end{align*}
where the inequality follows by the Cauchy-Schwarz estimate
$1\le \|\LL_{\UU}^{-1/2}\uu_k\|_2\|\LL_{\UU}^{1/2}\uu_k\|_2$ taken to the power two.
This is precisely the one-sided analogue of Lemma \ref{lemma:lower_bound_weighted_inner_product}.
The remaining spectral and scalar estimates are the same as in
the previous section, with the averaged quantity
$\innerproduct{\LL_{\UU},\uu_k\uu_k^\top}/2+\innerproduct{\LL_{\VV},\vv_k\vv_k^\top}/2$
replaced by $\innerproduct{\LL_{\UU},\uu_k\uu_k^\top}$.
Indeed, Wielandt's minimax principle gives
\[
    \sum_{k=1}^d \lambda_k
    \ge
    \sum_{k=1}^d \innerproduct{\LL_{\UU},\uu_k\uu_k^\top},
\]
and the scalar case distinction from the harmonic proof applies with
$\alpha_k=\innerproduct{\LL_{\UU},\uu_k\uu_k^\top}$.
This establishes the desired majorization
$Q_{\varepsilon}(\ZZ\mid\XX)\ge \mathcal{J}_{\varepsilon}(\ZZ)$.

\end{proof}

\subsubsection{Challenges for Harmonic-Mean Majorization Proofs} \label{sec:appendix:challenges:harmonic:majorization_proof}
It is natural to ask whether the variational proof from
\Cref{sec:appendix:variational:majorization_proof} can be adapted to harmonic-mean weights.
A proof attempt in this direction was made by \citet[Definition~13 and Lemma~14]{Kummerle-JMLR2018}
for a smoothed Schatten-$p$ surrogate, which is a generalization of the smoothed nuclear norm $\mathcal{J}_{\varepsilon}(\cdot)$ of \cref{eq:smoothedell1:objective} to include nonconvex Schatten-$p$ quasi-norms corresponding to $0 < p \leq 1$ (the Schatten-$1$ norm coincides with the nuclear norm). We briefly outline this proof strategy and point out why it does not directly provide a complete proof of the harmonic-mean majorization property.

The idea \citep[by][pp.~26--27]{Kummerle-JMLR2018} is to introduce, for fixed $\XX$ and smoothing parameter $\varepsilon$, an auxiliary matrix variable $\f{M}$ and the weight operator matrix $\widetilde{\f{W}}(\f{M})$ defined as
\[
    \widetilde{\f{W}}(\f{M})
    =
    2\big[\Id_{d_2}\otimes(\f{M}\f{M}^\top)^{1/2}\big]
    \big[(\f{M}\f{M}^\top)^{1/2}\oplus(\f{M}^\top\f{M})^{1/2}\big]^{-1}
    \big[(\f{M}^\top\f{M})^{1/2}\otimes\Id_{d_1}\big],
\]
where $\oplus$ denotes the Kronecker sum and $\otimes$ the Kronecker product. Here, $\widetilde{\f{W}}(\f{M})$ is of size $(d_1 d_2) \times (d_1 d_2)$. In the square non-singular case, this is rewritten as the harmonic mean $2(\widetilde{\f{W}}_1^{-1}+\widetilde{\f{W}}_2^{-1})^{-1}$, where $\widetilde{\f{W}}_1:\R^{d_1 d_2} \to \R^{d_1 d_2}$ is the matrix representation of the left-sided matrix multiplication operator $\Id_{d_2}\otimes(\f{M}\f{M}^\top)^{1/2}$ and $\widetilde{\f{W}}_2:\R^{d_1 d_2} \to \R^{d_1 d_2}$ of the right-sided matrix multiplication operator $(\f{M}^\top\f{M})^{1/2}\otimes\Id_{d_1}$, respectively. With this notation, an analogue of \cref{eq:variational:left:auxiliary:functional} can be defined as
\begin{equation} \label{eq:jmlr:auxiliary:functional:schematic}
    \Phi_{\varepsilon}(\f{Z},\f{M})
    :=
    \frac{1}{2}
    \left\langle
        \ZZ_{\operatorname{vec}},
        \widetilde{\f{W}}(\f{M})\ZZ_{\operatorname{vec}}
    \right\rangle
    +
    \frac{\varepsilon^2}{2}
    \sum_{i=1}^{d}\sigma_i(\f{M})
    +
    \frac{1}{2}
    \sum_{i=1}^{d}\sigma_i(\f{M})^{-1},
\end{equation}
which is essentially Definition 13 of \citet{Kummerle-JMLR2018} for $p=1$.
The claimed minimizer of $\Phi_{\varepsilon}(\f{Z},\cdot)$ with respect to $\f{M}$
\citep[stated by][Lemma~14]{Kummerle-JMLR2018} is aligned with the singular vectors of $\ZZ$,
namely
\begin{equation} \label{eq:jmlr:auxiliary:minimizer}
    \f{M}_{\operatorname{opt}}
    =
    \UU_{\ZZ}
    \diag\left(
        (\sigma_i(\ZZ)^2+\varepsilon^2)^{-1/2}
    \right)_{i=1}^{d}
    \VV_{\ZZ}^{\top}.
\end{equation}
If this variational characterization were available, freezing the auxiliary variable at the
minimizer $\f{M}_{\operatorname{opt}}$ associated with the base point $\f{Z}$ would provide an alternating-minimization explanation for the harmonic-mean weight update and prove a majorization statement akin to \cref{eq:majorization_property:proofsection} for a quadratic model function $Q_{\varepsilon}(\cdot\mid\ZZ)$ using a harmonic-mean weight operator \citep[with the minor technical difference that][use a slightly different Schatten-$p$ smoothing than $\mathcal{J}_{\varepsilon}(\cdot)$ from this paper]{Kummerle-JMLR2018}.

However, it turns out that \citet{Kummerle-JMLR2018} do not provide a complete proof of the optimality of \cref{eq:jmlr:auxiliary:minimizer} with respect to $\Phi_{\varepsilon}(\f{Z},\cdot)$ of \cref{eq:jmlr:auxiliary:functional:schematic}. The main issue is that the harmonic-mean weight depends simultaneously on the left and right singular spaces through an inverse Kronecker-sum, or equivalently through a Sylvester operator. Consequently, the auxiliary
minimization in $\f{M}$ does not decouple into independent scalar minimizations after one change of
basis, in contrast to the one-sided functional $\Phi_{\varepsilon}^{L}$ in
\Cref{sec:appendix:variational:majorization_proof}.

Specifically, there are at least two substantive gaps in the critical-point calculation of \citet[Lemma 14]{Kummerle-JMLR2018}: First, the differentiation of matrix square roots and inverse square
roots is treated as if the scalar chain rule applied directly to
$(\f{M}^\top\f{M})^{-1/2}$ and $(\f{M}\f{M}^\top)^{-1/2}$ \citep[see][p.~42, eq.~(58)]{Kummerle-JMLR2018}. This is
not valid for noncommuting matrix perturbations; the Fr\'echet derivative of a matrix power
involves divided differences or an equivalent Sylvester-type operator \citep{DaletskiiKrein-1965Integration} and \citep[Theorem 3.8]{Noferini17}. Thus, the stationarity equation \citep[Eq. (60)]{Kummerle-JMLR2018} used later in the proof of \citep[Section B.2]{Kummerle-JMLR2018} is not justified by the cited matrix calculus rules.

The second substantive gap is that, after deriving the stationarity equation, the proof of \citet[Section B.2]{Kummerle-JMLR2018} uses
informal linear-algebra implications to conclude that the singular vectors of the auxiliary variable $\f{M}$ must align with
those of $\ZZ$. For example, the commutation relation
$\SSigma^4\f{B}\f{B}^\top=\f{B}\f{B}^\top\SSigma^4$ obtained by \citet[Eq. (80)]{Kummerle-JMLR2018} does not force $\f{B}\f{B}^\top$ to be diagonal unless additional nondegeneracy assumptions on the diagonal entries of $\SSigma$ are imposed. Even diagonality of
$\f{B}\f{B}^\top$ would not by itself force $\f{B}$ to be diagonal, contrary to the third paragraph after eq.~(80) of \citet{Kummerle-JMLR2018}. Hence, the key alignment claim needed to identify the supposed global minimizer $\f{M}_{\operatorname{opt}}$ is not established.

Thus, while the variational strategy correctly identifies the formal harmonic-mean weight update,
it leaves open the essential inequality needed for majorization. The work done in \Cref{sec:appendix:harmonic:majorization_proof} provides a direct lower bound for the weighted inner product $\innerproduct{\W{\XX}{\varepsilon}(\ZZ),\ZZ}$, instead of relying exclusively on such a variational envelope.

We believe that the statement of \citet[Lemma 14]{Kummerle-JMLR2018} is correct for $p=1$, but likely incorrect for $0 < p < 1$. As the focus of this paper is the nuclear norm objective corresponding to $p=1$, we leave a clarification of majorization properties of IRLS-type quadratic models for nonconvex rank surrogates to future work.